\pdfoutput=1
\documentclass[twoside,11pt]{article}

\usepackage[utf8]{inputenc}
\usepackage[T1]{fontenc}
\usepackage[dvipsnames]{xcolor}
\usepackage[preprint]{jmlr2e}
\hypersetup{pdfborder={0 0 0}}
\makeatletter
\let\jmlr@maketitle\maketitle
\def\maketitle{{\let\@footnotemark\H@@footnotemark\let\@footnotetext\H@@footnotetext\jmlr@maketitle}}
\makeatother

\usepackage{amsmath}
\usepackage{mathrsfs}
\usepackage{bm}
\usepackage{cleveref}
\usepackage{enumitem}
\usepackage{multirow}
\usepackage{subcaption}
\usepackage{tikz}
\usepackage{bbm}
\usepackage{wrapfig}
\usepackage{multicol}
\usepackage{booktabs}
\usepackage{algorithm}
\usepackage{algpseudocode}
\usepackage{lastpage}
\usepackage{titletoc}

\definecolor{red}{HTML}{ca0020}
\definecolor{lightred}{HTML}{f4a582}
\definecolor{lightblue}{HTML}{92c5de}
\definecolor{green}{HTML}{008837}
\definecolor{blue}{HTML}{2c7bb6}

\usepackage{tcolorbox}
\tcbuselibrary{theorems,skins,breakable}

\makeatletter
\@ifundefined{c@lemma}{%
\usepackage{aliascnt}%
\@for\cref@tmpenv:=lemma,proposition,remark,corollary,definition\do{%
  \expandafter\let\csname\cref@tmpenv\endcsname\@undefined
  \expandafter\let\csname end\cref@tmpenv\endcsname\@undefined}%
\newaliascnt{lemma}{theorem}\newtheorem{lemma}[lemma]{Lemma}\aliascntresetthe{lemma}%
\newaliascnt{proposition}{theorem}\newtheorem{proposition}[proposition]{Proposition}\aliascntresetthe{proposition}%
\newaliascnt{remark}{theorem}\newtheorem{remark}[remark]{Remark}\aliascntresetthe{remark}%
\newaliascnt{corollary}{theorem}\newtheorem{corollary}[corollary]{Corollary}\aliascntresetthe{corollary}%
\newaliascnt{definition}{theorem}\newtheorem{definition}[definition]{Definition}\aliascntresetthe{definition}%
}{}
\makeatother

\newtcolorbox{keyeqbox}{
    colback=blue!6, colframe=gray!70,
    boxrule=1.0pt, boxsep=3pt,
    left=4pt, right=4pt,
    enlarge left by=-4pt, enlarge right by=-4pt
}

\newtcolorbox{thmbox}{
    colback=green!6, colframe=gray!70,
    boxrule=1.0pt, boxsep=3pt,
    left=4pt, right=4pt,
    enlarge left by=-4pt, enlarge right by=-4pt
}

\newtcolorbox{defbox}{
    colback=orange!6, colframe=gray!70,
    boxrule=1.0pt, boxsep=3pt,
    left=4pt, right=4pt,
    enlarge left by=-4pt, enlarge right by=-4pt
}

\renewenvironment{proof}[1][Proof]{\par\noindent{\bf #1\ }}{\hfill\BlackBox\\[2mm]}

\newenvironment{talign*}
 {\let\displaystyle\textstyle\csname align*\endcsname}
 {\endalign}
\newenvironment{talign}
 {\let\displaystyle\textstyle\csname align\endcsname}
 {\endalign}

\renewcommand{\eqref}[1]{(\ref{#1})}

\ShortHeadings{Rare Event Analysis via SOC}{Du et al.}
\firstpageno{1}

\title{Rare Event Analysis via Stochastic Optimal Control}

\author{\name Yuanqi Du\thanks{This work was Y.D.'s summer internship project at Microsoft Research New England, mentored by C.~D.-E. Correspondence: \href{mailto:yuanqidu@microsoft.com}{yuanqidu@microsoft.com}, \href{mailto:carlesd@microsoft.com}{carlesd@microsoft.com}.} \email yuanqidu@microsoft.com\\
\addr Microsoft Research New England\\
Cambridge, MA, USA\\
\addr Cornell University\\
Ithaca, NY, USA
\AND
\name Jiajun He\\
\addr University of Cambridge\\
Cambridge, United Kingdom
\AND
\name Dinghuai Zhang\\
\addr Microsoft Research New England\\
Cambridge, MA, USA
\AND
\name Eric Vanden-Eijnden \email eve2@nyu.edu\\
\addr Courant Institute of Mathematical Sciences, New York University\\
New York, NY, USA
\AND
\name Carles Domingo-Enrich\thanks{Corresponding author.} \email carlesd@microsoft.com\\
\addr Microsoft Research New England\\
Cambridge, MA, USA}

\begin{document}
\editor{To be assigned}
\maketitle

\begin{abstract}
Rare events, from biomolecular conformational changes and phase transitions to chemical reactions, are central to many physical systems, yet extremely difficult to study computationally because unbiased simulations seldom produce them. Transition Path Theory provides a rigorous statistical framework for analyzing such events: it characterizes the ensemble of reactive trajectories between two metastable states (reactant and product), and its central object, the committor function (the probability that the system next reaches the product rather than the reactant), encodes essential kinetic and thermodynamic information. We introduce a framework for committor estimation that builds on stochastic optimal control, with a principled formulation that grounds committor learning on capped hitting-time trajectories. In this framework the committor defines a feedback control, proportional to the gradient of its logarithm, that steers trajectories toward the reactive region, enabling sampling of reactive paths. As practical algorithms, we develop two training losses: a heuristic on-policy direct backpropagation loss and a principled off-policy Value Matching loss, for which we establish first-order optimality guarantees. Being off-policy, Value Matching can leverage simulated trajectories from alternative sampling processes and reuse them from a replay buffer, which enables efficient learning. In particular, we derive a reactive current-preserving sampling process that can accelerate exploration.
On reversible and non-reversible dynamics, the framework yields more accurate committor estimates, reaction rates, and transition path samples than existing methods.
\end{abstract}


\section{Introduction}
\label{sec:intro}

\looseness=-1
Rare events in physical systems, such as biomolecular conformational changes, phase transitions, and chemical reactions, are infrequent by nature yet key to how these systems evolve~\citep{onsager1944crystal,levenspiel1998chemical,bolhuis2002transition,seifert2008stochastic}. Their scarcity makes them notoriously difficult to observe in direct numerical simulations. Quantifying such transitions requires both thermodynamic information (e.g.\ free-energy differences) and kinetic information (e.g.\ reaction rates). While classical theories such as the Arrhenius equation and transition state theory relate rates to activation barriers~\citep{arrhenius1889dissociationswarme,eyring1935activated}, modern rare-event analysis is largely organized around the \emph{committor function}, a central object of Transition Path Theory (TPT)~\citep{vanden2006towards,vanden2010transition} that encodes the full ensemble of reactive trajectories.

\paragraph{Setup and dynamics.}
Let $X_t\in\mathcal{X}\subseteq\mathbb{R}^d$ denote the state of the system evolving according to the stochastic differential equation
\begin{equation}
\label{eq:reference0}
  \mathrm{d}X_t = b(X_t)\,\mathrm{d}t + \sigma\,\mathrm{d}W_t,
\end{equation}
where $b:\mathcal{X}\to\mathbb{R}^d$ is the drift, $\sigma\in\mathbb{R}^{d\times d}$ is a constant volatility coefficient (with $D = \frac{1}{2}\sigma\sigma^{\top}$ the diffusion tensor), and $W_t$ is a standard $d$-dimensional Wiener process. Under appropriate conditions on~$b$ and~$\sigma$, the process~\eqref{eq:reference0} is ergodic with respect to a unique stationary density~$\rho$ satisfying the Fokker-Planck equation $\nabla \cdot (b\rho - D\nabla\rho) = 0$.

An important special case is the \emph{reversible} (gradient) setting, in which the dynamics follow the overdamped Langevin equation $\mathrm{d}X_t = -\nabla U(X_t)\,\mathrm{d}t + \sqrt{2\beta^{-1}}\,\mathrm{d}W_t$ for a potential $U:\mathcal{X}\to\mathbb{R}$ and inverse temperature $\beta = 1/(k_BT)$, and the stationary density is the Boltzmann-Gibbs distribution
\begin{equation}
\label{eq:equilibrium}
  \rho(x) = Z^{-1}\,e^{-\beta U(x)}, \qquad Z = \int_{\mathcal{X}} e^{-\beta U(x)}\,\mathrm{d}x.
\end{equation}
Our framework applies equally to the underdamped (kinetic) Langevin equation, where the state includes both position and momentum and the diffusion tensor is degenerate (see~\Cref{subsec:examples}); the marginal of the stationary density over the positions is the same Boltzmann-Gibbs distribution~\eqref{eq:equilibrium}. When $\beta$ is large, $\rho$ concentrates around the local minima of~$U$, and transitions between them become rare. More generally, metastability can also arise in high-dimensional systems from entropic effects, even at moderate temperatures, as well as in non-reversible (non-gradient) dynamics where the stationary density is not known in closed form.

\paragraph{Committor function.}
Let $A,B\subset\mathcal{X}$ be two disjoint closed sets representing the reactant and product regions, respectively. In practice, these sets are chosen to be metastable states of the dynamics~\eqref{eq:reference0}: each should concentrate significant stationary probability mass, and they should be separated by a region where $\rho$ is very small, so that transitions between them are rare. Define the first hitting time $\tau = \inf\{t\ge0 : X_t\in A\cup B\}$, and let $\mathbb{P}_x$ denote the law of the process~\eqref{eq:reference0} started at $X_0 = x$. The \emph{committor function} (in short \emph{committor}, also known as splitting probability~\citep{onsager1938initial} or $p_{\mathrm{fold}}$~\citep{du1998transition}) is the function $q:\mathcal{X}\to[0,1]$ given by \looseness=-1
\begin{equation}
\label{eq:committor}
  q(x) \;=\; \mathbb{P}_x \bigl(X_{\tau}\in B\bigr),
\end{equation}
i.e.\ the probability that a trajectory of~\eqref{eq:reference0} starting at~$x$ reaches~$B$ before~$A$. By definition, $q=1$ in~$B$ and $q=0$ in~$A$; the complementary probability of reaching~$A$ before~$B$ is simply $1-q$. The level sets of~$q$ foliate the transition region $\mathcal{X}\setminus(A\cup B)$: in particular, the isocommittor surface $\{x:q(x)=\tfrac{1}{2}\}$ is the stochastic separatrix, and the stationary density restricted to this surface identifies the transition state ensemble. Once an accurate committor is available, key kinetic and thermodynamic quantities, including reaction rates, rate constants, and the density of reactive trajectories, follow from standard TPT identities~\citep{vanden2010transition} (see \Cref{sec:background}).

\paragraph{Challenges and prior work.}
Computing~$q$ in high dimensions is difficult because the committor is most informative in the transition region, which is rarely visited by unbiased dynamics. Existing approaches exploit several equivalent characterizations: $q$ solves a stationary backward Kolmogorov equation (BKE)~\citep{vanden2006towards}, admits a Feynman-Kac representation~\citep{li2020solving,strahan2023predicting}, and can be learned from first-hitting labels via maximum likelihood~\citep{peters2006obtaining}. In practice, each of these formulations requires effective sampling of the transition region, which has motivated the coupling of enhanced sampling mechanisms (umbrella sampling, metadynamics, and related techniques) with the training loss~\citep{rotskoff2022active,hasyim2022supervised,li2019computing,lin2024deep,kang2024computing,trizio2025everything,wang2025estimating}.
Control-based approaches provide a complementary perspective: the logarithm of the committor can be characterized as the value function of a stochastic optimal control (SOC) problem, and the gradient of the log-committor is related to the optimal control, which can be used to sample reactive trajectories.
This connection underlies earlier control formulations, variational learning of conditioned path ensembles, and adaptive importance sampling algorithms
\citep{hartmann2013characterization,hartmann2019variational,hartmann2026importance,singh2023variational,singh2024splitting,yuan2023optimal}. Nevertheless, most previous approaches lack theoretical guarantees, principled learning algorithms and scalable implementation that tackles high-dimensional problems. A more detailed discussion is provided in~\Cref{subsec:related_soc_committor}. For a broad review of committor estimation, we refer readers to \citep{weare2026machine}.

\paragraph{Our approach.}
We introduce REACT, a framework for committor learning that builds on the stochastic optimal control characterization of the problem. A rescaled committor $q_\xi$ defines a bounded value function ($\Phi=-\log q_\xi$), whose gradient induces the feedback control ($u^\star=\sigma^\top\nabla\log q_\xi$). We use this relationship to couple committor estimation with controlled sampling, and develop capped-horizon fixed-point formulations that preserve the underlying hitting-time solution through bootstrapping. The resulting framework permits short training rollouts without requiring trajectories to reach either target set.

On the algorithmic side, we develop two training losses for the resulting hitting-time control problem. First, we present a heuristic on-policy direct backpropagation (DBP) loss that directly optimizes the controlled-trajectory cost by backpropagating through the dynamics.
Second, we introduce Value Matching (VM), a principled off-policy loss for learning the value function from trajectories generated under an admissible sampling drift and noise level. We establish that the capped-horizon fixed-point equations uniquely characterize the true value function and that, under mild assumptions, first-order critical points of the VM loss in function space recover this value function. Being off-policy, VM also gives the flexibility to sample the training trajectories from alternative dynamics. In particular, we introduce a sampling process that preserves the reactive current via balancing the noise level and potential strength that can accelerate exploration.

The SOC formulation and the training losses are developed for the general dynamics~\eqref{eq:reference0}, without assuming reversibility. Our numerical experiments demonstrate the effectiveness of the proposed approach on both reversible and non-reversible dynamics, including a periodically-driven one.

\textbf{Contributions.} Our main contributions are:
\begin{itemize}[left=3pt,itemsep=-1pt]
  \item \textit{Practical committor learning through SOC}: a principled formulation that grounds committor learning on capped hitting-time trajectories and preserves the original solution.
  \item \textit{Two training losses:} a heuristic on-policy direct backpropagation loss (REACT-DBP) and a principled off-policy Value Matching loss (REACT-VM), which provides efficient reuse of simulated trajectories and gives the flexibility to sample from alternative dynamics, such as reactive current-preserving dynamics that can accelerate exploration in metastable systems.
  \item \textit{Theoretical guarantees:}
  uniqueness of the capped-horizon fixed-point solutions and first-order optimality of the VM loss in function space.
  \item \textit{Empirical results:} significantly improved accuracy over baseline approaches on benchmark systems for committor learning, reaction rate estimation, and transition path sampling, and significantly faster learning of the committor than the uncontrolled semigroup approach on underdamped Langevin dynamics for high-dimensional molecular systems, from alanine dipeptide isomerization in vacuum to the folding of a mini-protein in explicit solvent with thousands of atoms.
\end{itemize}

\paragraph{Outline.}
\Cref{sec:background} reviews transition path theory and standard committor formulations. \Cref{sec:methodology} develops the SOC connection and its sampling implications. \Cref{sec:solving_SOC} presents practical training losses and algorithms, including direct backpropagation and Value Matching, and we validate the framework empirically in~\Cref{sec:experiment}.

\section{Transition path theory and committor estimation} \label{sec:background}

\subsection{The committor and its role in TPT} \label{subsec:committor_tpt}

\paragraph{Setup.}
Consider the diffusion process
\begin{equation}\label{eq:SDE_bg}
\mathrm{d}X_t = b(X_t)\,\mathrm{d}t + \sigma\,\mathrm{d}W_t,
\end{equation}
with generator $\mathcal{L} = b\cdot\nabla + D:\nabla^2$, where $D = \tfrac{1}{2}\sigma\sigma^{\top}$ is the (constant) diffusion tensor\footnote{In the appendices, we consider state-dependent diffusion tensors for completeness.}, and $D:\nabla^2 \triangleq \sum_{i,j=1}^d D_{ij} \partial_{ij}$. We assume the process is ergodic with stationary density~$\rho$ satisfying the Fokker-Planck equation
\begin{equation}\label{eq:FPE_bg}
\nabla \cdot (b\rho - D\nabla\rho) = 0.
\end{equation}
The time-reversed process has generator $\tilde{\mathcal{L}} = \tilde b\cdot\nabla + D:\nabla^2$ with drift $\tilde b = -b + 2\rho^{-1}D\nabla\rho$. When detailed balance holds ($b\rho = D\nabla\rho$), the process is reversible and $\tilde{\mathcal{L}} = \mathcal{L}$.

\paragraph{Committor PDEs.}
Let $A$ and $B$ be two disjoint target sets. The \emph{forward committor} $q(x) = \mathbb{P}_x(X_{\tau} \in B)$, with $\tau$ the first hitting time of $A\cup B$ as in~\eqref{eq:committor}, is the probability of reaching~$B$ before~$A$, and satisfies
\begin{equation}\label{eq:BKE_bg}
\begin{cases}
\mathcal{L}\,q(x) = 0, & x \notin A \cup B, \\
q(x) = 0, & x \in A, \\
q(x) = 1, & x \in B.
\end{cases}
\end{equation}
The \emph{backward committor} $\tilde q(x)$ is the probability that the process last came from~$A$ rather than~$B$, and satisfies
\begin{equation}\label{eq:BKE_back_bg}
\begin{cases}
\tilde{\mathcal{L}}\,\tilde q(x) = 0, & x \notin A \cup B, \\
\tilde q(x) = 1, & x \in A, \\
\tilde q(x) = 0, & x \in B.
\end{cases}
\end{equation}
When detailed balance holds, $\tilde q = 1 - q$. The discrete-space, discrete-time analog of this setup, including finite-horizon, time-inhomogeneous chains in the spirit of \citet{helfmann2020extending}, is developed in \Cref{sec:discrete_mc}.

\paragraph{Reactive trajectories.}
Let $(X_t)_{t\in\mathbb{R}}$ be a stationary trajectory of~\eqref{eq:SDE_bg}. For a time $t$ with $X_t \notin A\cup B$, let $t^- = \sup\{s \le t : X_s \in A\cup B\}$ be the last exit time from $A\cup B$ before~$t$, and $t^+ = \inf\{s \ge t : X_s \in A\cup B\}$ the first entrance time into $A\cup B$ after~$t$. The time $t$ is \emph{reactive} if $X_{t^-} \in A$ and $X_{t^+} \in B$. The \emph{reactive trajectories} are the segments $(X_s)_{s\in[t^-,t^+]}$ of reactive times: the pieces of the stationary trajectory that leave~$A$ and reach~$B$ without returning to~$A$. In this notation, $q(x) = \mathbb{P}(X_{t^+}\in B \mid X_t = x)$ and $\tilde q(x) = \mathbb{P}(X_{t^-}\in A \mid X_t = x)$. Throughout, we consider the $A\to B$ direction; the $B\to A$ case follows by exchanging $A$ and $B$, which replaces $q$ and $\tilde q$ by $1-q$ and $1-\tilde q$.

\paragraph{Reactive density.}
Once $q$ and $\tilde q$ are known, transition path theory~\citep{vanden2006towards,vanden2010transition} provides a complete statistical characterization of reactive trajectories. The \emph{reactive density}
\begin{equation}\label{eq:reactive_density_bg}
\rho_\mathrm{R}(x) = \frac{\rho(x)\,q(x)\,\tilde q(x)}{Z_{AB}}, \qquad Z_{AB} = \int_{\mathcal{X}} \rho(x)\,q(x)\,\tilde q(x) \, \mathrm{d}x,
\end{equation}
gives the probability density of observing the system at~$x$, conditioned on it being on a reactive trajectory. In the reversible case, $\tilde q = 1-q$ and $\rho_\mathrm{R} \propto \rho\,q(1-q)$.

\paragraph{Doob's $h$-transform.}Conditioning the dynamics~\eqref{eq:SDE_bg} on reaching~$B$ before~$A$ produces a process with a modified drift~\citep{hartmann2013characterization,lu2015reactive}:
\begin{equation}\label{eq:doob_bg}
\mathrm{d}X_t = \big(b(X_t) + 2D\nabla\log q(X_t)\big)\,\mathrm{d}t + \sigma\,\mathrm{d}W_t.
\end{equation}
This is \emph{Doob's $h$-transform} of the reference process \eqref{eq:SDE_bg} with $h = q$: the additional drift $2D\nabla\log q$ steers the process toward~$B$, and the resulting path measure is exactly the law of~\eqref{eq:SDE_bg} conditioned on the event $\{X_{\tau} \in B\}$ (see \Cref{app:doob_derivation} for derivation).
To obtain reactive trajectories, \eqref{eq:doob_bg} is started on~$\partial A$ from the reactive flux out of~$A$, with density proportional to $\rho\, n_A\cdot D\nabla q$, where $n_A$ is the outward unit normal of~$\partial A$, and run until it hits~$\partial B$.

The reactive density~$\rho_\mathrm{R}$ is the stationary density of this process in $(A\cup B)^c$ (with source on~$\partial A$ and sink on~$\partial B$). Equivalently, its probability current
\begin{equation}\label{eq:mod_FPE_bg}
J_\mathrm{R} := (b + 2D\nabla\log q)\,\rho_\mathrm{R} - D\nabla\rho_\mathrm{R}
\end{equation}
is divergence-free in $(A\cup B)^c$: $\nabla \cdot J_\mathrm{R} = 0$ (proof in \Cref{thm:stationary}). An algebraic computation (see \Cref{lem:flux_reactive}) shows that $J_\mathrm{R}$ admits the explicit form
\begin{equation}\label{eq:reactive_current_bg}
J_\mathrm{R} = \frac{1}{Z_{AB}} \big(q\tilde q\,J + \rho\tilde q\,D\nabla q - \rho q\,D\nabla\tilde q \big),
\end{equation}
where $J = b\rho - D\nabla\rho$ is the stationary probability current of the original process. In the reversible case ($J = 0$, $\tilde q = 1-q$), this reduces to $J_\mathrm{R} = \frac{1}{Z_{AB}} \rho\,D\nabla q$. $J_{\mathrm{R}}$ is discontinuous at $\partial A$ and $\partial B$, because $\nabla q$, $\nabla \tilde q$ are zero on $A \cup B$, and discontinuous at their boundaries. At stationarity, the current discontinuity must be balanced out with mass creation at $\partial A$ (the \emph{source}) and mass destruction at $\partial B$ (the \emph{sink}); see \Cref{lem:mass_creation_destruction} for the source and sink rates.

More generally, for any $\kappa \in (0,+\infty)$, the process
\begin{equation}\label{eq:reactive_SDE_kappa_gen}
\mathrm{d}X_t = b_\kappa(X_t)\,\mathrm{d}t + \sqrt{\kappa}\,\sigma\,\mathrm{d}W_t, \qquad b_\kappa = \frac{1+\kappa}{2}(b + 2D\nabla\log q) + \frac{\kappa-1}{2}(\tilde{b} + 2D\nabla\log\tilde q)
\end{equation}
also has stationary density $\rho_\mathrm{R}$ and stationary current $J_\mathrm{R}$ under the same source-sink conditions (see \Cref{prop:kappa_app}).
The drift~$b_\kappa$ interpolates between the forward Doob's drift $b + 2D\nabla\log q$ (at $\kappa = 1$) and the backward Doob's drift $\tilde b + 2D\nabla\log\tilde q$, which enters with a minus sign as $\kappa \to 0$.
In the limit $\kappa \to 0$, the noise vanishes and the dynamics reduce to the deterministic ODE $\dot X_t = J_\mathrm{R}(X_t)/\rho_\mathrm{R}(X_t)$, whose trajectories follow the streamlines of the reactive current. Choosing~$\kappa\neq1$ changes the effective noise and energy barrier, which can accelerate exploration, but requires knowledge of~$\rho$ and~$\tilde q$ in addition to~$q$.

\paragraph{Reaction statistics.}
The \emph{reaction rate} $\nu_\mathrm{R}$ measures the total reactive flux from~$A$ to~$B$ through any dividing surface $\Sigma$ separating~$A$ from~$B$:
\begin{equation}\label{eq:reaction_rate_bg}
\nu_\mathrm{R} = Z_{AB}\int_\Sigma J_\mathrm{R} \cdot n\,\mathrm{d}S.
\end{equation}
The result is independent of~$\Sigma$ since $\nabla \cdot J_\mathrm{R} = 0$. It admits a volume integral representation that depends only on the forward committor~\citep{vanden2006towards,lu2015reactive}:
\begin{equation}\label{eq:reaction_rate_vol}
\nu_\mathrm{R} = \int_{\mathcal{X}} \rho\,\nabla q \cdot D\nabla q\,\mathrm{d}x.
\end{equation}
To see this, integrate $\nabla \cdot (\rho\,q\,D\nabla q) = \rho\,\nabla q \cdot D\nabla q
+ q \nabla \cdot (\rho\,D\nabla q)$ over $(A\cup B)^c$: the left-hand side becomes a boundary integral that equals $\nu_\mathrm{R}$ (since at~$\partial A$, where $q = 0$ and $\tilde q = 1$, we have $Z_{AB} J_\mathrm{R} = \rho\,D\nabla q$), and the second term on the right vanishes because $\mathcal{L}q = 0$ and $\nabla \cdot J = 0$.

The quantities
\begin{equation}\label{eq:p_AB}
p_B = \int_{\mathcal{X}} q(x)\,\rho(x)\,\mathrm{d}x, \qquad p_A = \int_{\mathcal{X}} (1-q(x))\,\rho(x)\,\mathrm{d}x = 1 - p_B
\end{equation}
give the probability that the process will next reach~$B$ or~$A$, respectively, when it is at equilibrium.
The \emph{directional rate constants}
\begin{equation}\label{eq:rate_constants_bg}
k_{AB} = \frac{\nu_\mathrm{R}}{p_A}, \qquad k_{BA} = \frac{\nu_\mathrm{R}}{p_B}
\end{equation}
quantify the reaction frequency conditioned on the process heading toward~$A$ or toward~$B$.

\emph{In summary}, an accurate forward committor~$q$ suffices to compute the reaction rate~$\nu_\mathrm{R}$, the rate constants $k_{AB}$ and $k_{BA}$, and the reactive drift $b + 2D\nabla\log q$. The backward committor~$\tilde q$ is additionally needed for the reactive density~$\rho_\mathrm{R}$ and the reactive current~$J_\mathrm{R}$.

\subsection{Examples}\label{subsec:examples}

\paragraph{Overdamped Langevin dynamics.}
A canonical instance of~\eqref{eq:SDE_bg} is the overdamped Langevin equation with potential~$U$ and inverse temperature $\beta = 1/(k_BT)$:
\begin{equation}\label{eq:SDE_rev}
\mathrm{d}X_t = -\nabla U(X_t)\,\mathrm{d}t + \sqrt{2\beta^{-1}}\,\mathrm{d}W_t,
\end{equation}
for which $b = -\nabla U$, $\sigma = \sqrt{2\beta^{-1}}\,I_d$, and $D = \beta^{-1} I_d$. The stationary density is $\rho \propto e^{-\beta U}$, and detailed balance holds ($b\rho = D\nabla\rho$, equivalently $J = 0$), so the process is reversible. Consequently $\tilde{\mathcal{L}} = \mathcal{L}$, $\tilde q = 1-q$, and the generator can be written in the self-adjoint form $\mathcal{L}q = \beta^{-1}\rho^{-1}\nabla\cdot(\rho\nabla q)$. The committor minimizes the Dirichlet energy
\begin{equation}\label{eq:kbe_variational}
q = \arg\min_{\substack{\hat q|_A = 0,\, \hat q|_B = 1}} \beta^{-1}\int_{\mathcal{X}} |\nabla \hat q|^2\,\rho\,\mathrm{d}x.
\end{equation}
Since $\tilde q = 1-q$, the reactive density, current, and velocity simplify to
\begin{align}
\rho_\mathrm{R}(x) &= \frac{\rho(x)\,q(x)\,(1-q(x))}{Z_{AB}}, \label{eq:reactive_density_rev} \\
J_\mathrm{R}(x) &= \frac{\beta^{-1}\rho(x)\,\nabla q(x)}{Z_{AB}}, \label{eq:reactive_current_rev} \\
b_\mathrm{R}(x) &= \frac{J_\mathrm{R}(x)}{\rho_\mathrm{R}(x)} = \beta^{-1}\,\frac{\nabla q(x)}{q(x)(1-q(x))} = \beta^{-1}\big(\nabla\log q(x) - \nabla\log(1-q(x))\big). \label{eq:reactive_velocity_rev}
\end{align}
The decomposition of the reactive velocity into $\nabla\log q$ and $-\nabla\log(1-q)$ reflects the two forces driving the transition: attraction toward~$B$ and repulsion from~$A$.

The Doob's $h$-transformed process~\eqref{eq:doob_bg} becomes $\mathrm{d}X_t = (-\nabla U + 2\beta^{-1}\nabla\log q)\,\mathrm{d}t + \sqrt{2\beta^{-1}}\,\mathrm{d}W_t$. The $\kappa$-family~\eqref{eq:reactive_SDE_kappa_gen} simplifies to
\begin{equation}\label{eq:reactive_SDE_kappa}
\begin{split}
\mathrm{d}X_t &= \kappa\, b(X_t)\,\mathrm{d}t + \beta^{-1}\big((1+\kappa)\nabla\log q(X_t) - (1-\kappa)\nabla\log(1-q(X_t))\big)\,\mathrm{d}t
\\ &\quad
+ \sqrt{2\kappa\beta^{-1}}\,\mathrm{d}W_t,
\end{split}
\end{equation}
and in the limit $\kappa \to 0$ reduces to the deterministic ODE
\begin{equation}\label{eq:reactive_ODE}
\dot{X}_t = \beta^{-1}\big(\nabla\log q(X_t) - \nabla\log(1-q(X_t))\big) = b_\mathrm{R}(X_t),
\end{equation}
whose trajectories follow the streamlines of~$J_\mathrm{R}$, with $Z_{AB} = \int_{\mathcal{X}} q(1-q)\rho\,\mathrm{d}x$.

The reaction rate becomes $\nu_\mathrm{R} = \beta^{-1}\int_{\mathcal{X}} |\nabla q|^2\,\rho\,\mathrm{d}x$, and the equilibrium constant
\begin{equation}\label{eq:K_eq}
K_{\mathrm{eq}} = \frac{k_{AB}}{k_{BA}} = \frac{p_B}{p_A} = \frac{\int_{\mathcal{X}} q\,\rho\,\mathrm{d}x}{\int_{\mathcal{X}} (1-q)\,\rho\,\mathrm{d}x}
\end{equation}
can be estimated from~$q$ alone. This is the setting used in our numerical experiments (\Cref{sec:result_rates}). The extension to the state-dependent setting with position-dependent (anisotropic) diffusion is developed in \Cref{eq:subsection_overdamped}.

\paragraph{Underdamped Langevin dynamics.}
A second important instance is the underdamped (kinetic) Langevin equation, in which the state variable is the pair $(r,p) \in \mathbb{R}^d \times \mathbb{R}^d$ of position and momentum (with unit mass):
\begin{equation}\label{eq:SDE_underdamped}
\begin{cases}
\mathrm{d}R_t = P_t\,\mathrm{d}t, \\[3pt]
\mathrm{d}P_t = -\nabla U(R_t)\,\mathrm{d}t - \gamma\, P_t\,\mathrm{d}t + \sqrt{2\gamma\beta^{-1}}\,\mathrm{d}W_t,
\end{cases}
\end{equation}
where $U$ is the potential, $\gamma > 0$ the friction coefficient, and $\beta = 1/(k_BT)$ the inverse temperature. This is a special case of~\eqref{eq:SDE_bg} with full state
$X = (R,P)$, drift $b(r,p) = (p,\, -\nabla U(r) - \gamma p)$,
and noise matrix~$\sigma$ having zero upper block and $\sqrt{2\gamma\beta^{-1}}\,I_d$ lower block. The diffusion tensor $D = \tfrac{1}{2}\sigma\sigma^{\top}$ is degenerate: it acts only on the momentum components, so the generator is hypoelliptic rather than elliptic.

The stationary density is the Boltzmann-Gibbs distribution
$\rho(r,p) \propto \exp\!\big(-\beta\big(U(r) + \tfrac{1}{2}|p|^2\big)\big)$.
The dynamics are \emph{not} reversible: the stationary current $J = b\rho - D\nabla\rho$ does not vanish (the Hamiltonian part of the drift is antisymmetric). However, the process satisfies a \emph{momentum-reversal symmetry}: letting
$\mathcal{R}(r,p) = (r,-p)$
denote the momentum-flip involution, the time-reversed process is obtained by flipping the sign of the momentum. Consequently, the backward committor is related to the forward one by
\begin{equation}\label{eq:qtilde_underdamped}
\tilde q(r,p) = 1 - q(r,-p),
\end{equation}
provided the sets $A$ and $B$ are symmetric under $\mathcal{R}$ (i.e.\
$(r,p) \in A \iff (r,-p) \in A$,
and similarly for $B$). In many applications the sets $A$ and $B$ are defined by position alone
($A = A_{r} \times \mathbb{R}^d$, $B = B_{r} \times \mathbb{R}^d$),
in which case this symmetry holds automatically.

The relation~\eqref{eq:qtilde_underdamped} implies that the reactive density and current can be expressed in terms of the forward committor alone, just as in the reversible case, though the formulas are different:
\begin{equation}\label{eq:reactive_density_underdamped}
\rho_\mathrm{R}(r,p) = \frac{\rho(r,p)\,q(r,p)\,(1 - q(r,-p))}{Z_{AB}}.
\end{equation}
In particular, $\rho_\mathrm{R}$ is not symmetric under momentum reversal, reflecting the directional character of reactive trajectories. The state-dependent generalization, with position-dependent diffusion and anisotropic mass and friction matrices, is developed in \Cref{eq:subsection_underdamped}.

\subsection{Existing approaches to committor estimation} \label{subsec:existing_approaches}

\paragraph{PDE/variational formulation.} The committor satisfies the backward Kolmogorov equation $\mathcal{L}q = 0$ in $(A\cup B)^c$ with boundary conditions $q|_A = 0$, $q|_B = 1$. In the reversible case, self-adjointness of~$\mathcal{L}$ implies that this is the Euler-Lagrange equation of the variational problem
\begin{equation}
\label{eq:kbe_variational_ea}
q = \arg\min_{\hat q}\int_{\mathcal{X}}\nabla\hat q(x) \cdot D\nabla\hat q(x)\,\rho(x)\,\mathrm{d}x,
\end{equation}
subject to $\hat q|_A=0$ and $\hat q|_B=1$. Neural networks and tensor networks~\citep{chen2023committor} have been used to parameterize $\hat q$, with boundary conditions enforced either as soft constraints~\citep{khoo2019solving} or through the parameterization itself~\citep{li2019computing}. Since the committor varies most steeply in the transition region $\mathcal{X}\setminus(A\cup B)$, accurate learning requires sampling configurations that are seldom visited by unbiased dynamics. To address this, committor learning has been coupled with enhanced sampling techniques such as umbrella sampling~\citep{rotskoff2022active,hasyim2022supervised}, metadynamics~\citep{li2019computing,lin2024deep}, and biased potentials constructed from intermediate committor estimates~\citep{kang2024computing}. Several approaches exploit an iterative loop in which a provisional committor guides sampling, which in turn refines the estimate~\citep{rotskoff2022active,kang2024computing,wang2025estimating}. Despite their empirical success, these methods leave open the question of how to optimally couple committor estimation and sampling in a principled and theoretically grounded manner. Moreover, the variational formulation~\eqref{eq:kbe_variational_ea} is specific to the reversible case; for non-reversible dynamics, one must work directly with the BKE or with probabilistic representations. Alternatively, one can directly minimize the PDE residual $\int|\mathcal{L}\hat q|^2\,\rho\,\mathrm{d}x$ (plus boundary penalties), an approach that applies to the general (possibly non-reversible) case since it does not require self-adjointness~\citep{khoo2019solving,raissi2019physics}.\looseness=-1

\paragraph{Feynman-Kac formulation.} Because $q$ satisfies the BKE, one can cast learning as matching conditional expectations under the Markov semigroup:
\begin{equation}
\label{eq:fk}
\begin{split}
&q = \arg\min_{\hat q}\;\mathbb{E}_{X_0\sim \mu_0}\bigl[\bigl|\hat q(X_0)-\mathbb{E}\bigl[\hat q(X_{\tau(T)})\mid X_0\bigr]\bigr|^2\bigr] \\
&\text{s.t.}\quad \hat q|_A=0,\quad \hat q|_B=1,
\end{split}
\end{equation}
where $\mu_0$ is an arbitrary initial distribution with full support on $\mathcal{X}\setminus(A\cup B)$ and $\tau(T):=\min\{T,\,\tau\}$ is the hitting time capped at horizon $T$. \cite{li2022semigroup,strahan2021long,strahan2023predicting} proposed losses based on~\eqref{eq:fk}, while \cite{mitchell2024committor} considered a logarithmic variant.\footnote{\cite{li2022semigroup,strahan2023predicting,mitchell2024committor} used a fixed horizon $T$ rather than $\tau(T)$. Using $\tau(T)$ is preferable because it assigns a reward to every trajectory that reaches $A\cup B$ before $T$.} Prior to our work, these approaches lacked firm theoretical grounding: (i) uniqueness of the committor as the minimizer of~\eqref{eq:fk} (or its logarithmic version) had not been established, and (ii) no algorithmic guarantees such as first-order optimality conditions were known. A further limitation is scalability: while~\eqref{eq:fk} admits a scalable loss (see~\Cref{subsec:strahan}), the logarithmic version does not (see~\Cref{subsec:committor}).

\paragraph{Maximum-likelihood formulation.} \cite{peters2006obtaining} formulated committor estimation as maximum likelihood:
\begin{equation}
\label{eq:mle}
q = \arg\min_{\hat q}\;-\,\mathbb{E}_{X_0\sim \mu_0}\bigl[\mathbf{1}_B(X_\tau)\log\hat q(X_0) + \mathbf{1}_A(X_\tau)\log\bigl(1-\hat q(X_0)\bigr)\bigr],
\end{equation}
where $\mu_0$ is again an arbitrary initial distribution with full support on $\mathcal{X}\setminus(A\cup B)$. Neural networks have been used to parameterize $\hat q$~\citep{jung2019artificial,sun2022multitask,jung2023machine}, but this formulation requires reactive trajectories, which are typically generated by transition path sampling (TPS) shooting~\citep{hummer2004transition}.

\paragraph{The sampling bottleneck.}
All three formulations ultimately hinge on obtaining informative samples from the transition region, a difficulty that manifests regardless of whether the loss is posed in PDE, probabilistic, or likelihood form. This motivates methods that actively bias sampling toward reactive configurations. The stochastic optimal control framework we develop next provides a principled mechanism to do so: it uses a running committor estimate to define a control that guides sampling, and it yields training losses with theoretical guarantees.

\section{Committor estimation via stochastic optimal control}
\label{sec:methodology}

We now show how stochastic optimal control (SOC) provides a unified framework for committor estimation\footnote{\citet[Sec.~6.1,~7.2]{hartmann2013characterization} were the first to formulate committor estimation as a SOC problem; we compare their framework with ours in~\Cref{subsec:hartmann}.}. A key ingredient is a rescaling of the boundary conditions that removes the logarithmic singularities inherent in the Cole-Hopf transform and in the optimal control drifts.

\paragraph{Rescaled committors.} Throughout, we write $\mathcal{S} := A \cup B$ for the union of the two target sets and $\mathcal{S}^c := \mathcal{X}\setminus\mathcal{S}$. Let $\xi \in (0,\tfrac{1}{2})$ be a small rescaling parameter.
Rather than imposing $q|_A = 0$ and $q|_B = 1$ (resp. $\tilde q|_A = 1$ and $\tilde q|_B = 0$) as in~\eqref{eq:BKE_bg}-\eqref{eq:BKE_back_bg}, we consider the solutions of the BKEs with shifted boundary conditions:
\begin{equation}
\label{eq:kbe_xi}
    \begin{cases}
        \mathcal{L} q_{\xi}(x) = 0, &x
        \notin A \cup B, \\
        q_{\xi}(x) = \xi, &x \in A, \\
        q_{\xi}(x) = 1 - \xi, &x \in B.
    \end{cases}
    \qquad\qquad\qquad
    \begin{cases}
        \tilde{\mathcal{L}} \tilde q_{\xi}(x) = 0, &x
        \notin A \cup B, \\
        \tilde q_{\xi}(x) = 1- \xi, &x \in A, \\
        \tilde q_{\xi}(x) = \xi, &x \in B.
    \end{cases}
\end{equation}

We refer to the solutions of \eqref{eq:kbe_xi} as the rescaled forward and backward committor functions. The rescaled committors are related to the original committors $q, \tilde q$ that solve~\eqref{eq:BKE_bg}-\eqref{eq:BKE_back_bg} by the affine maps
\begin{equation} \label{eq:rescaled_committor_committor}
\begin{aligned}
q_{\xi} &= \xi (1-q) + (1-\xi)\,q = \xi + (1-2\xi)\,q, \\
\tilde{q}_{\xi} &= \xi(1-\tilde{q}) + (1-\xi)\,\tilde{q} = \xi + (1-2\xi)\,\tilde{q},
\end{aligned}
\end{equation}
so that learning $q_{\xi}$ (resp. $\tilde q_{\xi}$) is equivalent to learning $q$ (resp. $\tilde{q}$). This rescaling is equivalent, up to a constant normalization, to the shifted-committor formulation of \citet{hartmann2026importance}.

Below we describe the SOC framework for learning forward committors; backward committors can be learned analogously, replacing $\mathcal{L}$, $b$ by $\tilde{\mathcal{L}}$, $\tilde{b}$, and flipping the boundary conditions, although that is not considered explicitly in this paper.

\paragraph{Cole-Hopf transform and HJB equation.} Since $q_{\xi} \geq \xi > 0$ everywhere, the value function
\begin{keyeqbox}
\begin{equation} \label{eq:cole_hopf}
    \Phi(x) = -\log q_{\xi}(x)
\end{equation}
\end{keyeqbox}
is bounded, with $\Phi|_A = -\log\xi$ and $\Phi|_B = -\log(1-\xi)$. The Cole-Hopf transform turns the linear BKE~\eqref{eq:kbe_xi} into a stationary Hamilton-Jacobi-Bellman (HJB) equation:
\begin{equation}
\label{eq:HJB_setup}
    \begin{cases}
        0 = - \mathcal{L} \Phi(x) + |\nabla \Phi(x)|^2_D, \qquad &x \in \mathcal{S}^c, \\
        \Phi(x) = -\log \xi, \qquad &x \in A, \\
        \Phi(x) = -\log(1-\xi), \qquad &x \in B,
    \end{cases}
\end{equation}
where $|\nabla\Phi|_D^2 = \nabla\Phi \cdot D\nabla\Phi$.

\begin{proof}[Proof sketch of the equivalence BKE $\leftrightarrow$ HJB]
Since $q_{\xi} = e^{-\Phi}$,
\begin{equation}
    \nabla q_{\xi} = -\nabla\Phi\, e^{-\Phi}, \qquad \frac{\partial^2 q_{\xi}}{\partial x_i \partial x_j} = \Big(-\frac{\partial^2 \Phi}{\partial x_i \partial x_j} + \frac{\partial \Phi}{\partial x_i}\frac{\partial \Phi}{\partial x_j}\Big) e^{-\Phi}.
\end{equation}
Substituting into~\eqref{eq:kbe_xi} and dividing by $e^{-\Phi}$ yields~\eqref{eq:HJB_setup}, and conversely.
\end{proof}

\paragraph{SOC hitting-time problem.} Equation~\eqref{eq:HJB_setup} arises as the HJB equation of the following SOC problem.

\begin{thmbox}
\begin{proposition} \label{prop:SOC}
Let $\mu_0$ be any distribution with full support on $\mathcal{S}^c$. Consider the controlled process
\begin{equation}
    \label{eq:controlled_SDE_main}
    \mathrm{d}X^u_t = \big(b(X^u_t) + \sigma\,u(X^u_t)\big) \, \mathrm{d}t + \sigma\,\mathrm{d}W_t, \qquad X^u_0 \sim \mu_0,
\end{equation}
where $\tau$ is the first hitting time of $\mathcal{S}$, as in~\eqref{eq:committor}. Let $\mathbb{E}$ denote expectation over both $X^u_0 \sim \mu_0$ and the Brownian noise in~\eqref{eq:controlled_SDE_main}. Then the minimizer of
\begin{equation} \label{eq:control_problem_def_main}
    \min_{u \in \mathcal{U}} \mathbb{E}\Big[\frac{1}{2} \int_0^{\tau} |u(X^u_t)|^2 \, \mathrm{d}t + g(X^u_{\tau})\Big], \qquad X^u_0 \sim \mu_0,
\end{equation}
with the bounded terminal cost
\begin{equation} \label{eq:terminal_cost}
    g(x) = \begin{cases} -\log\xi, & x \in A, \\ -\log(1-\xi), & x \in B, \end{cases}
\end{equation}
is the Markovian control
\begin{equation} \label{eq:u_star}
    u^{\star}(x) = -\sigma^{\top}\nabla \Phi(x) = \sigma^{\top}\nabla \log q_{\xi}(x),
\end{equation}
where $\Phi$ is the value function for \eqref{eq:control_problem_def_main}, which is defined as
\begin{equation} \label{eq:value_function}
    \Phi(x) = \min_{u \in \mathcal{U}} \mathbb{E}_x\Big[ \frac{1}{2}\int_0^{\tau} |u(X^u_t)|^2 \, \mathrm{d}t + g(X^u_{\tau})\Big],
\end{equation}
and is related to the rescaled committor $q_{\xi}$ through equation \eqref{eq:cole_hopf}.
\end{proposition}
\end{thmbox}

The proof of this proposition is given in~\Cref{eq:vf_regularity}. A discrete-time, discrete-space analog of \Cref{prop:SOC}, with tilted transition kernels in place of drift controls and a log-sum-exp Bellman equation in place of~\eqref{eq:HJB_setup}, is developed in \Cref{sec:discrete_mc}. The condition that $\mu_0$ has full support on $\mathcal{S}^c$ ensures that the control objective~\eqref{eq:control_problem_def_main} constrains $u$ everywhere on $\mathcal{S}^c$; the choice of $\mu_0$ does not affect the optimizer $u^{\star}$ but does affect the conditioning of the training losses in practice, a point we return to in \cref{sec:solving_SOC}. The value function \eqref{eq:value_function}
admits a useful path-integral characterization:
\begin{equation} \label{eq:value_function_path_integral}
    \Phi(x) = -\log \mathbb{E}_x\big[ \exp \big( - g(X_{\tau}) \big)\big],
\end{equation}
where $X$ solves~\eqref{eq:reference0} and $\mathbb{E}_x$ denotes expectation under $\mathbb{P}_x$, the law of the process started at $X_0 = x$. This follows from the Feynman-Kac formula: setting $\varphi(x) = \mathbb{E}_x[\exp(-g(X_\tau))]$ yields $\mathcal{L}\varphi = 0$ on $\mathcal{S}^c$ with boundary data $\varphi|_{\partial\mathcal{S}} = e^{-g}$, and the Cole-Hopf transform $\varphi = e^{-\Phi}$ recovers~\eqref{eq:HJB_setup}.

\begin{remark}[Probabilistic view of SOC] \label{rem:SOC_probabilistic}
The control objective~\eqref{eq:control_problem_def_main} can be obtained from a KL divergence between path measures. Let $\mathbb{P}$ denote the path measure of the uncontrolled process~\eqref{eq:reference0} with $X_0 \sim \mu_0$, and let $\mathbb{P}^{u}$ denote the path measure of the controlled process~\eqref{eq:controlled_SDE_main}. Define the target path measure $\mathbb{P}^{\star}$ via the Radon-Nikodym derivative
\begin{equation}
    \frac{\mathrm{d}\mathbb{P}^{\star}}{\mathrm{d}\mathbb{P}}(X) = \frac{\exp(-g(X_{\tau}))}{Z(X_0)},
    \qquad Z(x) = \mathbb{E}_x\big[\exp(-g(X_\tau))\big] = e^{-\Phi(x)} = q_{\xi}(x).
\end{equation}
The last equality follows by separating the cases $X_\tau \in A$ and $X_\tau \in B$: since $e^{-g} = \xi$ on~$A$ and $e^{-g} = 1-\xi$ on~$B$ by~\eqref{eq:terminal_cost},
\begin{equation*}
Z(x) = \xi\,\mathbb{P}_x(X_\tau \in A) + (1-\xi)\,\mathbb{P}_x(X_\tau \in B) = \xi\,\big(1-q(x)\big) + (1-\xi)\,q(x) = q_{\xi}(x).
\end{equation*}
Using Girsanov's theorem (see~\Cref{subsec:girsanov}), we have
\begin{equation}
    \mathrm{KL}(\mathbb{P}^u \,\|\, \mathbb{P}^{\star})
    = \mathbb{E}\Big[\tfrac{1}{2}\int_0^\tau |u(X^u_t)|^2\,\mathrm{d}t + g(X^u_\tau)\Big] + \mathrm{const.}
\end{equation}
For a detailed review, see \citet{singh2025variational}.
\end{remark}

\paragraph{The optimally controlled SDE.} Per \Cref{prop:SOC}, the optimal control is given by $u^{\star}(x) = \sigma^{\top}\nabla \log q_{\xi}(x)$.
Since $q_{\xi} \in [\xi, 1-\xi]$,
the drift $2D\nabla\log q_\xi$ is non-singular everywhere; this is the main benefit of working with the rescaled committor. Any learned approximation $\phi \approx \Phi$ yields a feedback control $u(x) = -\sigma^{\top}\nabla \phi(x)$, and the optimally controlled drift in~\eqref{eq:controlled_SDE_main} becomes $b(x) + 2D\,\nabla\log q_{\xi}(x)$. That is, the optimally controlled SDE is
\begin{equation}
\label{eq:X_u_star_main}
    \mathrm{d}X^{u^{\star}}_t = \big(b(X^{u^{\star}}_t) + 2D\,\nabla \log q_{\xi}(X^{u^{\star}}_t) \big) \, \mathrm{d}t + \sigma\, \mathrm{d}W_t.
\end{equation}
This is formally similar to Doob's $h$-transform SDE~\eqref{eq:doob_bg}, with $q_\xi$ in place of $q$. The two are not equivalent:
since the drift $2D\nabla\log q_\xi$ is bounded everywhere, trajectories of~\eqref{eq:X_u_star_main} are neither fully repelled from $\partial A$ nor fully absorbed at $\partial B$, so they are not the law of
\eqref{eq:reference0} conditioned on reaching $B$ before $A$. Nevertheless, given $q_\xi$ we can recover $\nabla\log q$ by inverting the linear map~\eqref{eq:rescaled_committor_committor} and
simulate~\eqref{eq:doob_bg} to obtain reactive trajectories.

\paragraph{Rescaled reactive density.} We define the rescaled reactive density $\rho^{(\xi)}_{\mathrm{R}}$ in analogy with the reactive density \eqref{eq:reactive_density_bg}, but with the rescaled committors instead of the original ones:
\begin{equation}\label{eq:rescaled_reactive_density}
\rho^{(\xi)}_\mathrm{R}(x) = \frac{\rho(x)\,q_{\xi}(x)\,\tilde q_{\xi}(x)}{Z^{(\xi)}_{AB}}, \qquad Z^{(\xi)}_{AB} = \int_{\mathcal{X}} \rho(x)\,q_{\xi}(x)\,\tilde q_{\xi}(x) \, \mathrm{d}x.
\end{equation}
Just like $\rho_\mathrm{R}$ for Doob's SDE \eqref{eq:doob_bg}, $\rho^{(\xi)}_{\mathrm{R}}$ is stationary on $(A\cup B)^c$ for the optimally controlled SDE~\eqref{eq:X_u_star_main}, but only with a nontrivial
 source on $\partial A$ and sink on $\partial B$ (see \Cref{subsec:reactive}). The associated probability current $J^{(\xi)}_{\mathrm{R}}$ admits expressions analogous to \eqref{eq:mod_FPE_bg} and
 \eqref{eq:reactive_current_bg}, written in terms of $q_\xi$, $\tilde q_\xi$ and their gradients. As in the case $\xi = 0$ (cf. discussion after \eqref{eq:reactive_current_bg}), $J^{(\xi)}_{\mathrm{R}}$ is
 discontinuous at $\partial A$ and $\partial B$ because $\nabla q_\xi$ and $\nabla\tilde q_\xi$ vanish inside $A\cup B$ but are nonzero just outside, and at stationarity this discontinuity must be balanced by
mass
 creation on $\partial A$ (the source, sustaining the reactive current) and mass destruction on $\partial B$ (the sink, absorbing it); explicit source and sink rates are given
in \Cref{lem:mass_creation_destruction}.
 Mirroring~\eqref{eq:reactive_SDE_kappa_gen}, for any $\kappa > 0$, the solution of
\begin{equation}\label{eq:rescaled_reactive_SDE_kappa}
\mathrm{d}X_t = b_\kappa(X_t)\,\mathrm{d}t + \sqrt{\kappa}\,\sigma\,\mathrm{d}W_t, \qquad b_\kappa =
\frac{1+\kappa}{2}(b + 2D\nabla\log q_{\xi}) + \frac{\kappa-1}{2}(\tilde{b} + 2D\nabla\log\tilde q_{\xi})
\end{equation}
also has stationary density $\rho^{(\xi)}_\mathrm{R}$ and stationary current $J^{(\xi)}_\mathrm{R}$ under the same source-sink conditions (\Cref{prop:kappa_app}) as when $\kappa=1$.

\subsection{SOC problems with capped hitting times}
\label{subsec:SOC_hitting_times}

The trajectories of~\eqref{eq:control_problem_def_main}-\eqref{eq:controlled_SDE_main} may be very long when both $A$ and $B$ are hard to reach, making direct solution expensive. Following~\cite{strahan2023predicting,mitchell2024committor}, we cap trajectories at a horizon $T$ to keep costs bounded, and for paths that do not hit $A \cup B$ before $T$ we use the current estimate as terminal cost. This leads to the fixed-point problems
\begin{enumerate}[left=-5pt,label=(\roman*)]
\item \emph{Variational fixed-point problem:}
\begin{equation*} \label{eq:fixed_point_1} \tag{FP1}
\begin{cases}
    \forall x \in \mathcal{S}^c: \quad \Phi_1(x) = \inf_{u \in \mathcal{U}} \big\{ \mathbb{E}_x\big[\int_0^{\tau(T)} \frac{1}{2}|u(X_t^{u})|^2 \, \mathrm{d}t + \Phi_1(X^{u}_{\tau(T)})\big] \big\}, \\
    \forall x \in \partial \mathcal{S}: \quad \Phi_1(x) = g(x),
\end{cases}
\end{equation*}
\item \emph{Path-integral fixed-point problem:}
\begin{equation*} \label{eq:fixed_point_2} \tag{FP2}
\begin{cases}
    \forall x \in \mathcal{S}^c: \quad \Phi_2(x) = - \log \mathbb{E}_x\big[ \exp \big( - \Phi_2(X_{\tau(T)}) \big) \big], \\
    \forall x \in \partial \mathcal{S}: \quad \Phi_2(x) = g(x).
\end{cases}
\end{equation*}
\end{enumerate}
The following theorem shows that solutions to~\eqref{eq:fixed_point_1} and~\eqref{eq:fixed_point_2} are unique and coincide with $\Phi$ under mild conditions (proof in~\Cref{subsec:optimality}):
\begin{thmbox}
\begin{theorem} \label{thm:V_1_V_2}
    Suppose that $b$ is smooth with at most polynomial growth of its derivatives, that $\mathcal{S} = A \cup B \subset \mathbb{R}^d$ is a closed set, and that $g$ is as in~\eqref{eq:terminal_cost}.
    \begin{enumerate}[left=0pt,label=(\roman*)]
        \item If, in addition, $\Phi \in C^{2+\alpha}(\mathcal{S}^c)$ solves the HJB equation~\eqref{eq:HJB_setup} for some $\alpha \in (0,1)$ and the equilibrium distribution $\rho$ of~\eqref{eq:reference0} has a finite $m$-th moment, then every \emph{continuous} solution $\Phi_1$ of Eq.~\eqref{eq:fixed_point_1} that is bounded below with polynomial growth of degree at most $m$ equals $\Phi$, and the infimum in Eq.~\eqref{eq:fixed_point_1} is attained by $u^{\star} = -\sigma^{\top}\nabla \Phi$.
        \item If $\Phi_2$ is a solution of Eq.~\eqref{eq:fixed_point_2} that is bounded below and has at most polynomial growth, then $\Phi_2 = \Phi$ as defined by Eqs.~\eqref{eq:value_function}-\eqref{eq:value_function_path_integral}.
    \end{enumerate}
\end{theorem}
\end{thmbox}
Thus, finding fixed points of either~\eqref{eq:fixed_point_1} or~\eqref{eq:fixed_point_2} suffices. The regularity assumption in (i) holds, e.g., when $\sigma\sigma^{\top}$ is uniformly elliptic, or hypoelliptic in the sense of H\"ormander, provided that $\partial\mathcal{S}$ is $C^{2+\alpha}$ and $\tau < \infty$ almost surely. \Cref{thm:V_1_V_2} is the specialization to the committor setting of the general \Cref{thm:V_1_V_2_appendix}, which allows a running state cost $f$ and a general terminal cost $g$; here $f \equiv 0$ and $g \geq -\log(1-\xi) > 0$ by~\eqref{eq:terminal_cost}, so the sign conditions on $f$ and $g$ required there hold automatically. To prove \Cref{thm:V_1_V_2}, we first show, by applying the tower property of conditional expectation in a telescoping argument, that the only solution of~\eqref{eq:fixed_point_2} is~\eqref{eq:value_function_path_integral}; we then show that any solution of~\eqref{eq:fixed_point_1} is bounded below and, on average, above by $\Phi$, so the two coincide.
Building on this result, in~\Cref{subsec:kl_method} we propose a heuristic direct backpropagation (DBP) loss to learn $\Phi$, and in~\Cref{sec:value_matching} we introduce Value Matching (off-policy, with guarantees). Since learned value functions $\phi$ induce controls of the form $-\sigma^{\top}\nabla \phi$, all losses we consider take scalar-valued functions $\phi$ as argument, and the goal is to obtain $\phi \approx \Phi = -\log q_{\xi}$ after training.

\section{Solving the SOC committor estimation problem}
\label{sec:solving_SOC}

This section describes practical training losses for learning the value function $\Phi$ (and therefore the committor) from simulated trajectories. Recall from~\Cref{sec:methodology} that the SOC structure implies that any value estimate $\phi \approx \Phi$ induces a feedback control $u(x)=-\sigma^{\top}\nabla \phi(x)$. In practice, we combine three ingredients: (i) capped trajectories via $\tau(T)$ to control compute, (ii) losses that can be optimized with stochastic gradients, and (iii) a neural network parameterization that enforces boundary conditions (\Cref{subsec:parameterizing_committor}). We present two complementary approaches, direct backpropagation (\Cref{subsec:kl_method})
and Value Matching (\Cref{sec:value_matching}), that offer trade-offs between computational cost, theoretical guarantees, and flexibility in the choice of sampling process. \looseness=-1

\begin{algorithm}[t]
\caption{REACT training procedure}
\label{alg:react}
\begin{algorithmic}[1]
\Require Dynamics $(b,\sigma)$, sets $A,B$, horizon $T$, initial distribution $\mu_0$, rescaling parameter $\xi$, optimizer, training steps $N$, noise level $\kappa$, $K$ update steps per batch of trajectories
\State Initialize neural network parameters $\theta$ and value estimate $\phi_\theta$ using the committor parameterization in~\Cref{subsec:parameterizing_committor}.
\For{$n=1,\dots,N$}
  \State Sample initial states $X_0 \sim \mu_0$, e.g.\ from the rescaled reactive density~\eqref{eq:rescaled_reactive_density} (\Cref{subsec:choosing_mu0}).
  \State \textbf{REACT-DBP (on-policy, backprop through SDE)}:
  \State \hspace{2mm} Sample capped trajectories from the controlled SDE~\eqref{eq:DBP_SDE}.
  \State \hspace{2mm} Take one optimizer step on the DBP loss~\eqref{eq:DBP_1}.
  \State \textbf{REACT-VM (off-policy)}:
  \State \hspace{2mm} Form the detached value function $\bar{\phi}\gets\texttt{sg}(\phi_\theta)$ and sampling drift $v$ as in~\eqref{eq:approximate_v}.
  \State \hspace{2mm} Sample capped trajectories $X^{v,\kappa}$ from the SDE~\eqref{eq:VM_SDE}, and add them to a replay buffer $\mathcal{B}$ together with their sampling drift.
  \State \hspace{2mm} Take $K$ optimizer steps on the Value Matching loss~\eqref{eq:final_vm}, each on a minibatch of trajectories drawn from $\mathcal{B}$, with the sampling drift of each stored trajectory.
\EndFor
\State \Return learned rescaled committor $\hat{q}_{\xi} = e^{-\phi_\theta}$ (and the corresponding unrescaled committor), plus derived quantities (rates, equilibrium constants, sampled transition paths).
\end{algorithmic}
\end{algorithm}

\subsection{Direct backpropagation (DBP)}
\label{subsec:kl_method}

We begin with a direct on-policy loss, obtained by treating the terminal cost in the fixed-point formulation~\eqref{eq:fixed_point_1} as fixed during each update. This loss can be differentiated by backpropagating through the SDE dynamics (a standard approach in deep SOC; see~\cite[Sec.~5.1.1]{domingoadjoint2025}). Concretely, for a candidate value function $\phi$ (meant to approximate $\Phi$), we consider:
\begin{equation}
\label{eq:DBP_1}
    \mathcal{L}_{\mathrm{DBP}}(\phi) = \mathbb{E}\Big[\int_0^{\texttt{sg}(\tau(T))} |\nabla \phi(X_t^{\phi})|_D^2 \, \mathrm{d}t +
    \texttt{sg}(\phi)(X^{\phi}_{\texttt{sg}(\tau(T))})\Big],
\end{equation}
where the controlled trajectories satisfy
\begin{equation}
\label{eq:DBP_SDE}
    \mathrm{d}X_t^{\phi} = \big(b(X^{\phi}_t) - 2D\,\nabla \phi(X^{\phi}_t)\big) \, \mathrm{d}t + \sigma\,\mathrm{d}W_t, \quad X^{\phi}_0 \sim \mu_0.
\end{equation}
Here DBP stands for direct backpropagation and $\texttt{sg}$ denotes a stopped-gradient copy (detached from the computational graph). Note that at optimality $\phi = \Phi = -\log q_{\xi}$, the drift in~\eqref{eq:DBP_SDE} reduces to $b + 2D\nabla\log q_{\xi}$, i.e.\ the bounded rescaled counterpart of the Doob's $h$-transformed drift~\eqref{eq:doob_bg}. The practical choices in~\eqref{eq:DBP_1} are worth highlighting:
\begin{enumerate}[left=-5pt,label=(\roman*)]
    \item \emph{Gradient flow through trajectories.} Even though $\texttt{sg}(\phi)$ detaches the explicit terminal evaluation, $\phi$ still affects the loss through the controlled trajectory $X^{\phi}$.
    \item \emph{Boundary conditions.} The loss does not explicitly include the terminal cost $g$, so boundary conditions must be enforced by the parameterization of $\phi$ (see~\Cref{subsec:parameterizing_committor}).
    \item \emph{Capped hitting time.} Since $\tau(T)$ depends on the trajectory (and therefore on $\phi$), we use $\texttt{sg}(\tau(T))$ to avoid differentiating through a stopping time; in general $\tau(T)$ is not differentiable with respect to the parameters of $\phi$.
\end{enumerate}
These features make~\eqref{eq:DBP_1} practical, but they also limit theory: the effective gradient is biased due to the stopped-gradient choices; consequently, DBP has no consistency or convergence guarantee, and its stationary points are not known to solve the original hitting-time problem, so we present it as a heuristic on-policy baseline, and the method requires backpropagating through the full SDE trajectory, which is memory-intensive for long horizons. Both limitations are addressed by the Value Matching method described next.

\subsection{Value Matching}
\label{sec:value_matching}

We now introduce Value Matching (VM), a loss function designed to learn $\Phi$ from trajectories generated by an arbitrary sampling process. VM is built from the path-integral characterization~\eqref{eq:value_function_path_integral} and the fixed-point formulation~\eqref{eq:fixed_point_2}.
For simplicity and ease of presentation, we consider a state space $\mathcal{X} = \mathbb{R}^d$ going forward, although the extension to $\mathcal{X} \subset \mathbb{R}^d$ is possible.

To motivate VM, we characterize the value function via likelihood ratios between controlled path measures. Let $\mathbb{P}^{u^{\star}}$ be the path measure of the optimally controlled process~\eqref{eq:X_u_star_main}, with initial distribution $\mu_0$.
For a scalar function $\phi$, let $\mathbb{P}^{\phi}$
denote the path measure induced by~\eqref{eq:controlled_SDE_main} with control $u(x)=-\sigma^{\top}\nabla \phi(x)$, and with the same initial distribution.
Consider the capped process $X^{v}=(X^v_t)_{0\le t\le\tau(T)}$ that solves $\mathrm{d}X^{v}_t = v(X^{v}_t) \, \mathrm{d}t + \sigma\, \mathrm{d}W_t$, also with initial distribution $\mu_0$.
Combining Girsanov's theorem with the path-integral characterization of the optimal control yields:
\begin{equation}
\label{eq:log_RN}
\begin{split}
    \log \frac{\mathrm{d}\mathbb{P}^{u^{\star}}}{\mathrm{d}\mathbb{P}^{\phi}}(X^{v}) &= \log \frac{\mathrm{d}\mathbb{P}}{\mathrm{d}\mathbb{P}^{\phi}}(X^{v}) + \log \frac{\mathrm{d}\mathbb{P}^{u^{\star}}}{\mathrm{d}\mathbb{P}}(X^{v}) \\
    &= \int_0^{\tau(T)} \big\langle \sigma^{\top}\nabla \phi(X^{v}_t),\, \mathrm{d}W_t \big\rangle
    + \int_0^{\tau(T)} \big\langle \nabla \phi(X^{v}_t),\, v(X^v_t) - b(X^v_t) +
    D \nabla \phi (X^{v}_t) \big\rangle \, \mathrm{d}t
    \\
    &\quad - \mathbf{1}_{\mathcal{S}}(X^{v}_{\tau(T)})\, g(X^{v}_{\tau(T)}) - \mathbf{1}_{\mathcal{S}^c}(X^{v}_{\tau(T)})\, \Phi(X^{v}_{\tau(T)}) + \Phi(X^{v}_0).
\end{split}
\end{equation}
At optimality $u^{\star}=-\sigma^{\top}\nabla \Phi$, which means that $\mathbb{P}^{u^{\star}}=\mathbb{P}^{\Phi}$. Hence, the right-hand side vanishes identically when $\phi=\Phi$. This observation motivates a loss function that penalizes deviations from zero using trajectories generated under an arbitrary drift $v$. Since $\Phi$ is unknown, VM replaces the initial and terminal evaluations of $\Phi$ by evaluations of $\phi$, yielding a self-consistent loss:
\begin{equation}
\label{eq:VM_kappa_1}
\begin{split}
    \mathcal{L}_{\mathrm{VM}_{v}}(\phi) &= \mathbb{E} \Bigl[ \tfrac{1}{2}\Bigl( \int_0^{\tau(T)} \big\langle \sigma^{\top}\nabla \phi(X^{v}_t),\, \mathrm{d}W_t \big\rangle
    + \int_0^{\tau(T)} \big\langle \nabla \phi(X^{v}_t),\, v(X^v_t) - b(X^v_t) +
    D \nabla \phi (X^{v}_t) \big\rangle \, \mathrm{d}t
    \\
    &\qquad\quad - \mathbf{1}_{\mathcal{S}}(X^{v}_{\tau(T)})\, g(X^{v}_{\tau(T)}) - \mathbf{1}_{\mathcal{S}^c}(X^{v}_{\tau(T)})\, \phi(X^{v}_{\tau(T)}) + \phi(X^{v}_0)
    \Bigr)^2 \Bigr].
\end{split}
\end{equation}
The key advantage is that $v$ is free: we can sample trajectories off-policy, including at different noise levels,
as shown below in the general form of the VM loss.

\begin{defbox}
\begin{definition}[Value Matching loss function]
\label{def:VM}
Given $\kappa \in (0,\infty)$ and a drift $v:
\mathbb{R}^d\to\mathbb{R}^d$ that is locally Lipschitz with linear growth, consider trajectories $X^{v,\kappa}$ satisfying
\begin{equation} \label{eq:VM_SDE}
    \mathrm{d}X^{v,\kappa}_t = v(X^{v,\kappa}_t) \, \mathrm{d}t + \sqrt{\kappa}\,\sigma\, \mathrm{d}W_t,
\end{equation}
with $X^{v,\kappa}_0 \sim \mu_0$.
For $T>0$, the VM loss is
\begin{equation}
\label{eq:L_VM_v_kappa}
\begin{split}
    \mathcal{L}_{\mathrm{VM}_{v,\kappa}}(\phi) &=
    \mathbb{E} \Bigl[ \tfrac{1}{2} \Bigl( \int_0^{\tau(T)} \big\langle \nabla \phi(X^{v,\kappa}_t),\, \tfrac{1}{\kappa}\, \mathrm{d}X^{v,\kappa}_t -\big( b(X^{v,\kappa}_t) - D\,\nabla \phi(X^{v,\kappa}_t) \big) \, \mathrm{d}t \big\rangle \\
    &\qquad\quad + \tfrac{1}{\kappa}\, \phi(X^{v,\kappa}_0) + \big(\mathbf{1}_{\mathcal{S}}(X^{v,\kappa}_{\tau(T)}) - \tfrac{1}{\kappa}\big)\,\phi(X^{v,\kappa}_{\tau(T)}) - \mathbf{1}_{\mathcal{S}}(X^{v,\kappa}_{\tau(T)})\, g(X^{v,\kappa}_{\tau(T)})
    \Bigr)^2 \Bigr].
\end{split}
\end{equation}
\end{definition}
\end{defbox}

The next theorem states that VM is principled: first-order critical points correspond to the true value function.
\begin{thmbox}
\begin{theorem}[First-order optimality of the VM loss with capped hitting times] \label{thm:capped_vm_loss}
Fix $\alpha \in (0,1)$ and assume that:
\begin{enumerate}[left=0pt,label=(\roman*)]
    \item $b$ and $\sigma$ are locally Lipschitz with linear growth and belong to $C^{\alpha}(\mathbb{R}^d)$;
    \item $\sigma\sigma^{\top}$ is uniformly elliptic: there exists $c > 0$ such that $\zeta^{\top}\sigma(x)\sigma(x)^{\top}\zeta \geq c\,\|\zeta\|^2$ for all $x, \zeta \in \mathbb{R}^d$;
    \item the boundary $\partial\mathcal{S}$ is $C^{2+\alpha}$;
    \item the sampling drift $v \in C^1(\mathbb{R}^d;\mathbb{R}^d)$ is locally Lipschitz with linear growth, and $\kappa \in (0,+\infty)$.
\end{enumerate}
Define
\begin{equation}
\label{eq:Phi_def}
\begin{split}
\mathcal{F} &= \{ \phi : \mathbb{R}^d \to \mathbb{R} \mid \phi \in C^{2+\alpha}(\mathbb{R}^d),\; \phi,\, \nabla \phi,\, \nabla^2 \phi \text{ grow at most polynomially} \}.
\end{split}
\end{equation}
If $\hat{\phi} \in \mathcal{F}$ is a first-order critical point of $\mathcal{L}_{\mathrm{VM}_{v,\kappa}}$, i.e.\ for any perturbation $\eta \in C^{2+\alpha}(\mathbb{R}^d)$ we have $\frac{\partial}{\partial \varepsilon} \mathcal{L}_{\mathrm{VM}_{v,\kappa}}(\hat{\phi} + \varepsilon \eta) \big|_{\varepsilon = 0} = 0$, then $\hat{\phi}(x) = \Phi(x)$ and $u^{\star}(x) = -\sigma^{\top}\nabla \hat{\phi}(x)$ for all $x \in \mathbb{R}^d$.
\end{theorem}
\end{thmbox}

While $v$ and $\kappa$ are in principle arbitrary, they should be chosen to make training efficient. In the idealized case where $q_{\xi}$ and $\tilde{q}_{\xi}$ are known, a natural choice is the drift of the $\kappa$-rescaled reactive process~\eqref{eq:rescaled_reactive_SDE_kappa} introduced in~\Cref{sec:methodology}:
\begin{equation}
\begin{split} \label{eq:v_rescaled_reactive}
v(x) &=
\tfrac{1+\kappa}{2}\big(b(x) + 2D\nabla\log q_{\xi}(x)\big) + \tfrac{\kappa-1}{2}\big(\tilde b(x) + 2D\nabla\log\tilde q_{\xi}(x)\big).
\end{split}
\end{equation}
During training we only have access to approximate committors $\hat{q}_{\xi}$, $\hat{\tilde{q}}_{\xi}$ and value function estimates $\phi=-\log \hat{q}_{\xi}$, $\tilde{\phi}=-\log \hat{\tilde{q}}_{\xi}$. For overdamped and underdamped Langevin, these quantities can be expressed solely in terms of the approximate forward committors: $\hat{\tilde{q}}_{\xi}(x) = 1 - \hat{q}_{\xi}(x)$ and $\hat{\tilde{q}}_{\xi}(r,p) = 1 - \hat{q}_{\xi}(r,-p)$, respectively.
Writing $\hat{q}_{\xi}=\exp(-\bar{\phi})$, $\hat{\tilde{q}}_{\xi}=\exp(-\bar{\tilde{\phi}})$ with $\bar{\phi}=\texttt{sg}(\phi)$, $\bar{\tilde{\phi}}=\texttt{sg}(\tilde{\phi})$ detached from the computational graph yields the practical choice
\begin{equation} \label{eq:approximate_v}
    v(x) =
    \tfrac{1+\kappa}{2}\big(b(x) - 2D\nabla \bar{\phi}(x)\big) + \tfrac{\kappa-1}{2}\big(\tilde b(x) - 2D\nabla \bar{\tilde{\phi}}(x) \big).
\end{equation}
With this choice, the VM loss~\eqref{eq:L_VM_v_kappa} can be rewritten in a form convenient for implementation (see~\Cref{cor:particular_VM}):
\begin{equation}
\label{eq:final_vm}
\begin{split}
    &\mathcal{L}_{\mathrm{VM}_{v,\kappa}}(\phi) \\ &= \mathbb{E} \Bigl[ \tfrac{1}{2} \Bigl( \int_0^{\tau(T)} \big\langle \nabla \phi(X^{v,\kappa}_t), \tfrac{1}{\sqrt{\kappa}} \sigma \, \mathrm{d}W_t
    + \big( \tfrac{1}{\kappa} - 1 \big) \big( b(X^{v,\kappa}_t) - D \nabla \log \rho(X^{v,\kappa}_t)
    \big) \, \mathrm{d}t \\ &\qquad\qquad\qquad\qquad + D \big(
    \nabla \phi(X^{v,\kappa}_t) \!-\!(\tfrac{1}{\kappa}\!+\!1)\!\nabla \bar{\phi}(X^{v,\kappa}_t)
    \!+ \!(\tfrac{1}{\kappa}\!-\!1) \nabla \bar{\tilde{\phi}}(X^{v,\kappa}_t)
    \big) \, \mathrm{d}t
    \big\rangle \\ &\qquad\quad + \! \tfrac{1}{\kappa} \phi(X^{v,\kappa}_0)
    \! + \! (\mathbf{1}_{\mathcal{S}}(X^{v,\kappa}_{\tau(T)}) \! - \! \tfrac{1}{\kappa})
    \phi(X^{v,\kappa}_{\tau(T)})
    \! - \! \mathbf{1}_{\mathcal{S}}(X^{v,\kappa}_{\tau(T)}) g(X^{v,\kappa}_{\tau(T)})
    \Bigr)^2 \Bigr].
\end{split}
\end{equation}
Note that the term $b - D \nabla \log \rho$ vanishes in the reversible case (overdamped Langevin). In \Cref{eq:subsection_overdamped} and \Cref{eq:subsection_underdamped} we particularize this loss function for overdamped and underdamped Langevin, respectively.

The Value Matching loss is a general-purpose SOC loss, not limited to hitting-time problems. We develop VM for finite-horizon, time-dependent SOC problems in~\Cref{subsec:vm_finite} and for uncapped hitting-time problems ($T=+\infty$) in~\Cref{subsec:vm_loss}, where first-order optimality is established. These arguments also underpin the more delicate capped hitting-time result of~\Cref{thm:capped_vm_loss}.

A common thread across these proofs is a non-standard generalization of the Dynkin representation theorem (\Cref{thm:parabolic-dynkin,thm:pde-dynkin}), which to our knowledge has not appeared in the literature. We use it to write the HJB-residual functional $F$ associated to a candidate value function $\hat{\phi}$ as an It\^o-type sum $F(\bar X^{\kappa}) = \eta^{\star}(\bar X^{\kappa}_0) + \int_0^T \langle \sqrt{\kappa}\sigma^{\top}\nabla\eta^{\star}(\bar X^{\kappa}_t,t), \mathrm{d}W_t\rangle + (\text{drift correction})$, where $\eta^{\star}$ solves a specific parabolic PDE. Choosing the perturbation $\eta = -\eta^{\star}$ then collapses the directional derivative into the manifestly non-positive form $\partial_\varepsilon \mathcal{L}_{\mathrm{VM}_{v,\kappa}}(\hat{\phi}+\varepsilon\eta)|_{\varepsilon=0} = -\mathbb{E}\big[F(X^{v,\kappa})^2\big]$, so that the first-order condition forces $F \equiv 0$ almost surely, which is precisely the HJB equation, and hence $\hat{\phi} = \Phi$.

The proof of \Cref{thm:capped_vm_loss} uses uniform ellipticity in the Dynkin representation (\Cref{thm:pde-dynkin}) and in the full-support property of diffusion bridges (\Cref{thm:full_supp_bridge}). The guarantee therefore does not cover underdamped Langevin dynamics~\eqref{eq:SDE_underdamped}, whose noise is degenerate. We leave the extension to hypoelliptic dynamics to future work.

\looseness=-1
The finite-horizon, time-dependent version of the VM loss is closely related to the BSDE loss introduced by \cite{nusken2023interpolating} (for which no theoretical guarantees are provided). We further relate VM to existing SOC and RL losses, including the FBSDE loss~\citep{weinan2017deep,han2018solving}, the moment and log-variance losses~\citep{nusken2021solving}, and the Soft Actor-Critic (SAC) losses~\citep{haarnoja2018soft}, in~\Cref{subsec:comparison_SOC_losses,subsec:SAC}. In~\Cref{subsec:girsanov_gradient}, we review the REINFORCE-style loss introduced in \cite{hua2024efficient} to tackle the same SOC problems.

Under a reproducing-kernel-Hilbert-space (RKHS) parameterization of the value function, $\phi = \langle \theta, \psi \rangle_{\mathcal{G}}$ with feature map $\psi$ and feature space $\mathcal{G}$, the resulting VM losses become \emph{quartic} functions of the parameters~$\theta$. To our knowledge, the optimization-landscape geometry of SOC-derived path-integral losses, including the BSDE, FBSDE, moment, log-variance, and SAC losses discussed above, as well as our Value Matching loss, has received little attention in the literature; \Cref{sec:convergence_proof} takes a first step by exposing this quartic structure.

Finally, discrete-time, discrete-space analogs of VM, one for the time-homogeneous capped-hitting-time setting (mirroring \Cref{thm:capped_vm_loss}) and one for the time-inhomogeneous finite-horizon setting (mirroring \Cref{thm:vm_finite}), are developed in~\Cref{sec:discrete_mc}. \looseness=-1

\subsection{Choosing the initial distribution $\mu_0$ and the rollout sampling scheme}
\label{subsec:choosing_mu0}
Although the optimal control in \Cref{prop:SOC} is agnostic to the initial distribution $\mu_0$, and theoretically the algorithms described in this section can work with any initial distribution, choosing $\mu_0$ well is fundamental to learn $\Phi = -\log q_{\xi}$ in regions of interest.

A sensible criterion is to set $\mu_0$ such that in the idealized case where $\phi=-\log q_{\xi}$, the occupancy measure of the process $X^{v,\kappa}$ (with $v$ given by \eqref{eq:approximate_v}) is equal to the rescaled reactive density $\rho^{(\xi)}_\mathrm{R}$ defined in \eqref{eq:rescaled_reactive_density}. The reason for this is that when $\xi \ll 1$, the factors $q_{\xi} \,\tilde q_{\xi}$ tilt the mass of $\rho^{(\xi)}_\mathrm{R}$ towards transition regions, where trajectories are similarly likely to end up at (or come from) $A$ and $B$\footnote{In fact, we want to choose $\xi$ small enough that $\rho^{(\xi)}_\mathrm{R}$ has substantial mass in transition regions.}. To accomplish this, we use that under $\phi=-\log q_{\xi}$, the process $X^{v,\kappa}$ solves the SDE \eqref{eq:rescaled_reactive_SDE_kappa}, and thus, has stationary distribution equal to $\rho^{(\xi)}_\mathrm{R}$ as long as the source and sink conditions are enforced appropriately.
We develop two theoretically exact birth-death mechanisms to implement these conditions:
\begin{itemize}[left=3pt]
\item \emph{Boundary birth-death} (\Cref{subsec:different_noise_levels}): particles are born on~$\partial A$ from a distribution determined by the committors and killed on~$\partial B$ via an exponential clock driven by boundary local times, which are also computed from the committors.
\item \emph{Interior birth-death} (\Cref{sec:mollified_tpt_appendix}): the rescaled committor is extended to a global solution of a mollified equation (\Cref{prop:xi_extension}), with smooth killing rates $\lambda_A \geq 0$ supported in~$A$ and $\lambda_B \geq 0$ supported in~$B$. The reactive density then satisfies the global Fokker Planck equation (\Cref{prop:mod_FPE_soft})
\begin{equation} \label{eq:mod_FPE_soft_main}
    \nabla\!\cdot\! J^{(\xi)}_\mathrm{R} = \frac{\lambda_A\,\rho\,q_{\xi}}{Z^{(\xi)}_{AB}} - \frac{\lambda_B}{q_{\xi}}\,\rho^{(\xi)}_\mathrm{R},
\end{equation}
where the source term is implemented through respawning from $p_{\mathrm{birth}} \propto \lambda_A\,\rho\,q_{\xi}$ inside~$A$ and the sink term is handled via volumetric killing at rate $\lambda_B/q_{\xi}$ inside~$B$. This avoids boundary local times entirely; in the limit $\lambda_A,\lambda_B \to \infty$, it recovers the boundary approach. The exact rates $\lambda_A$, $\lambda_B$ that ensure that the mollified committor matches the rescaled committor over $(A \cup B)^c$ are not easily computable in general, but using constant rates provides an acceptable approximation.
\end{itemize}
That is, to sample rollouts for the algorithms in Sections~\ref{subsec:kl_method}-\ref{sec:value_matching}, we can simulate the SDE \eqref{eq:VM_SDE} with $v$ in \eqref{eq:approximate_v}, and use either of the birth-death mechanisms above, approximating $q_{\xi} \approx e^{-\phi}$: when a particle is killed at $\partial B$ or $B$, it respawns at $\partial A$ or $A$ accordingly. We obtain rollouts with horizon $\tau(T) := \min\{T,\tau\}$ by breaking up instances of this Markov process into chunks of length $T$, or $\tau$ if the process hits $\partial A$ or $\partial B$. Observe that this procedure is implicitly setting an initial distribution $\mu_0$, although it is not possible to write it down explicitly. This rollout initial distribution should not be confused with the birth distribution of the particles. The \emph{birth distribution} is the law of the point at which a new particle is placed after the previous one is killed: a distribution on~$\partial A$ for the boundary mechanism, and $p_{\mathrm{birth}}$ inside~$A$ for the interior mechanism. In contrast, $\mu_0$ is the law of the first state of a rollout: a rollout starts at a newly born particle or, if the particle of the previous rollout is still alive, at the state where that rollout was stopped (for instance at the horizon~$T$). Thus $\mu_0$ mixes the birth distribution with the distribution of these stopping points.

\paragraph{Practical sampling of rollouts.} For underdamped and overdamped Langevin, we have access to the unnormalized log-densities of $\rho$ and $\rho^{(\xi)}_{\mathrm{R}}$ (assuming that $\phi = - \log q_{\xi}$).
Observe that the forward (resp.\ backward) committors from $B$ to $A$ are equal to one minus the forward (resp.\ backward) committors from $A$ to $B$. These two observations motivate the following Markov process, which does not have occupancy measure equal to the rescaled reactive density $\rho^{(\xi)}_{\mathrm{R}}$ but works well in practice:
\begin{enumerate}[label=(\roman*)]
    \item Sample a starting point $y \sim \rho^{(\xi)}_{\mathrm{R}}$, e.g. using an MCMC method.
    \item Pick $A$ or $B$ as the target with probability proportional to the current committor estimate.
    If the target is $B$, run the SDE \eqref{eq:VM_SDE}, and if the target is $A$, run the SDE \eqref{eq:VM_SDE}, replacing $q_{\xi}$, $\tilde{q}_{\xi}$ by $1-q_{\xi}$, $1-\tilde{q}_{\xi}$ respectively. Kill particles when they reach either $A$ or $B$.
    \item When the particle is killed, return to (i).
\end{enumerate}
Like before, when running this procedure we must approximate $q_{\xi} \approx e^{-\phi}$, and we obtain rollouts by breaking a process instance into chunks of length $\tau(T)$. Here the birth distribution is $\rho^{(\xi)}_{\mathrm{R}}$, and $\mu_0$ again mixes it with the states at which rollouts are stopped at the horizon.
While the occupancy measure of this simple approach does not match $\rho^{(\xi)}_{\mathrm{R}}$, it still places mass in the relevant regions of the space, which is what the off-policy VM loss needs.

\subsection{Parameterizing the committor}
\label{subsec:parameterizing_committor}

The learned rescaled committor must satisfy the shifted boundary conditions $\hat{q}_{\xi}|_A = \xi$ and $\hat{q}_{\xi}|_B = 1-\xi$ (cf.~\Cref{sec:methodology}), respect the $A$-$B$ symmetry, and be initialized stably. Assume we have access to a distance-like function $\mathrm{dist}$ such that $\mathrm{dist}(x,A)$ and $\mathrm{dist}(x,B)$ are efficiently computable. Note this can generalize to arbitrary collective variables. We parameterize the committor via
\begin{align} \label{eq:q_parameterization}
\hat{q}_{\xi}(x) &= \xi\, \frac{\mathrm{dist}(x, B)\,\psi^{(2)}_{\theta}(x)}{\mathrm{dist}(x, A)\,\psi_{\theta}^{(1)}(x) +\mathrm{dist}(x, B)\,\psi_{\theta}^{(2)}(x)} + (1-\xi)\, \frac{\mathrm{dist}(x, A)\,\psi^{(1)}_{\theta}(x)}{\mathrm{dist}(x, A)\,\psi_{\theta}^{(1)}(x) +\mathrm{dist}(x, B)\,\psi_{\theta}^{(2)}(x)},\nonumber\\
&= \xi + (1-2\xi)\, \frac{\mathrm{dist}(x, A)\,\psi^{(1)}_{\theta}(x)}{\mathrm{dist}(x, A)\,\psi_{\theta}^{(1)}(x) +\mathrm{dist}(x, B)\,\psi_{\theta}^{(2)}(x)},
\end{align}
where $\psi_{\theta}=(\psi^{(1)}_{\theta},\psi^{(2)}_{\theta}) : \mathbb{R}^d \to [0,+\infty)^2$ is a neural network (e.g.\ with a softplus output). The distance factors enforce $\hat{q}_{\xi} = \xi$ on $A$ and $\hat{q}_{\xi} = 1-\xi$ on $B$ by construction, and if $\psi_{\theta}$ is initialized near-constant, then $\hat{q}_{\xi}$ starts from a geometric guess induced purely by the configuration-space geometry. This provides a stable and symmetry-respecting initialization.

The value-function estimate used in the losses of~\Cref{subsec:kl_method,sec:value_matching}
is then
\begin{equation}
\label{eq:phi_parameterization}
\phi(x) = -\log \hat{q}_{\xi}(x),
\end{equation}
which is bounded since $\hat{q}_{\xi} \in [\xi, 1-\xi]$ by construction (cf.\ the discussion of the $\xi$-rescaling in~\Cref{sec:methodology}). In particular, $\phi|_A = -\log\xi = g|_A$ and $\phi|_B = -\log(1-\xi) = g|_B$, so the boundary conditions of the SOC problem~\eqref{eq:terminal_cost} are automatically satisfied.

\section{Experiments} \label{sec:experiment}

We organize the numerical results into two broad classes of dynamics. The first class, \emph{(generalized) reversible dynamics}, includes overdamped Langevin dynamics and underdamped Langevin dynamics, the latter being reversible after momentum reversal. The second class, \emph{non-reversible dynamics}, includes stationary non-equilibrium dynamics with a steady probability current and periodically-driven dynamics with a time-dependent force field. Across these settings, we test three capabilities of REACT: accurate committor estimation, reliable estimation of reaction statistics derived from the committor, and direct sampling of transition paths from the learned SOC control. Full system details are provided in~\Cref{appendix:system_details}.

\subsection{(Generalized) reversible dynamics}

\subsubsection{Overdamped Langevin systems}

Overdamped Langevin dynamics provide a canonical testbed for committor estimation. We consider three analytical two-dimensional potentials that isolate complementary difficulties: a symmetric triple-well landscape, the asymmetric Müller-Brown potential, and a rugged Müller-Brown potential with additional metastable structure. Reference committors are computed numerically using the finite-element method. We compare against three baselines: the variational BKE loss~\eqref{eq:kbe_variational} (BKE), its importance sampling variant~\citep{kang2024computing} (BKE-IS), and the Feynman-Kac formulation~\eqref{eq:fk} in its semigroup form~\citep{li2022semigroup,strahan2023predicting} (semigroup). For REACT, we evaluate the direct backpropagation loss (REACT-DBP, \Cref{subsec:kl_method}) and the off-policy Value Matching loss (REACT-VM, \Cref{sec:value_matching}). All methods share the network, the boundary condition, the optimizer and the budget of energy evaluations; the detailed setups and baselines are included in~\Cref{appendix:baseline}.

\begin{figure}[t]
\centering
\includegraphics[width=\textwidth]{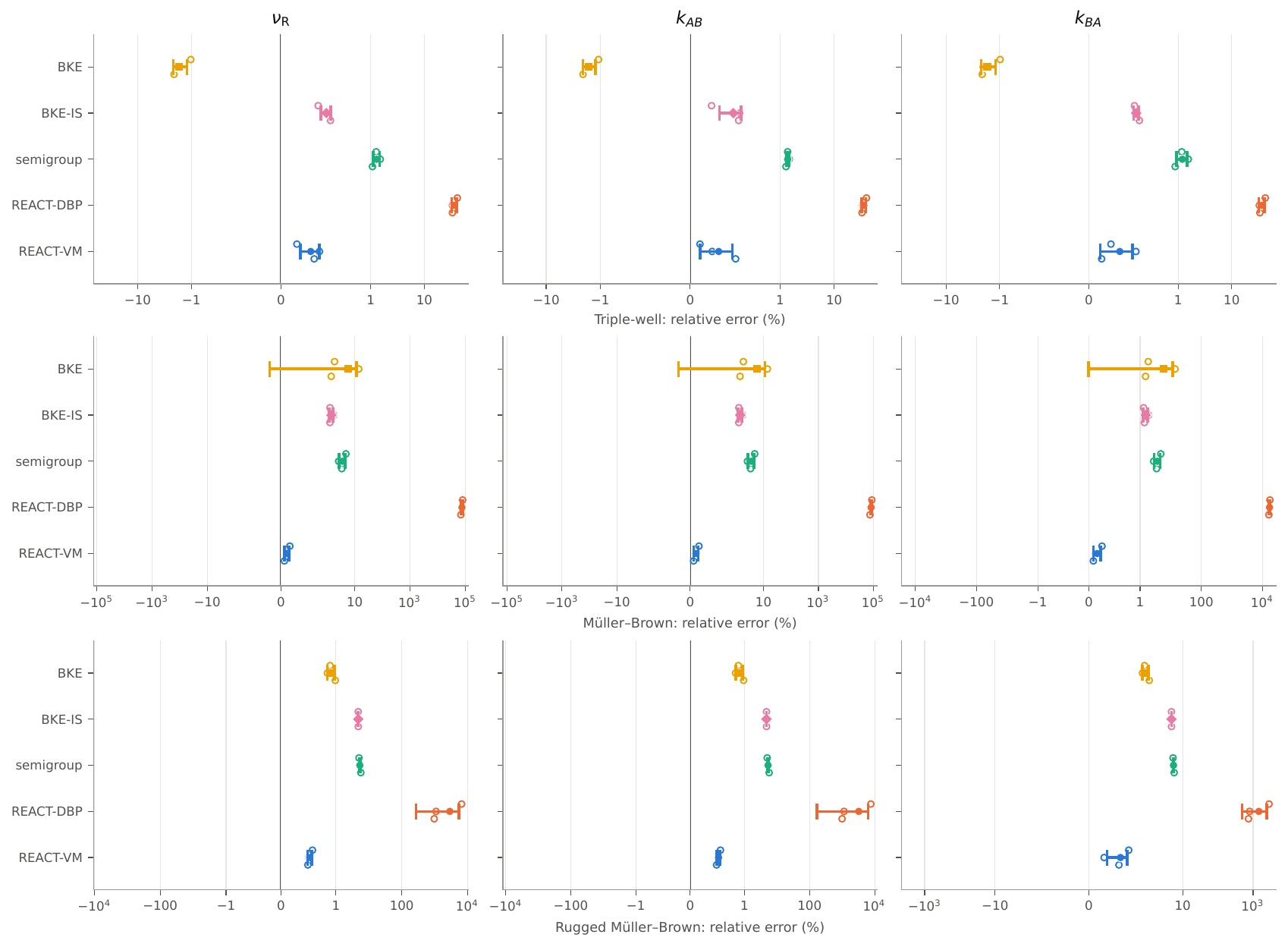}
\caption{Quantitative evaluation on the three overdamped systems at the same number of energy evaluations: relative error of the reaction rate $\nu_\mathrm{R}$ and the rate constants $k_{AB}$ and $k_{BA}$ at the end of training; mean $\pm$ standard deviation over three runs and the three runs. The full results are given in \Cref{tab:main_result}.}
\label{fig:final_errors}
\end{figure}

\label{sec:result_committor}\label{sec:result_rates}

Against the finite-element committor evaluated on a grid over the domain, REACT-VM learns the committor accurately, while the BKE baselines remain competitive. Full committor panels and error maps are provided in~\Cref{fig:viz_committor_all,fig:error_committor}. Once the committor is learned, we estimate the reaction rate $\nu_\mathrm{R}$, directional rate constants $k_{AB}$ and $k_{BA}$, basin weights $p_A$ and $p_B$, reactive weight $Z_{AB}$, and equilibrium constant $K_{\mathrm{eq}}$ using the TPT formulas introduced in~\Cref{subsec:committor_tpt} (Eqs.~\eqref{eq:reactive_density_bg},~\eqref{eq:reaction_rate_bg}-\eqref{eq:rate_constants_bg} and \eqref{eq:K_eq}). In two dimensions we evaluate these integrals on a fine grid; details are given in~\Cref{appendix:rate}. These statistics are the more important validation, as they are the quantities of interest in applications.

REACT-VM matches the reference reaction rates most closely on all three systems, with relative errors in $\nu_\mathrm{R}$, $k_{AB}$ and $k_{BA}$ well below 1\% (\Cref{fig:final_errors}). The next best is BKE or BKE-IS, but their errors fluctuate between systems. The semigroup baseline is reasonably stable but less accurate at the same energy evaluations. REACT-DBP gives the least accurate rates, because it learns on-policy, using each simulated trajectory only once, and its objective is heuristic (\Cref{subsec:kl_method}); the off-policy VM loss instead keeps learning from stored trajectories simulated with earlier iterates of the control. Additional results, including every reaction statistic, are given in~\Cref{tab:main_result}, and ablations of the boundary conditions, of $\xi$ and of $\kappa$ in~\Cref{appendix:ablations}.

\paragraph{Sampling transition paths.}
\label{sec:result_paths}

Once the committor is learned, we sample transition paths by running Doob's $h$-transform~\eqref{eq:doob_bg} of the reference dynamics~\eqref{eq:SDE_bg} with $\nabla\log q$ replaced by $\nabla\log\hat{q}$, which is the optimal importance-sampling strategy for generating $A\to B$ reactive trajectories. \Cref{fig:viz_tps} compares the path ensembles and hitting-time distributions of TPS and REACT-VM for the three overdamped systems, and \Cref{fig:viz_flux} shows the deterministic ($\kappa=0$) reactive-current visualizations. Overall, REACT-VM recovers comparable transition corridors and hitting-time distributions while avoiding an explicit MCMC procedure over path space, i.e., TPS \citep{dellago1998transition}.

\subsubsection{Underdamped Langevin systems}

Many physical systems are naturally second order and are better modeled by underdamped rather than overdamped Langevin dynamics. We evaluate this setting on a one-dimensional double-well potential, where the phase-space committor can be computed by the finite-difference method. The variational baselines do not apply here, since~\eqref{eq:kbe_variational} requires reversible dynamics; we compare REACT with the semigroup baseline under the same setting of the overdamped systems (\Cref{appendix:hyperparameters}).

\begin{figure}[t]
\centering
\includegraphics[width=0.78\textwidth]{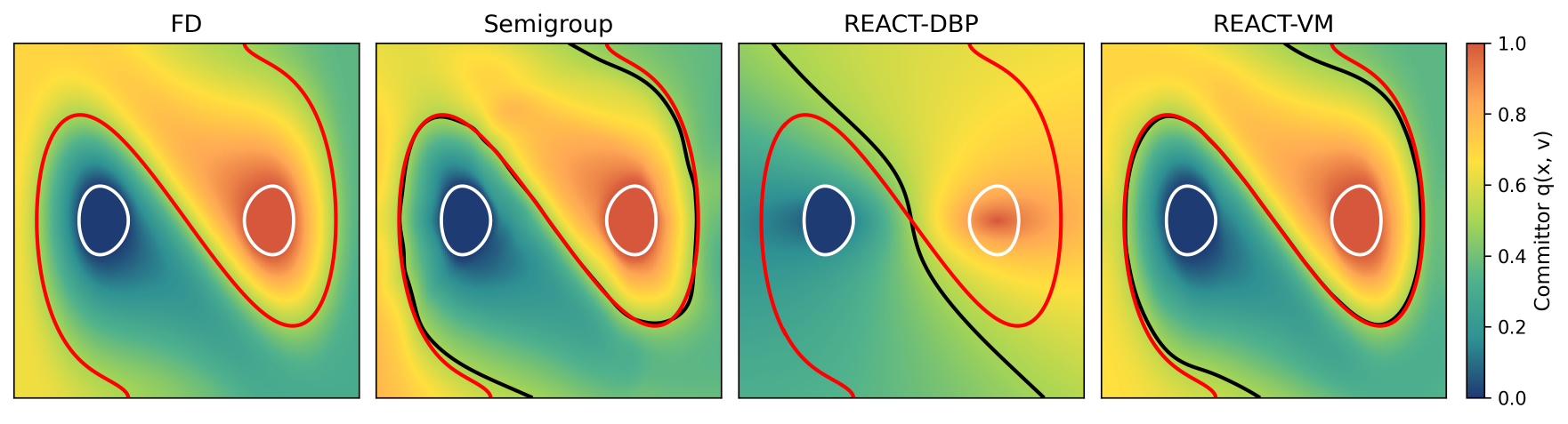}
\vspace{-4mm}
\caption{Ground truth (finite-difference) and learned committors on the 1D underdamped double-well system for the semigroup baseline, REACT-DBP and REACT-VM, with position $r$ on the horizontal axis and momentum $p$ on the vertical axis. The regions enclosed by the white curves denote the two target states; red and black curves mark the reference and learned $q=0.5$ surfaces, respectively.}
\label{fig:viz_underdamped_tps}
\end{figure}

\begin{table}[t]
\centering
\caption{Quantitative evaluation on the underdamped double-well, Maier-Stein and periodically-driven triple-well systems: mean committor error, and fraction of the trajectories started on the learned $q=0.5$ surface that reach $A$ first (ideal $0.5$). Mean $\pm$ standard deviation over three runs.}
\label{tab:react_extra}\label{tab:react_underdamped_doublewell}\label{tab:react_maier_stein}\label{tab:react_results}
\resizebox{\textwidth}{!}{
\setlength{\tabcolsep}{4pt}
\begin{tabular}{lccccc}
\toprule
 & Underdamped DW & \multicolumn{2}{c}{Maier-Stein} & \multicolumn{2}{c}{Periodic TW} \\
\cmidrule(lr){2-2}\cmidrule(lr){3-4}\cmidrule(lr){5-6}
Method & Mean $q$ error & Mean $q$ error & $q=0.5$ & Mean $q$ error & $q=0.5$ \\
\midrule
Semigroup & $0.0224\!\pm\!0.0008$ & $0.017\!\pm\!0.004$ & $0.489\!\pm\!0.006$ & $0.0113\!\pm\!0.0006$ & $0.498\!\pm\!0.008$ \\
REACT-DBP & $0.1603\!\pm\!0.0002$ & $0.028\!\pm\!0.003$ & {\boldmath$0.493\!\pm\!0.003$} & $0.012\!\pm\!0.002$ & $0.481\!\pm\!0.027$ \\
REACT-VM & {\boldmath$0.0142\!\pm\!0.0005$} & {\boldmath$0.0046\!\pm\!0.0002$} & $0.492\!\pm\!0.005$ & {\boldmath$0.0061\!\pm\!0.0002$} & {\boldmath$0.498\!\pm\!0.002$} \\
\bottomrule
\end{tabular}
}
\end{table}

\paragraph{1D double-well potential.} As shown in~\Cref{tab:react_extra} and~\Cref{fig:viz_underdamped_tps}, REACT-VM learns the most accurate phase-space committor for the underdamped double-well system, followed by the semigroup baseline. REACT-DBP is far less accurate at the same number of energy evaluations: it learns on-policy, using each simulated trajectory only once, and typically needs many more energy evaluations to converge. This setting is more stringent than the overdamped benchmarks because the noise is degenerate: it acts directly on momentum, while position evolves through the deterministic transport term. Consequently, the committor depends not only on position but also on momentum direction; two states at the same position can have different probabilities of reaching the product well depending on whether their momentum points toward or away from the barrier. The visualization in~\Cref{fig:viz_underdamped_tps} reflects this phase-space asymmetry.

\subsubsection{Underdamped Langevin molecular system: alanine dipeptide}
\label{sec:aldp}

To test REACT beyond low-dimensional systems, we learn the committor of the C7$_{\mathrm{eq}}\to$C7$_{\mathrm{ax}}$ isomerization of alanine dipeptide in vacuum under underdamped Langevin dynamics, a standard benchmark for reaction coordinates and committors~\citep{bolhuis2000reaction}. Following the commonly used assumption in the field~\citep{jung2023machine,kang2024computing,tang2026breaking}, we treat the dynamics of its $10$ heavy atoms (without hydrogens) as Markovian and learn the committor as a function of their positions and momenta, while the hydrogens enter only through the dynamics. We evaluate on held-out configurations stratified over reference committor values; the details of this setting and the experimental setup are given in~\Cref{appendix:aldp}.

\begin{figure}[t]
\centering
\includegraphics[width=\textwidth]{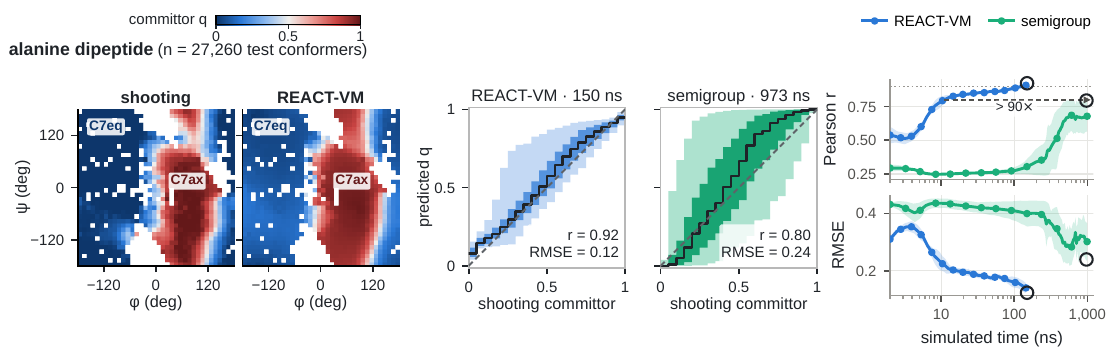}
\caption{Evaluation of REACT-VM and the semigroup baseline on alanine dipeptide. (a) Committor on the $(\phi,\psi)$ plane: shooting reference and REACT-VM. (b) Learned against reference committor on the test configurations, for the better of the two runs of each method: in each bin of the reference, the median of the learned values (black) and their central $50\%$ and $90\%$ (shaded). (c) Pearson correlation and RMSE against simulated time: mean and range of two independent runs; circles mark the committors shown in (a) and (b), and the arrow compares the simulated times at which the mean curves reach $0.8$ correlation.}
\label{fig:aldp}
\end{figure}

We compare REACT-VM with the semigroup baseline under the same setting. \Cref{fig:aldp} shows the results. REACT-VM converges far faster than the semigroup loss (\Cref{fig:aldp}c): within about $150$\,ns of simulated time it is more accurate than both of the two semigroup runs after $1\,\mu$s (\Cref{fig:aldp}b). Its convergence is also stable: the two independent runs follow similar curves, whereas the two semigroup runs differ widely and remain less accurate.

\subsubsection{Protein folding in explicit solvent: chignolin}
\label{sec:chignolin}

To test REACT at the scale of a solvated protein, we learn the folding committor of the mini-protein chignolin (CLN025, ten residues)~\citep{honda2008crystal,lindorff2011fast} in explicit water: $166$ protein atoms, $2{,}100$ water molecules and two ions, $6{,}468$ atoms in total, under underdamped Langevin dynamics. The folded and the unfolded state are defined by the fraction $Q$ of native contacts~\citep{best2013native}, $Q\geq0.8$ and $Q\leq0.2$. We do not learn the committor of the entire system: as for alanine dipeptide, we follow the commonly used assumption that the dynamics of the protein heavy atoms is Markovian~\citep{jung2023machine,kang2024computing,tang2026breaking}, and learn the committor as a function of their positions and momenta, while the solvent and the hydrogens enter only through the dynamics. We evaluate on $2{,}134$ conformers whose committors were estimated by shooting; the details of this setting and the experimental setup are given in~\Cref{appendix:chignolin}.

\begin{figure}[t]
\centering
\includegraphics[width=\textwidth]{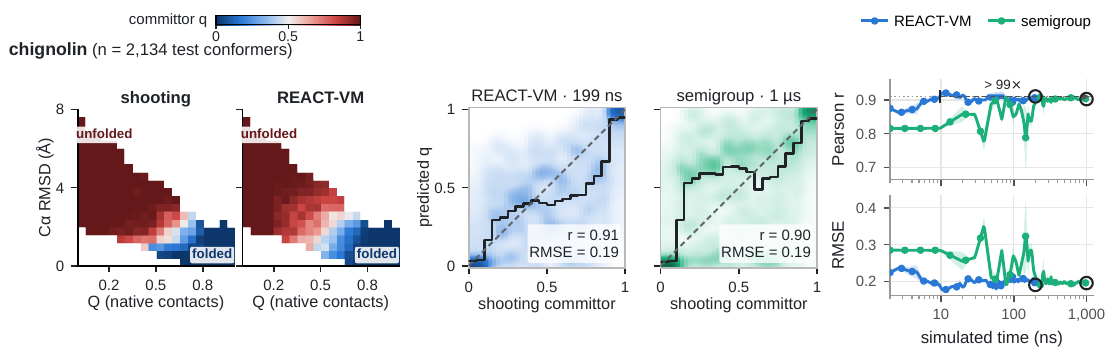}
\caption{Evaluation of REACT-VM and the semigroup baseline on chignolin in explicit solvent. (a) Committor on the plane of the fraction of native contacts $Q$ and the C$_\alpha$ RMSD to the native structure: shooting reference and REACT-VM. (b) Learned against reference committor on the test conformers, for the last network of the better of the two runs of each method: in each bin of the reference, the median of the learned values (black) and their distribution (shaded). (c) Pearson correlation and RMSE against simulated time: mean and range of two independent runs; circles mark the committors shown in (a) and (b). The dashed arrow starts where the mean REACT-VM curve reaches $0.91$ while the mean of semigroup curve never reaches this level.}
\label{fig:chignolin}
\end{figure}

We compare REACT-VM with the semigroup baseline under the same setting. \Cref{fig:chignolin} shows the results. REACT-VM learns the committor far faster than the semigroup loss (\Cref{fig:chignolin}c): within $10$\,ns of simulated time it reaches an accuracy that the semigroup loss does not reach within its budget of $1\,\mu$s (\Cref{fig:chignolin}b). Training is also more stable with REACT-VM: its correlation with the reference changes smoothly during training, whereas that of the semigroup loss has large variance.

\subsection{Non-reversible dynamics}

Beyond detailed balance, many systems of interest exhibit stationary probability currents or explicitly time-dependent forcing. We therefore consider two non-reversible settings: the stationary Maier-Stein model, which has a non-zero steady current, and a periodically-driven triple-well system, whose committor is time dependent. Reference committors for both systems are computed using finite-difference methods. The variational baselines do not apply here either; we compare REACT with the semigroup baseline.

\subsubsection{Stationary non-equilibrium dynamics: Maier-Stein system}

\begin{figure}[t]
\centering
\includegraphics[width=0.78\textwidth]{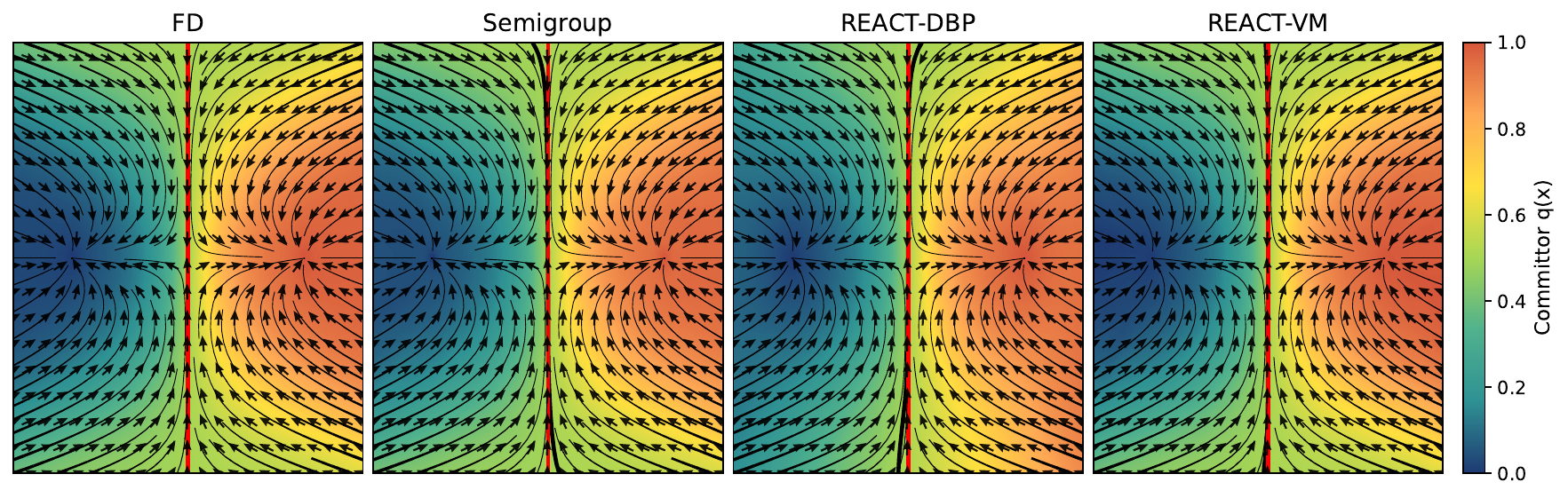}
\vspace{-4mm}
\caption{Ground truth (finite-difference) and learned committors on the Maier-Stein system for the semigroup baseline, REACT-DBP and REACT-VM, overlaid with the force field of the system.}
\label{fig:viz_maier_stein}
\end{figure}

\begin{figure}[t]
\centering
\includegraphics[width=0.84\textwidth]{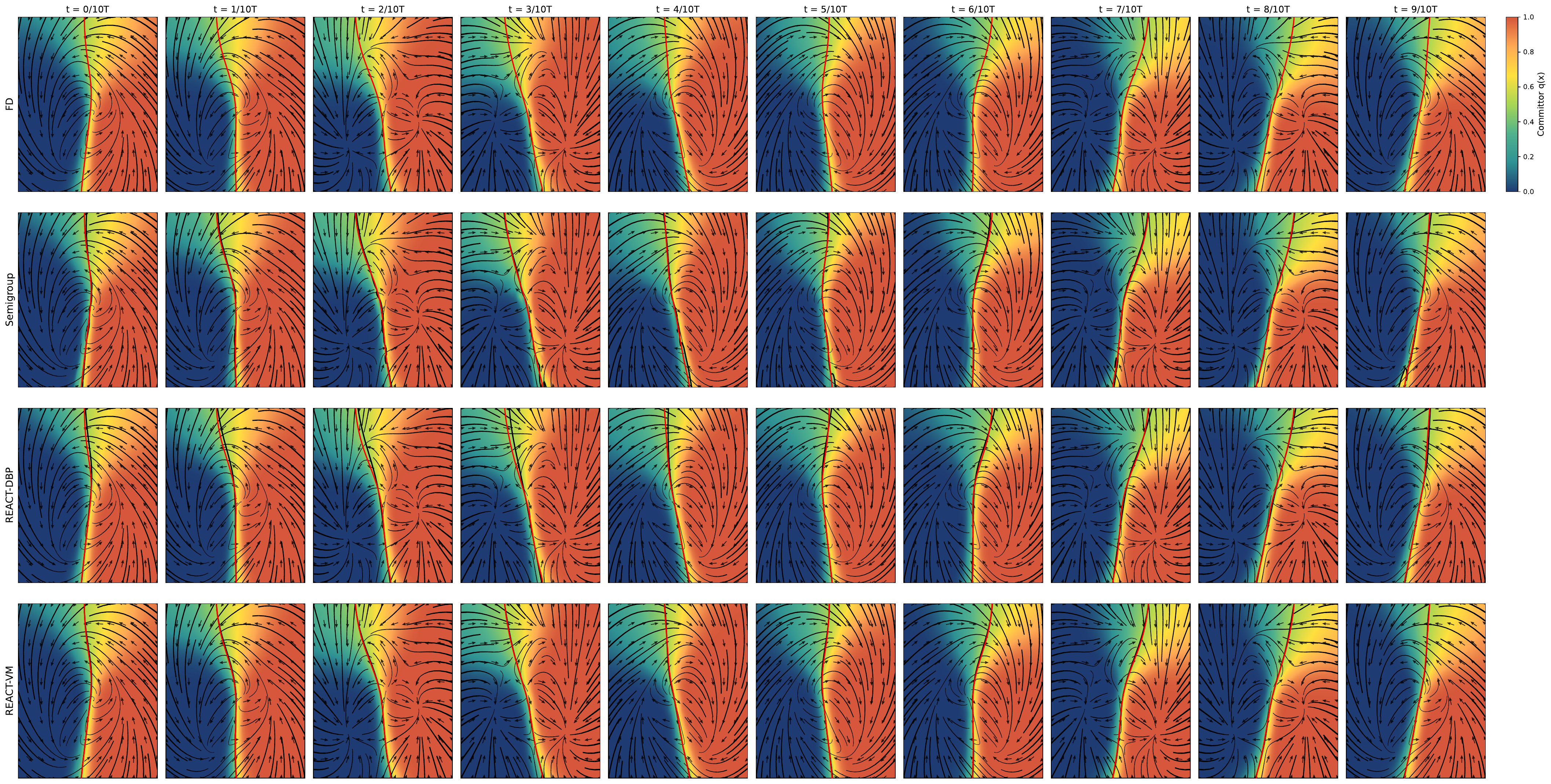}
\vspace{-3mm}
\caption{Ground truth (finite-difference) and learned committors on the periodically-driven triple-well system for the semigroup baseline, REACT-DBP and REACT-VM (mean over three runs), overlaid with the time-dependent force field of the system.}
\label{fig:periodic_committor}
\end{figure}

\Cref{tab:react_extra,fig:viz_maier_stein} show that REACT and the semigroup baseline recover the Maier-Stein transition surface accurately despite the absence of detailed balance. This is a qualitatively different test from the reversible Langevin examples: the stationary current is nonzero, so the reactive flow is not simply a gradient-driven crossing of an energy barrier. REACT-VM is the most accurate of the three methods in this example, with less than a third of the mean error of the semigroup baseline and of REACT-DBP, and all three capture the transition surface of this stationary non-equilibrium system: shots from each learned $q=0.5$ surface split nearly evenly between $A$ and $B$. The transition path visualizations in~\Cref{fig:tp_maier} (\Cref{appendix:additional_viz}) further show that the learned controls of REACT produce coherent reactive trajectories aligned with the non-equilibrium force field.

\subsubsection{Periodically-driven dynamics: triple-well system}

For the periodically-driven system, the committor is time dependent: the probability of reaching $B$ before $A$ depends on both the current state and the phase of the external forcing. We evaluate the learned committor at ten evenly spaced snapshots over one forcing period and report both the mean committor error and the learned $q=0.5$ transition surface in~\Cref{tab:react_results}. All three methods remain accurate in this non-stationary setting, and REACT-VM learns the most accurate committor at every phase, with about half the mean error of the semigroup baseline and of REACT-DBP (\Cref{fig:error_committor_periodic} in~\Cref{appendix:additional_viz}). This result is important because the forcing continuously deforms the reactive channel, so the method must learn a moving separatrix rather than a single stationary dividing surface. The visual comparison in~\Cref{fig:periodic_committor} shows that the learned committors track this phase-dependent deformation across the period, supporting the use of the VM formulation beyond stationary reversible dynamics.

\section{Conclusion}
\label{sec:conclusion}

We presented REACT, a framework for committor estimation that builds on the stochastic optimal control characterization of the committor: a rescaled committor defines a bounded value function through the Cole-Hopf transform, and the optimal feedback control, which takes the form of a (rescaled) Doob $h$-transform, steers trajectories toward the transition region. For practical committor learning, we introduced a principled formulation that grounds learning on capped hitting-time trajectories and preserves the original solution, so that the training trajectories can be short and need not reach either target set. As practical algorithms, we developed two training losses: a heuristic on-policy direct backpropagation (DBP) loss and a principled off-policy Value Matching (VM) loss. For VM, we proved that, under mild assumptions, first-order critical points recover the exact committor, both in the standard and capped-horizon settings. VM is clearly the better of the two: it learns more accurate committors than DBP in all our comparisons, and being off-policy, it can leverage simulated trajectories from alternative sampling processes and reuse them from a replay buffer, which enables efficient learning. In particular, we introduced $\kappa$-rescaled sampling dynamics that preserve the reactive current and can accelerate exploration, with the rescaled reactive density $\rho^{(\xi)}_\mathrm{R}$ as stationary distribution.

The numerical results support this stochastic optimal control perspective across both (generalized) reversible and non-reversible dynamics. On all three overdamped benchmarks, the triple-well, Müller-Brown, and rugged Müller-Brown potentials, REACT-VM gives the most accurate reaction rates compared with the BKE, BKE-IS and semigroup baselines in the same setting, and its learned control generates transition paths whose corridors and hitting-time distributions are comparable to those of transition path sampling. The underdamped double-well experiment shows that the same framework extends to phase-space committors, where inertia and degenerate noise make the transition surface depend on both position and momentum. The Maier-Stein and periodically-driven triple-well experiments further demonstrate that the method is not tied to detailed balance: it can learn committors and sample reactive trajectories in the presence of stationary probability currents and explicitly time-dependent forcing. On the underdamped double-well and the two non-reversible systems, REACT-VM learns a more accurate committor than the semigroup baseline and REACT-DBP at the same number of energy evaluations. On high-dimensional molecular systems under underdamped Langevin dynamics, from alanine dipeptide isomerization to the folding of the mini-protein chignolin in explicit solvent, REACT-VM learns the committor significantly faster than the uncontrolled semigroup approach in the same setting. On chignolin, a system of thousands of atoms whose committor we learn on the protein heavy atoms, REACT-VM reaches within $10$\,ns of simulated time an accuracy that the semigroup loss does not reach with a hundred times as much simulation.

Several directions remain open. The chignolin experiment shows that REACT scales to a protein in explicit solvent; the most important next step is larger proteins with slower transitions, where rare event analysis is most needed but where both committor representation and trajectory generation are substantially harder, and where the committor can depend on slow degrees of freedom that short trajectories resolve only weakly. This will likely require combining the SOC formulation with physically meaningful collective variables, symmetry-aware architectures, and sampling schemes that amortize the cost of long controlled trajectories. A second direction is to further develop the underdamped and non-reversible settings, where phase-space structure, degenerate diffusion, and time dependence introduce additional numerical and theoretical challenges. In particular, extending the first-order optimality guarantee of VM from uniformly elliptic to hypoelliptic noise would cover underdamped Langevin dynamics. Finally, the learned committor and the associated reactive current-preserving dynamics suggest a route toward adaptive enhanced sampling methods: using $\rho_\mathrm{R}$ or the $\kappa$-rescaled dynamics~\eqref{eq:reactive_SDE_kappa_gen} as proposals could connect committor learning directly with practical umbrella sampling, replica exchange, or molecular simulation workflows.

\acks{Y.D.\ acknowledges support from Cornell University. Y.D.\ thanks Carsten Hartmann and Roberto Car for the pointer to relevant work in committor estimation.
J.H.\ acknowledges support from the University of Cambridge Harding Distinguished Postgraduate Scholars Programme.
The authors declare no competing interests.}

\newpage
\appendix
\crefalias{section}{appendix}
\crefalias{subsection}{subappendix}
\crefalias{subsubsection}{subsubappendix}
\clearpage
\appendix
\onecolumn
\vspace*{1pt}

\vbox{%
    \hsize\textwidth
    \linewidth\hsize
    \vskip 0.1in
    \hrule height 4pt
  \vskip 0.25in
  \vskip -\parskip%
    \centering
    {\LARGE\bf
    {Rare Event Analysis via Stochastic Optimal Control \\Appendix}
    \par}
     \vskip 0.29in
  \vskip -\parskip
  \hrule height 1pt
  \vskip 0.09in%
  }

\startcontents[sections]
\printcontents[sections]{l}{1}{\setcounter{tocdepth}{3}}

\section{Learning committor functions for overdamped and underdamped Langevin} \label{sec:overdamped_underdamped_langevin}

\subsection{The committor problems for state-dependent diffusions} \label{eq:committor_state_dependent}

Consider the diffusion process with state-dependent volatility coefficient
\begin{talign}\label{eq:SDE_bg_state_dependent}
\mathrm{d}X_t = b(X_t)\,\mathrm{d}t + \sigma(X_t) \,\mathrm{d}W_t,
\end{talign}
with generator $\mathcal{L} = b\cdot\nabla + D:\nabla^2$, where $D = \tfrac{1}{2}\sigma\sigma^{\top}$ is the diffusion tensor, and $D:\nabla^2 \triangleq \sum_{i,j=1}^d D_{ij} \partial_{ij}$. The Fokker-Planck equation corresponding to \eqref{eq:SDE_bg_state_dependent} is
\begin{talign} \label{eq:FPE_state_dependent}
    \partial_t \rho_t = - \nabla \cdot (b \rho_t) + \nabla^2 : (D \rho_t),
\end{talign}
where
\begin{talign} \label{eq:nabla2_D_rho}
    \nabla^2 \!:\!(D \rho_t) \triangleq \nabla \cdot \big( \nabla \cdot (D \rho_t) \big) = \sum_{i,j=1}^{d} \frac{\partial^2}{\partial x_i \partial x_j} \big( D_{ij} \rho_t \big), \qquad \nabla \cdot (D \rho_t) \triangleq [\sum_{i=1}^{d} \partial_i (D \rho_t)_{ij} ]_{j=1}^{d}.
\end{talign}
The associated probability current $j_t$, which lets us write the Fokker-Planck equation compactly as $\partial_t \rho_t = - \nabla \cdot j_t$, takes the form
\begin{talign} \label{eq:flux_def}
    j_t = b \rho_t - \nabla \cdot \big( D \rho_t \big).
\end{talign}
By the divergence theorem, for any region $\Omega$ with outward normal $n$, the probability mass exiting $\Omega$ per unit time can be written as
\begin{talign*}
    \frac{\mathrm{d}}{\mathrm{d}t} \int_{\Omega} \rho_t = - \int_{\partial \Omega} \langle j_t, \mathrm{d}n \rangle.
\end{talign*}
We assume the process \eqref{eq:SDE_bg_state_dependent} is ergodic with stationary density~$\rho$ satisfying the stationary Fokker-Planck equation
\begin{talign}\label{eq:FPE_bg_state_dependent}
\nabla \cdot (b\rho - \nabla\cdot (D \rho)) = 0.
\end{talign}
The time-reversed process has generator $\tilde{\mathcal{L}} = \tilde b\cdot\nabla + D:\nabla^2$ with drift
\begin{talign} \label{eq:reverse_drift}
\tilde b = -b + 2\rho^{-1}D\nabla\rho + 2 \nabla \cdot D.
\end{talign}
We can derive this by expanding $\nabla^2 : (D \rho_t) = \nabla \cdot (D \nabla \rho_t + \nabla \cdot D \rho_t)$ and rewriting the Fokker-Planck equation \eqref{eq:FPE_state_dependent} as
\begin{talign*}
    \partial_t \rho_t = \nabla \cdot \big((-b + 2\rho^{-1}D\nabla\rho_t + 2 \nabla \cdot D) \rho_t \big) - \nabla^2 : (D \rho_t),
\end{talign*}
When detailed balance holds ($b\rho = \nabla\cdot(D\rho)$), the process is reversible and $\tilde{\mathcal{L}} = \mathcal{L}$.

The forward and backward (rescaled) committor problems for state-dependent diffusion coefficients are defined as in equation \eqref{eq:kbe_xi}.

\subsection{State-dependent overdamped Langevin} \label{eq:subsection_overdamped}
The overdamped
Langevin SDEs presented in equation \eqref{eq:SDE_rev}
can be generalized to the case in which the volatility coefficient $\sigma$ is state-dependent and time-dependent. For simplicity, we omit the time dependency, as it does not change the equations. The state-dependent overdamped Langevin dynamics reads:
\begin{talign} \label{eq:reference_anisotropic}
    \mathrm{d}X_t = \big( - \beta D(X_t) \nabla U(X_t) + \nabla \cdot D(X_t) \big) \mathrm{d}t + \sigma(X_t) \mathrm{d}W_t,
\end{talign}
where $D(x) = \frac{1}{2} \sigma \sigma^{\top}(x)$, and $\nabla \cdot D(x)$ is defined as in \eqref{eq:nabla2_D_rho}. That is,
\begin{talign} \label{eq:b_reference_anisotropic}
b(x) = - \beta D(x) \nabla U(x) + \nabla \cdot D(x).
\end{talign}
Note that setting $\sigma(x) = \sqrt{2 \beta^{-1}} \mathrm{I}$ we recover \eqref{eq:SDE_rev}.

\begin{lemma}[Stationary density and current, overdamped Langevin] \label{lem:overdamped_langevin}
   The stationary distribution for \eqref{eq:reference_anisotropic} is $\rho \propto \exp(-\beta U)$, and it satisfies detailed balance, i.e.\ the probability current $J$ corresponding to $\rho$ is zero.
\end{lemma}
\begin{proof}
Applying \eqref{eq:flux_def}, the current for $\rho \propto \exp(-\beta U)$ reads
\begin{talign*}
    J = (- \beta D \nabla U + \nabla \cdot D) \rho - \nabla \cdot \big( D \rho \big) = 0,
\end{talign*}
which directly implies that $\rho$ solves the stationary Fokker-Planck equation for \eqref{eq:reference_anisotropic}.
\end{proof}

\begin{remark} \label{eq:detailed_SOL}
Rearranging the proof of \Cref{lem:overdamped_langevin}, we see that requiring that $\rho$ satisfies detailed balance for the Fokker-Planck equation \eqref{eq:FPE_state_dependent} is equivalent to imposing that the dynamics is state-dependent overdamped Langevin.
\end{remark}

Applying \eqref{eq:reverse_drift}, the drift of the time-reversed process for the stationary density $\rho$ reads
\begin{talign} \label{eq:tilde_b_equals_b}
    \tilde{b}(x) = \beta D(x) \nabla U(x) - \nabla \cdot D(x) - 2\beta D(x) \nabla U(x) + 2 \nabla \cdot D(x) = b(x),
\end{talign}
which means that the backward SDE matches the forward SDE, and that the forward and time-reversed generators $\mathcal{L}$, $\tilde{\mathcal{L}}$ are equal. Plugging this into equation \eqref{eq:kbe_xi}, we obtain that the forward and backward committor functions are related by
\begin{talign} \label{eq:tilde_q_equals_q}
    \tilde{q}(x) = 1 - q(x).
\end{talign}
Using the relations \eqref{eq:tilde_b_equals_b} and \eqref{eq:tilde_q_equals_q}, for overdamped Langevin, the (rescaled) reactive SDE defined in equations \eqref{eq:rescaled_reactive_SDE_kappa} and \eqref{eq:v_rescaled_reactive} simplifies to
\begin{talign}
\begin{split}
\label{eq:rescaled_reactive_SDE_kappa_app}
dX_t &= b_\kappa(X_t)\,dt + \sqrt{\kappa}\,\sigma(X_t)\,dW_t, \\
\text{where } b_\kappa &=
\kappa b + (1+\kappa)D\nabla\log q_{\xi} - (1-\kappa)D\nabla\log(1-q_{\xi}) \\ &= \kappa \big( - \beta D \nabla U + \nabla \cdot D \big) + (1+\kappa)D\nabla\log q_{\xi} - (1-\kappa)D\nabla\log(1-q_{\xi}).
\end{split}
\end{talign}
And the drift $v$ and the Value Matching loss defined in equations \eqref{eq:approximate_v}-\eqref{eq:final_vm} (see \Cref{cor:particular_VM} for the derivation for state-dependent $\sigma$) simplify to
\begin{talign}
\begin{split} \label{eq:approximate_v_app_overdamped}
    v(x) &=
    \kappa \big( - \beta D(x) \nabla U(x) + \nabla \cdot D(x) \big) \\ &\quad - (1+\kappa)D(x)\nabla\bar{\phi}(x) - (1-\kappa)D(x)\nabla\log\big(1-e^{-\bar{\phi}(x)}\big).
\end{split}
\end{talign}
and
\begin{talign}
\label{eq:final_vm_app_overdamped}
\begin{split}
    &\mathcal{L}_{\mathrm{VM}_{v,\kappa}}(\phi) \\ &= \mathbb{E} \Bigl[ \tfrac{1}{2} \Bigl( \int_0^{\tau(T)} \big\langle \nabla \phi(X^{v,\kappa}_t), \tfrac{1}{\sqrt{\kappa}} \sigma(X^{v,\kappa}_t) \, \mathrm{d}W_t
    \\ &\qquad\qquad\qquad\qquad + D \big(
    \nabla \phi(X^{v,\kappa}_t) \!-\!(\tfrac{1}{\kappa}\!+\!1)\!\nabla \bar{\phi}(X^{v,\kappa}_t)
    \!- \!(\tfrac{1}{\kappa}\!-\!1) \nabla \log (1\!-\!e^{-\bar{\phi}(X^{v,\kappa}_t)})
    \big) \, \mathrm{d}t
    \big\rangle \\ &\qquad\quad + \! \tfrac{1}{\kappa} \phi(X^{v,\kappa}_0)
    \! + \! (\mathbf{1}_{\mathcal{S}}(X^{v,\kappa}_{\tau(T)}) \! - \! \tfrac{1}{\kappa})
    \phi(X^{v,\kappa}_{\tau(T)})
    \! - \! \mathbf{1}_{\mathcal{S}}(X^{v,\kappa}_{\tau(T)}) g(X^{v,\kappa}_{\tau(T)})
    \Bigr)^2 \Bigr].
\end{split}
\end{talign}

\subsection{State-dependent underdamped Langevin} \label{eq:subsection_underdamped}
A generalized version of the underdamped Langevin SDE \eqref{eq:SDE_underdamped}
reads:
\begin{talign}
\begin{split} \label{eq:state_dependent_underdamped}
\begin{cases}
\mathrm{d}R_t = \frac{\beta}{\gamma} \tilde{D}(R_t) M^{-1}P_t\,\mathrm{d}t,\\
\mathrm{d}P_t =  \big( \! - \!\frac{\beta}{\gamma} \tilde{D}(R_t) \nabla U(R_t) \! + \! \frac{1}{\gamma} \nabla \cdot \tilde{D}(R_t) \! - \! \beta D(R_t, P_t) M^{-1}P_t \! + \! \nabla_{p} \cdot D(R_t, P_t) \big) \,\mathrm{d}t
\! + \! \sigma(R_t, P_t) \,\mathrm{d}W_t.
\end{cases}
\end{split}
\end{talign}
Here,
\begin{itemize}
\item $D = \frac{1}{2} \sigma \sigma^{\top} \in \mathbb{R}^{d \times d}$. We assume that $\sigma(r,p) = \sigma(r,-p)$, and consequently the analogous equality holds for $D$,
\item the preconditioner matrix $\tilde{D} \in \mathbb{R}^{d \times d}$ is another symmetric matrix that may or may not be equal to $D$,
\item the mass matrix $M \in \mathbb{R}^{d \times d}$ is symmetric and positive definite,
\item and the scalar $\gamma > 0$ is the friction coefficient.
\end{itemize}
If we set $\sigma(r,p) = \sqrt{2 \gamma \beta^{-1}} \mathrm{I}$, $\tilde{D}(r) = D(r,p) = \gamma \beta^{-1} \mathrm{I}$, $M = \mathrm{I}$ we recover the SDE \eqref{eq:SDE_underdamped}.

The density \(\rho_t(r,p)\) of the joint random variable $(R_t,P_t)$ obeys the Fokker-Planck equation
\begin{talign} \label{eq:fp_underdamped}
\begin{split}
\partial_t \rho_t &= -\nabla_{r}\!\cdot\bigl(\frac{\beta}{\gamma} \tilde{D} M^{-1}p\,\rho_t \bigr)
+\nabla_{p}\!\cdot\Bigl( (\frac{\beta}{\gamma} \tilde{D} \nabla U - \frac{1}{\gamma} \nabla \cdot \tilde{D} + \beta D M^{-1} p - \nabla_{p} \cdot D)\,\rho_t + \nabla_{p} \cdot \big( D \,\rho_t \big) \Bigr).
\end{split}
\end{talign}
We can write \eqref{eq:state_dependent_underdamped} compactly as
\begin{talign} \label{eq:b_Sigma}
\mathrm{d}X_t = b(X_t) \mathrm{d}t + \Sigma(X_t) \, \mathrm{d}W_t,
\end{talign}
where $X_t = (R_t, P_t)$ and
\begin{talign} \label{eq:b_sigma_underdamped}
\begin{split}
    b(r,p) \! &= \!
    \begin{bmatrix}
        b^{(r)}(r,p) \\
        b^{(p)}(r,p)
    \end{bmatrix} \! = \!
    \begin{bmatrix}
        \frac{\beta}{\gamma} \tilde{D}(r) M^{-1}p \\
        - \frac{\beta}{\gamma} \tilde{D}(r) \nabla U(r) + \frac{1}{\gamma} \, \nabla \cdot \tilde{D}(r) - \beta \,D(r,p)\,M^{-1} p + \nabla_{p} \cdot D(r,p)
    \end{bmatrix}, \\
    \Sigma(r,p) \! &= \!
    \begin{bmatrix}
        0 & 0 \\
        0 & \sigma(r,p)
    \end{bmatrix}.
\end{split}
\end{talign}

\begin{lemma}[Stationary density and current, underdamped Langevin]
    The stationary distribution for \eqref{eq:state_dependent_underdamped} is
    \begin{talign*}
    \rho(r,p)
    = Z^{-1}\exp\bigl(-\beta H(r,p)\bigr),
    \qquad
    H(r,p)=U(r)+\tfrac12\,p^{\top} M^{-1} p.
    \end{talign*}
    In general, $\rho$ does not satisfy detailed balance for \eqref{eq:state_dependent_underdamped}; its probability current $J = b \rho - \nabla \cdot \big( D \rho \big)$ reads
    \begin{talign} \label{eq:flux_underdamped}
    J(r,p) &=
    \begin{bmatrix}
        \frac{\beta}{\gamma} \tilde{D}(r) M^{-1}p \ \rho(r,p) \\
        \big( - \frac{\beta}{\gamma} \tilde{D}(r) \nabla U(r) + \frac{1}{\gamma} \, \nabla \cdot \tilde{D}(r) \big) \rho(r,p)
    \end{bmatrix}.
\end{talign}
\end{lemma}
\begin{proof}
We split the right-hand side of the Fokker-Planck equation \eqref{eq:fp_underdamped} into two parts:
\begin{talign} \label{eq:two_parts}
    \underbrace{-\nabla_{r}\cdot(\frac{\beta}{\gamma} \tilde{D}\,M^{-1} p \,\rho)\! + \! \nabla_{p} \cdot\bigl( \big( \frac{\beta}{\gamma} \tilde{D}\,\nabla U \! - \! \frac{1}{\gamma} \nabla \cdot \tilde{D} \big) \,\rho\bigr)}_{\text{Transport part}} \! + \! \underbrace{\nabla_{p}\!\cdot\Bigl( \big( \beta\,D \,M^{-1} p \! - \! \nabla_{p} \cdot D  \big) \,\rho \! + \! \nabla_{p} \cdot \big( D \,\rho \big) \Bigr)}_{\text{Friction-diffusion part}}.
\end{talign}
We prove that each of these parts are zero.

\emph{Transport part}.
Since $\nabla_{r} \rho = - \beta \nabla_{r} U$,
and
$\nabla_{p} \rho = - \beta M^{-1} p \rho$,
we compute
\begin{talign*}
\begin{split}
-\nabla_{r}\cdot(\tilde{D}\,M^{-1} p\,\rho)
&= - \langle \tilde{D} \,M^{-1} p, \nabla_{r} \rho \rangle - \langle \nabla \cdot \tilde{D}, M^{-1} p \rangle \rho
\\ &= \beta \langle \tilde{D} \,M^{-1} p, \nabla U \rangle \rho - \langle \nabla \cdot \tilde{D}, M^{-1} p \rangle \rho
\end{split} \\
\begin{split}
\nabla_{p}\cdot \big( ( -\tilde{D}\,\nabla U + \frac{1}{\beta} \nabla \cdot \tilde{D}) \,\rho \big)
&= - \langle \tilde{D}\,\nabla U, \nabla_{p} \rho \rangle + \frac{1}{\beta} \langle \nabla \cdot \tilde{D}, \nabla_{p} \rho \rangle
\\ &= \beta \langle \tilde{D} \,M^{-1} p, \nabla U \rangle \rho - \langle \nabla \cdot \tilde{D}, M^{-1} p \rangle \rho
\end{split}
\end{talign*}
and thus, their sum cancels.

\emph{Friction-diffusion part}. Observe that
\begin{talign*}
\big( \beta\,D \,M^{-1} p - \nabla_{p} \cdot D  \big) \,\rho
+ \nabla_{p}\cdot \big( D \rho \big)
=\big( \beta\,D \,M^{-1} p - \nabla_{p} \cdot D  \big) \,\rho
-\beta D M^{-1} p \,\rho + \nabla_{p} \cdot D \,\rho
=0,
\end{talign*}
which means that this part does not contribute to the current, and consequently has zero contribution in \eqref{eq:two_parts}.

The expression \eqref{eq:flux_underdamped} for the current follows directly from the fact that the friction-diffusion part does not contribute to the current.
\end{proof}

Applying equation \eqref{eq:reverse_drift} to the SDE \eqref{eq:b_Sigma}, the drift of the time-reversed process for the stationary density $\rho$ reads
\begin{talign} \label{eq:b_tilde_b}
\begin{split}
    &\tilde{b}(r,p) = - b(r,p) +  \rho(r,p)^{-1} \big( \Sigma \Sigma^{\top} \big)(r,p) \nabla \rho(r,p) + \nabla \cdot \big( \Sigma \Sigma^{\top} \big)(r,p) \\ &=
    \begin{bmatrix}
        -b^{(r)}(r,p) \\
        -b^{(p)}(r,p) + 2 \rho(r,p)^{-1} D(r,p) \nabla_{p} \rho(r,p) + 2 \nabla_{p} \cdot D(r,p)
    \end{bmatrix} \\ &=
    \begin{bmatrix}
        - \frac{\beta}{\gamma} \tilde{D}(r) M^{-1}p \\
        \frac{\beta}{\gamma} \tilde{D}(r) \nabla U(r) \! - \! \frac{1}{\gamma} \nabla \cdot \tilde{D}(r) \! + \! \beta D(r,p) M^{-1} p \! - \! \nabla_{p} \cdot D(r,p) \! - \! 2 \beta \,D(r,p)\,M^{-1} p \!+ \! 2 \nabla_{p} \cdot D(r,p)
    \end{bmatrix}
    \\ &=
    \begin{bmatrix}
        -\frac{\beta}{\gamma} \tilde{D}(r) M^{-1} p \\
        \frac{\beta}{\gamma} \tilde{D}(r) \nabla U(r) \! - \! \frac{1}{\gamma} \nabla \cdot \tilde{D}(r) \! - \! \beta D(r,-p) M^{-1} p \! + \! \nabla_{p} \cdot D(r,-p)
    \end{bmatrix} = \begin{bmatrix}
    b^{(r)}(r,-p) \\
    -b^{(p)}(r,-p)
    \end{bmatrix},
\end{split}
\end{talign}
where the second-to-last equality holds because $D(r,-p) = D(r,p)$ by assumption.

The stationary backward Kolmogorov equation for the forward committor reads
\begin{talign} \label{eq:BKE_forward_app}
\begin{cases}
    &\mathrm{Tr}\big( D(r,p) \nabla_{p,p}^2 q(r,p) \big) +
    \langle b^{(r)}(r,p),
    \nabla_{r} q(r,p) \rangle
    + \langle b^{(p)}(r,p),
    \nabla_{p} q(r,p) \rangle = 0, \\
    &q(r,p) = 0, \quad \forall (r,p) \in A, \\ &q(r,p) = 1, \quad \forall (r,p) \in B.
\end{cases}
\end{talign}
And using equation \eqref{eq:b_tilde_b},
the BKE
for the backward committor reads
\begin{talign} \label{eq:BKE_backward_app}
\begin{cases}
    &\mathrm{Tr}\big( D(r,p) \nabla_{p,p}^2 \tilde{q}(r,p) \big)
    + \langle b^{(r)}(r,-p),
    \nabla_{r} \tilde{q}(r,p) \rangle
    - \langle b^{(p)}(r,-p),
    \nabla_{p} \tilde{q}(r,p) \rangle = 0, \\
    &\tilde{q}(r,p) = 1, \quad \forall (r,p) \in A, \\ &\tilde{q}(r,p) = 0, \quad \forall (r,p) \in B.
\end{cases}
\end{talign}
Observe that if $q(r,p)$ is a solution of \eqref{eq:BKE_forward_app}, then
\begin{align} \label{eq:tilde_q_q_underdamped}
\tilde q(r,p) = 1 - q(r,-p)
\end{align}
is a solution of \eqref{eq:BKE_backward_app}. This holds because $\nabla_{r} \tilde{q}(r,p) = - \nabla_{r} q(r,-p)$, $\nabla_{p} \tilde{q}(r,p) = \nabla_{p} q(r,-p)$ and $\nabla_{p,p}^2 \tilde{q}(r,p) = - \nabla_{p,p}^2 q(r,-p)$, which means that the first equation in \eqref{eq:BKE_backward_app} is equivalent to
\begin{talign*}
- \mathrm{Tr}\big( D(r,p) \nabla_{p,p}^2 q(r,-p) \big)
    - \langle b^{(r)}(r,-p),
    \nabla_{r} q(r,-p) \rangle
    - \langle b^{(p)}(r,-p),
    \nabla_{p} q(r,-p) \rangle = 0,
\end{talign*}
which matches the first equation in \eqref{eq:BKE_forward_app} up to a change of variable $p \gets - p$ and a sign flip.

Using equations \eqref{eq:b_sigma_underdamped}, \eqref{eq:b_tilde_b} and \eqref{eq:tilde_q_q_underdamped}, for underdamped Langevin, the (rescaled) reactive SDE defined in equations \eqref{eq:rescaled_reactive_SDE_kappa} and \eqref{eq:v_rescaled_reactive} reads
\begin{talign}
\begin{split}
\label{eq:rescaled_reactive_SDE_kappa_app_underdamped}
dX_t &= b_\kappa(X_t)\,dt + \sqrt{\kappa}\,\Sigma(X_t)\,dW_t, \qquad \text{where }\\
b_\kappa(r,p) \! &= \!
\begin{bmatrix}
    b^{(r)}(r,p) \\
    \frac{1+\kappa}{2} \big( b^{(p)}(r,p) \! + \! 2D(r,p) \nabla_{p} \log q_{\xi}(r,p) \big) \! + \! \frac{\kappa-1}{2} \big(\! - \! b^{(p)}(r,-p) \! + \! 2D(r,-p) \nabla_{p} \log \big(1 \! - \! q_{\xi}(r,-p) \big) \big)
\end{bmatrix}.
\end{split}
\end{talign}
In this equation and below, $\nabla_{p} \log \big(1-q_{\xi}(r,-p) \big)$ denotes the momentum gradient of the map $p \mapsto \log \big(1-q_{\xi}(r,-p)\big) = \log \tilde q_{\xi}(r,p)$, i.e.\ $-(\nabla_{p} \log(1-q_{\xi}))(r,-p)$; the same convention applies to $\nabla_{p} \log \big(1-\exp(-\bar{\phi}(r,-p))\big)$.
And the drift $v$ and the Value Matching loss defined in equations \eqref{eq:approximate_v}-\eqref{eq:final_vm} (see \Cref{cor:particular_VM} for the derivation for state-dependent $\sigma$) simplify to
\begin{talign}
\begin{split} \label{eq:approximate_v_underdamped_app}
    &v(r,p) =
    \begin{bmatrix}
        v^{(r)}(r,p) \\
        v^{(p)}(r,p)
    \end{bmatrix}
    \\ &=
    \begin{bmatrix}
    b^{(r)}(r,p)
    \\
    \tfrac{1+\kappa}{2}\big(b^{(p)}(r,p) \! - \! 2D(r,p)\nabla_{p}\bar \phi(r,p)\big) \\ \quad + \tfrac{\kappa-1}{2}\big(-b^{(p)}(r,-p) \! + \! 2D(r,-p)\nabla_{p} \log \big( 1 \! - \! \exp(-\bar{\phi}(r,-p)) \big) \big)
    \end{bmatrix}
\end{split}
\end{talign}
\begin{talign}
\label{eq:final_vm_underdamped_app}
\begin{split}
    &\mathcal{L}_{\mathrm{VM}_{v,\kappa}}(\phi) \\ &= \mathbb{E} \Bigl[ \tfrac{1}{2} \Bigl( \int_0^{\tau(T)} \Big\langle \nabla_p \phi(R^{v,\kappa}_t,P^{v,\kappa}_t), \tfrac{1}{\sqrt{\kappa}} \sigma(R^{v,\kappa}_t,P^{v,\kappa}_t) \, \mathrm{d}W_t
    \\ &\qquad\qquad\qquad\qquad
    + \big( \tfrac{1}{\kappa} - 1 \big) \big(
    - \frac{\beta}{\gamma} \tilde{D}(R^{v,\kappa}_t) \nabla U(R^{v,\kappa}_t) + \frac{1}{\gamma} \, \nabla \cdot \tilde{D}(R^{v,\kappa}_t)
    \big) \, \mathrm{d}t \\ &\qquad\qquad\qquad\qquad + D(R^{v,\kappa}_t,P^{v,\kappa}_t) \big(
    \nabla_p \phi(R^{v,\kappa}_t,P^{v,\kappa}_t) \!-\!(\tfrac{1}{\kappa}\!+\!1)\!\nabla_p \bar{\phi}(R^{v,\kappa}_t,P^{v,\kappa}_t)
    \\ &\qquad\qquad\qquad\qquad\qquad\qquad\qquad - \!(\tfrac{1}{\kappa}\!-\!1) \nabla_p \log \big( 1 - \exp(-\bar{\phi}(R^{v,\kappa}_t,-P^{v,\kappa}_t)) \big)
    \big) \, \mathrm{d}t
    \Big\rangle \\ &\qquad\quad + \big( \tfrac{1}{\kappa} - 1 \big) \int_0^{\tau(T)} \big\langle \nabla_r \phi(R^{v,\kappa}_t,P^{v,\kappa}_t), b^{(r)}(R^{v,\kappa}_t,P^{v,\kappa}_t) \big\rangle \, \mathrm{d}t \\ &\qquad\quad + \! \tfrac{1}{\kappa} \phi(R^{v,\kappa}_0,P^{v,\kappa}_0)
    \! + \! (\mathbf{1}_{\mathcal{S}}(R^{v,\kappa}_{\tau(T)}, P^{v,\kappa}_{\tau(T)}) \! - \! \tfrac{1}{\kappa})
    \phi(R^{v,\kappa}_{\tau(T)}, P^{v,\kappa}_{\tau(T)})
    \! - \! \mathbf{1}_{\mathcal{S}}(R^{v,\kappa}_{\tau(T)}, P^{v,\kappa}_{\tau(T)}) g(R^{v,\kappa}_{\tau(T)}, P^{v,\kappa}_{\tau(T)})
    \Bigr)^2 \Bigr].
\end{split}
\end{talign}
In the last equation, we used that the position component of $\tfrac{1}{\kappa}\mathrm{d}X^{v,\kappa}_t - b(X^{v,\kappa}_t)\,\mathrm{d}t$ equals $(\tfrac{1}{\kappa}-1)\,b^{(r)}\,\mathrm{d}t$ (because $v^{(r)} = b^{(r)}$ and the noise acts only on the momenta), and that by equation \eqref{eq:flux_underdamped},
\begin{align}
\begin{split}
    b^{(p)}(r,p)
        - D(r,p) \nabla_p \log \rho(r,p) - \nabla_p \cdot D(r,p)
    =
        - \frac{\beta}{\gamma} \tilde{D}(r) \nabla U(r) + \frac{1}{\gamma} \, \nabla \cdot \tilde{D}(r).
\end{split}
\end{align}

\section{Sampling from the reactive density}

The reactive distribution and the ways to sample from it were first studied by \cite{vanden2006towards,vanden2010transition,lu2015reactive}. All these works considered the case $\xi = 0$. In this section, we generalize their results to $\xi \in [0,1/2)$, which is critical for our framework, and provide algorithmic sampling procedures for the reactive density, assuming access to the forward and backward committors.

\subsection{\texorpdfstring{Doob's $h$-transform and reactive trajectories}{Doob's h-transform and reactive trajectories}} \label{app:doob_derivation}
This subsection takes a path-space viewpoint on Doob's $h$-transform~\eqref{eq:doob_bg} from \Cref{subsec:committor_tpt}, for the unrescaled committor~$q$ ($\xi = 0$). \Cref{subsec:reactive,subsec:different_noise_levels,sec:mollified_tpt_appendix} take a density viewpoint (stationary Fokker-Planck equations with sources and sinks) and treat general $\xi \in [0,1/2)$.

\paragraph{Derivation of the conditioned drift.}
Let $X$ solve~\eqref{eq:SDE_bg} with $X_0 = x \notin A\cup B$, and let $\tau$ be the first hitting time of $A\cup B$. By Bayes' rule, the law $\mathbb{P}^B_x$ of $X$ conditioned on the event $\{X_\tau\in B\}$ has density
\begin{talign*}
\frac{\mathrm{d}\mathbb{P}^B_x}{\mathrm{d}\mathbb{P}_x} = \frac{\mathbf{1}_B(X_\tau)}{q(x)}
\end{talign*}
with respect to the unconditioned law $\mathbb{P}_x$. By the strong Markov property, its density process (the density restricted to $\mathcal{F}_t$) is $M_t = \mathbb{E}_x[\mathbf{1}_B(X_\tau)\mid\mathcal{F}_t]/q(x) = q(X_{t\wedge\tau})/q(x)$. Since $\mathcal{L}q = 0$ in $(A\cup B)^c$, It\^o's formula gives $\mathrm{d}q(X_t) = \langle \sigma^{\top}\nabla q(X_t), \mathrm{d}W_t\rangle$ for $t<\tau$, so that
\begin{talign*}
\mathrm{d}M_t = M_t\,\langle \sigma^{\top}\nabla\log q(X_t), \mathrm{d}W_t\rangle, \qquad t < \tau,
\end{talign*}
i.e.\ $M$ is the stochastic exponential of $\int_0^{t\wedge\tau}\langle\sigma^{\top}\nabla\log q(X_s),\mathrm{d}W_s\rangle$. By Girsanov's theorem (\Cref{subsec:girsanov}), $\tilde W_t = W_t - \int_0^{t} \sigma^{\top}\nabla\log q(X_s)\,\mathrm{d}s$ is a Brownian motion under $\mathbb{P}^B_x$ up to time~$\tau$, and therefore
\begin{talign*}
\mathrm{d}X_t = \big(b(X_t) + \sigma\sigma^{\top}\nabla\log q(X_t)\big)\,\mathrm{d}t + \sigma\,\mathrm{d}\tilde W_t,
\end{talign*}
which is~\eqref{eq:doob_bg} since $\sigma\sigma^{\top} = 2D$. Equivalently, the generator of the conditioned process is $q^{-1}\mathcal{L}(q\,\cdot) = \mathcal{L} + 2D\nabla\log q\cdot\nabla$. The same argument with $\mathbf{1}_B(X_\tau)$ replaced by $e^{-g(X_\tau)}$ and $q$ by $q_\xi$ shows that, up to time~$\tau$, the optimally controlled SDE~\eqref{eq:X_u_star_main} has the law $\mathbb{P}^{\star}$ of \Cref{rem:SOC_probabilistic}. This is a reweighting rather than a conditioning, since paths ending in~$A$ keep weight~$\xi$.

\paragraph{Starting distribution.}
The $h$-transform determines how paths conditioned to reach~$B$ before~$A$ evolve, but it does not determine their starting points on~$\partial A$. In the stationary ensemble of \Cref{subsec:committor_tpt}, reactive trajectories leave~$A$ according to the reactive current through~$\partial A$. Since $q = 0$ and $\tilde q = 1$ on~$\partial A$, \eqref{eq:reactive_current_bg} gives $J_\mathrm{R}\cdot n_A = Z_{AB}^{-1}\,\rho\, n_A\cdot D\nabla q$ there, where $n_A$ is the outward unit normal of~$\partial A$; this is the source term $m_\xi$ of \Cref{lem:mass_creation_destruction} with $\xi = 0$. Starting~\eqref{eq:doob_bg} on~$\partial A$ with density proportional to $\rho\, n_A\cdot D\nabla q$ and stopping it at~$\partial B$ reproduces the reactive ensemble~\citep{lu2015reactive}; other starting distributions give paths that reach~$B$ before~$A$, but not with the statistics of the reactive ensemble. Since $q = 0$ on~$\partial A$, the drift $2D\nabla\log q$ is singular there and $\{X_\tau \in B\}$ has probability zero for $X_0 \in \partial A$, so starting on~$\partial A$ is understood as the limit of starting points $x \to \partial A$ from outside~$A$. Working with the rescaled committor $q_\xi$, $\xi > 0$, removes this singularity, since $\log q_\xi$ is bounded.

The resulting sampling schemes, with births on~$\partial A$ or inside~$A$, are given in \Cref{subsec:different_noise_levels,sec:mollified_tpt_appendix}.

\subsection{The reactive density as the solution of a stationary Fokker-Planck PDE} \label{subsec:reactive}

For $\xi \in [0,1/2)$,
let $q_{\xi}$, $\tilde{q}_{\xi}$
be the rescaled forward and backward committor functions as defined in equation \eqref{eq:kbe_xi}.
Consider the SDE controlled by the optimal control
$u^{\star}(x) = \sigma^{\top}(x) \nabla \log q_{\xi}
$:
\begin{talign} \label{eq:optimally_controlled_SDE}
    \mathrm{d}X^{u^{\star}}_t = (b(X^{u^{\star}}_t) + 2 D(X^{u^{\star}}_t) \nabla \log q_{\xi}
    (X^{u^{\star}}_t)) \, \mathrm{d}t + \sigma(X^{u^{\star}}_t) \, \mathrm{d}W_t,
\end{talign}
where we used that $D = \frac{1}{2} \sigma \sigma^{\top}$,
and its corresponding Fokker-Planck equation:
\begin{talign} \label{eq:FPE_rho_t}
    \partial_t \rho_t = - \nabla \cdot \big( (b + 2 D \nabla \log q_{\xi}) \rho_t \big) + \nabla^2 \!:\!(D \rho_t).
\end{talign}
The following result characterizes the stationary densities of this dynamics.

\begin{theorem} \label{thm:stationary}
Assume the uncontrolled diffusion has smooth invariant density $\rho$ and define
\begin{talign} \label{eq:def_z_xi_c_xi}
Z^{(\xi)}_{AB}:=\int_{\mathbb{R}^d}\rho(x)\,q_{\xi}(x)\tilde q_{\xi}(x)\,\mathrm{d}x, \qquad
C^{(\xi)}_{AB}:=\int_{\mathbb{R}^d}\rho(x)q_{\xi}(x)^2\,\mathrm{d}x.
\end{talign}
Then:
\begin{enumerate}[left=3pt,label=(\roman*)]
    \item The \emph{reactive stationary density}
    \begin{talign} \label{eq:react_dist}
        \rho^{(\xi)}_{\mathrm{R}}(x)=\frac{\rho(x)\,q_{\xi}(x)\tilde q_{\xi}(x)}{Z^{(\xi)}_{AB}}
    \end{talign}
    is stationary for \eqref{eq:FPE_rho_t} on $\mathbb{R}^d\setminus(A\cup B)$.
    For overdamped Langevin,
    we have that $\tilde q=1-q$, and equation \eqref{eq:react_dist} reduces to
    \begin{talign*}
    \rho^{(\xi)}_{\mathrm{R}} \propto \rho \,q_{\xi}(1-q_{\xi}).
    \end{talign*}
    For underdamped Langevin, we have that $\tilde q(r,p)=1-q(r,-p)$, and consequently,
    \begin{talign*}
    \rho^{(\xi)}_{\mathrm{R}}(r,p) \propto \rho(r,p) \,q_{\xi}(r,p) \big(1-q_{\xi}(r,-p) \big).
    \end{talign*}
    \item For overdamped Langevin, the \emph{tilted stationary density}
    \begin{talign*}
        \rho^{(\xi)}_{\mathrm{teq}}(x)=\frac{\rho(x)q_{\xi}(x)^2}{C^{(\xi)}_{AB}}
    \end{talign*}
    is stationary for \eqref{eq:FPE_rho_t} over $\mathbb{R}^d$.
\end{enumerate}
\end{theorem}
\begin{proof}
For simplicity, throughout the proof we use the short-hands $q = q_{\xi}$, $\tilde{q} = \tilde{q}_{\xi}$, $\rho_{\mathrm{R}} = \rho^{(\xi)}_{\mathrm{R}}$, $J_{\mathrm{R}} = J^{(\xi)}_{\mathrm{R}}$, $Z_{AB} = Z^{(\xi)}_{AB}$, $\rho_{\mathrm{teq}} = \rho^{(\xi)}_{\mathrm{teq}}$, $J_{\mathrm{teq}} = J^{(\xi)}_{\mathrm{teq}}$, $C_{AB} = C^{(\xi)}_{AB}$.
We can rewrite the Fokker-Planck equation \eqref{eq:FPE_rho_t} as follows:
\begin{talign}
\begin{split} \label{eq:partial_rho_t_0}
    \partial_t \rho_t
    &= \nabla \cdot \big( (-b - 2D \nabla \log q) \rho_t
    + \nabla \cdot (D \rho_t)
    \big)
    \\ &= \nabla \cdot \big( (-b - 2D \nabla \log q) \rho_t
    + \nabla \cdot D \rho_t + D \nabla \rho_t
    \big),
\end{split}
\end{talign}
which means that the associated current is
\begin{talign} \label{eq:current_j_t}
j_t = (b + 2D \nabla \log q) \rho_t
    - \nabla \cdot D \rho_t - D \nabla \rho_t.
\end{talign}
We prove \emph{(i)} first. For $\rho_{\mathrm{R}}=\rho q \tilde q /Z_{AB}$, the current \eqref{eq:current_j_t} reads
\begin{talign}
\begin{split} \label{eq:reactive_flux}
    J_{\mathrm{R}}
    &=
    \frac{1}{Z_{AB}} \Big( q \tilde q \,\big( J  + 2D \nabla \log q \rho \big)
    - \rho\,\tilde q D\nabla q
    - \rho\,q D\nabla\tilde q \Big) \\ &= \frac{1}{Z_{AB}} \Big( q \tilde q J
    + \rho\,\tilde q D\nabla q
    - \rho\,q D\nabla\tilde q \Big),
\end{split}
\end{talign}
where $J:=b\rho - \nabla \cdot D \rho -D\nabla\rho$ is the current of the invariant density $\rho$ for the reference dynamics \eqref{eq:reference0}.
Taking divergence on both sides of the equality, and using $\nabla\cdot J=0$ by construction,
\begin{talign} \label{eq:nabla_j_xi}
    Z_{AB} \nabla \cdot J_{\mathrm{R}} = \nabla \cdot \big( q \tilde q \big) J + \nabla \cdot \big(\rho\,\tilde q D\nabla q \big) - \nabla \cdot \big( \rho\,q D\nabla\tilde q \big).
\end{talign}
We develop the last two terms in the right-hand side separately.

\textbf{Term $\nabla \cdot \big(\rho\,\tilde q D\nabla q \big)$}. We have that
\begin{talign}
\begin{split} \label{eq:second_term_divergence}
    \nabla \cdot \big(\rho\,\tilde q D\nabla q \big) &= \tilde q \nabla \rho \cdot D\nabla q + \rho\, \nabla \tilde q \cdot D\nabla q + \rho\,\tilde q \nabla \cdot \big( D\nabla q \big) \\ &\overset{(i)}{=} \tilde q \nabla \rho \cdot D\nabla q + \rho\, \nabla \tilde q \cdot D\nabla q - \rho\,\tilde q b \cdot \nabla q + \rho\,\tilde q (\nabla \cdot D) \cdot \nabla q \\ &\overset{(ii)}{=} \tilde q \nabla q \cdot \big( - b \rho D + \nabla \cdot D \rho + D \nabla \rho \big) + \rho\, \nabla \tilde q \cdot D\nabla q
    \\ &\overset{(iii)}{=} - \tilde q \nabla q \cdot J + \rho\, \nabla \tilde q \cdot D\nabla q,
\end{split}
\end{talign}
where equality $(i)$ holds because $D : \nabla^2 q = - b \cdot \nabla q$ by the forward committor BKE in \eqref{eq:kbe_xi}, equality $(ii)$ holds because $D$ is symmetric, and equality $(iii)$ holds by the definition of $J$.

\textbf{Term $\nabla \cdot \big( \rho\,q D\nabla\tilde q \big)$}. Analogously,
\begin{talign}
\begin{split} \label{eq:third_term_divergence}
    \nabla \cdot \big( \rho\,q D\nabla\tilde q \big) &= q \nabla \rho \cdot D\nabla \tilde q + \rho\, \nabla \tilde q \cdot D\nabla q + \rho\, q \nabla \cdot \big( D\nabla \tilde q \big) \\ &\overset{(i)}{=} \nabla \rho \cdot D\nabla \tilde q + \rho\, \nabla \tilde q \cdot D\nabla q + \rho\, q \big( b - 2 D \rho^{-1} \nabla \rho - 2 \nabla \cdot D \big) \cdot \nabla \tilde q + \rho\,q (\nabla \cdot D) \cdot \nabla \tilde q \\ &\overset{(ii)}{=} q \nabla \tilde q \cdot \big( b \rho D - \nabla \cdot D \rho - D \nabla \rho \big) + \rho\, \nabla \tilde q \cdot D\nabla q \\ &\overset{(iii)}{=} \tilde q \nabla q \cdot J + \rho\, \nabla \tilde q \cdot D\nabla q,
\end{split}
\end{talign}
where equality $(i)$ holds because $D : \nabla^2 \tilde q = - \tilde b \cdot \nabla \tilde q = \big(b - 2 \rho^{-1} D \nabla \rho - 2 \nabla \cdot D \big) \cdot \nabla \tilde q$ by the backward committor BKE in \eqref{eq:kbe_xi}, and equalities $(ii)$ and $(iii)$ hold by the same arguments as in \eqref{eq:second_term_divergence}.

\textbf{Combining terms}. Plugging \eqref{eq:second_term_divergence} and \eqref{eq:third_term_divergence} into the right-hand side of \eqref{eq:nabla_j_xi} yields
\begin{talign}
    Z_{AB} \nabla \cdot J_{\mathrm{R}} = \nabla \cdot \big( q \tilde q \big) J - \tilde q \nabla q \cdot J + \rho\, \nabla \tilde q \cdot D\nabla q - \tilde q \nabla q \cdot J - \rho\, \nabla \tilde q \cdot D\nabla q = 0,
\end{talign}
which proves that $\rho_{\mathrm{R}}$ is stationary for \eqref{eq:FPE_rho_t} on $\mathbb{R}^d\setminus(A\cup B)$. The statements for overdamped and underdamped Langevin follow from the derivations in \Cref{sec:overdamped_underdamped_langevin}.
\\ Next, we prove \emph{(ii)}. For $\rho_{\mathrm{teq}}(x)=\rho(x)q(x)^2/C_{AB}$, the current \eqref{eq:current_j_t} reads
\begin{talign}
\begin{split} \label{eq:j_x_teq}
    J_{\mathrm{teq}}
    &=
    \frac{1}{C_{AB}} \big( (b + 2D \nabla \log q) \rho q^2
    - \nabla \cdot D \rho q^2
    \\ &\qquad\quad - D \big( \nabla \rho q^2 + 2 \nabla q \rho q \big) \big)
    \\ &= \frac{1}{C_{AB}} \big( q^2 J + 2D \nabla \log q \rho q^2 - 2 D \nabla q \rho q \big) = \frac{1}{C_{AB}} q^2 J.
\end{split}
\end{talign}
As shown in \Cref{sec:overdamped_underdamped_langevin}, for overdamped Langevin, $J = b(x)\,\rho(x)-
\nabla \cdot\bigl(D(x)\,\rho(x)\bigr) = 0$, which implies that $J_{\mathrm{teq}} = 0$.
\end{proof}

\begin{remark}[Comparison between the tilted equilibrium and reactive distribution]
    For overdamped Langevin, observe that $q_{\xi}^2$ takes value $(1-\xi)^2$ on $B$ and value $\xi^2$ on $A$ and values in between elsewhere, and that $q_{\xi} \tilde q_{\xi} = q_{\xi} (1-q_{\xi})$ takes value $\xi (1-\xi)$ on $A \cup B$, value $\frac{1}{4}$ on the hypersurface $\{x \, | \, q_{\xi}(x) = \frac{1}{2} \}$, and values in between elsewhere. Hence, comparatively, $\rho^{(\xi)}_{\mathrm{teq}}$ allocates more mass around $B$ and $\rho^{(\xi)}_{\mathrm{R}}$ allocates more mass around $\{x \, | \, q_{\xi}(x) = \frac{1}{2} \}$. Since learning the committor function well around $\{x \, | \, q_{\xi}(x) = \frac{1}{2} \}$ is more interesting for our algorithms and other applications, sampling from $\rho^{(\xi)}_{\mathrm{R}}$ is more relevant than sampling from $\rho^{(\xi)}_{\mathrm{teq}}$.
    By \Cref{thm:stationary}, sampling from $\rho^{(\xi)}_{\mathrm{teq}}$ is straightforward as it can be accomplished simply by simulating the SDE \eqref{eq:optimally_controlled_SDE}. However, sampling from $\rho^{(\xi)}_{\mathrm{R}}$ is not as simple, as the behavior at the boundaries $\partial A$, $\partial B$ must be handled carefully.
    \Cref{subsec:different_noise_levels} is devoted to developing algorithms to sample from $\rho^{(\xi)}_{\mathrm{R}}$ in various ways. In the remainder of this section, we study the probability current of $\rho^{(\xi)}_{\mathrm{R}}$ as a preliminary.
\end{remark}

The following lemma shows that $\rho^{(\xi)}_{\mathrm{R}}$ does not satisfy detailed balance, as its probability current is not zero on $\mathbb{R}^d \setminus (A \cup B)$.

\begin{lemma} [Probability current of $\rho^{(\xi)}_{\mathrm{R}}$] \label{lem:flux_reactive}
The probability current $J^{(\xi)}_{\mathrm{R}}$ of the rescaled reactive distribution $\rho^{(\xi)}_{\mathrm{R}}$
is discontinuous at the boundary $\partial (A \cup B)$:
\begin{enumerate}[label=(\roman*),left=3pt]
\item On $\mathbb{R}^d \setminus (A \cup B)$, the current is
\begin{talign} \label{eq:flux_expression}
J^{(\xi)}_{\mathrm{R}} = \frac{1}{Z^{(\xi)}_{AB}} \Big( q_{\xi} \tilde q_{\xi} J + \rho \,\tilde q_{\xi} D\nabla q_{\xi} - \rho \,q_{\xi} D\nabla\tilde q_{\xi} \Big),
\end{talign}
where $Z^{(\xi)}_{AB}$ is defined in \eqref{eq:def_z_xi_c_xi}, and $J :=b\rho - \nabla \cdot D \rho -D\nabla\rho$ is the current of the invariant density $\rho$ for the reference dynamics \eqref{eq:reference0}.
For overdamped Langevin, this simplifies to
\begin{talign} \label{eq:j_overdamped_outside}
J^{(\xi)}_{\mathrm{R}} = \frac{1}{Z^{(\xi)}_{AB}} D \nabla q_{\xi} \rho.
\end{talign}
\item On $\mathrm{int}(A \cup B)$, the current is
\begin{talign} \label{eq:j_xi_re}
J^{(\xi)}_{\mathrm{R}} = \frac{1}{Z^{(\xi)}_{AB}} q_{\xi} \tilde q_{\xi} J = \frac{1}{Z^{(\xi)}_{AB}} \xi (1-\xi) J.
\end{talign}
For overdamped Langevin, this simplifies to
$J^{(\xi)}_{\mathrm{R}} = 0$.
\end{enumerate}
For $x \in \partial (A \cup B)$, the current discontinuity $J^{(\xi)}_{\partial}$ is given by
\begin{talign} \label{eq:j_xi_partial_def}
    J^{(\xi)}_{\partial}(x) \! \triangleq \! \lim_{y \to x, y \in \mathbb{R}^d \setminus (A \cup B)} J^{(\xi)}_{\mathrm{R}}(y) \! - \! \lim_{y' \to x, y' \in A \cup B} J^{(\xi)}_{\mathrm{R}}(y') \! = \! \frac{1}{Z^{(\xi)}_{AB}} \big( \rho \,\tilde q_{\xi} D\nabla q_{\xi} \! - \! \rho \,q_{\xi} D\nabla\tilde q_{\xi} \big).
\end{talign}
\end{lemma}
\begin{proof}
    We prove \emph{(i)} first. Equation \eqref{eq:flux_expression} is just a restatement of \eqref{eq:reactive_flux}.
    To prove \eqref{eq:j_overdamped_outside}, we use that for overdamped Langevin, detailed balance holds: $J = 0$, and that $\tilde q_{\xi} = 1- q_{\xi}$.
    To show \emph{(ii)}, we use that by the definitions of $q$ and $\tilde q$,
    $q_{\xi}(x) = \xi$, $\tilde q_{\xi}(x) = 1 - \xi$ for all $x \in A$, and $q_{\xi}(x) = 1-\xi$, $\tilde q_{\xi}(x) = \xi$ for all $x \in B$. Thus, in the interior of $A \cup B$, we have that $q_{\xi} = \xi \mathbf{1}_{A} + (1-\xi) \mathbf{1}_{B}$, $\tilde q_{\xi} = (1-\xi) \mathbf{1}_{A} + \xi \mathbf{1}_{B}$, which means that on $\mathrm{int}(A \cup B)$,
    \begin{talign*}
    \begin{split}
        \rho^{(\xi)}_{\mathrm{R}} &= \frac{1}{Z^{(\xi)}_{AB}} \rho \xi (1-\xi), \\
        \nabla q_{\xi} &= 0, \quad \nabla \tilde q_{\xi} = 0,
    \end{split}
    \end{talign*}
    and this implies equation \eqref{eq:j_xi_re}.
\end{proof}

The following lemma characterizes the hypersurface integral of the current discontinuity along the boundaries $\partial A$, $\partial B$. It also specifies the rate of mass creation and destruction that must happen at these boundaries to balance out the flux discontinuity.

\begin{lemma} \label{lem:mass_creation_destruction}
Let $n(x)$ be the outward unit normal of $\partial A$, and let $\mathrm{d}n(x) = n(x) \mathrm{d}S(x)$, $\mathrm{d}S(x)$ being the surface differential. The integral of $J_{\partial}$ across $\partial A$ is
\begin{talign} \label{eq:integral_partial_A}
    \int_{\partial A} \langle J^{(\xi)}_{\partial}, \mathrm{d}n \rangle = \frac{1}{2Z^{(\xi)}_{AB}} \int_{\partial A} \Big( \frac{\tilde q_{\xi}(x) \| \sigma^{\top} \nabla q_{\xi}(x) \|^2}{\|\nabla q_{\xi}(x)\|} + \frac{q_{\xi}(x) \| \sigma^{\top} \nabla \tilde q_{\xi}(x) \|^2}{\|\nabla \tilde q_{\xi}(x)\|} \Big) \rho(x) \, \mathrm{d}S(x)
\end{talign}
To make up for this discontinuity, the rate of mass creation per unit time and unit surface at $x \in \partial A$ is
\begin{talign} \label{eq:m_xi_def}
    m_{\xi}(x) = \frac{1}{2Z^{(\xi)}_{AB}} \Big( \frac{\tilde q_{\xi}(x) \| \sigma^{\top} \nabla q_{\xi}(x) \|^2}{\|\nabla q_{\xi}(x)\|} + \frac{q_{\xi}(x) \| \sigma^{\top} \nabla \tilde q_{\xi}(x) \|^2}{\|\nabla \tilde q_{\xi}(x)\|} \Big) \rho(x).
\end{talign}
Analogously, if $n(x)$ is the outward unit normal of $\partial B$, the integral of $J^{(\xi)}_{\partial}$ across $\partial B$ is
\begin{talign} \label{eq:integral_partial_B}
    \int_{\partial B} \langle J^{(\xi)}_{\partial}, \mathrm{d}n \rangle = - \frac{1}{2Z^{(\xi)}_{AB}} \int_{\partial B} \Big( \frac{\tilde q_{\xi}(x) \| \sigma^{\top} \nabla q_{\xi}(x) \|^2}{\|\nabla q_{\xi}(x)\|} + \frac{q_{\xi}(x) \| \sigma^{\top} \nabla \tilde q_{\xi}(x) \|^2}{\|\nabla \tilde q_{\xi}(x)\|} \Big) \rho(x) \, \mathrm{d}S(x).
\end{talign}
And to make up for the discontinuity at $\partial B$, the rate of mass destruction per unit time and unit surface at $x \in \partial B$ is also given by $m_{\xi}(x)$.

For overdamped Langevin, the rate of mass creation simplifies to
\begin{talign} \label{eq:m_xi_def_overdamped}
    m_{\xi}(x) = \frac{1}{2Z^{(\xi)}_{AB}} \frac{\| \sigma^{\top} \nabla q_{\xi}(x) \|^2}{\|\nabla q_{\xi}(x)\|} \rho(x).
\end{talign}
\end{lemma}
\begin{proof}
    The outward normal $n$ of $\partial A$ has the same direction as $\nabla q_{\xi}$ and $-\nabla \tilde q_{\xi}$ for all points in $\partial A$, because $\partial A$ is a level set of $q_{\xi}$ and $\tilde q_{\xi}$. That is,
    \begin{talign} \label{eq:dn_expressions_partial_A}
        \mathrm{d}n(x) = \nabla q_{\xi}(x) / \|\nabla q_{\xi}(x)\| \, \mathrm{d}S(x), \qquad \mathrm{d}n(x) = -\nabla \tilde{q}_{\xi}(x) / \|\nabla \tilde{q}_{\xi}(x)\| \, \mathrm{d}S(x).
    \end{talign}
    Equation \eqref{eq:integral_partial_A} follows from using the expression of $J^{(\xi)}_{\partial}$ in \eqref{eq:j_xi_partial_def} and the expressions of $\mathrm{d}n$ in \eqref{eq:dn_expressions_partial_A}, as well as the definition $D = \frac{1}{2} \sigma \sigma^{\top}$. The rate of mass creation in equation \eqref{eq:m_xi_def} can be read off as the integrand of the integral in the right-hand side of \eqref{eq:integral_partial_A}. Mass must be created at $\partial A$ to make up for the current discontinuity.

    To prove \eqref{eq:integral_partial_B}, we follow the same structure, but in this case use that the outward normal $n$ of $\partial B$ has the same direction as $-\nabla q_{\xi}$ and $\nabla \tilde q_{\xi}$ for all points in $\partial B$, which explains the sign difference. In this case, mass needs to be destroyed at $\partial B$, because the integrand in the right-hand side of \eqref{eq:integral_partial_B} is $- m_{\xi}(x) < 0$.
\end{proof}

\subsection{Sampling from the reactive distribution with birth-death at $\partial A$, $\partial B$}
\label{subsec:different_noise_levels}

Following \Cref{lem:mass_creation_destruction}, to sample the reactive distribution we must simulate the SDE \eqref{eq:optimally_controlled_SDE} while ensuring that mass is created and destroyed at a rate $m_{\xi}$ at the boundaries $\partial A$ and $\partial B$, respectively. Mass creation, or \emph{birth}, can be handled by reinitializing particles with the distribution over $\partial A$ proportional to $m_{\xi}(x)$. For mass destruction, or \emph{death}, we need a criterion to decide when to kill a particle. For that we rely on \emph{boundary local times}, which keep track of the amount of time a particle has spent near a boundary.

\begin{definition}[Boundary local time] \label{def:boundary_local_time}
Let $X$ be the solution of the SDE \eqref{eq:optimally_controlled_SDE}. The boundary local time of $X$ on $\partial B$ is given by
\begin{talign*}
L_t^{\partial B}
:=
\lim_{\varepsilon\downarrow 0}
\frac{1}{2\varepsilon}
\int_0^t \mathbf{1}_{\{|d_{\partial B}(X_s)|<\varepsilon\}}\,ds,
\end{talign*}
whenever the limit exists, where $d_{\partial B}$ is a signed distance function to $\partial B$ (by definition, signed distance functions to $\partial B$ satisfy that $\|\nabla d_{\partial B}(x)\| = 1$ for all $x \in \partial B$).
\end{definition}

\begin{remark}[Practical estimation of boundary local time]
In practice, if a trajectory $(X_{t_n})_{n=0}^N$ is simulated on a time grid $t_n = n\Delta t$, for $n=0,\dots,N$,
a natural estimator of the boundary local time is obtained by discretizing the occupation-time formula:
\begin{talign*}
L_t^{\partial B}
\approx
\widehat L_t^{\partial B,\varepsilon}
:=
\frac{\Delta t}{2\varepsilon}
\sum_{t_n \le t}
\mathbf{1}_{\{|d_{\partial B}(X_{t_n})|<\varepsilon\}},
\end{talign*}
where $\varepsilon>0$ is a small boundary-layer width. Equivalently, one may use a kernel-smoothed version,
\begin{talign*}
\widehat L_t^{\partial B,\varepsilon,K}
:=
\Delta t \sum_{t_n \le t} K_\varepsilon(d_{\partial B}(X_{t_n})),
\qquad
K_\varepsilon(z)=\frac{1}{\varepsilon}K\!\left(\frac{z}{\varepsilon}\right),
\end{talign*}
for a mollifier $K$ with unit mass. In practice, $\varepsilon$ should be chosen small relative to the geometry of $\partial B$, but large enough to be resolved by the time discretization; a common rule of thumb is to take $\varepsilon$ of order $\sqrt{\Delta t}$, matching the typical diffusive displacement over one time step.
\end{remark}

\begin{proposition}[Particle killing criterion] \label{prop:killing_time_tau_B}
Let $X$ be the solution of the SDE \eqref{eq:optimally_controlled_SDE}, $L_t^{\partial B}$ be the boundary local time of $X$ on $\partial B$, defined in \Cref{def:boundary_local_time}, and $m_{\xi}(x)$ be the boundary absorption rate per unit time and unit surface defined in \eqref{eq:m_xi_def}. For $x \in \partial B$, define the boundary killing rate per unit local time as
\begin{talign} \label{eq:lambda_xi_def}
\lambda_{\xi}(x) := \frac{m_{\xi}(x)}{\rho^{(\xi)}_{\mathrm{R}}(x)} = \frac{1}{2} \Big( \frac{\| \sigma^{\top} \nabla q_{\xi}(x) \|^2}{q_{\xi}(x) \|\nabla q_{\xi}(x)\|} + \frac{\| \sigma^{\top} \nabla \tilde q_{\xi}(x) \|^2}{\tilde q_{\xi}(x) \|\nabla \tilde q_{\xi}(x)\|} \Big)
\end{talign}
Then a killing time $\tau_B$ realizing the boundary absorption rate $m_{\xi}$ is given by
\begin{talign} \label{eq:killing_time_tau_B}
\tau_B
:=
\inf\big\{
t\ge 0:
\int_0^t \lambda_{\xi}(X_s)\, \mathrm{d}L_s^{\partial B}
\ge E
\big\},
\qquad
E\sim \mathrm{Exp}(1),
\end{talign}
where $E$ is independent of $X$.
\end{proposition}
\begin{proof}[Proof sketch]
The killing mechanism is implemented by an exponential clock run with respect to boundary local time. Hence the infinitesimal killing probability along a trajectory is $\lambda_{\xi}(X_t)\, dL_t^{\partial B}$.
Since $\rho^{(\xi)}_{\mathrm{R}}(x)$ is the boundary density at $x\in \partial B$, the absorbed mass per unit time and unit surface is given by $\lambda_{\xi}(x)\,\rho^{(\xi)}_{\mathrm{R}}(x)$.
Therefore, prescribing the absorbed mass to be $m_{\xi}(x)$ requires
\begin{talign*}
\lambda_{\xi}(x)\,\rho^{(\xi)}_{\mathrm{R}}(x)=m_{\xi}(x),
\end{talign*}
which yields $\lambda_{\xi}(x)= m_{\xi}(x) / \rho^{(\xi)}_{\mathrm{R}}(x)$. To prove the second equality in \eqref{eq:lambda_xi_def}, we substitute in the expressions for $m_{\xi}(x)$ and $\rho^{(\xi)}_{\mathrm{R}}(x)$ in equations \eqref{eq:m_xi_def} and \eqref{eq:react_dist}, and simplify.
\end{proof}

The following proposition shows that we can introduce a parameter $\kappa \geq 0$ that tunes the noise level of the SDE, while preserving $\rho^{(\xi)}_{\mathrm{R}}$ as the stationary distribution. Note that when $\kappa = 1$, we recover the SDE \eqref{eq:optimally_controlled_SDE}.
\begin{proposition} \label{prop:kappa_app}
For any $\kappa \geq 0$, the reactive density $\rho^{(\xi)}_{\mathrm{R}}$ is a stationary distribution of the SDE
\begin{talign} \label{eq:optimal_kappa_new}
    \mathrm{d}X_t &= \big(
    \tfrac{1+\kappa}{2}\big(b(X_t) + 2D\nabla\log q_{\xi}(X_t)\big) + \tfrac{\kappa-1}{2}\big(\tilde b(X_t) + 2D\nabla\log\tilde q_{\xi}(X_t)\big)
    \big) \, \mathrm{d}t + \sqrt{\kappa} \sigma \, \mathrm{d}W_t,
\end{talign}
provided that the mass creation and destruction rates at $\partial A$ and $\partial B$ are $m_{\xi}$.
\end{proposition}
\begin{proof}
Adding and subtracting $(1-\kappa) \nabla \cdot \big( \nabla \cdot \big( D \rho_t \big) \big)$ on the right-hand side of the Fokker-Planck equation \eqref{eq:FPE_rho_t} yields
\begin{talign}
\begin{split} \label{eq:FPE_rho_t_2}
    \partial_t \rho_t &= - \nabla \cdot \big( (b + 2 D \nabla \log q_{\xi}) \rho_t - (1-\kappa) \nabla \cdot \big( D \rho_t \big) \big) + \kappa \nabla \cdot \big( \nabla \cdot \big( D \rho_t \big) \big) \\
    &= - \nabla \cdot \big( (b + 2 D \nabla \log q_{\xi} - (1-\kappa)(\nabla \cdot D + D \nabla \log \rho_t) ) \rho_t \big) + \kappa \nabla \cdot \big( \nabla \cdot \big( D \rho_t \big) \big)
\end{split}
\end{talign}
Note that since $\rho_{\mathrm{R}}$ is a stationary solution of \eqref{eq:FPE_rho_t} over $\mathbb{R}^d \setminus (A \cup B)$, it is also a stationary solution of \eqref{eq:FPE_rho_t_2}. Plugging $\rho_t = \rho^{(\xi)}_{\mathrm{R}}$, and using that $\nabla \log \rho^{(\xi)}_{\mathrm{R}} = \nabla \log \rho + \nabla \log q_{\xi} + \nabla \log \tilde q_{\xi}$, yields
\begin{talign}
\begin{split} \label{eq:stationary_optimal_2}
    0 &= - \nabla \cdot \Big( \big(b + 2 D \nabla \log q_{\xi} \\ &\qquad\qquad - (1-\kappa)(\nabla \cdot D + D (\nabla \log \rho + \nabla \log q_{\xi} + \nabla \log \tilde q_{\xi})) \big)
    \rho^{(\xi)}_{\mathrm{R}}
    \Big) + \kappa \nabla \cdot \big( \nabla \cdot \big( D \rho^{(\xi)}_{\mathrm{R}} \big) \big) \\ &= - \nabla \cdot \Big( \big(b - (1-\kappa)(\nabla \cdot D + D \nabla \log \rho ) \big) \rho^{(\xi)}_{\mathrm{R}} \Big) \\ &\quad - \nabla \cdot \Big( \big( (1 + \kappa) D \nabla \log q_{\xi} - (1 - \kappa) D \nabla \log \tilde q_{\xi} \big)
    \rho^{(\xi)}_{\mathrm{R}}
    \Big) + \kappa \nabla \cdot \big( \nabla \cdot \big( D \rho^{(\xi)}_{\mathrm{R}} \big) \big)
\end{split}
\end{talign}
Since $\tilde b = - b + 2 D \nabla \log \rho + 2 \nabla \cdot D$, this simplifies to
\begin{talign*}
    0 = - \nabla \cdot \Big( \big(
    \tfrac{1+\kappa}{2}\big(b + 2D\nabla\log q_{\xi}\big) + \tfrac{\kappa-1}{2}\big(\tilde b + 2D\nabla\log\tilde q_{\xi}\big)
    \big) \rho^{(\xi)}_{\mathrm{R}} \Big) + \kappa \nabla \cdot \big( \nabla \cdot \big( D \rho^{(\xi)}_{\mathrm{R}} \big) \big).
\end{talign*}
And we can view this as the stationary Fokker-Planck equation for the SDE \eqref{eq:optimal_kappa_new}, which concludes the proof.
\end{proof}

Assembling \Cref{prop:killing_time_tau_B} and \Cref{prop:kappa_app} lets us construct a process to sample the reactive density, shown in the following result.
\begin{theorem} \label{thm:boundary_birth_death}
Consider the birth-death process $X$ specified by
\begin{itemize}
    \item Birth: $X$ is initialized at $\partial A$ according to the distribution proportional to $m_{\xi}$ (defined in equation \eqref{eq:m_xi_def}).
    \item SDE: $X$ solves the SDE \eqref{eq:optimal_kappa_new}, for any $\kappa \in [0,+\infty)$.
    \item Death: $X$ is killed at the killing time $\tau_B$ defined in equation \eqref{eq:killing_time_tau_B}.
\end{itemize}
Then, the occupancy measure of $X$ is the rescaled reactive density $\rho^{(\xi)}_{\mathrm{R}}$.
\end{theorem}

For $\xi = 0$ and $\kappa = 1$, \Cref{thm:boundary_birth_death} recovers the construction of \Cref{app:doob_derivation}: the SDE is~\eqref{eq:doob_bg}, births follow the reactive flux $\rho\, n_A\cdot D\nabla q$, and since $\lambda_\xi \to \infty$ on~$\partial B$ as $\xi \to 0$ (because $\tilde q_\xi = \xi$ there), particles are killed as soon as they hit~$\partial B$. For $\xi = 0$, the drift $2D\nabla\log q$ is singular on~$\partial A$, so births on~$\partial A$ are understood as the limit of births just outside~$A$ (see \Cref{app:doob_derivation}). For $\xi > 0$, the drift $2D\nabla\log q_\xi$ is bounded and particles can be born directly on~$\partial A$.

\subsection{Sampling from the reactive distribution with birth-death inside $A$, $B$}
\label{sec:mollified_tpt_appendix}

In the preceding subsection, birth and death events occur on the boundaries $\partial A$ and $\partial B$, requiring the estimation of boundary local times (\Cref{def:boundary_local_time}), and the ability to sample from a distribution on the boundary $\partial A$. An alternative approach, which avoids boundary local times entirely, is to extend the committor smoothly into~$A$ and~$B$ and use volumetric killing rates inside these regions. We develop this \emph{mollified} framework below.

\paragraph{Mollified committors.}
\label{app:mollified}
We defined the rescaled committors by imposing boundary conditions at $\partial A$, $\partial B$, and a BKE over $\mathbb{R}^d \setminus (A \cup B)$. To define the mollified committors, we consider PDEs over $\mathbb{R}^d$ which involve smooth functions $\lambda_A,\lambda_B\ge 0$ supported in $A$ and $B$ respectively (with $\lambda_A$ and $\lambda_B$ not both identically zero). Specifically, the \emph{mollified forward committor} satisfies
\begin{talign}\label{eq:BKE_soft}
\mathcal{L}q - (\lambda_A+\lambda_B)q + \lambda_B = 0, \qquad \forall x \in \mathbb{R}^d,
\end{talign}
and the \emph{mollified backward committor} satisfies
\begin{talign}\label{eq:BKE_back_soft}
\tilde{\mathcal{L}} \tilde q - (\lambda_A+\lambda_B)\tilde q + \lambda_A = 0,  \qquad \forall x \in \mathbb{R}^d.
\end{talign}
Observe that over $\mathbb{R}^d \setminus (A \cup B)$, we are still imposing $\mathcal{L}q = 0$, $\tilde{\mathcal{L}} \tilde q = 0$.
Since $\lambda_A + \lambda_B \ge 0$ (not identically zero), the operators $\mathcal{L} - (\lambda_A+\lambda_B)$ and $\tilde{\mathcal{L}} - (\lambda_A+\lambda_B)$ are invertible (the killing breaks the zero eigenvalue of $\mathcal{L}$), so equations~\eqref{eq:BKE_soft}-\eqref{eq:BKE_back_soft} are uniquely solvable. These act as soft killing rates: $\lambda_A$ absorbs into~$A$ and $\lambda_B$ absorbs into~$B$.
By the Feynman-Kac formula, the solution to~\eqref{eq:BKE_soft} is
\begin{talign}\label{eq:FK_q}
q(x) = \mathbb{E}_x\!\left[\int_0^\infty \lambda_B(X_t)\,e^{-\int_0^t(\lambda_A+\lambda_B)(X_s)\,ds}\,dt\right],
\end{talign}
where $X_t$ solves the reference SDE~\eqref{eq:SDE_bg_state_dependent}. This is manifestly non-negative and bounded above by~$1$ (since $1-q$ satisfies the same equation with $\lambda_A$ and $\lambda_B$ swapped). Similarly,
\begin{talign}\label{eq:FK_qR}
\tilde q(x) = \mathbb{E}_x^{\tilde b}\!\left[\int_0^\infty \lambda_A(\tilde X_t)\,e^{-\int_0^t(\lambda_A+\lambda_B)(\tilde X_s)\,ds}\,dt\right],
\end{talign}
where $\tilde X_t$ solves the time-reversed SDE with drift $\tilde b$ defined in~\eqref{eq:reverse_drift}. The interpretation is: $q(x)$ is the probability that a particle started at~$x$, killed at rate $\lambda_A+\lambda_B$, is absorbed by~$\lambda_B$ (into~$B$) rather than by~$\lambda_A$ (into~$A$).

\begin{remark}[Well-posedness]\label{rem:well_posedness}
    The integral in~\eqref{eq:FK_q} converges because, by ergodicity, for large~$t$:
\begin{talign*}
\frac{1}{t}\int_0^t (\lambda_A+\lambda_B)(X_s)\,ds \;\longrightarrow\; \langle\lambda_A+\lambda_B\rangle_\rho = \int(\lambda_A+\lambda_B)\,\rho\,dx > 0,
\end{talign*}
so the exponential factor decays like $e^{-t\langle\lambda_A+\lambda_B\rangle_\rho}$, making the time integral finite. Uniqueness follows because any bounded solution $u$ of the homogeneous equation $[\mathcal{L} - (\lambda_A+\lambda_B)]u = 0$ satisfies $u(x) = \mathbb{E}_x[u(X_t)\,e^{-\int_0^t(\lambda_A+\lambda_B)(X_s)\,ds}]$, which tends to~$0$ as $t\to\infty$, forcing $u\equiv 0$.
\end{remark}

\paragraph{Modified Fokker-Planck equation with source and sink.}
\label{app:mod_FPE_soft}
In analogy with \Cref{subsec:reactive} and \Cref{subsec:different_noise_levels}, the reactive density is defined as
$\rho_{\mathrm{R}} = \frac{1}{Z}\rho\,q\,\tilde q$,
where $q$ and $\tilde q$ now solve the mollified equations~\eqref{eq:BKE_soft}-\eqref{eq:BKE_back_soft}, and $Z = \int_{\mathbb{R}^d} \rho(x)\,q(x)\,\tilde q(x) \, \mathrm{d}x$. The reactive current $J_{\mathrm{R}} =
(b + 2D\nabla\log q)\,\rho_\mathrm{R} - \nabla \cdot (D\rho_\mathrm{R})
$ still satisfies the identity
\begin{talign}\label{eq:JR_flux_new}
J_{\mathrm{R}} =
\frac{1}{Z} \big(q\tilde q\,J + \rho\tilde q\,D\nabla q - \rho q\,D\nabla\tilde q \big)
\end{talign}
since this is the same algebraic computation as in \Cref{lem:flux_reactive} (equation~\eqref{eq:reactive_flux}), and does not depend on the BKEs.

\begin{proposition}[Mollified reactive density PDE]
\label{prop:mod_FPE_soft}
The mollified reactive density satisfies
\begin{talign}\label{eq:mod_FPE_soft_new}
\nabla \cdot \big((b + 2D\nabla\log q)\,\rho_{\mathrm{R}} - \nabla \cdot (D\rho_{\mathrm{R}}) \big) = \frac{1}{Z}\lambda_A\,\rho\, q - \frac{\lambda_B}{q}\,\rho_{\mathrm{R}}.
\end{talign}
Equivalently, since $\rho_\mathrm{R} = \frac{1}{Z}\rho\,q\,\tilde q$:
\begin{talign}\label{eq:mod_FPE_soft_alt_new}
\nabla \cdot J_\mathrm{R} = \frac{1}{Z} \big( \lambda_A\,q - \lambda_B\,\tilde q \big) \rho = \big(\frac{\lambda_A}{\tilde q} - \frac{\lambda_B}{q}\big)\rho_\mathrm{R}.
\end{talign}
Since $\lambda_A = \lambda_B = 0$ in $(A\cup B)^c$, over $(A\cup B)^c$
we have that $\nabla \cdot J_\mathrm{R} = 0$, recovering \Cref{thm:stationary}(i).
\end{proposition}

\begin{proof}
The computation follows the proof of \Cref{thm:stationary}(i). By~\eqref{eq:JR_flux_new} and $\nabla \cdot J = 0$, we have $\nabla \cdot J_\mathrm{R} = \frac{1}{Z} \big( \nabla(q\tilde q)\cdot J + \nabla \cdot (\rho\tilde q\,D\nabla q) - \nabla \cdot (\rho q\,D\nabla\tilde q) \big)$, as in~\eqref{eq:nabla_j_xi}. The two divergence terms are expanded exactly as in~\eqref{eq:second_term_divergence}-\eqref{eq:third_term_divergence}, with the only difference that the BKEs now include killing terms: $D:\nabla^2 q = -b\cdot\nabla q + (\lambda_A+\lambda_B)q - \lambda_B$ and $D:\nabla^2 \tilde q = -\tilde b\cdot\nabla \tilde q + (\lambda_A+\lambda_B)\tilde q - \lambda_A$. This gives
\begin{align*}
\nabla \cdot (\rho\tilde q\,D\nabla q) &= -\tilde q\nabla q\cdot J + \rho\nabla\tilde q\cdot D\nabla q + \rho\tilde q\big[(\lambda_A+\lambda_B)q - \lambda_B\big], \\
\nabla \cdot (\rho q\,D\nabla\tilde q) &= q\nabla\tilde q\cdot J + \rho\nabla q\cdot D\nabla\tilde q + \rho q\big[(\lambda_A+\lambda_B)\tilde q - \lambda_A\big].
\end{align*}
The terms without $\lambda_A,\lambda_B$ cancel exactly as before, leaving only the source terms:
\begin{talign*}
\nabla \cdot J_\mathrm{R} = \frac{1}{Z} \Big( \rho\tilde q\big[(\lambda_A+\lambda_B)q - \lambda_B\big] - \rho q\big[(\lambda_A+\lambda_B)\tilde q - \lambda_A\big] \Big) = \frac{1}{Z} \big( \lambda_A\rho q - \lambda_B\rho\tilde q \big).
\end{talign*}
\end{proof}

Integrating~\eqref{eq:mod_FPE_soft_alt_new} over the domain gives the conservation law
\begin{talign}\label{eq:rate_balance_mollified}
\int \lambda_A\,\rho\, q\,dx = \int \lambda_B\,\rho\, \tilde q\,dx,
\end{talign}
the mollified analogue of the TPT rate formula.

\paragraph{The reactive process.}
\label{app:doob_reactive}
We can identify equations \eqref{eq:mod_FPE_soft_new}-\eqref{eq:mod_FPE_soft_alt_new} as the stationary Fokker-Planck equation for a process $X$ that satisfies the SDE
\begin{talign}\label{eq:SDE_soft}
dX_t = \big(b(X_t) + 2D\nabla\log q(X_t)\big)\,dt + \sigma\,dW_t,
\end{talign}
with the following birth-death mechanism:
\begin{itemize}
\item \textbf{Death:} a particle at position~$x$ is killed at rate $\lambda_B(x)/q(x)$.
\item \textbf{Birth:} upon killing, the particle respawns at a new position $y$ drawn from
\begin{talign}
p_{\mathrm{birth}}(y) = \frac{\lambda_A(y)\,\rho(y)\,q(y)}{Z_{\mathrm{birth}}}, \qquad Z_{\mathrm{birth}} = \int \lambda_A\,\rho\,q\,dx.
\end{talign}
\end{itemize}
The birth distribution $p_{\mathrm{birth}} \propto \lambda_A\,\rho\,q$ is supported in~$A$ (where $\lambda_A > 0$), weighted by the stationary density~$\rho$ and the forward committor~$q$. In the hard-boundary limit ($\lambda_A\to\infty$ on $A$, $\lambda_B\to\infty$ on~$B$), the birth and death concentrate on $\partial A$ and $\partial B$, recovering the boundary birth-death process of \Cref{subsec:different_noise_levels}.

\begin{remark}[Relation between the killed and reactive processes]\label{rem:killed_vs_reactive}
The mollified framework involves two distinct processes:
\begin{itemize}
\item The \emph{killed process}: the original SDE~\eqref{eq:SDE_bg_state_dependent} with killing at rate $\lambda_A+\lambda_B$. Particles are absorbed and never replaced; this process has no stationary density and is used only to define the committors via~\eqref{eq:FK_q}-\eqref{eq:FK_qR}.
\item The \emph{reactive process}: the $h$-transformed SDE~\eqref{eq:SDE_soft} with drift $b + 2D\nabla\log q$, killing at rate $\lambda_B/q$, and respawning from $p_{\mathrm{birth}} \propto \lambda_A\rho q$. The birth term recycles particles absorbed near~$B$ back to~$A$, maintaining the stationary density~$\rho_\mathrm{R}$.
\end{itemize}
\end{remark}

\paragraph{Connection to the rescaled
committor.}
\label{app:xi_connection}
We now show that the rescaled committor of~\eqref{eq:kbe_xi} (with $q|_A = \xi$, $q|_B = 1-\xi$) can always be extended to a global solution of the mollified equation~\eqref{eq:BKE_soft}.

\begin{proposition}\label{prop:xi_extension}
For any $\xi \in (0,\tfrac{1}{2})$, let $q$ solve~\eqref{eq:kbe_xi}. Then $q$ can be extended to a smooth function on all of~$\mathbb{R}^d$ satisfying the mollified equation $\mathcal{L}q - (\lambda_A+\lambda_B)q + \lambda_B = 0$ for some non-negative $\lambda_A$ supported in~$A$ and $\lambda_B$ supported in~$B$.
\end{proposition}

\begin{proof}[Proof sketch]
In $(A\cup B)^c$, the equation $\mathcal{L}q = 0$ holds and $\lambda_A = \lambda_B = 0$, so the mollified equation is automatically satisfied. It remains to construct smooth extensions inside~$A$ and~$B$ with $0 < q < 1$ such that the implied killing rates
\[
\lambda_A = \frac{\mathcal{L}q}{q} \quad \text{inside } A, \qquad \lambda_B = \frac{\mathcal{L}q}{q-1} \quad \text{inside } B
\]
are non-negative. Since $q > 0$ in~$A$, the condition $\lambda_A \ge 0$ requires $\mathcal{L}q \ge 0$ ($q$ is $\mathcal{L}$-subharmonic in~$A$). Since $q < 1$ in~$B$, the condition $\lambda_B \ge 0$ requires $\mathcal{L}q \le 0$ ($q$ is $\mathcal{L}$-superharmonic in~$B$). Such extensions exist: one can smoothly continue~$q$ so that it decreases from~$\xi$ on~$\partial A$ toward a positive minimum inside~$A$ (subharmonic), and increases from~$1-\xi$ on~$\partial B$ toward a maximum below~$1$ inside~$B$ (superharmonic). The matching of~$q$ and its normal derivative across~$\partial A$ and~$\partial B$ is ensured by construction.
\end{proof}

The extension inside~$A$ and~$B$ is not unique (different choices yield different killing rates~$\lambda_A$ and~$\lambda_B$), but this non-uniqueness is harmless: the reactive density~$\rho_\mathrm{R}$ and the reactive drift in $(A\cup B)^c$ depend only on~$q$ there, which is determined by~\eqref{eq:kbe_xi} independently of the extension. The hard-boundary limit $\xi \to 0$ recovers the classical formulation: the killing rates~$\lambda_A,\lambda_B$ concentrate into hard absorption on~$\partial A$ and~$\partial B$.

\paragraph{Constant killing rates.}
\label{app:constant_killing}
A natural choice is to take $\lambda_A$ and $\lambda_B$ constant inside~$A$ and~$B$:
\[
\lambda_A = \frac{\delta}{T_A}, \qquad \lambda_B = \frac{\delta}{T_B},
\]
where $T_A = \max_{x\in A}\mathbb{E}_x[\tau_{\partial A}]$ and $T_B = \max_{x\in B}\mathbb{E}_x[\tau_{\partial B}]$ are the maximal mean exit times from~$A$ and~$B$ for the original process~\eqref{eq:SDE_bg_state_dependent}, and $\delta > 0$ is a free parameter. The Feynman-Kac representation inside~$A$ gives $q(x) = \xi\,\mathbb{E}_x[e^{-\lambda_A\tau_{\partial A}}]$. By Jensen's inequality,
\[
\xi\,e^{-\delta} \le q(x) \le \xi \qquad \text{inside } A,
\]
and similarly $1-\xi \le q(x) \le 1-\xi\,e^{-\delta}$ inside~$B$, so $q$ remains bounded away from~$0$ and~$1$. The birth distribution $p_{\mathrm{birth}}(y) \propto \rho(y)\,\mathbb{E}_y[e^{-\lambda_A\tau_{\partial A}}]$ inside~$A$ interpolates between:
\begin{itemize}
\item $\delta \to 0$: $p_{\mathrm{birth}} \to \rho|_A / \int_A \rho\,dx$ (respawn uniformly w.r.t.~$\rho$);
\item $\delta \gg 1$: $p_{\mathrm{birth}}$ concentrates near~$\partial A$, recovering the boundary picture of \Cref{subsec:different_noise_levels}.
\end{itemize}

\begin{remark}[Approximation from constant killing rates]\label{rem:approx}
Constant killing rates generically introduce a mismatch in~$\nabla q$ at~$\partial A$ and~$\partial B$ (unless $\lambda_A,\lambda_B$ are chosen by derivative matching, as in the proof of \Cref{prop:xi_extension}). This mismatch is controlled: $\rho_\mathrm{R}$ in $(A\cup B)^c$ depends only on the exterior~$q$ (independently of the extension), so the error only affects the birth-death mechanism near~$\partial A$ and~$\partial B$. A respawned particle exits~$A$ in time~$O(T_A)$ regardless of where it starts, so the effect on sampled reactive trajectories in the exterior is small.
\end{remark}

\begin{remark}[Role of $\delta$]\label{rem:delta}
The parameter~$\delta$ does not affect $\rho_\mathrm{R}$ in the exterior (which is determined by~$q$ on $\mathcal{S}^c$), but it controls the throughput of the birth-death process: the killing and birth rates scale as~$\delta/T_B$ and~$\delta/T_A$ respectively. Taking~$\delta$ large concentrates the birth distribution near~$\partial A$, while $\delta\to 0$ gives the simple birth distribution $\rho|_A/\!\int_A\rho\,dx$.
\end{remark}

\paragraph{Example: pure diffusion in 1D.}
\label{app:example_1D}
Consider $b = 0$, $D = 1$ ($dX_t = \sqrt{2}\,dW_t$) on $[-1,2]$ with reflecting boundaries, $A = [-1,0]$, $B = [1,2]$. In the exterior $(0,1)$, $q'' = 0$ with $q(0) = \xi$, $q(1) = 1-\xi$ gives $q(x) = \xi + (1-2\xi)x$. Inside~$A$, solving $q'' = \lambda_A q$ with $q(0) = \xi$, $q'(-1) = 0$ yields
\[
q(x) = \xi\,\frac{\cosh[\sqrt{\lambda_A}\,(x+1)]}{\cosh(\sqrt{\lambda_A})}, \qquad x \in [-1,0],
\]
with $\lambda_A = 2\delta$ (since $T_A = 1/2$). The birth distribution $p_{\mathrm{birth}}(x) \propto \cosh[\sqrt{\lambda_A}\,(x+1)]$ interpolates between the uniform distribution on~$A$ ($\delta \to 0$) and a point mass at $\partial A = \{0\}$ ($\delta \to \infty$). Matching the derivative $q'(0^-) = q'(0^+) = 1-2\xi$ determines $\lambda_A$ via
$\xi\sqrt{\lambda_A}\tanh(\sqrt{\lambda_A}) = 1-2\xi$; for small~$\xi$, $\lambda_A \approx 1/\xi^2$ and the penetration depth is $\ell \approx \xi$. By symmetry, the same analysis applies to~$B = [1,2]$ with $1-q$ replacing~$q$ and $\lambda_B = 2\delta$.

\section{Proofs}
\label{sec:first_order_optimality}

Throughout this section, $V$ denotes the value function of the general SOC problem (which may include a running cost~$f$); in the main text, $\Phi$ denotes the specialization to $f = 0$. The coefficients $b(x,t)$ and $\sigma(x,t)$ are allowed to depend on time for generality; the main text considers the autonomous (time-independent) case.

\subsection{Preliminaries}

\subsubsection{Admissible controls}
\begin{definition}[Admissible Control] \label{def:admissible_controls}
Fix an initial time \(t\in[0,T]\). A process $u = (u_{s})_{s\in[t,T]}$ is called an \emph{admissible control} (denoted \(u\in\mathcal{U}\)) if it satisfies:
\begin{enumerate}[label=(\roman*),left=2pt]
    \item \textbf{Values in the action set:}
    \begin{talign*}
      u_{s}(\omega)\in \mathsf{U},\quad \text{for all }s\in[t,T],\ \omega\in\Omega,
    \end{talign*}
    where \(\mathsf{U}\subset\mathbb{R}^{m}\) is the given control set.

    \item \textbf{Adaptedness / non-anticipativity:}
    \(u\) is progressively measurable with respect to the filtration \(\{\mathcal{F}_{s}\}_{s\ge t}\) generated by the Brownian motion \(W\).  Equivalently,
    \begin{talign*}
      u_{s} \text{ is } \mathcal{F}_{s}\text{-measurable for each }s\,.
    \end{talign*}

    \item \textbf{Integrability / well-posedness:}  There exists \(p>1\) (corresponding to the coercivity exponent of the running cost) such that
    \begin{talign} \label{eq:alpha_bound}
      \mathbb{E}\!\Bigl[\int_{t}^{T} \bigl|u_{s}\bigr|^{p}\,\mathrm{d}s\Bigr] < \infty.
    \end{talign}
    If we assume that $b : [0,T]\times\mathbb{R}^d\times \mathsf{U} \to \mathbb{R}^d$, and $\sigma : [0,T]\times \mathbb{R}^d\times \mathsf{U} \to \mathbb{R}^{d\times k}$ satisfy:
    \begin{enumerate}[label=(\roman*)]
      \item Measurability: $(t,a)\mapsto b(x,t,a), \sigma(x,t,a)$ are Borel measurable, $x\mapsto b(x,t,a),\ \sigma(x,t,a)$ are continuous.
      \item Local Lipschitz in \(x\):  For each $R > 0$, there exists \(L_R>0\) such that $|b(x,t,a)-b(y,t,a)|+\|\sigma(x,t,a)-\sigma(y,t,a)\|\le L_R \| x-y\|$, for all \(t\in[0,T]\), \(a\in \mathsf{U}\), and \(x,y\in\mathbb{R}^d\) such that $\|x\|,\|y\| < R$.
      \item Polynomial growth:  There exists \(K>0\) and $r > 1$ such that $|b(x,t,a)| + \|\sigma(x,t,a)\| \le K\,\big(1 + \lvert x\rvert^r + \lvert a\rvert^r\big)$, for all \((x,t,a)\).
    \end{enumerate}
    then equation \eqref{eq:alpha_bound} ensures that the stochastic differential equation
    \begin{talign*}
      \mathrm{d}X_{s}^{x,t,u}
      = b\bigl(X_{s}^{x,t,u},s,u_{s}\bigr)\,\mathrm{d}s
      \;+\;\sigma\bigl(X_{s}^{x,t,u},s,u_{s}\bigr)\,\mathrm{d}W_{s},
      \quad X_{t}^{x,t,u}=x,
    \end{talign*}
    admits a (strong or weak) solution on \([t,T]\).
\end{enumerate}
\end{definition}

\subsubsection{Girsanov theorem and KL divergence between path measures} \label{subsec:girsanov}
\begin{theorem}[Girsanov's Theorem for SDEs] \label{thm:girsanov}
Let \((\Omega,\mathcal F,(\mathcal F_t)_{t\ge0},\mathbb{Q})\) be a filtered probability space satisfying the usual conditions, and let $W=(W_t)_{t\ge0}$
be an $\mathbb{R}^m$-valued $\mathcal{F}_t$-Brownian motion under $\mathbb{Q}$.  Suppose $\theta=(\theta_t)_{t\ge0}$
is an $\mathbb{R}^m$-valued, progressively measurable process
and let \(\tau\) be an \((\mathcal F_t)\)-stopping time with \(\tau<\infty\) a.s.  Suppose the ``random-horizon'' Novikov condition holds:
\begin{talign*}
\mathbb{E}_{\mathbb{Q}}\Bigl[\exp\bigl(\tfrac12\!\int_0^{\tau}\|\theta_s\|^2\,\mathrm{d}s\bigr)\Bigr] < \infty.
\end{talign*}
Let $b\colon \mathbb{R}^d \times [0,T] \to\mathbb{R}^d$, $\sigma\colon \mathbb{R}^d \times [0,T] \to\mathbb{R}^{d\times m}$
be measurable functions satisfying the usual conditions \footnote{Measurability, local Lipschitzness, polynomial growth.} guaranteeing a unique strong solution $X$ of the SDE
\begin{talign*}
\begin{cases}
\mathrm{d}X_t = b(X_t,t)\,\mathrm{d}t \;+\;\sigma(X_t,t)\,\mathrm{d}W_t,\\
X_0 = x\in\mathbb{R}^d.
\end{cases}
\end{talign*}
Define the stopped exponential
\begin{talign} \label{eq:Z_t}
Z_t
\;=\;
\exp \big(
\int_0^{t\wedge\tau}\theta_s^\top\, \mathrm{d}W_s
      -\tfrac12\int_0^{t\wedge\tau}\|\theta_s\|^2\,\mathrm{d}s \big),
\quad t\ge0.
\end{talign}
Then:
\begin{enumerate}[label=(\roman*),left=2pt]
  \item \((Z_t)_{t\ge0}\) is a uniformly integrable \(\mathbb{Q}\)-martingale and in particular
  \(\mathbb{E}_{\mathbb{Q}}[Z_\tau]=1\).
  \item One may define a new measure $\mathbb{Q}^{\theta}$
  on \(\mathcal F_\tau\) by
  \begin{talign} \label{eq:dP_theta_dP}
    \frac{\mathrm{d}\mathbb{Q}^{\theta}}{\mathrm{d}\mathbb{Q}}(W) := \frac{\mathrm{d}\mathbb{Q}^{\theta}}{\mathrm{d}\mathbb{Q}}\Bigm\vert_{\mathcal F_\tau}
    = Z_\tau,
  \end{talign}
  and \(\mathbb{Q}^{\theta}\sim \mathbb{Q}\) on \(\mathcal F_\tau\).
  \item Under \(\mathbb{Q}^{\theta}\), the stopped process
  \begin{talign} \label{eq:W_theta}
    W^{\theta}_{t\wedge\tau}
    \;:=\;
    W_{t\wedge\tau}
    \;
    -
    \;\int_0^{t\wedge\tau}\theta_s\,\mathrm{d}s
  \end{talign}
  is an \(m\)-dimensional \(\mathcal F_t\)-Brownian motion (up to time \(\tau\)).
  \item
  Let $T > 0$. Given $W$, let $X^{\theta}$ and $X$ be the stopped processes that satisfy the following SDE up to the stopping time $\tau$:
  \begin{talign} \label{eq:X_theta_girsanov}
    \mathrm{d}X^{\theta}_t
    &= \big( b(X^{\theta}_t,t) + \sigma(X^{\theta}_t,t)\,\theta_t \big)\,\mathrm{d}t + \sigma(X^{\theta}_t,t)\,\mathrm{d}W_t,
    \quad X_0=x, \\
    \mathrm{d}X_t
    &= b(X_t,t)\,\mathrm{d}t + \sigma(X_t,t)\,\mathrm{d}W_t,
    \quad X_0=x. \label{eq:X_girsanov}
  \end{talign}
  Then for any functional $F$, we can write
  \begin{talign}
  \begin{split} \label{eq:exp_girsanov}
      \mathbb{E}_{W \sim \mathbb{Q}}[F(X^{\theta})] &=
      \mathbb{E}_{W \sim \mathbb{Q}}[F(X) \frac{\mathrm{d}\mathbb{Q}^{\theta}}{\mathrm{d}\mathbb{Q}}
      (W)],
  \end{split}
  \end{talign}
  which means that
  \begin{talign} \label{eq:dP_dQ}
  \frac{\mathrm{d}\mathbb{P}^{\theta}}{\mathrm{d}\mathbb{P}}(X) = \frac{d\mathbb{Q}^{\theta}}{d\mathbb{Q}}(W).
  \end{talign}
  \item In particular, if we let $b \equiv 0$ and $\sigma \equiv 1$, we have that $X^{\theta}_{t\wedge \tau} = W_{t\wedge \tau} + \int_0^{t\wedge \tau} \theta_s \, \mathrm{d}s$, which following the notation in equation \eqref{eq:W_theta} means that $X^{\theta} = W^{-\theta}$. Hence, by equation \eqref{eq:dP_dQ}, the law of $W^{-\theta}$ is $\mathbb{Q}^{\theta}$, and similarly, the law of $W^{\theta}$ is $\mathbb{Q}^{-\theta}$.
\end{enumerate}
\end{theorem}
\begin{proof}
    Point (i) is sometimes known as Novikov's theorem. Point (ii) is a direct consequence of point (i). Point (iii) is properly Girsanov's theorem. Next, we show that point (iv) follows from point (iii).
    In particular, in \eqref{eq:exp_girsanov} the equality $\mathbb{E}_{W \sim \mathbb{Q}}[F(X^{\theta})] = \mathbb{E}_{W \sim \mathbb{Q}}[F(X) \frac{\mathrm{d}\mathbb{Q}^{\theta}}{\mathrm{d}\mathbb{Q}}
    (W)]$ holds because when we replace $W$ by $W^{\theta}$ in the SDE \eqref{eq:X_theta_girsanov}, we obtain that
    \begin{talign*}
    \begin{split}
        \mathrm{d}X^{\theta}_t
        &= \big( b(X^{\theta}_t,t) + \sigma(X^{\theta}_t,t)\,\theta_t \big)\,\mathrm{d}t + \sigma(X^{\theta}_t,t)\,\mathrm{d}W^{\theta}_t \\ &= \big( b(X^{\theta}_t,t) + \sigma(X^{\theta}_t,t)\,\theta_t - \sigma(X^{\theta}_t,t)\,\theta_t\big)\,\mathrm{d}t + \sigma(X^{\theta}_t,t)\,\mathrm{d}W_t \\ &= b(X^{\theta}_t,t) \,\mathrm{d}t + \sigma(X^{\theta}_t,t)\,\mathrm{d}W_t,
    \end{split}
    \end{talign*}
    which is equal to the SDE \eqref{eq:X_girsanov} for $X$. The equality then holds because taking an expectation with respect to $W$ is equivalent to taking an expectation with respect to $W^{\theta}$ with the importance weight $\frac{d\mathbb{Q}^{\theta}}{d\mathbb{Q}}
    (W)$.
\end{proof}
\begin{corollary}[KL divergence between path measures] \label{cor:KL_path}
    Consider the processes $X^{\theta}$ and $X$ defined in \eqref{eq:X_theta_girsanov} and \eqref{eq:X_girsanov}, and the process $W^{-\theta}$ defined in \eqref{eq:W_theta} (up to a a sign flip on $\theta$).
    Let $\mathbb{P}$ and $\mathbb{P}^{\theta}$ be the path measures of $X$ and $X^{\theta}$, respectively, and $\mathbb{Q}$ and
    $\mathbb{Q}^{\theta}$
    the path measures of the stopped Brownian motions $W$ and
    $W^{-\theta}$ (by \cref{thm:girsanov}(v)).
    We have that
    \begin{talign*}
    \begin{split}
        \mathrm{KL}(\mathbb{P}^{\theta}\,\|\,\mathbb{P}) = \mathbb{E}_{W \sim \mathbb{Q}}\big[\tfrac12\!\int_0^{\tau}\|\theta_s\|^2\,\mathrm{d}s \big].
    \end{split}
    \end{talign*}
\end{corollary}
\begin{proof}
    By the definition of $W^{-\theta}$, observe that $X^{\theta}$ satisfies that $\mathrm{d}X^{\theta}_t = b(t,X^{\theta}_t) \, \mathrm{d}t + \sigma(t,X^{\theta}_t) \, \mathrm{d}W^{-\theta}_t$.
    The map $W \mapsto X$
    which also maps
    $W^{-\theta} \mapsto X^{\theta}$,
    is a bijection. The pushforward of $\mathbb{Q}$ by this map is
    $\mathbb{P}$,
    and the pushforward of
    $\mathbb{Q}^{\theta}$ by it
    is
    $\mathbb{P}^{\theta}$.
    Hence,
    we have that
    \begin{talign*}
    \begin{split}
        \log\frac{\mathrm{d}\mathbb{P}^{\theta}}{\mathrm{d}\mathbb{P}}(X^{\theta}) &= \log\frac{\mathrm{d}\mathbb{Q}^{\theta}}{\mathrm{d}\mathbb{Q}}(W^{-\theta}) = \int_0^{\tau}\theta_s^\top\, \mathrm{d}W^{-\theta}_s
        -\tfrac12\int_0^{\tau}\|\theta_s\|^2\,\mathrm{d}s \\ &= \int_0^{\tau}\theta_s^\top\, (\mathrm{d}W_s + \theta_s \, \mathrm{d}s)
        -\tfrac12\int_0^{\tau}\|\theta_s\|^2\,\mathrm{d}s = \int_0^{\tau}\theta_s^\top\, \mathrm{d}W_s
        + \tfrac12\int_0^{\tau}\|\theta_s\|^2\,\mathrm{d}s,
    \end{split}
    \end{talign*}
    where the second equality holds by equations \eqref{eq:Z_t} and \eqref{eq:dP_theta_dP}.
    Hence,
    \begin{talign*}
    \begin{split}
        \mathrm{KL}(\mathbb{P}^{\theta}\,\|\,\mathbb{P}) &=\mathbb{E}_{W \sim \mathbb{Q}} \big[\log\frac{\mathrm{d}\mathbb{P}^{\theta}}{\mathrm{d}\mathbb{P}}(X^{\theta}) \big]
        = \mathbb{E}_{W \sim \mathbb{Q}}\big[ \int_0^{\tau}\theta_s^\top\, \mathrm{d}W_s+\tfrac12\int_0^{\tau}\|\theta_s\|^2\,\mathrm{d}s \big]
        \\ &= \mathbb{E}_{W \sim \mathbb{Q}}\big[\tfrac12\!\int_0^{\tau}\|\theta_s\|^2\,\mathrm{d}s \big].
    \end{split}
    \end{talign*}
\end{proof}
\begin{corollary} \label{cor:Girsanov_dX_t}
  Consider the processes $X$, $X^{\theta}$, $W$, $W^{\theta}$ as defined in \Cref{thm:girsanov}. We can alternatively write
  \begin{talign}
  \begin{split} \label{eq:dQ_theta_dQ}
      \frac{\mathrm{d}\mathbb{P}^{\theta}}{\mathrm{d}\mathbb{P}}(X) =
      \frac{\mathrm{d}\mathbb{Q}^{\theta}}{\mathrm{d}\mathbb{Q}}(W) &= \exp \big(\!\int_0^{\tau}\theta_s^\top\,\big(\sigma(X_s,s)^{-1} \big( \mathrm{d}X_s - b(X_s,s) \, \mathrm{d}s \big) - \tfrac12 \theta_s \, \mathrm{d}s \big) \big) \\ &=
      \exp
      \big(\!\int_0^{\tau}\theta_s^\top\,\big(\sigma(X^{\theta}_s,s)^{-1} \big( \mathrm{d}X^{\theta}_s - b(X^{\theta}_s,s) \, \mathrm{d}s \big) + \tfrac12 \theta_s \, \mathrm{d}s \big) \big).
  \end{split}
  \end{talign}
\end{corollary}
\begin{proof}
    Using that
    \begin{talign*}
    \begin{split}
    \mathrm{d}W_t &= \sigma(X_t,t)^{-1} \big( \mathrm{d}X_t - b(X_t,t) \, \mathrm{d}t \big),
    \end{split}
    \end{talign*}
    and
    \begin{talign*}
    \begin{split}
    \mathrm{d}W_t &= \sigma(X^{\theta}_t,t)^{-1} \big( \mathrm{d}X^{\theta}_t - ( b(X^{\theta}_t,t) + \sigma(X^{\theta}_t,t) \theta_t) \, \mathrm{d}t \big) \\ &= \sigma(X^{\theta}_t,t)^{-1} \big( \mathrm{d}X^{\theta}_t - b(X^{\theta}_t,t) \, \mathrm{d}t \big) + \theta_t \, \mathrm{d}t,
    \end{split}
    \end{talign*}
    we rewrite equation \eqref{eq:Z_t} as
    \begin{talign*}
    \begin{split}
        \frac{\mathrm{d}\mathbb{Q}^{\theta}}{\mathrm{d}\mathbb{Q}}(W) &= \exp \big(\!\int_0^{\tau}\theta_s^\top\,\big(\mathrm{d}W_s - \tfrac12 \theta_s \,\mathrm{d}s \big) \big) \\ &= \exp \big(\!\int_0^{\tau}\theta_s^\top\,\big(\sigma(X_s,s)^{-1} \big( \mathrm{d}X_s - b(X_s,s) \, \mathrm{d}s \big) - \tfrac12 \theta_s \, \mathrm{d}s \big) \big),
    \end{split}
    \end{talign*}
    and alternatively,
    \begin{talign*}
    \begin{split}
        \frac{\mathrm{d}\mathbb{Q}^{\theta}}{\mathrm{d}\mathbb{Q}}(W) &= \exp \big(\!\int_0^{\tau}\theta_s^\top\,\big(\mathrm{d}W_s - \tfrac12 \theta_s \,\mathrm{d}s \big) \big) \\ &= \exp \big(\!\int_0^{\tau}\theta_s^\top\,\big(\sigma(X^{\theta}_s,s)^{-1} \big( \mathrm{d}X^{\theta}_s - b(X^{\theta}_s,s) \, \mathrm{d}s \big) + \tfrac12 \theta_s \, \mathrm{d}s \big) \big),
    \end{split}
    \end{talign*}
    which proves
    the second and third equalities of
    equation \eqref{eq:dQ_theta_dQ}. To prove the first equality of \eqref{eq:dQ_theta_dQ}, we use equation \eqref{eq:dP_dQ} from \Cref{thm:girsanov}.
\end{proof}

\subsubsection{Two characterizations of the value function} \label{subsubsec:char_value}
\begin{theorem}[Characterizations of the value function] \label{thm:value_functions}
Let $\mathcal{U}$ be the set of admissible controls as defined in Definition \ref{def:admissible_controls}, and let $\tau$ be an arbitrary stopping time which is almost surely finite. Consider the value function
\begin{talign}
\label{eq:control_problem_def}
    V(x,t) &= \inf_{u \in \mathcal{U}} \{ \mathbb{E}[\int_t^{\tau} \big( \frac{1}{2} \|u_s\|^2 + f(X^{u}_{s},s) \big) \, \mathrm{d}s + g(X^{u}_{\tau}) | X^{u}_t = x ] \}, \\
    \text{where } \mathrm{d}X^{u}_t &= (b(X^{u}_t,t) + \sigma(X^{u}_t,t) u_t) \, \mathrm{d}t + \sigma(X^{u}_t,t) \mathrm{d}W_t.
\label{eq:controlled_SDE}
\end{talign}
Suppose that $b\colon\mathbb{R}^d\to\mathbb{R}^d$, $\sigma\colon\mathbb{R}^d\to\mathbb{R}^{d\times d}$
are functions satisfying the usual conditions guaranteeing a unique strong solution $X$ of the SDE (measurability, local Lipschitzness, polynomial growth), and that $f$ and $g$ are bounded below and have at most polynomial growth.
Then, we can also write
\begin{talign*}
    V(x,t) = - \log \mathbb{E}\big[ \exp \big( - \int_t^{\tau} f(X_{s},s) \, \mathrm{d}s - g(X_{\tau}) \big) \, | \, X_t = x \big].
\end{talign*}
Moreover, the path probability measure $\mathbb{P}^{u^{\star}}$ of the process $X^{u^{\star}}$ induced by the optimal control $u^{\star}$ satisfies
\begin{talign} \label{eq:path_integral_characterization}
    \frac{\mathrm{d}\mathbb{P}^{u^{\star}}}{\mathrm{d}\mathbb{P}}(X) = \exp\big( - \int_0^{\tau} f(X_{s},s) \, \mathrm{d}s - g(X_{\tau}) + V(X_0,0)\big).
\end{talign}
\end{theorem}
\begin{proof}
Let $X$ and $X^{u}$ solve the following SDEs up to $\tau$:
\begin{talign*}
\begin{split}
    \mathrm{d}X_t &= b(X_t,t)\,\mathrm{d}t + \sigma(X_t,t)\,\mathrm{d}W_t, \quad X_0=x, \\
    \mathrm{d}X^{u}_t &= \big( b(X^{u}_t,t) + \sigma(X^{u}_t,t)\,u_t \big)\,\mathrm{d}t + \sigma(X^{u}_t,t)\,\mathrm{d}W_t, \quad X_0=x.
\end{split}
\end{talign*}
Let $\mathbb{P}_{(x,t)}$ and $\mathbb{P}^{u}_{(x,t)}$ be the path measures of the processes $X$ and $X^{u}$ with starting point $(t,x)$, respectively. By Corollary \ref{cor:KL_path}, we obtain that:
\begin{talign}
\begin{split} \label{eq:V_x_t}
    V(x,t) &= \inf_{u \in \mathcal{U}} \{ \mathbb{E}[\int_t^{\tau} \big( \frac{1}{2} \|u_s\|^2 + f(X^{u}_{s},s) \big) \, \mathrm{d}s + g(X^{u}_{\tau}) | X^{u}_t = x ] \} \\ &= \inf_{u \in \mathcal{U}} \{ \mathrm{KL}(\mathbb{P}^{u}_{(x,t)}\,\|\,\mathbb{P}_{(x,t)}) + \mathbb{E}[\int_t^{\tau} f(X^{u}_{s},s) \, \mathrm{d}s + g(X^{u}_{\tau}) | X^{u}_t = x ] \} \\ &= \inf_{u \in \mathcal{U}} \{ \mathbb{E}_{\mathbb{P}^{u}_{(x,t)}}[ \log\frac{\mathrm{d}\mathbb{P}^{u}_{(x,t)}}{\mathrm{d}\mathbb{P}_{(x,t)}}(X^{u}) +\int_t^{\tau} f(X^{u}_{s},s) \, \mathrm{d}s + g(X^{u}_{\tau}) | X^{u}_t = x ] \}
    \\ &= \inf_{u \in \mathcal{U}} \{ \mathbb{E}_{\mathbb{P}_{(x,t)}}[ \big( \log\frac{\mathrm{d}\mathbb{P}^{u}_{(x,t)}}{\mathrm{d}\mathbb{P}_{(x,t)}}(X) +\int_t^{\tau} f(X_{s},s) \, \mathrm{d}s + g(X_{\tau}) \big) \frac{\mathrm{d}\mathbb{P}^{u}_{(x,t)}}{\mathrm{d}\mathbb{P}_{(x,t)}}(X) | X_t = x ] \}
\end{split}
\end{talign}
Observe that this is a variational problem over the restricted space of path measures induced by a control $u$. Namely,
for any admissible control $u\in\mathcal{U}$, let
$L^u(X) \;=\;\frac{\mathrm{d}\mathbb{P}^u_{(x,t)}}{\mathrm{d}\mathbb{P}_{(x,t)}}(X)$, which means that $\mathbb{E}_{\mathbb{P}_{(x,t)}}\bigl[L^u\bigr]=1$.
Set $\mathcal{W}(X)\;:=\;\int_t^{\tau}f\bigl(X_s,s\bigr)\,\mathrm{d}s \;+\; g\bigl(X_\tau\bigr)$,
so that the control objective can be rewritten as
\begin{talign*}
J(u)
=\mathbb{E}_{X \sim \mathbb{P}_{(x,t)}}\!\bigl[L^u(X) \bigl(\log L^u(X) + \mathcal{W}(X)\bigr)\bigr].
\end{talign*}
Consider the variational problem
\begin{talign} \label{eq:mathcal_J}
&\inf_{L} \mathcal{J}(L)
\;=\;
\mathbb{E}_{\mathbb{P}_{(x,t)}}\bigl[L\log L\bigr]
\;+\;
\mathbb{E}_{\mathbb{P}_{(x,t)}}\bigl[L\,\mathcal{W}\bigr], \\
&\text{subject to } L\ge0,\;\mathbb{E}_{\mathbb{P}_{(x,t)}}[L]=1.
\end{talign}
The optimal value of this problem is smaller or equal than $V(x,t)$, since the objective functional matches the one in the right-hand side of \eqref{eq:V_x_t} and the optimization space is larger. The following lemma characterizes the optimal solution of problem \eqref{eq:mathcal_J}, and shows that the optimal values are actually equal.

\begin{lemma}[Existence, Uniqueness, and Characterization] \label{lem:variational_path}
The minimization problem \eqref{eq:mathcal_J}
admits a unique solution $L^{\star}$, and the corresponding control $u^{\star}$ satisfying
\(\mathrm{d}\mathbb{P}_{(x,t)}^{u^{\star}}/\mathrm{d}\mathbb{P}_{(x,t)}=L^{\star}\)
is the unique minimizer of the original variational problem.  Moreover
\[
L^{\star}(X)
\;=\;
\frac{\exp \bigl(-\mathcal{W}(X)\bigr)}{\mathbb{E}_{\mathbb{P}_{(x,t)}}\big[e^{-\mathcal{W}}\big]}.
\]
\end{lemma}
When we substitute the $L^{\star}$ given by Lemma \ref{lem:variational_path} into \eqref{eq:mathcal_J}, we obtain that
\begin{talign*}
\begin{split}
    V(x,t)\;&=\;\mathbb{E}_{\mathbb{P}_{(x,t)}}\bigl[L^{\star}\log L^{\star}\bigr]
\;+\;
\mathbb{E}_{\mathbb{P}_{(x,t)}}\bigl[L^{\star}\,\mathcal{W}\bigr] \\ &= \frac{\mathbb{E}_{\mathbb{P}_{(x,t)}}\big[\exp \big(-\mathcal{W}(X)\big) \big( -\mathcal{W}(X) - \log \mathbb{E}_{\mathbb{P}_{(x,t)}}\big[e^{-\mathcal{W}}\big] \big) \big]
\;+\;
\mathbb{E}_{\mathbb{P}_{(x,t)}}\big[\exp \big(-\mathcal{W}(X)\big)\,\mathcal{W}(X)\big]}{\mathbb{E}_{\mathbb{P}_{(x,t)}}\big[e^{-\mathcal{W}}\big]} \\ &= - \log \mathbb{E}_{\mathbb{P}_{(x,t)}}\big[e^{-\mathcal{W}}\big].
\end{split}
\end{talign*}
\end{proof}

\begin{remark}[Equivalence between adapted controls and Markov controls] \label{rem:equivalence}
Under standard compactness/continuity assumptions, restricting to feedback laws does not change the value function, because one can convert any adapted control $u = (u_s)_{s \geq 0}$ into a Markov control that attains the same value for $J(t,x;\cdot)$ through the mapping
    \begin{talign*}
      \tilde u(t,x)
      :=
      \mathbb{E} \bigl[
      u_t
      \,\big|\,X^{u}_t=x\bigr],
      \qquad t\ge 0,\;x\in\mathbb{R}^d.
    \end{talign*}
\end{remark}

\paragraph{Proof of Lemma \ref{lem:variational_path}} We use the shorthand $\mathbb{E}$ for $\mathbb{E}_{\mathbb{P}_{(x,t)}}$ when it is clear by the context.

\textit{Step 1: Convexity and Existence.}
The map $L\mapsto\mathbb{E}[L\log L]$ is strictly convex on
\(\{L\ge0:\,\mathbb{E}[L]=1\}\),
and $\mathcal{J}(L)$ adds a linear term $\mathbb{E}[L\,\mathcal{W}]$.  By the direct method (lower-semi-continuity + coercivity) there is a unique minimizer $L^{\star}$.

\textit{Step 2: First-Order Optimality.}
Introduce a Lagrange multiplier $\lambda\in\mathbb{R}$ for the constraint $\mathbb{E}[L]=1$.  For any variation $\delta L$ with $\mathbb{E}[\delta L]=0$ consider
\begin{talign*}
\frac{\mathrm{d}}{\mathrm{d}\varepsilon}
\Bigl\{
\mathcal{J}(L^{\star}+\varepsilon\,\delta L)
\;-\;\lambda\,\mathbb{E}[L^{\star}+\varepsilon\,\delta L]
\Bigr\}\Bigm|_{\varepsilon=0}
=0.
\end{talign*}
A direct calculation yields $\mathbb{E}\Bigl[\delta L\;\bigl(1 + \log L^{\star} + \mathcal{W} - \lambda\bigr)\Bigr]
=0$, for all $\delta L$ with mean zero.  Hence
$1 + \log L^{\star}(X) + \mathcal{W}(X) - \lambda = 0$,
so up to normalization $L^{\star}(X)\;\propto\;\exp(-\mathcal{W}(X) )$,
and enforcing $\mathbb{E}[L^{\star}]=1$ gives the stated formula.

\textit{Step 3: Inducing the Control via Girsanov.}
Define the strictly positive $\mathbb{P}_{(x,t)}$-martingale
\begin{talign*}
Z_s
:=
\frac{\mathbb{E}_{\mathbb{P}_{(x,t)}}\bigl[e^{-\mathcal{W}} \mid \mathcal{F}_s\bigr]}
     {\mathbb{E}_{\mathbb{P}_{(x,t)}}\bigl[e^{-\mathcal{W}}\bigr]},
\qquad t\le s\le \tau,
\end{talign*}
so that $Z_t=1$ and $Z_\tau=L^{\star}$. By the Martingale Representation Theorem there is a unique $\mathbb{R}^d$-valued, $\{\mathcal{F}_s\}$-predictable process $u^{\star}_s$ such that
\begin{talign*}
Z_s
\;=\;
\exp\Bigl(-\int_t^s u^{\star}_r\,\mathrm{d} W_r \;-\;\tfrac12\!\int_t^s\|u^{\star}_r\|^2\,\mathrm{d} r\Bigr).
\end{talign*}
Equivalently, in differential form,
$\mathrm{d} Z_s\;=\; -\,Z_s\,u^{\star}_s\,\mathrm{d} W_s$.
By Girsanov's theorem this $u^{\star}$ is exactly the drift-change that produces $\frac{\mathrm{d}\mathbb{P}_{(x,t)}^{u^{\star}}}{\mathrm{d}\mathbb{P}_{(x,t)}}=Z_\tau=L^{\star}$,
and strict convexity guarantees no other control can attain the same minimum.

\subsubsection{The Feynman-Kac formula for parabolic and elliptic problems
} \label{subsubsec:feynman_kac}
\begin{definition}[Hölder and parabolic Hölder spaces] \label{def:holder_space}
Let $\Omega \subset \mathbb{R}^d$ be an open domain,
$k \in \mathbb{N}_0$ an integer, and $0<\alpha\le1$.

\medskip
\noindent
\emph{(Elliptic case).}
A function $f:\Omega \to \mathbb{R}$ belongs to the Hölder space
$C^{k+\alpha}(\Omega)$ if
\begin{enumerate}[label=(\roman*),left=2pt]
  \item All partial derivatives $D^\beta f$ exist and are continuous on $\Omega$
  for every multi-index $\beta$ with $|\beta|\le k$.
  \item For every multi-index $\beta$ with $|\beta|=k$,
  the derivative $D^\beta f$ is Hölder continuous with exponent $\alpha$, i.e.
  \begin{talign*}
  [D^\beta f]_{C^\alpha(\Omega)}
  := \sup_{\substack{x,y\in\Omega \\ x\neq y}}
     \frac{|D^\beta f(x)-D^\beta f(y)|}{|x-y|^\alpha}
     < \infty.
  \end{talign*}
\end{enumerate}

\medskip
\noindent
\emph{(Parabolic case).}
Let $Q_T := [0,T]\times\Omega$.
A function $f:Q_T \to \mathbb{R}$ belongs to the parabolic Hölder space
$C^{k+\alpha,\,(k+\alpha)/2}(Q_T)$ if
\begin{enumerate}[label=(\roman*),left=2pt]
  \item All mixed derivatives $\partial_t^m D_x^\beta f$ exist and are continuous on $Q_T$
  whenever $2m+|\beta|\le k$.
  \item For each $2m+|\beta|=k$, the derivative $\partial_t^m D_x^\beta f$ is Hölder continuous
  with respect to the \emph{parabolic distance}
  \begin{talign*}
     d\bigl((t,x),(s,y)\bigr) := |x-y| + |t-s|^{1/2}.
  \end{talign*}
  That is,
  \begin{talign*}
  [\partial_t^m D_x^\beta f]_{C^{\alpha,\alpha/2}(Q_T)}
  := \sup_{\substack{(t,x)\neq(s,y)\\ (t,x),(s,y)\in Q_T}}
     \frac{|\partial_t^m D_x^\beta f(t,x)-\partial_t^m D_x^\beta f(s,y)|}
          {d\bigl((t,x),(s,y)\bigr)^{\alpha}}
     < \infty.
  \end{talign*}
\end{enumerate}
\end{definition}
\begin{theorem}[Feynman-Kac for Parabolic Terminal-Value Problems]\label{thm:feynman_kac_parabolic}
Fix $T>0$. Let $b,\sigma : \mathbb{R}^d\times[0,T] \to\mathbb{R}^d,\mathbb{R}^{d\times d}$ be locally Lipschitz in $x$, Hölder continuous in $(x,t)$, and of at most linear growth, and assume uniform ellipticity of $\sigma\sigma^\top$, i.e., there exists $\lambda>0$ such that
\begin{talign*}
\xi^\top \sigma\sigma^\top(x,t)\,\xi \;\ge\; \lambda\,|\xi|^2
\qquad\text{for all }(x,t)\in\mathbb{R}^d\times[0,T],\ \xi\in\mathbb{R}^d.
\end{talign*}
Let $c,f\in C^{\alpha,\alpha/2}(\mathbb{R}^d\times[0,T])$ be bounded for some $0<\alpha<1$, and let $g\in C^{2+\alpha}(\mathbb{R}^d)$ with at most polynomial growth. For $(x,t)\in\mathbb{R}^d\times[0,T]$, let $X^{x,t}$ be the unique strong solution of
\begin{talign*}
X_s^{x,t}
\;=\; x \;+\; \int_t^s b(X_r^{x,t},r)\,\mathrm{d}r \;+\; \int_t^s \sigma(X_r^{x,t},r)\,\mathrm{d}W_r,
\qquad s\in[t,T],
\end{talign*}
and define
\begin{talign*}
u(x,t)
\;:=\;
\mathbb{E}\!\Bigg[
  e^{-\displaystyle\int_t^T c(X_r^{x,t},r)\,\mathrm{d}r}\,g\!\bigl(X_T^{x,t}\bigr)
  \;+\;\int_t^T
    e^{-\displaystyle\int_t^s c(X_r^{x,t},r)\,\mathrm{d}r}\,f\!\bigl(X_s^{x,t},s\bigr)\,\mathrm{d}s
\Bigg].
\end{talign*}
Then $u\in C^{2+\alpha,\,1+\alpha/2}\bigl(\mathbb{R}^d\times[0,T)\bigr)\cap C\bigl(\mathbb{R}^d\times[0,T]\bigr)$ and $u$ is the unique classical solution with at most polynomial growth of
\begin{talign*}
\begin{cases}
\partial_t u(x,t)\;+\;\mathcal{L}_t u(x,t)\;-\;c(x,t)\,u(x,t)\;=\;-\,f(x,t),
& (x,t)\in\mathbb{R}^d\times[0,T),\\[4pt]
u(x,T)\;=\;g(x), & x\in\mathbb{R}^d,
\end{cases}
\end{talign*}
where
\begin{talign*}
\mathcal{L}_t \varphi(x)
\;:=\; \frac{1}{2} \operatorname{Tr} \bigl(\sigma \sigma^{\top}(x,t)\,\nabla^2 \varphi(x)\bigr) \;+\; b(x,t)\cdot \nabla\varphi(x).
\end{talign*}
Moreover, the derivatives $\nabla u$ also have at most polynomial growth.
\end{theorem}
\begin{proof}[Sketch of proof]
Define $u$ by the stochastic representation above. Fix $(x,t)$ and abbreviate $X_s:=X_s^{x,t}$. For $R>0$ set the localization $\tau_R:=\inf\{s\ge t:\,|X_s|\ge R\}\wedge T$. For $s\in[t,T]$, consider
\begin{talign*}
Y_s
\;:=\;
e^{-\displaystyle\int_t^{s\wedge\tau_R} c(X_r,r)\,\mathrm{d}r}\,
u\bigl(X_{s\wedge\tau_R},s\wedge\tau_R\bigr)
\;+\;
\int_t^{s\wedge\tau_R}
e^{-\displaystyle\int_t^{r} c(X_\ell,\ell)\,\mathrm{d}\ell}\,
f\bigl(X_r,r\bigr)\,\mathrm{d}r.
\end{talign*}
By the time-dependent Itô formula applied to $u(X_s,s)$ and then multiplying by the Doléans factor $e^{-\int_t^{s}c}$, we obtain
\begin{talign*}
\mathrm{d}Y_s
\;=\;
e^{-\int_t^{s\wedge\tau_R} c(X_r,r)\,\mathrm{d}r}\,
\bigl(\partial_s u + \mathcal{L}_s u - c\,u + f\bigr)\bigl(X_{s\wedge\tau_R},s\wedge\tau_R\bigr)\,\mathrm{d}s
\;+\; \mathrm{d}M_s,
\end{talign*}
where $M$ is a local martingale. Taking expectations yields, for $s\in[t,T]$,
\begin{talign*}
\mathbb{E}[Y_s]
\;=\;
u(x,t)
\;+\;
\int_t^{s}
\mathbb{E}\!\Bigl[
e^{-\int_t^{r\wedge\tau_R} c(X_\ell,\ell)\,\mathrm{d}\ell}\,
\bigl(\partial_r u + \mathcal{L}_r u - c\,u + f\bigr)\bigl(X_{r\wedge\tau_R},r\wedge\tau_R\bigr)
\Bigr]\,\mathrm{d}r.
\end{talign*}
By boundedness of $c,f$ and polynomial growth of $u$ together with finite horizon $T$, the family $\{Y_s: s\in[t,T], R\ge1\}$ is uniformly integrable. Letting $s\uparrow T$ and using $u(T,\cdot)=g(\cdot)$ in the stochastic representation gives
\begin{talign*}
\mathbb{E}[Y_T]
\;=\;
\mathbb{E}\!\Bigl[
  e^{-\int_t^{T} c(X_r,r)\,\mathrm{d}r}\,g(X_T)
  \;+\;\int_t^{T}
    e^{-\int_t^{r} c(X_\ell,\ell)\,\mathrm{d}\ell}\,f(X_r,r)\,\mathrm{d}r
\Bigr]
\;=\; u(x,t).
\end{talign*}
Comparing the last two displays at $s=T$ and then letting $R\to\infty$ (dominated convergence via linear-growth/ellipticity moment bounds) yields
\begin{talign*}
\int_t^T
\mathbb{E}\!\Bigl[
e^{-\int_t^{r} c(X_\ell,\ell)\,\mathrm{d}\ell}\,
\bigl(\partial_r u + \mathcal{L}_r u - c\,u + f\bigr)\bigl(X_{r},r\bigr)
\Bigr]\,\mathrm{d}r
\;=\;0.
\end{talign*}
The integrand is continuous in $r$ with value $\bigl(\partial_t u + \mathcal{L}_t u - c\,u + f\bigr)(x,t)$ at $r=t$, hence this value must vanish; that is,
\begin{talign*}
\bigl(\partial_t u + \mathcal{L}_t u - c\,u + f\bigr)(x,t)=0.
\end{talign*}
The terminal condition follows from the definition of $u$. Standard parabolic Schauder theory then implies $u\in C^{2+\alpha,\,1+\alpha/2}$ on $\mathbb{R}^d \times [0,T)$, and uniqueness in this class with at most polynomial growth follows from the parabolic maximum principle.
\end{proof}
\begin{theorem}[Feynman-Kac for Elliptic Boundary-Value Problems with Exit Time] \label{thm:feynman_kac}
Let $\mathcal{D}\subset\mathbb{R}^d$ be a domain with $C^{2+\alpha}$ boundary $\partial\mathcal{D}$ for some $0<\alpha<1$.  Suppose
\begin{talign*}
X_t^x \;=\; x \;+\; \int_0^t b\bigl(X_s^x\bigr)\,\mathrm{d}s \;+\; \int_0^t \sigma\bigl(X_s^x\bigr)\,\mathrm{d}W_s,
\end{talign*}
where $b,\sigma\in C^{\alpha}(\overline{\mathcal D})$ and the generator
\begin{talign*}
\mathcal{L}u(x)
=\tfrac12\operatorname{Tr} \bigl(\sigma\sigma^\top(x)\,\nabla^2u(x)\bigr)
 \;+\; b(x)\cdot \nabla u(x)
\end{talign*}
is uniformly elliptic.  Define the exit time $\tau
=\inf\{\,t\ge0 : X_t^x\notin\mathcal D\}$,
and let
\begin{talign*}
c,f\in C^{\alpha}(\overline{\mathcal D}),
\quad
g\in C^{2+\alpha}(\partial\mathcal D).
\end{talign*}
For each $x\in\mathcal D$, set
\begin{talign*}
u(x)
\;=\;
\mathbb{E}_x\big[
  e^{-\displaystyle\int_0^{\tau} c\bigl(X_s\bigr)\,\mathrm{d}s}\,
  g\bigl(X_{\tau}\bigr)
  \;+\;
  \int_0^{\tau}
    e^{-\displaystyle\int_0^s c\bigl(X_r\bigr)\,\mathrm{d}r}\,
    f\bigl(X_s\bigr)\,\mathrm{d}s
\big].
\end{talign*}
Suppose that either
\begin{itemize}
    \item $\mathcal{D}$ is bounded, or
    \item $\tau$ is finite almost surely, and $u$ has at most polynomial growth.
\end{itemize}
Then $u\in C^{2+\alpha}(\mathcal D)\cap C^0(\overline{\mathcal D})$ and is the unique classical solution of
\begin{talign*}
\begin{cases}
\mathcal{L}u(x)\;-\;c(x)\,u(x)\;=\;-\,f(x), & x\in\mathcal D,\\[6pt]
u(x)=g(x), & x\in\partial\mathcal D.
\end{cases}
\end{talign*}
Moreover, $u$ and its derivatives $\nabla u$ have at most polynomial growth.
\end{theorem}
\begin{proof}[Sketch of proof]
We write the proof in the case in which $\mathcal{D}$ is bounded, as the other case just requires slight modifications. Let
\begin{talign*}
u(x)
=\mathbb{E}_x\Bigl[
  e^{-\displaystyle\int_0^{\tau}c(X_s)\,\mathrm{d}s}\,g(X_{\tau})
  + \int_0^{\tau}
    e^{-\displaystyle\int_0^s c(X_r)\,\mathrm{d}r}\,f(X_s)\,\mathrm{d}s
\Bigr].
\end{talign*}
Fix \(x\in\mathcal D\), and for \(r>0\) set \(\tau_r=\tau\wedge r\).  Define the process
\begin{talign*}
Y_t
=
e^{-\displaystyle\int_0^{t\wedge\tau_r}c(X_s)\,\mathrm{d}s}\,u\bigl(X_{t\wedge\tau_r}\bigr)
\;+\;
\int_0^{t\wedge\tau_r}
e^{-\displaystyle\int_0^s c(X_r)\,\mathrm{d}r}\,f(X_s)\,\mathrm{d}s.
\end{talign*}
By Itô's formula,
\begin{talign*}
dY_t
=
e^{-\int_0^{t\wedge\tau_r}c(X_s)\,\mathrm{d}s}\,
\bigl(\mathcal{L}u - c\,u + f\bigr)\bigl(X_{t\wedge\tau_r}\bigr)\,\mathrm{d}t
\;+\;\mathrm{d}M_t,
\end{talign*}
where \(M_t\) is a local martingale.
By the Itô computation in the proof we have, for each fixed $r>0$ and $t\ge0$,
\begin{talign*}
\mathbb{E}_x[Y_t]
\;=\;
u(x)
\;+\;
\int_0^t
\mathbb{E}_x\!\Bigl[
  e^{-\int_0^{s\wedge\tau_r}c(X_u)\,\mathrm{d}u}\,
  (\mathcal{L}u - c\,u + f)\bigl(X_{s\wedge\tau_r}\bigr)
\Bigr]
\,\mathrm{d}s.
\end{talign*}
Here $\tau_r=\tau\wedge r$.  Since $c,f$ are bounded on $\overline{\mathcal D}$ and $u|_{\partial\mathcal D}=g$ is continuous (hence bounded), the family $\{Y_t\}$ is uniformly integrable.  Thus we may first let $t\to\infty$ (so $t\wedge\tau_r\to\tau_r$) and apply dominated convergence to get
\begin{talign*}
\mathbb{E}_x[Y_\infty]
\;=\;
u(x)
\;+\;
\int_0^\infty
\mathbb{E}_x\!\Bigl[
  e^{-\int_0^{s\wedge\tau_r}c(X_u)\,\mathrm{d}u}\,
  (\mathcal{L}u - c\,u + f)\bigl(X_{s\wedge\tau_r}\bigr)
\Bigr]
\,\mathrm{d}s.
\end{talign*}
Next let $r\to\infty$.  By the boundary condition $u(X_{\tau})=g(X_{\tau})$ and dominated convergence again,
\begin{talign*}
\mathbb{E}_x[Y_\infty]
\;=\;
\mathbb{E}_x\!\Bigl[
  e^{-\int_0^{\tau}c(X_s)\,\mathrm{d}s}\,g(X_{\tau})
  +\!\int_0^{\tau}
    e^{-\int_0^s c(X_u)\,\mathrm{d}u}\,f(X_s)\,\mathrm{d}s
\Bigr]
\;=\;u(x),
\end{talign*}
by the very definition of $u$.  Comparing the two expressions for $\mathbb{E}_x[Y_\infty]$ shows
\begin{talign*}
\int_0^\infty
\mathbb{E}_x\!\Bigl[
  e^{-\int_0^{s\wedge\tau}c(X_u)\,\mathrm{d}u}\,
  (\mathcal{L}u - c\,u + f)\bigl(X_{s\wedge\tau}\bigr)
\Bigr]
\,\mathrm{d}s
=0.
\end{talign*}
But the integrand
\begin{talign*}
s\;\mapsto\;
\mathbb{E}_x\!\Bigl[
  e^{-\int_0^{s\wedge\tau}c(X_u)\,\mathrm{d}u}\,
  (\mathcal{L}u - c\,u + f)\bigl(X_{s\wedge\tau}\bigr)
\Bigr]
\end{talign*}
is continuous at $s=0$ with value $(\mathcal{L}u - c\,u + f)(x)$.  The only way its integral over $[0,\infty)$ can vanish is if its value at $s=0$ vanishes.  Hence
\begin{talign*}
(\mathcal{L}u - c\,u + f)(x)\;=\;0,
\end{talign*}
as required.
Continuity of \(u\) up to \(\partial\mathcal D\) gives the boundary condition \(u=g\).  Uniqueness follows from the (strong) maximum principle.
\end{proof}

\begin{remark}[Local vs.\ global Hölder assumptions]
In Theorems~\ref{thm:feynman_kac_parabolic} and \ref{thm:feynman_kac}, the global assumptions
$c,f\in C^{\alpha,\alpha/2}$ (parabolic) or $c,f\in C^\alpha$ (elliptic) can be relaxed to
\emph{local} Hölder continuity:
\begin{talign*}
b,\sigma \in C^{\alpha,\alpha/2}_{\mathrm{loc}},\qquad
c,f \in C^{\alpha,\alpha/2}_{\mathrm{loc}}
\quad\text{(parabolic)},
\end{talign*}
or
\begin{talign*}
b,\sigma \in C^{\alpha}_{\mathrm{loc}},\qquad
c,f \in C^{\alpha}_{\mathrm{loc}}
\quad\text{(elliptic)}.
\end{talign*}
Assume uniform ellipticity, linear growth for $b,\sigma$, and that $c$ is bounded below
($c \ge -C$) so that $e^{-\int c}$ has controlled moments; also assume $f,g$ have at most polynomial growth.
Then the Feynman-Kac representation yields a function $u$ which satisfies
\begin{talign*}
\partial_t u + \mathcal{L}_t u - c u = -f \quad\text{in the classical sense on } (0,T)\times\mathbb{R}^d,
\end{talign*}
resp.\ $\mathcal{L}u - c u = -f$ on $\mathcal{D}$, with
\begin{talign*}
u \in C^{2+\alpha,1+\alpha/2}_{\mathrm{loc}}((0,T)\times\mathbb{R}^d)
\quad\text{or}\quad
u \in C^{2+\alpha}_{\mathrm{loc}}(\mathcal{D}).
\end{talign*}
If, in addition, the coefficients and $g$ are Hölder \emph{up to the boundary} and
$\partial\mathcal{D}\in C^{2+\alpha}$, then $u$ enjoys the corresponding global/up-to-boundary
regularity ($C^{2+\alpha,1+\alpha/2}$ in the parabolic case, $C^{2+\alpha}$ in the elliptic case)
and the boundary condition holds classically. Without the up-to-boundary Hölder control, the
boundary condition holds in the sense of continuous trace, and the interior regularity remains local.
Uniqueness in these classes follows from the (parabolic/elliptic) maximum principle under
$c \ge -C$ and standard growth conditions.
\end{remark}

\subsubsection{Value function regularity and relation to optimal control} \label{eq:vf_regularity}

\begin{theorem} \label{thm:vf_regularity}
    Suppose that $b\colon\mathbb{R}^d\to\mathbb{R}^d$, $\sigma\colon\mathbb{R}^d\to\mathbb{R}^{d\times d}$ satisfy the usual conditions guaranteeing a unique strong solution $X$ of the SDE (measurability, local Lipschitzness, polynomial growth), that $\mathcal{S} = A \cup B \subset \mathbb{R}^d$ is a closed set, and that $g : \partial \mathcal{S} \to \mathbb{R}$, $f : \overline{\mathcal{S}^c} \to \mathbb{R}$ are bounded below and have at most polynomial growth, where $\mathcal{S}^c = \mathbb{R}^d \setminus \mathcal{S}$. Consider the hitting time SOC problem with state costs (a more general version of \eqref{eq:control_problem_def_main}-\eqref{eq:controlled_SDE_main}):
    \begin{talign}
    \label{eq:control_problem_def_app}
        &\min_{u \in \mathcal{U}} \mathbb{E}[\int_0^{\tau} \big( \frac{1}{2} \|u(X^u_t)\|^2 + f(X^u_t) \big) \, \mathrm{d}t + g(X^u_{\tau})], \\
        &\text{where } \mathrm{d}X^u_t = (b(X^u_t) + \sigma(X^u_t) u(X^u_t)) \, \mathrm{d}t + \sigma(X^u_t) \mathrm{d}W_t, \quad X^u_0 \sim \mu_0,
        \label{eq:controlled_SDE_app}
    \end{talign}
    and the corresponding value function, which is a more general version of \eqref{eq:value_function}:
    \begin{talign}
    \label{eq:value_function_general}
        V(x) = \min_{u \in \mathcal{U}} \mathbb{E}[ \int_0^{\tau} \big(\frac{1}{2}\|u(X^u_t)\|^2 + f(X^u_t)\big) \, \mathrm{d}t + g(X^u_{\tau})| X_0 = x],
    \end{talign}
    By \Cref{thm:value_functions}, \Cref{rem:equivalence}, and the time-independence of the problem, $V$ also admits the following path integral characterization:
    \begin{talign*}
        V(x,t) = - \log \mathbb{E}\big[ \exp \big( - \int_t^{\tau} f(
        X_{s}) \, \mathrm{d}s - g(X_{\tau}) \big) \, | \, X_t = x \big].
    \end{talign*}
    If additionally, for some $0 < \alpha < 1$,
    \begin{itemize}
    \item $f \in C^{\alpha}(\overline{\mathcal{S}^c})$ and $g \in C^{2+\alpha}(\partial \mathcal{S})$,
    \item $b, \sigma \in C^{\alpha}(\overline{\mathcal{S}^c})$,
    \item There exist $\Lambda, \lambda > 0$ such that $\lambda < \|\sigma \sigma^{\top}(x) \| < \Lambda$, for all $x \in \overline{\mathcal{S}^c}$,
    \item The boundary $\partial\mathcal{S}$ is $C^{2+\alpha}$,
    \item Either $\mathcal{S}^c$ is bounded, or we assume that $\tau$ is finite almost surely,
    \end{itemize}
    then we have that $V \in C^{2+\alpha}(\mathcal{S}^c) \cap C(\overline{\mathcal{S}^c})$, and that it solves a stationary Hamilton-Jacobi-Bellman equation:
    \begin{talign} \label{eq:stationary_HJB}
    \begin{cases}
    \tfrac12\,\mathrm{Tr}\bigl(\sigma(x)\,\sigma(x)^\top\,\nabla^2 V(x)\bigr)
     \;+\; b(x)\cdot \nabla V(x)
     \;-\;\tfrac12\,\big\|\sigma(x)^\top \nabla V(x)\big\|^2
     = -f(x),
     & x\in \mathcal{S}^c,\\[6pt]
    V(x) = g(x),
     & x\in \partial\mathcal{S}.
    \end{cases}
    \end{talign}
    Also, for any $x \in \mathcal{S}^c$, the minimum in equations \eqref{eq:control_problem_def_app} and \eqref{eq:value_function_general}
    is attained by the control function $u^{\star} : \mathbb{R}^d \to \mathbb{R}^d$, defined for all $x\in \mathcal{S}^c$ as
    \begin{talign} \label{eq:u_star_hat_v}
        u^{\star}(x) = - \sigma^{\top}(x) \nabla V(x).
    \end{talign}
\end{theorem}

\begin{proof}
\emph{Step 1: Applying Feynman-Kac to show that HJB holds.} Using the additional regularity assumptions, and looking at the formula $V(x) = - \log \mathbb{E}\big[ \exp \big( - \int_0^{\tau} f(X_t) \, \mathrm{d}t - g(X_{\tau}) \big) \, | \, X_0 = x \big]$, Theorem \ref{thm:feynman_kac} implies that $\exp(-V) \in C^{2+\alpha}(\mathcal{S}^c) \cap C^{0}(\overline{\mathcal{S}^c})$, and that it solves the following PDE, which is a stationary backward Kolmogorov equation with killing rate $f$.
\begin{talign*}
\begin{cases}
\tfrac12\mathrm{Tr} \bigl(\sigma\sigma^\top(x)\,\nabla^2 u(x)\bigr)
 + b(x)\cdot \nabla u(x) - f(x)\,u(x) = 0, & x\in\mathcal S^c,\\[6pt]
u(x)=e^{-g(x)}, & x\in\partial\mathcal S.
\end{cases}
\end{talign*}
Applying the Cole-Hopf transform, we obtain that $V \in C^{2+\alpha}(\mathcal{S}^c) \cap C^{0}(\overline{\mathcal{S}^c})$, and that it solves the stationary Hamilton-Jacobi-Bellman equation \eqref{eq:stationary_HJB}.

\emph{Step 2: Deriving the expression of the optimal control.}
Since $V \in C^{2+\alpha}(\mathcal{S}^c)$, for any control $u \in \mathcal{U}$, by Itô's lemma, we have that
\begin{talign}
\begin{split} \label{eq:verification_1}
    V(X^u_{\tau}) - V(X^u_0) &= \int_0^{\tau} \big( \langle b(X^u_s,s) + \sigma(X^u_s,s) u(X^u_s,s), \nabla V(X^u_s,s) \rangle \\ &\qquad + \frac{1}{2} \sum_{i,j=1}^{d}  (\sigma \sigma^{\top})_{ij} (X^u_s,s) \partial_{x_i} \partial_{x_j} V(X^u_s,s) \big) \, \mathrm{d}s + S^u_{\tau},
\end{split}
\end{talign}
where $S^u_{\tau} = \int_0^{\tau} \nabla V(X^u_s,s)^{\top} \sigma (X^u_s,s) \, \mathrm{d}W_s$. Note that by \eqref{eq:stationary_HJB},
\begin{talign}
\begin{split} \label{eq:verification_2}
    &\langle b(X^u_s,s) + \sigma(X^u_s,s) u(X^u_s), \nabla V(X^u_s) \rangle \\ &\qquad + \frac{1}{2} \sum_{i,j=1}^{d}  (\sigma \sigma^{\top})_{ij} (X^u_s) \partial_{x_i} \partial_{x_j} V(X^u_s) \\ &= \frac{1}{2} \| (\sigma^{\top} \nabla V)(X^u_s) \|^2 - f(X^u_s) + \langle \sigma(X^u_s,s) u(X^u_s), \nabla V(X^u_s) \rangle \\ &= \frac{1}{2}\| (\sigma^{\top} \nabla V)(X^u_s) + u(X^u_s) \|^2 - \frac{1}{2}\|u(X^u_s) \|^2 - f(X^u_s),
\end{split}
\end{talign}
and this implies that
\begin{talign} \label{eq:no_exp}
    g(X^u_{\tau}) - V(X^u_0) = \int_0^{\tau} \big( \frac{1}{2}\| (\sigma^{\top} \nabla V)(X^u_s) + u(X^u_s) \|^2 - \frac{1}{2}\|u(X^u_s) \|^2 - f(X^u_s) \big) \, \mathrm{d}s + S^u_{\tau}.
\end{talign}
Since $\mathbb{E} [ S^u_{\tau} \, | \, X^u_0 = x ] = 0$, rearranging \eqref{eq:no_exp} and taking the conditional expectation with respect to $X^u_0$ yields the following result:
\begin{talign*}
    J(x;u) = \mathbb{E} \big[ \int_0^{\tau} \frac{1}{2}\| (\sigma^{\top} \nabla V)(X^u_s) + u(X^u_s) \|^2 \, \mathrm{d}s \, | \, X^u_0 = x \big] + V(x),
\end{talign*}
where the control objective $J$ is defined as
\begin{talign*}
    J(x;u) = \mathbb{E} \big[ \int_0^{\tau} \big( \frac{1}{2} \|u(X^u_t)\|^2 + f(X^u_t) \big) \, \mathrm{d}t + g(X^u_{\tau}) \, | \, X^u_0 = x \big]
\end{talign*}
Since $V(x) = \inf_{u \in \mathcal{U}} J(x;u)$, the infimum is attained at the optimal control $u^{\star}(x) = - (\sigma^{\top} \nabla V)(x)$.
\end{proof}

\subsubsection{Non-standard Dynkin representation theorems}

We develop non-standard (generalized) variations of the Dynkin representation of functionals of stochastic processes, which we will use in our proofs. To our knowledge, these specific results have not been stated or used before.

\begin{theorem}[Parabolic PDE/Dynkin representation (finite horizon)]
\label{thm:parabolic-dynkin}
Fix $T>0$. Let $X$ solve the time-inhomogeneous SDE
\begin{talign} \label{eq:X_td}
\mathrm{d}X_t \;=\; b(X_t,t)\,\mathrm{d}t \;+\; \sigma(X_t,t)\,\mathrm{d}W_t,\qquad X_0=x\in\mathbb{R}^d,
\end{talign}
where $b,\sigma$ are locally Lipschitz in $x$ (uniformly in $t\in[0,T]$) with linear growth, and set $a=\sigma \sigma^\top$.
Let $h_1:\mathbb{R}^d\times[0,T]\to\mathbb{R}$ and $h_2:\mathbb{R}^d\to\mathbb{R}$ be continuous.

Assume there exists $\eta\in C^{1,2}(\mathbb{R}^d \times [0,T])$ with at most polynomial growth such that
\begin{talign}\label{eq:parabolic-PDE}
\partial_t \eta(x,t)\;+\;\mathcal{L}_t \eta(x,t)\;+\;h_1(x,t)\;=\;0 \quad\text{for }(x,t)\in\mathbb{R}^d\times[0,T),
\qquad
\eta(x,T)=h_2(x),
\end{talign}
where $\mathcal{L}_t \eta := b(\cdot,t)\!\cdot\!\nabla \eta \;+\; \tfrac{1+\chi}{2}\mathrm{Tr}(a(\cdot,t)\nabla^2\eta)$, with $\chi \in (-1,+\infty)$ arbitrary.
Suppose moreover that
\begin{talign*}
\mathbb{E}\!\left[\int_0^{T}\!\!\|\sigma(X_t,t)^\top\nabla \eta(X_t,t)\|^2\, \mathrm{d}t\right]<\infty
\end{talign*}
which holds, for example, if $\nabla\eta$ has at most polynomial growth and $X$ has finite moments (the latter holds by the assumptions on $b$, $\sigma$).
Define
\begin{talign} \label{eq:F_def}
F \;:=\; \int_0^{T} h_1(X_t,t)\,\mathrm{d}t \;+\; h_2(X_T).
\end{talign}
Then the following representation holds:
\begin{talign}\label{eq:parabolic-dynkin-repr}
F
\;=\;
\eta(x,0)
\;+\;
\int_0^{T} \!\big\langle \sigma^{\top}\nabla \eta(X_t,t), \, \mathrm{d}W_t \big\rangle - \frac{\chi}{2} \int_0^{T} \mathrm{Tr}\big(\sigma \sigma^{\top} \nabla^2\eta(X_t,t) \big) \, \mathrm{d}t,
\end{talign}
and, in particular, $\;\mathbb{E}[F]=\eta(x,0) - \frac{\chi}{2} \mathbb{E}\big[\int_0^{T} \mathrm{Tr}\big(\sigma \sigma^{\top} \nabla^2\eta(X_t,t) \big) \, \mathrm{d}t \big]$.
\end{theorem}

\begin{proof}
By Itô's formula for time-dependent test functions,
\begin{talign*}
\begin{split}
\mathrm{d}\eta(X_t,t)
&\;=\;
\big(\partial_t\eta(X_t,t) + \langle b(X_t,t), \nabla \eta(X_t,t) \rangle + \tfrac{1}{2}\mathrm{Tr}(a(X_t,t)\nabla^2\eta(X_t,t) ) \big)\,\mathrm{d}t
\\ &\quad +
\big\langle \sigma(X_t,t)^\top\nabla \eta(X_t,t),\, \mathrm{d}W_t \big\rangle \\
&\;=\;
\big(\partial_t\eta + \mathcal{L}_t\eta\big)(X_t,t)\,\mathrm{d}t
+
\big\langle \sigma(X_t,t)^\top\nabla \eta(X_t,t),\, \mathrm{d}W_t \big\rangle \\ &\quad - \tfrac{\chi}{2}\mathrm{Tr}(a(X_t,t)\nabla^2\eta(X_t,t) ) \,\mathrm{d}t.
\end{split}
\end{talign*}
Integrating from $0$ to $T$ and using the terminal condition $\eta(X_T,T)=h_2(X_T)$ yields
\begin{talign*}
\begin{split}
h_2(X_T)-\eta(x,0)
&\;=\;
\int_0^T \!\big(\partial_t\eta + \mathcal{L}_t\eta\big)(X_t,t)\,dt
\;+\;
\int_0^T \!\big\langle \sigma(X_t,t)^\top\nabla \eta(X_t,t),\, \mathrm{d}W_t \big\rangle \\ &\qquad - \tfrac{\chi}{2} \int_0^T \mathrm{Tr}(a(X_t,t)\nabla^2\eta(X_t,t) ) \,\mathrm{d}t.
\end{split}
\end{talign*}
Invoking the PDE \eqref{eq:parabolic-PDE}, we have $\partial_t\eta+\mathcal{L}_t\eta=-h_1$, hence
\begin{talign*}
\begin{split}
\int_0^{T} h_1(X_t,t)\,dt + h_2(X_T)
 &=
\eta(x,0)
+ \!
\int_0^{T} \big\langle \sigma(X_t,t)^\top\nabla \eta(X_t,t),\, \mathrm{d}W_t \big\rangle \\ &\quad - \tfrac{\chi}{2} \int_0^T \mathrm{Tr}(a(X_t,t)\nabla^2\eta(X_t,t) ) \,\mathrm{d}t,
\end{split}
\end{talign*}
which is \eqref{eq:parabolic-dynkin-repr}. The stochastic integral is a square-integrable martingale with zero mean by the stated integrability condition, giving $\mathbb{E}[F]=\eta(x,0) - \frac{\chi}{2} \mathbb{E}\big[\int_0^{T} \mathrm{Tr}\big(\sigma \sigma^{\top} \nabla^2\eta(X_t,t) \big) \, \mathrm{d}t \big]$.
\end{proof}

\begin{remark}[Conditions for existence of a solution to the terminal-value problem \eqref{eq:parabolic-PDE}] \label{rem:regularity_PDP}
Several alternative conditions are available for the existence of a solution to the terminal-value problem \eqref{eq:parabolic-PDE}. Applying the Feynman-Kac formula for parabolic terminal-value problems (\Cref{thm:feynman_kac_parabolic}), we have that the following conditions are sufficient:
\begin{itemize}[left=3pt]
    \item $b,\sigma : \mathbb{R}^d\times[0,T] \to\mathbb{R}^d,\mathbb{R}^{d\times d}$ are locally Lipschitz in $x$, Hölder continuous in $(x,t)$, and of at most linear growth, and $a:=\sigma\sigma^\top$ is uniformly elliptic.
    \item $h_1\in C^{\alpha,\alpha/2}(\mathbb{R}^d \times [0,T])$, and $h_2\in C^{2+\alpha}(\mathbb{R}^d)$, and have at most polynomial growth.
\end{itemize}
In particular, these conditions imply that $\eta \in C^{2+\alpha,1+\alpha/2}(\mathbb{R}^d \times [0,T))\cap C(\mathbb{R}^d \times [0,T])$ defined as $\eta(x,t) = \mathbb{E}\big[\int_t^{T}\!h_1(X_s,s)\,ds + h_2(X_{T})\, | \, X_t = x \big]$ is the unique classical solution of the terminal-value problem \eqref{eq:parabolic-PDE}.
\end{remark}

\begin{remark}[Alternative approach through the Clark-Ocone theorem from Malliavin calculus]
    Given a functional $F$ of the form \eqref{eq:F_def} with the appropriate regularity conditions on $h_1$, $h_2$, $b$, $\sigma$, the Clark-Ocone theorem provides the following representation:
    \begin{talign} \label{eq:F_Clark_Ocone}
      F
      \;=\;
      \mathbb{E}[F|X_0]
      \;+\;
      \int_0^T \langle \mathbb{E}\bigl[D_t F \mid \mathcal{F}_t\bigr], \mathrm{d}W_t \rangle,
    \end{talign}
    where $D_t F$ is the Malliavin derivative of $F$. By the definition of the Malliavin derivative, and exchanging the gradient and the conditional expectation, it is possible to prove that $\mathbb{E}\bigl[D_t F \mid \mathcal{F}_t\bigr] = \sigma^{\top}(X_t,t) \nabla_{X_t} \mathbb{E}\big[ \int_t^{T} h_1(X_s,s)\,\mathrm{d}s + h_2(X_T) \big] := \sigma^{\top}(X_t,t) \nabla_{X_t} \mathbb{E}[F_t(X)|X_t]$, where we define $F_t(X) = \int_t^{T} h_1(X_s,s)\,\mathrm{d}s \;+\; h_2(X_T)$. Plugging this into \eqref{eq:F_Clark_Ocone} yields
    \begin{talign} \label{eq:F_Clark_Ocone_2}
      F
      \;=\;
      \mathbb{E}[F_0|X_0]
      \;+\;
      \int_0^T \langle \sigma^{\top}(X_t,t) \nabla_{X_t} \mathbb{E}[F_t(X)|X_t], \mathrm{d}W_t \rangle,
    \end{talign}
    which means we recover the result of \Cref{thm:parabolic-dynkin} in the particular case $\chi = 0$ by setting $\eta(x) = \mathbb{E}[F_t(X)\mid X_t = x]$. The Dynkin approach is preferable because it applies to the case $\chi \neq 0$ and also to the case in which the horizon is a stopping time (\cref{thm:pde-dynkin}), while the Clark-Ocone approach does not translate as easily.
\end{remark}

\begin{theorem}[PDE/Dynkin representation for exit-time functionals]
\label{thm:pde-dynkin}
Let $X$ solve the SDE
\[
dX_t=b(X_t)\,dt+\sigma(X_t)\, \mathrm{d}W_t,\qquad X_0=x\in\mathbb{R}^d,
\]
with $b,\sigma$ locally Lipschitz with linear growth and $a=\sigma\sigma^\top$ uniformly elliptic.
Let $\mathcal{D}\subset\mathbb{R}^d$ be an open set and $\tau:=\tau_{\mathcal{D}}=\inf\{t\ge0:\,X_t\notin \mathcal{D}\}$.
Let $h_1:\mathcal{D}\to\mathbb{R}$ and $h_2:\partial \mathcal{D}\to\mathbb{R}$ be continuous.

Assume there exists $\eta \in C^2(\mathcal{D})\cap C(\overline{\mathcal{D}})$
solving the Dirichlet-Poisson problem
\begin{talign} \label{eq:Dirichlet-Poisson}
\mathcal{L} \eta(x)+h_1(x)=0\quad\text{for }x\in \mathcal{D},
\qquad
\eta(x)=h_2(x)\quad\text{for }x\in\partial \mathcal{D},
\end{talign}
where $\mathcal{L} \eta=b\!\cdot\!\nabla \eta+\tfrac{1+\chi}{2}\mathrm{Tr}(a\nabla^2 \eta)$, with $\chi \in (-1,+\infty)$ arbitrary.
Suppose moreover that
\begin{talign*}
\mathbb{E}\!\left[\int_0^{\tau}\!\!\|\sigma(X_t)^\top\nabla \eta(X_t)\|^2\,dt\right]<\infty
\quad\text{(e.g.\ holds if $\mathcal{D}$ is bounded and $\nabla \eta$ is bounded on $\overline{\mathcal{D}}$).}
\end{talign*}
Define
\begin{talign*}
F \;:=\; \int_0^{\tau} h_1(X_t)\,dt \;+\; h_2(X_\tau).
\end{talign*}
Then
\begin{talign*}
F \;=\; \eta(x) \;+\; \int_0^{\tau} \langle \sigma(X_t)^{\!\top}\nabla \eta(X_t), \mathrm{d}W_t \rangle
- \frac{\chi}{2}
\int_0^{\tau} \mathrm{Tr}\big( \sigma \sigma^{\top} \nabla^2_x \eta(X^{v,\kappa}_t) \big) \, \mathrm{d}t,
\end{talign*}
and, in particular, $\;\mathbb{E}[F]=\eta(x) - \frac{\chi}{2}
\mathbb{E}\big[ \int_0^{\tau} \mathrm{Tr}\big( \sigma \sigma^{\top} \nabla^2_x \eta(X^{v,\kappa}_t) \big) \, \mathrm{d}t \big]$.
\end{theorem}

\begin{proof}
Fix $n\in\mathbb{N}$ and set $\tau_n:=\tau\wedge n$. Itô's formula applied to $\eta(X_{t\wedge\tau_n})$ gives
\begin{talign*}
\begin{split}
\eta(X_{t\wedge\tau_n})
&=
\eta(x)
+
\int_0^{t\wedge\tau_n}\!\!
\big( \langle b(X_s), \nabla \eta(X_s) \rangle +\tfrac{1}{2}\mathrm{Tr}(a\nabla^2 \eta(X_s)) \big)
\, \mathrm{d}s
\\ &\quad +
\int_0^{t\wedge\tau_n}\!\langle \sigma(X_s)^\top\nabla \eta(X_s), \mathrm{d}W_s \rangle \\
&=
\eta(x)
+
\int_0^{t\wedge\tau_n}\!\! \big( \mathcal{L} \eta(X_s) - \tfrac{\chi}{2} \mathrm{Tr}(a\nabla^2 \eta(X_s)) \big)\, \mathrm{d}s
\\ &\quad +
\int_0^{t\wedge\tau_n}\!\langle \sigma(X_s)^\top\nabla \eta(X_s), \mathrm{d}W_s \rangle.
\end{split}
\end{talign*}
Using $\mathcal{L} \eta=-h_1$ on $\mathcal{D}$,
\begin{talign*}
\eta(X_{t\wedge\tau_n})
\! = \!
\eta(x)
\! - \!
\int_0^{t\wedge\tau_n}\!\!\big( h_1(X_s) \! + \! \tfrac{\chi}{2} \mathrm{Tr}(a\nabla^2 \eta(X_s)) \big)\, \mathrm{d}s
\! + \!
\int_0^{t\wedge\tau_n}\!\langle \sigma(X_s)^\top\nabla \eta(X_s), \mathrm{d}W_s \rangle.
\end{talign*}
Let $t\to\infty$. Then $t\wedge\tau_n\uparrow \tau_n$, $X_{t\wedge\tau_n}\to X_{\tau_n}$, and
\begin{talign*}
\eta(X_{\tau_n})
=
\eta(x)
-
\int_0^{\tau_n}\!\big( h_1(X_s) + \tfrac{\chi}{2} \mathrm{Tr}(a\nabla^2 \eta(X_s))\big)\, \mathrm{d}s
+
\int_0^{\tau_n}\!\langle \sigma(X_s)^\top\nabla \eta(X_s), \mathrm{d}W_s \rangle.
\end{talign*}
Since $X_{\tau}\in\partial \mathcal{D}$ and $\eta|_{\partial \mathcal{D}}=h_2$, we have $\eta(X_{\tau})=h_2(X_\tau)$ and $\eta(X_{\tau_n})=h_2(X_{\tau_n})$ on $\{\tau\le n\}$, while $\eta(X_{\tau_n})=\eta(X_n)$ on $\{\tau>n\}$.

By the square-integrability assumption, $\big(\int_0^{t\wedge\tau} \langle \sigma^\top\nabla \eta, \mathrm{d}W \rangle \big)_{t\ge0}$ is a square-integrable martingale, hence optional stopping justifies taking expectations in the previous display with $t=\infty$:
\begin{talign*}
\mathbb{E}\!\left[\eta(X_{\tau_n})+\int_0^{\tau_n}\!\big( h_1(X_s) + \tfrac{\chi}{2} \mathrm{Tr}(a\nabla^2 \eta(X_s))\big)\,ds\right] = \eta(x).
\end{talign*}
Let $n\to\infty$. The monotone convergence theorem yields
\(
\int_0^{\tau_n} h_1(X_s)\,ds \uparrow \int_0^{\tau} h_1(X_s)\,ds
\),
and $\eta(X_{\tau_n})\to \eta(X_{\tau})=h_2(X_\tau)$ a.s.; dominated convergence (using continuity of $\eta,h_2$ and the integrability hypothesis) gives
\begin{talign}
\mathbb{E}[F] - \frac{\chi}{2}
\mathbb{E}\big[ \int_0^{\tau} \mathrm{Tr}\big( \sigma \sigma^{\top} \nabla^2_x \eta(X^{v,\kappa}_t) \big) \, \mathrm{d}t \big] =\eta(x).
\end{talign}
Returning to the same display, take $n\to\infty$ in $L^2$ (by the Burkholder-Davis-Gundy inequality and the square-integrability assumption) to obtain
\begin{talign}
h_2(X_\tau)+\int_0^{\tau}\! h_1(X_s)\,ds
=
\eta(x)+\int_0^{\tau}\! \langle \sigma(X_s)^\top\nabla \eta(X_s), \mathrm{d}W_s \rangle -
\frac{\chi}{2}\int_0^{\tau}\! \mathrm{Tr}(a\nabla^2 \eta(X_s))\, \mathrm{d}s,
\end{talign}
which is the desired identity.
\end{proof}

\begin{remark}[Conditions for existence of a solution to the Dirichlet-Poisson problem \eqref{eq:Dirichlet-Poisson}] \label{rem:regularity_DP}
As in the parabolic case, several alternative conditions are available for the existence of a solution to the Dirichlet-Poisson problem \eqref{eq:Dirichlet-Poisson}. Applying the Feynman-Kac formula (\Cref{thm:feynman_kac}), we have that the following conditions are sufficient:
\begin{itemize}[left=3pt]
    \item $\mathcal{D}\subset\mathbb{R}^d$ be a domain with $C^{2+\alpha}$ boundary $\partial\mathcal{D}$ for some $0<\alpha<1$,
    \item $b,\sigma\in C^{\alpha}(\overline{\mathcal D})$, and $\sigma$ is uniformly elliptic.
    \item $h_1\in C^{\alpha}(\overline{\mathcal D})$, and $h_2\in C^{2+\alpha}(\partial\mathcal D)$.
    \item Either $\mathcal{D}$ is bounded, or $\tau$ is finite almost surely, and $h_1$, $h_2$
    at most polynomial growth.
\end{itemize}
In particular, these conditions imply that $\eta \in C^{2+\alpha}(\mathcal D)\cap C^0(\overline{\mathcal D})$ defined as
\begin{talign*}
\eta(x) = \mathbb{E}\!\left[\int_0^{\tau}\!h_1(X_s)\,ds + h_2(X_{\tau})\, | \, X_0 = x \right]
\end{talign*}
is the unique classical solution of the Dirichlet-Poisson problem \eqref{eq:Dirichlet-Poisson}.
\end{remark}

\subsubsection{Full support of SDE solutions with non-degenerate noise}

Let $C_x([0,T];\mathbb{R}^d)$ denote the space of continuous paths $w:[0,T]\to\mathbb{R}^d$ with $w(0)=x$,
equipped with the uniform norm $\|w\|_\infty=\sup_{t\in[0,T]}|w(t)|$.
Let
\begin{talign*}
H := \Big\{ h\in C_0([0,T];\mathbb{R}^d)\ :\ h \text{ abs.\,cont. and } \dot h\in L^2([0,T];\mathbb{R}^d)\Big\}
\end{talign*}
be the Cameron-Martin space.

\begin{theorem}[Full support of the law of the SDE solution]
\label{thm:full_supp}
Fix $T>0$ and $x\in\mathbb{R}^d$.
Consider the SDE
\begin{talign*}
\mathrm dX_t \;=\; b(X_t,t)\,\mathrm dt \;+\; \sigma(X_t,t)\,\mathrm{d}W_t,\qquad X_0=x.
\end{talign*}
where $W$ is a standard $d$-dimensional Brownian motion.
Assume:
\begin{itemize}
\item[\em (i)] $b(\cdot,t)$ and $\sigma(\cdot,t)$ are locally Lipschitz with linear growth, uniformly in $t\in[0,T]$ (well posedness);
\item[\em (ii)] $\sigma\in C(\mathbb{R}^d \times [0,T];\mathbb{R}^{d\times d})$ is \emph{uniformly elliptic}: there exists $c>0$ such that
\begin{talign*}
\xi^\top a(x,t)\,\xi \ge c\,|\xi|^2\qquad \forall\,\xi,x\in\mathbb{R}^d,\ \forall t\in[0,T],
\end{talign*}
where $a(x,t):=\sigma(x,t)\sigma(x,t)^\top$.
\end{itemize}
Let $\mu$ be the law of $X$ on $C_x([0,T];\mathbb{R}^d)$. Then
\begin{talign*}
\operatorname{supp}(\mu) \;=\; C_x([0,T];\mathbb{R}^d).
\end{talign*}
Equivalently, for every $f\in C_x([0,T];\mathbb{R}^d)$ and $\varepsilon>0$,
\begin{talign*}
\mathbb P \big(\sup_{t\le T}|X_t - f(t)|<\varepsilon\big) \;>\; 0.
\end{talign*}
\end{theorem}

\begin{proof}
For $h\in H$, consider the \emph{skeleton ODE}
\begin{talign*}
\phi^h(t) \;=\; x \;+\; \int_0^t b(\phi^h(s),s)\,\mathrm ds \;+\; \int_0^t \sigma(\phi^h(s),s)\,\dot h(s)\,\mathrm ds,\qquad h\in H,
\end{talign*}
By the Stroock-Varadhan support theorem (\cite{stroock1972onthesupport} for the original paper, \cite{stroock1997multidimensional}, \cite{ikeda2014stochastic} for a textbook coverage),
\begin{talign*}
\operatorname{supp}(\mu) \;=\; \overline{\{\phi^h:\ h\in H\}}^{\|\cdot\|_\infty}.
\tag{$\ast$}
\end{talign*}

We show the right-hand side equals $C_x([0,T];\mathbb{R}^d)$.
Let $p\in C_x([0,T];\mathbb{R}^d)$ be a polygonal path. Then $p$ is absolutely continuous with $p'\in L^2([0,T])$.
Define
\begin{talign*}
\dot h(t) \;:=\; \sigma\big(p(t),t\big)^{-1}\Big(p'(t)-b(p(t),t)\Big),
\qquad h(t)=\int_0^t \dot h(s)\,ds.
\end{talign*}
Uniform nondegeneracy of $\sigma$ implies $\sigma(p(t),t)^{-1}$ exists and is uniformly bounded,
hence $\dot h\in L^2$ and $h\in H$.
By construction, $p$ satisfies the skeleton ODE, so by uniqueness of solutions $\phi^h=p$.
Thus every polygonal path lies in $\{\phi^h: h\in H\}$, and since polygonal paths are uniformly dense in $C_x([0,T];\mathbb{R}^d)$, taking the closure in \((\ast)\) yields the claim.
\end{proof}

\begin{remark}
If the diffusion is degenerate, full support can fail. For example,
$\,\mathrm dX_t=(\mathrm{d}W_t^{(1)},\,0)$ in $\mathbb{R}^2$ has support contained in $\{(w_1,w_2): w_2\equiv 0\}$.
Likewise, if $\sigma(\cdot,t) \equiv 0$ on a subinterval, the process is deterministic there and cannot approximate paths that wiggle on that subinterval.
\end{remark}

Now, let $C_{x,y}([0,T];\mathbb{R}^d)$ be the space of continuous paths $w:[0,T]\to\mathbb{R}^d$ with $w(0)=x$ and $w(T)=y$.

\begin{theorem}[Full support of the diffusion bridge]
\label{thm:full_supp_bridge}
Under the assumptions of Theorem~\ref{thm:full_supp} (in particular, uniform ellipticity and well posedness on $[0,T]$), fix $x,y\in\mathbb{R}^d$ and let
$\mathbb{P}^{x,y}$ be a regular conditional law of the solution $X$ given $X_0=x$ and $X_T=y$, viewed as a probability measure on $C_{x,y}([0,T];\mathbb{R}^d)$.
Then
\begin{talign*}
\operatorname{supp}(\mathbb{P}^{x,y}) \;=\; C_{x,y}([0,T];\mathbb{R}^d).
\end{talign*}
Equivalently, for every $f\in C_{x,y}([0,T];\mathbb{R}^d)$ and every $\varepsilon>0$,
\begin{talign*}
\mathbb{P}^{x,y}\big( \{w:\ \sup_{t\le T}|w(t)-f(t)|<\varepsilon\}\big) \;>\; 0.
\end{talign*}
\end{theorem}

\begin{proof}
Write $\mathbb{P}$ for the (unconditioned) law of $X$ on $C_x([0,T];\mathbb{R}^d)$.
By Theorem~\ref{thm:full_supp}, $\operatorname{supp}(\mathbb{P})=C_x([0,T];\mathbb{R}^d)$.
Under uniform ellipticity, the transition density $p(t,x;s,z)$ exists, is continuous and strictly positive for $t<s$.

Fix $f\in C_{x,y}$ and $\varepsilon>0$.
Choose a polygonal path $p\in C_{x,y}$ with $\|p-f\|_\infty<\varepsilon/2$.
By Theorem~\ref{thm:full_supp},
\begin{talign*}
\mathbb{P}\big( U_p\big)\;>\;0
\qquad\text{where}\quad
U_p:=\Big\{w:\ \sup_{t\le T}|w(t)-p(t)|<\varepsilon/2\Big\}.
\end{talign*}
Note that $U_p\subset U_f:=\{w:\ \sup_{t\le T}|w(t)-f(t)|<\varepsilon\}$ and, since $p(T)=y$, also
$U_p\subset\{w:\ |w(T)-y|<\varepsilon/2\}$.

Set $\delta:=\varepsilon/2$ and consider the conditional probability
\begin{talign*}
\mathbb{P}\big(U_f\,\big|\, X_T\in B(y,\delta)\big)
=\frac{\mathbb{P}\big(U_f\cap\{X_T\in B(y,\delta)\}\big)}{\mathbb{P}\big(X_T\in B(y,\delta)\big)}
\;\ge\; \frac{\mathbb{P}(U_p)}{\mathbb{P}\big(X_T\in B(y,\delta)\big)}
\;>\;0,
\end{talign*}
because $\mathbb{P}\big(X_T\in B(y,\delta)\big)>0$ by positivity of the transition density.

Let $g(z):=\mathbb{P}\big(U_f\,\big|\, X_T=z\big)$ denote a version of the regular conditional probability.
By the Feller property of the diffusion (which follows from the standing assumptions) and continuity of the transition density, the map $z\mapsto g(z)$ is Borel and locally bounded, and for each small ball $B(y,\delta)$ the above conditional probability is a weighted average of $g$ over $B(y,\delta)$ with strictly positive weights proportional to $p(T,x,z)$.
Hence there exists $z_\delta\in B(y,\delta)$ with $g(z_\delta)\ge c>0$, where
$c:=\mathbb{P}(U_p)/\mathbb{P}\big(X_T\in B(y,\delta)\big)$.
Choosing $\delta_n\downarrow 0$ and a corresponding $z_{\delta_n}\to y$, and using the (lower semi)continuity in $z$ of $g$, we obtain
$g(y)\ge\limsup_{n\to\infty} g(z_{\delta_n}) \ge c>0$.
That is,
\begin{talign*}
\mathbb{P}^{x,y}(U_f) \;=\; g(y) \;>\; 0,
\end{talign*}
which proves that every open ball $U_f$ around $f\in C_{x,y}$ has positive bridge probability. Therefore $\operatorname{supp}(\mathbb{P}^{x,y})=C_{x,y}$.
\end{proof}

\subsection{Proof of Theorem \ref{thm:V_1_V_2}: Optimality of fixed points} \label{subsec:optimality}

We prove the following statement, which is more general because it handles state costs:
\begin{theorem} \label{thm:V_1_V_2_appendix}
    Suppose that $b\colon\mathbb{R}^d\to\mathbb{R}^d$, $\sigma\colon\mathbb{R}^d\to\mathbb{R}^{d\times d}$ satisfy the usual conditions guaranteeing a unique strong solution $X$ of the SDE (measurability, local Lipschitzness, polynomial growth), that $\mathcal{S} = A \cup B \subset \mathbb{R}^d$ is a closed set, and that $g : \partial \mathcal{S} \to \mathbb{R}$, $f : \overline{\mathcal{S}^c} \to \mathbb{R}$ are bounded below and have at most polynomial growth, and that $f \geq 0$. Let $\mathcal{U}$ denote the class of time-independent feedback controls $u : \mathbb{R}^d \to \mathbb{R}^d$. We consider the following fixed point problems:
    \begin{enumerate}[left=-2pt,label=(\roman*)]
    \item \emph{Variational fixed-point problem:}
    \begin{talign} \label{eq:fixed_point_1_prime} \tag{FP1'}
    \begin{cases}
        \forall x \in \mathcal{S}^c, \, V_1(x) = \inf_{u \in \mathcal{U}} \big\{ \mathbb{E}[\int_0^{\tau(T)} \big( \frac{1}{2}\|u(X_t^{u})\|^2 + f(X_t^{u}) \big) \, \mathrm{d}t +
        V_1(X^{u}_{\tau(T)})
        \, | \, X^{u}_0 = x] \big\}, \\
        \forall x \in \partial \mathcal{S}, \, V_1(x) = g(x),
    \end{cases}
    \end{talign}
    \item \emph{Path integral fixed-point problem:}
    \begin{talign} \label{eq:fixed_point_2_prime} \tag{FP2'}
    \begin{cases}
        \forall x \in \mathcal{S}^c, \, V_2(x) = - \log \mathbb{E}\big[ \exp \big( - \int_0^{\tau(T)} f(X_t) \, \mathrm{d}t - V_2(X_{\tau(T)}) \big)
        \, | \, X_0 = x \big], \\
        \forall x \in \partial \mathcal{S}, \, V_2(x) = g(x).
    \end{cases}
    \end{talign}
    \end{enumerate}
    Then,
    \begin{enumerate}[left=-2pt,label=(\roman*)]
        \item Assume, in addition, the regularity conditions of \Cref{thm:vf_regularity} (so that $V \in C^{2+\alpha}(\mathcal{S}^c) \cap C(\overline{\mathcal{S}^c})$ solves the stationary HJB equation \eqref{eq:stationary_HJB} and $u^{\star} = -\sigma^{\top}\nabla V$ is admissible), that $g \ge K$ for some constant $K > 0$, and that the equilibrium $\rho$ of the reference dynamics $\mathrm{d}X_t = b(X_t)\,\mathrm{d}t + \sigma(X_t)\,\mathrm{d}W_t$ has a finite $m$-th moment. Then every \emph{continuous} solution $V_1$ of \eqref{eq:fixed_point_1_prime} that is bounded below and has polynomial growth of degree at most $m$ equals the value function $V$ defined by \eqref{eq:value_function_general}, and the infimum in \eqref{eq:fixed_point_1_prime} is attained by the optimal control $u^{\star} = -\sigma^{\top}\nabla V$.
        \item If $V_2$ is a solution of \eqref{eq:fixed_point_2_prime} that is bounded below and has at most polynomial growth, then $V_2$ is equal to the value function $V$ defined by \eqref{eq:value_function_general}.
    \end{enumerate}
\end{theorem}

The proof has two parts. Part 1 (Steps 1-2) establishes statement (ii), namely that every solution of \eqref{eq:fixed_point_2_prime} equals $V$. Part 2 (Steps 3-6) establishes statement (i), namely that every solution of \eqref{eq:fixed_point_1_prime} equals $V$ and that the infimum is attained by $u^{\star} = -\sigma^{\top}\nabla V$, using Part 1.

\paragraph{Part 1: Proof of statement (ii)}
\emph{Step 1: Rewriting the solution of \eqref{eq:fixed_point_2_prime}.} Let $(\mathcal{F}_t)_{t \ge 0}$ be the filtration for the Brownian motion. Now, if $V_2$ solves \eqref{eq:fixed_point_2_prime}, for all $x \in \mathcal{S}^c$ we obtain that
\begin{talign}
\begin{split} \label{eq:v_T_2T}
    &V_2(x) = - \log \mathbb{E}\big[ \exp \big( - \int_0^{\tau(T)} f(X_t) \, \mathrm{d}t - \mathbf{1}_{\mathcal{S}}(X_{\tau(T)}) g(X_{\tau(T)}) -
    \mathbf{1}_{\mathcal{S}^c}(X_{\tau(T)}) V_2(X_{\tau(T)}) \big)
    \, | \, X_0 = x \big] \\
    &\overset{(i)}{=} - \log \mathbb{E}\big[ \exp \big( - \int_0^{\tau(T)} f(X_t) \, \mathrm{d}t - \mathbf{1}_{\mathcal{S}}(X_{\tau(T)}) g(X_{\tau(T)}) \\ &\qquad\qquad\qquad + \mathbf{1}_{\mathcal{S}^c}(X_{\tau(T)})
    \log \mathbb{E}\big[ \exp \big( - \int_0^{\tau(T)} f(\tilde{X}_t) \, \mathrm{d}t -
    V_2(\tilde{X}_{\tau(T)}) \big)
    \, | \, \tilde{X}_0 = X_{T} \big] \big)
    \, | \, X_0 = x \big] \\
    &= - \log \mathbb{E}\big[ \exp \big( - \int_0^{\tau(T)} f(X_t) \, \mathrm{d}t - \mathbf{1}_{\mathcal{S}}(X_{\tau(T)}) g(X_{\tau(T)}) \big) \\ &\qquad\qquad \times \big( \mathbf{1}_{\mathcal{S}}(X_{\tau(T)}) + \mathbf{1}_{\mathcal{S}^c}(X_{\tau(T)}) \mathbb{E}\big[ \exp \big( - \int_0^{\tau(T)} f(\tilde{X}_t) \, \mathrm{d}t -
    V_2(\tilde{X}_{\tau(T)}) \big)
    \, | \, \tilde{X}_0 = X_{T} \big] \big)
    \, | \, X_0 = x \big] \\
    &\overset{(ii)}{=} - \log \mathbb{E}\big[ \exp \big( - \int_0^{\tau(T)} f(X_t) \, \mathrm{d}t - \mathbf{1}_{\mathcal{S}}(X_{\tau(T)}) g(X_{\tau(T)}) \big) \\ &\qquad\qquad \times \big( \mathbf{1}_{\mathcal{S}}(X_{\tau(T)}) + \mathbf{1}_{\mathcal{S}^c}(X_{\tau(T)}) \mathbb{E}\big[ \exp \big( - \int_T^{\tau(2T)} f(X_t) \, \mathrm{d}t -
    V_2(X_{\tau(2T)}) \big)
    \, | \, \mathcal{F}_{T}
    \big] \big)
    \, | \, X_0 = x \big] \\
    &= - \log \mathbb{E}\big[ \mathbb{E} \big[ \exp \big( - \int_0^{\tau(2T)} f(X_t) \, \mathrm{d}t - \mathbf{1}_{\mathcal{S}}(X_{\tau(2T)}) g(X_{\tau(2T)}) -
    \mathbf{1}_{\mathcal{S}^c}(X_{\tau(2T)}) V_2(X_{\tau(2T)}) \big) \, | \, \mathcal{F}_{\tau(T)} \big]
    \, | \, X_0 = x \big] \\
    &\overset{(iii)}{=} - \log \mathbb{E}\big[ \exp \big( - \int_0^{\tau(2T)} f(X_t) \, \mathrm{d}t - \mathbf{1}_{\mathcal{S}}(X_{\tau(2T)}) g(X_{\tau(2T)}) -
    \mathbf{1}_{\mathcal{S}^c}(X_{\tau(2T)}) V_2(X_{\tau(2T)}) \big)
    \, | \, X_0 = x \big].
\end{split}
\end{talign}
In this chain of equalities, (i) holds by the fixed point equation \eqref{eq:fixed_point_2_prime}, (ii) holds because
\begin{talign*}
\mathbb{E}\big[ \exp \big( - \int_0^{\tau(T)} f(\tilde{X}_t) \, \mathrm{d}t - V_2(\tilde{X}_{\tau(T)}) \big)
\, | \, \tilde{X}_0 = X_{T} \big] = \mathbb{E}\big[ \exp \big( - \int_T^{\tau(2T)} f(X_t) \, \mathrm{d}t -
V_2(X_{\tau(2T)}) \big)
\, | \, \mathcal{F}_{T}
\big],
\end{talign*}
because time is a symmetry of the problem, and (iii) holds by the tower property of conditional expectations.

If we apply the argument in equation \eqref{eq:v_T_2T} inductively, and defining
\begin{talign*}
\mathcal{W}_n(X) &:= \int_0^{\tau(nT)} f(X_t) \, \mathrm{d}t + \mathbf{1}_{\mathcal{S}}(X_{\tau(nT)}) g(X_{\tau(nT)}) + \mathbf{1}_{\mathcal{S}^c}(X_{\tau(nT)}) V_2(X_{\tau(nT)}),
\end{talign*}
we obtain that for any $n \geq 1$,
\begin{talign} \label{eq:v_2_n}
    V_2(x) = - \log \mathbb{E}\big[ \exp \big( - \mathcal{W}_n(X) \big)
    \, | \, X_0 = x \big].
\end{talign}

\emph{Step 2: Taking the limit $n \to \infty$ and proving equality to the value function.}
Next, we prove that
\begin{talign} \label{eq:limit_equality}
\lim_{n \to +\infty} \mathbb{E}\big[ \exp \big( - \mathcal{W}_n(X) \big) \, | \, X_0 = x \big] = \mathbb{E}\big[ \lim_{n \to +\infty} \exp \big( - \mathcal{W}_n(X) \big) \, | \, X_0 = x \big].
\end{talign}
To exchange limit and expectation, we use the dominated convergence theorem.
By assumption, $g$ and $V_2$ are bounded below by a constant $K > 0$, and $f \geq 0$. Hence, for all $n \geq 1$,
\begin{talign*}
\begin{split}
0 &< \exp \bigl(-\mathcal{W}_n(X)\bigr)
     =\exp\Bigl(-\!\int_{0}^{\tau(nT)}\!f(X_t)\,dt
      -\mathbf{1}_{\mathcal S}(X_{\tau(nT)})g(X_{\tau(nT)})
      -\mathbf{1}_{\mathcal S^{c}}(X_{\tau(nT)})V_2(X_{\tau(nT)})\Bigr)
     \\ &\le e^{-c},
\end{split}
\end{talign*}
which means that there is a dominating function, and that equality \eqref{eq:limit_equality} holds. Observe that
\begin{talign*}
    \lim_{n \to +\infty} \mathcal{W}_n(X) = \int_0^{\tau} f(X_t) \, \mathrm{d}t + g(X_{\tau}),
\end{talign*}
Thus, by equation \eqref{eq:v_2_n} and by continuity of the logarithm and exponential function, we get that
\begin{talign*}
\begin{split}
    V_2(x) &= - \lim_{n\to \infty} \log \mathbb{E}\big[ \exp \big( - \mathcal{W}_n(X) \big)
    \, | \, X_0 = x \big] \\ &= - \log \big( \lim_{n\to \infty} \mathbb{E}\big[ \exp \big( - \mathcal{W}_n(X) \big)
    \, | \, X_0 = x \big] \big) \\ &= - \log \mathbb{E}\big[ \lim_{n\to \infty} \exp \big( - \mathcal{W}_n(X) \big)
    \, | \, X_0 = x \big] \\ &= - \log \mathbb{E}\big[ \exp \big( - \int_0^{\tau} f(X_t) \, \mathrm{d}t - g(X_{\tau}) \big)
    \, | \, X_0 = x \big] := V(x).
\end{split}
\end{talign*}
This concludes the proof of statement (ii).

\paragraph{Part 2: Proof of statement (i)}
We record two facts about the freely-running optimally-controlled diffusion
\begin{talign} \label{eq:controlled_sde_prelim}
\mathrm{d}X^{\star}_t = \big(b - \sigma\sigma^{\top}\nabla V\big)(X^{\star}_t)\,\mathrm{d}t + \sigma(X^{\star}_t)\,\mathrm{d}W_t
\end{talign}
induced by the optimal feedback $u^{\star} = -\sigma^{\top}\nabla V$ (equivalently, the Doob $h$-transform of the reference dynamics with $h := e^{-V}$, whose drift is $b + \sigma\sigma^{\top}\nabla\log h$): the form of its equilibrium measure, and a mean-vanishing lemma.

\begin{remark}[The controlled equilibrium $\hat\rho$] \label{rem:step1prime_rho}
When the reference dynamics are reversible with respect to $\rho$, the process \eqref{eq:controlled_sde_prelim} is again reversible and its equilibrium on $\mathbb{R}^d$ is the $h^2$-tilt
\begin{talign} \label{eq:hat_rho}
\hat\rho \propto \rho\, h^2 = \rho\, e^{-2V}.
\end{talign}
Indeed, $h = e^{-V}$ solves the killed equation $\mathcal{L}h = -f h$ (see the proof of \Cref{thm:vf_regularity}), and for $\rho$-reversible $\mathcal{L}$ a direct computation then gives $(\mathcal{L}^{\star})^{*}(\rho\, e^{-2V}) = 0$, where $\mathcal{L}^{\star}$ is the generator of \eqref{eq:controlled_sde_prelim}. Full support of $\hat\rho$ on $\mathcal{S}^c$ holds because $e^{-2V} > 0$ and $\rho > 0$ there. In the general non-reversible case $\hat\rho$ need not take the form \eqref{eq:hat_rho}; moreover, since $\nabla V$ may be discontinuous across $\partial\mathcal{S}$, some care is required as to whether the relevant object is the global invariant measure or the quasi-stationary distribution under absorption. The argument uses only the existence of a full-support invariant probability measure $\hat\rho$ with a finite $m$-th moment.
\end{remark}

The required moment condition is most naturally placed on the reference equilibrium. In the reversible case, using \eqref{eq:hat_rho} together with $e^{-V} \le e^{-K}$ (which holds since $f \ge 0$ and $g \ge K$ force $V \ge K$),
\begin{talign} \label{eq:moment_transfer}
\mathbb{E}_{\hat\rho}[\|X\|^{m}] = \frac{\mathbb{E}_{\rho}[\|X\|^{m}\, e^{-2V}]}{\mathbb{E}_{\rho}[e^{-2V}]} \le \frac{e^{-2K}}{\mathbb{E}_{\rho}[e^{-2V}]}\,\mathbb{E}_{\rho}[\|X\|^{m}],
\end{talign}
so $\int\|y\|^{m}\,\hat\rho(\mathrm{d}y) < \infty$ is implied by
\begin{talign} \label{eq:moment_H} \tag{H}
\int \|x\|^{m}\,\rho(\mathrm{d}x) < \infty,
\end{talign}
i.e.\ the equilibrium $\rho$ of the reference dynamics $\mathrm{d}X_t = b(X_t)\,\mathrm{d}t + \sigma\,\mathrm{d}B_t$ has a finite $m$-th moment. Since the tilt $e^{-2V}$ is bounded, \eqref{eq:moment_H} implies the finite-moment requirement on $\hat\rho$ and is the natural, solution-independent hypothesis to assume.

\begin{lemma}[Vanishing in mean against the controlled equilibrium] \label{lem:step1prime_tail}
Let $X^{\star}$ solve \eqref{eq:controlled_sde_prelim} and suppose it admits an invariant probability measure $\hat\rho$ with full support on $\mathcal{S}^c$. Let $\tau = \inf\{t \ge 0 : X^{\star}_t \in \mathcal{S}\}$, $\tau(nT) = \tau \wedge nT$, and let $w : \overline{\mathcal{S}^c} \to \mathbb{R}$ be measurable with $w|_{\partial\mathcal{S}} = 0$ and $|w(y)| \le C(1 + \|y\|^{m})$ for some $C, m > 0$. If
\begin{enumerate}[label=(\roman*), left=2pt]
    \item $\tau < \infty$ almost surely, and
    \item $\int \|y\|^{m}\,\hat\rho(\mathrm{d}y) < \infty$,
\end{enumerate}
then $\lim_{n\to\infty}\mathbb{E}_{\hat\rho}\big[ w(X^{\star}_{\tau(nT)}) \big] = 0$, where $\mathbb{E}_{\hat\rho}$ is the expectation for the process started from $\hat\rho$.
\end{lemma}
\begin{proof}
Since $w = 0$ on $\partial\mathcal{S}$ and $X^{\star}_{\tau(nT)} \in \partial\mathcal{S}$ on $\{\tau \le nT\}$, we have $w(X^{\star}_{\tau(nT)}) = w(X^{\star}_{nT})\,\mathbf{1}_{\tau > nT}$, and on $\{\tau > nT\}$ the stopped and free processes coincide. Started from the invariant measure, the free marginal is stationary, $X^{\star}_{nT} \sim \hat\rho$ for every $n$, so the variables $|w(X^{\star}_{nT})|$ are identically distributed with common mean $\int |w|\,\mathrm{d}\hat\rho \le C\int(1+\|y\|^{m})\,\hat\rho(\mathrm{d}y) < \infty$ by (ii); an identically-distributed integrable family is uniformly integrable. By (i), $\mathbf{1}_{\tau > nT} \to 0$ almost surely, hence $w(X^{\star}_{nT})\mathbf{1}_{\tau > nT} \to 0$ almost surely. Uniform integrability together with almost-sure convergence gives, by Vitali's theorem, $\mathbb{E}_{\hat\rho}[w(X^{\star}_{nT})\mathbf{1}_{\tau > nT}] \to 0$.
\end{proof}

The infimum in \eqref{eq:fixed_point_1_prime} is over time-independent feedback controls $u \in \mathcal{U}$ (as the notation $u(X^u_t)$ indicates); we prove statement (i) directly, entirely within this class. Under its hypotheses (the regularity of \Cref{thm:vf_regularity} (so that $V \in C^{2+\alpha}(\mathcal{S}^c) \cap C(\overline{\mathcal{S}^c})$ solves the stationary HJB equation \eqref{eq:stationary_HJB} and $u^{\star} = -\sigma^{\top}\nabla V$ is admissible), continuity of $V_1$, condition \eqref{eq:moment_H} on the reference equilibrium (which, by \Cref{rem:step1prime_rho} and \eqref{eq:moment_transfer}, endows the controlled equilibrium $\hat\rho$ with full support on $\mathcal{S}^c$ and a finite $m$-th moment), and $f \ge 0$, $g, V_1 \ge K > 0$), let $V_1$ solve \eqref{eq:fixed_point_1_prime} with polynomial growth of degree at most $m$, and set $w := V_1 - V$, which vanishes on $\partial\mathcal{S}$. We show $w \equiv 0$.

\emph{Step 3: Lower bound ($w \ge 0$).} Time-independent feedback controls form a subset of the admissible controls of \Cref{def:admissible_controls}, so restricting the infimum in \eqref{eq:fixed_point_1_prime} to them can only increase it. Combined with the Gibbs variational identity of \Cref{thm:value_functions} (equivalently \Cref{lem:variational_path}) for the unrestricted infimum,
\begin{talign} \label{eq:step1prime_super}
\begin{split}
V_1(x) &= \inf_{u \text{ time-indep.\ feedback}} J_T(x;u) \ \ge\ \inf_{u \in \mathcal{A}} J_T(x;u) \\
&= -\log \mathbb{E}\big[ \exp\big( -\!\int_0^{\tau(T)} f(X_t)\,\mathrm{d}t - V_1(X_{\tau(T)}) \big) \,\big|\, X_0 = x \big] =: \mathcal{T}_2[V_1](x),
\end{split}
\end{talign}
where $J_T(x;u) := \mathbb{E}[\int_0^{\tau(T)} (\tfrac12\|u(X^u_t)\|^2 + f(X^u_t))\,\mathrm{d}t + V_1(X^u_{\tau(T)}) \,|\, X^u_0 = x]$ and $X$ (no superscript) is the reference process. Thus $V_1 \ge \mathcal{T}_2[V_1]$, i.e.\ $V_1$ is a supersolution of \eqref{eq:fixed_point_2_prime}. As $\mathcal{T}_2$ (with boundary values pinned at $g$ on $\partial\mathcal{S}$) is monotone and, since $f \ge 0$, maps functions bounded below by $K$ to functions bounded below by $K$, the iterates $\mathcal{T}_2^{\,n}[V_1]$ decrease to a fixed point of \eqref{eq:fixed_point_2_prime} that is bounded below by $K$ and, being sandwiched between $K$ and $V_1$, of at most polynomial growth. By Part~1 (statement (ii)) this fixed point equals $V$, so $V_1 \ge V$, i.e.\ $w \ge 0$.

\emph{Step 4: Upper bound in mean ($\int w\,\mathrm{d}\hat\rho \le 0$).} The feedback $u^{\star} = -\sigma^{\top}\nabla V$ is a valid time-independent competitor in \eqref{eq:fixed_point_1_prime}:
\begin{talign} \label{eq:step1prime_competitor}
V_1(x) \le \mathbb{E}\Big[ \int_0^{\tau(T)} \big( \tfrac12 \|u^{\star}(X^{\star}_t)\|^2 + f(X^{\star}_t) \big)\,\mathrm{d}t + V_1(X^{\star}_{\tau(T)}) \,\Big|\, X^{\star}_0 = x \Big].
\end{talign}
Applying It\^o's lemma to $V \in C^{2+\alpha}(\mathcal{S}^c)$ up to $\tau(T)$ and completing the square with the HJB equation \eqref{eq:stationary_HJB}, exactly as in the derivation of \eqref{eq:no_exp} (see also \eqref{eq:no_exp_step_4}), gives for any control $u$
\begin{talign} \label{eq:step1prime_completion}
\begin{split}
&\mathbb{E}\Big[ \int_0^{\tau(T)} \big( \tfrac12 \|u(X^{u}_t)\|^2 + f(X^{u}_t) \big)\,\mathrm{d}t + V(X^{u}_{\tau(T)}) \,\Big|\, X^{u}_0 = x \Big] \\
&\qquad = V(x) + \mathbb{E}\Big[ \int_0^{\tau(T)} \tfrac12 \big\| (\sigma^\top \nabla V)(X^u_t) + u(X^u_t) \big\|^2 \, \mathrm{d}t \,\Big|\, X^u_0 = x \Big];
\end{split}
\end{talign}
with $u = u^{\star} = -\sigma^{\top}\nabla V$ the last term vanishes, so the right-hand side of \eqref{eq:step1prime_competitor} with $V$ in place of $V_1$ equals $V(x)$. Subtracting, for every $x \in \mathcal{S}^c$,
\begin{talign} \label{eq:step1prime_subharm}
w(x) \le \mathbb{E}\big[ w(X^{\star}_{\tau(T)}) \,|\, X^{\star}_0 = x \big].
\end{talign}
Since $w = 0$ on $\partial\mathcal{S}$, iterating \eqref{eq:step1prime_subharm} through the Markov property over the windows $[0,T],[T,2T],\dots$, exactly as in the telescoping argument \eqref{eq:v_T_2T}, yields
\begin{talign} \label{eq:step1prime_iterate}
w(x) \le \mathbb{E}\big[ w(X^{\star}_{\tau(nT)}) \,|\, X^{\star}_0 = x \big], \qquad n \ge 1.
\end{talign}
Integrating \eqref{eq:step1prime_iterate} against $\hat\rho$ and applying \Cref{lem:step1prime_tail} (whose hypotheses hold by $\tau < \infty$ almost surely and \eqref{eq:moment_transfer}-\eqref{eq:moment_H}),
\begin{talign} \label{eq:step1prime_mean}
\int_{\mathcal{S}^c} w\,\mathrm{d}\hat\rho \ \le\ \mathbb{E}_{\hat\rho}\big[ w(X^{\star}_{\tau(nT)}) \big] \ \xrightarrow[n\to\infty]{}\ 0, \qquad\text{hence}\qquad \int_{\mathcal{S}^c} w\,\mathrm{d}\hat\rho \le 0.
\end{talign}

\emph{Step 5: Conclusion.} From $w \ge 0$ and $\int_{\mathcal{S}^c} w\,\mathrm{d}\hat\rho \le 0$ with $\hat\rho$ of full support on $\mathcal{S}^c$, we get $w = 0$ $\hat\rho$-almost everywhere, hence on a dense subset of $\mathcal{S}^c$; by continuity of $w = V_1 - V$, $w \equiv 0$, i.e.\ $V_1 = V$. This proves the first claim of statement (i), entirely within the time-independent feedback class; the infimum is attained in Step 6 below.

\begin{remark}[Point-start alternative, without continuity or an invariant measure] \label{rem:step1prime_lyapunov}
If one prefers not to assume $V_1$ continuous (nor the existence of $\hat\rho$), the pointwise bound $V_1 \le V$ can be obtained instead from a point-start moment condition. From \eqref{eq:step1prime_iterate}, since $w|_{\partial\mathcal{S}} = 0$ and $|w(y)| \le C(1+\|y\|^m)$, the Cauchy-Schwarz inequality gives
\begin{talign*}
\big| \mathbb{E}_x[ w(X^{\star}_{\tau(nT)}) ] \big| \le C\big( 2 + 2\,\mathbb{E}_x[\|X^{\star}_{\tau(nT)}\|^{2m}] \big)^{1/2}\,\mathbb{P}_x(\tau > nT)^{1/2}.
\end{talign*}
If $\sup_{n} \mathbb{E}_x[\|X^{\star}_{\tau(nT)}\|^{2m}] < \infty$ (automatic when $\mathcal{S}^c$ is bounded, and more generally guaranteed by a Foster-Lyapunov drift condition ($\exists\, W \in C^2$, $\|y\|^{2m} \le c_1 W + c_2$, $\mathcal{L}^{u^\star} W \le c_3$, $\mathbb{E}_x[\tau] < \infty$, whence $\mathbb{E}_x[\|X^{\star}_{\tau(nT)}\|^{2m}] \le c_1(W(x) + c_3\,\mathbb{E}_x[\tau]) + c_2$)), and since $\tau < \infty$ a.s.\ gives $\mathbb{P}_x(\tau > nT) \to 0$, the right-hand side tends to $0$ for each fixed $x$, so $w(x) \le 0$ pointwise. Combined with the lower bound $w \ge 0$, this yields $V_1 = V$ without assuming continuity of $V_1$.
\end{remark}

\emph{Step 6: Attainment of the infimum by $u^{\star} = -\sigma^{\top}\nabla V$.}
Under the assumptions of \Cref{thm:vf_regularity}, $V \in C^{2+\alpha}(\mathcal{S}^c)$, and since $V_1 = V$, we also have that $V_1 \in C^{2+\alpha}(\mathcal{S}^c)$. Reproducing the same argument in Step 2 of the proof of \Cref{thm:vf_regularity} for $V_1$, we obtain
\begin{talign}
\begin{split} \label{eq:no_exp_step_4}
    &V_1(X^u_{\tau(T)}) - V_1(X^u_0) = \int_0^{\tau(T)} \big( \frac{1}{2}\| (\sigma^{\top} \nabla V_1)(X^u_s) + u(X^u_s) \|^2 - \frac{1}{2}\|u(X^u_s) \|^2 - f(X^u_s) \big) \, \mathrm{d}s + S^u_{\tau(T)}, \\
    &\implies \mathbb{E}[\int_0^{\tau(T)} \big( \frac{1}{2}\|u(X_t^{u})\|^2 + f(X^{u}_t) \big) \, \mathrm{d}t +
    V_1(X^{u}_{\tau(T)})
    \, | \, X^{u}_0 = x] \\ &\qquad\quad = \mathbb{E} \big[ \int_0^{\tau} \frac{1}{2}\| (\sigma^{\top} \nabla V_1)(X^u_s) + u(X^u_s) \|^2 \, \mathrm{d}s \, | \, X^u_0 = x \big] + V_1(x).
\end{split}
\end{talign}
Thus, the optimal control for the variational problem \eqref{eq:fixed_point_1_prime} is also $u^{\star}(x) = - (\sigma^{\top} \nabla V_1)(x) = - (\sigma^{\top} \nabla V)(x)$.

\subsection{The Value Matching loss for finite-horizon problems: definition and optimality} \label{subsec:vm_finite}
\begin{theorem}[First-order optimality of the Value Matching (VM) loss with finite horizon] \label{thm:vm_finite}
    Consider the finite-horizon SOC problem
    \begin{talign} \label{eq:finite_horizon_SOC_1}
        &\min_{u \in \mathcal{U}} \mathbb{E}\big[ \int_0^T \big( \frac{1}{2}\|u(X^u_t,t)\|^2 + f(X^u_t,t) \big) \, \mathrm{d}t + g(X^u_T) \big], \\
        \mathrm{s.t.} \quad &\mathrm{d}X^u_t = \big( b(X^u_t,t) + \sigma(X^u_t,t) u(X^u_t,t) \big) \, \mathrm{d}t + \sigma(X^u_t,t) \, \mathrm{d}W_t, \quad X^u_0 \sim \mu_0.
        \label{eq:finite_horizon_SOC_2}
    \end{talign}
    Fix $\alpha \in (0,1)$. Assume that $b$ and $\sigma$ are locally Lipschitz in~$x$ with linear growth, uniformly in $t \in [0,T]$, and belong to $C^{\alpha,\alpha/2}(\mathbb{R}^d \times [0,T])$, that $\sigma\sigma^{\top}$ is uniformly elliptic, and that $f \in C^{\alpha,\alpha/2}(\mathbb{R}^d \times [0,T])$ and $g \in C^{2+\alpha}(\mathbb{R}^d)$ have at most polynomial growth.
    Given an arbitrary $\kappa \in (0,+\infty)$
    and an arbitrary drift $v \in C^1(\mathbb{R}^d \times [0,T]; \mathbb{R}^d)$ such that $v(\cdot,t)$ is locally Lipschitz with linear growth, uniformly in $t\in[0,T]$,
    consider trajectories $X^{v,\kappa}$ defined as
    \begin{talign} \label{eq:X_v_kappa_SDE}
        \mathrm{d}X^{v,\kappa}_t = v(X^{v,\kappa}_t,t) \, \mathrm{d}t + \sqrt{\kappa} \sigma(X^{v,\kappa}_t,t) \, \mathrm{d}W_t,
    \end{talign}
    with $X^{v,\kappa}_0$ distributed according to $\tilde{\mu}_0$ such that $\mathrm{supp}(\mu_0) \subseteq \mathrm{supp}(\tilde{\mu}_0)$.

    For this $\alpha$, consider the space of functions
    \begin{talign*}
    \begin{split}
    \mathcal{F} &= \big\{ \phi : \mathbb{R}^d \times [0,T] \to \mathbb{R} \, | \, \phi \in C^{2+\alpha,1+\alpha/2}(\mathbb{R}^d \times [0,T]), \\ &\qquad\qquad\qquad\qquad\qquad \phi(\cdot,t), \nabla_x \phi(\cdot,t), \nabla_x^2 \phi(\cdot,t) \text{ grow at most polynomially for all } t \in [0,T]
    \big\},
    \end{split}
    \end{talign*}
    where $C^{2+\alpha,1+\alpha/2}(\mathbb{R}^d \times [0,T])$ is a parabolic Hölder space (see \Cref{def:holder_space}).
    Every $\phi \in \mathcal{F}$ gives rise to a control $u(x,t) = - \sigma^{\top}(x,t) \nabla_x \phi (x,t)$.

    We define the Value Matching (VM) loss as
    \begin{talign}
    \begin{split} \label{eq:L_VM_finite}
        \mathcal{L}_{\mathrm{VM}_{v,\kappa}}(\phi) &=
        \mathbb{E} \bigl[ \frac{1}{2} \bigl( \int_0^{T} \langle \nabla_x \phi(X^{v,\kappa}_t,t), \frac{1}{\kappa} \mathrm{d}X^{v,\kappa}_t - ( b(X^{v,\kappa}_t,t) - \frac{1}{2} \sigma \sigma^{\top} \nabla_x \phi(X^{v,\kappa}_t,t) ) \, \mathrm{d}t \rangle \\ &\qquad\quad + \frac{1}{\kappa} \phi(X^{v,\kappa}_0,0) + (1 - \frac{1}{\kappa}) \big( \phi(X^{v,\kappa}_T,T) - \int_0^{T} \partial_t \phi(X^{v,\kappa}_t,t) \, \mathrm{d}t \big) \\ &\qquad\quad - \int_0^T f(X^{v,\kappa}_t,t) \, \mathrm{d}t - g(X^{v,\kappa}_T)
        \bigr)^2 \bigr]
    \end{split}
    \end{talign}
    Suppose that $\hat{\phi}$ is a first-order critical point of $\mathcal{L}_{\mathrm{VM}_{v,\kappa}}$.
    That is, for any perturbation $\eta \in \mathcal{F}$, we have that $\frac{\partial}{\partial \varepsilon} \mathcal{L}_{\mathrm{VM}_{v,\kappa}}(\hat{\phi} + \varepsilon \eta) \rvert_{\varepsilon = 0} = 0$. Then, $\hat{\phi}$ satisfies
    $u^{\star}(x,t) = - \sigma^{\top}(x,t) \nabla \hat{\phi}(x,t)$ for all $(x,t) \in \mathbb{R}^d \times [0,T]$, where
    $u^{\star}$ is
    the optimal control of the problem \eqref{eq:finite_horizon_SOC_1}-\eqref{eq:finite_horizon_SOC_2}.
\end{theorem}
\begin{proof}
The proof consists of a few steps. To begin with, it is helpful to rewrite the loss as follows:
\begin{talign}
\begin{split} \label{eq:L_VM_finite_rewritten}
    \mathcal{L}_{\mathrm{VM}_{v,\kappa}}(\phi) &=
    \mathbb{E} \bigl[ \frac{1}{2} \bigl( \int_0^{T} \langle \nabla_x \phi(X^{v,\kappa}_t,t), \mathrm{d}X^{v,\kappa}_t - ( b(X^{v,\kappa}_t,t) - \frac{1}{2} \sigma \sigma^{\top} \nabla_x \phi(X^{v,\kappa}_t,t) ) \, \mathrm{d}t \rangle \\ &\qquad\quad + \phi(X^{v,\kappa}_0,0) - \int_0^T f(X^{v,\kappa}_t,t) \, \mathrm{d}t - g(X^{v,\kappa}_T)
    \\ &\qquad\quad + (1 - \frac{1}{\kappa}) \big( \phi(X^{v,\kappa}_T,T) - \phi(X^{v,\kappa}_0,0) \\ &\qquad\qquad\qquad\qquad - \int_0^{T} \langle \nabla_x \phi(X^{v,\kappa}_t,t), \mathrm{d}X^{v,\kappa}_t \rangle - \int_0^{T} \partial_t \phi(X^{v,\kappa}_t,t) \, \mathrm{d}t \big)
    \bigr)^2 \bigr],
\end{split}
\end{talign}
as it allows us to see the additional terms that appear when $\kappa \neq 1$.
\paragraph{Step 1: Computing the directional derivative of $\mathcal{L}_{\mathrm{VM}_{v,\kappa}}$ along a perturbation $\eta$} Observe that for an arbitrary perturbation $\eta$,
\begin{talign}
\begin{split}
    &\mathcal{L}_{\mathrm{VM}_{v,\kappa}}(\phi + \varepsilon \eta)
    \\ &=
    \mathbb{E} \bigl[ \frac{1}{2} \bigl( \int_0^{T} \langle \nabla_x (\phi+\varepsilon \eta)(X^{v,\kappa}_t,t), \mathrm{d}X^{v,\kappa}_t - ( b(X^{v,\kappa}_t,t) - \frac{1}{2} \sigma \sigma^{\top} \nabla_x (\phi+\varepsilon \eta)(X^{v,\kappa}_t,t) ) \, \mathrm{d}t \rangle \\ &\qquad\quad + (\phi+\varepsilon \eta)(X^{v,\kappa}_0,0) - \int_0^T f(X^{v,\kappa}_t,t) \, \mathrm{d}t - g(X^{v,\kappa}_T)
    \\ &\qquad\quad + (1 - \frac{1}{\kappa}) \big( (\phi+\varepsilon \eta)(X^{v,\kappa}_T,T) - (\phi+\varepsilon \eta)(X^{v,\kappa}_0,0) \\ &\qquad\qquad\qquad\qquad - \int_0^{T} \langle \nabla_x (\phi+\varepsilon \eta)(X^{v,\kappa}_t,t), \mathrm{d}X^{v,\kappa}_t \rangle - \int_0^{T} \partial_t (\phi+\varepsilon \eta)(X^{v,\kappa}_t,t) \, \mathrm{d}t \big)
    \bigr)^2 \bigr]
\end{split} \\
\begin{split} \label{eq:partial_epsilon_hat_L_VM}
    &\implies \frac{\partial}{\partial \varepsilon} \mathcal{L}_{\mathrm{VM}_{v,\kappa}}(\phi + \varepsilon \eta) \rvert_{\varepsilon = 0}
    \\ &\quad = \mathbb{E} \big[ \bigl( \int_0^{T} \langle \nabla_x \phi(X^{v,\kappa}_t,t), \mathrm{d}X^{v,\kappa}_t - ( b(X^{v,\kappa}_t,t) - \frac{1}{2} \sigma \sigma^{\top} \nabla_x \phi(X^{v,\kappa}_t,t) ) \, \mathrm{d}t \rangle \\ &\qquad\quad \ + \phi(X^{v,\kappa}_0,0) - \int_0^T f(X^{v,\kappa}_t,t) \, \mathrm{d}t - g(X^{v,\kappa}_T)
    \\ &\qquad\quad \ + (1 - \frac{1}{\kappa}) \big( \phi(X^{v,\kappa}_T,T) - \phi(X^{v,\kappa}_0,0)
    - \int_0^{T} \langle \nabla_x \phi(X^{v,\kappa}_t,t), \mathrm{d}X^{v,\kappa}_t \rangle - \int_0^{T} \partial_t \phi(X^{v,\kappa}_t,t) \, \mathrm{d}t \big)
    \bigr) \\ &\qquad\quad \times
    \bigl( \int_0^{T}
    \langle \nabla_x \eta(X^{v,\kappa}_t,t), \mathrm{d}X^{v,\kappa}_t - ( b(X^{v,\kappa}_t,t) - \sigma \sigma^{\top} \nabla_x \phi(X^{v,\kappa}_t,t) ) \, \mathrm{d}t \rangle
    + \eta(X^{v,\kappa}_0,0) \\ &\qquad\qquad + (1 - \frac{1}{\kappa}) \big( \eta(X^{v,\kappa}_T,T) - \eta(X^{v,\kappa}_0,0)
    - \int_0^{T} \langle \nabla_x \eta(X^{v,\kappa}_t,t), \mathrm{d}X^{v,\kappa}_t \rangle - \int_0^{T} \partial_t \eta(X^{v,\kappa}_t,t) \, \mathrm{d}t \big) \bigr) \big].
\end{split}
\end{talign}
Applying Itô's lemma to $\phi$ yields
\begin{talign}
\begin{split} \label{eq:phi_ito}
    &\int_0^{T} \langle \nabla_x \phi(X^{v,\kappa}_t,t), \mathrm{d}X^{v,\kappa}_t \rangle \\ &= \phi(X^{v,\kappa}_T,T) - \phi(X^{v,\kappa}_0,0) - \int_0^{T} \big( \partial_t \phi(X^{v,\kappa}_t,t) + \frac{\kappa}{2} \mathrm{Tr}\big( \sigma \sigma^{\top} \nabla^2_x \phi(X^{v,\kappa}_t,t) \big) \big) \, \mathrm{d}t,
    \\ &\int_0^{T} \langle \nabla_x \eta(X^{v,\kappa}_t,t), \mathrm{d}X^{v,\kappa}_t \rangle \\ &= \eta(X^{v,\kappa}_T,T) - \eta(X^{v,\kappa}_0,0) - \int_0^{T} \big( \partial_t \eta(X^{v,\kappa}_t,t) + \frac{\kappa}{2} \mathrm{Tr}\big( \sigma \sigma^{\top} \nabla^2_x \eta(X^{v,\kappa}_t,t) \big) \big) \, \mathrm{d}t,
\end{split}
\end{talign}
and plugging this into the right-hand side of \eqref{eq:partial_epsilon_hat_L_VM}, we obtain
\begin{talign*}
\begin{split}
    &\frac{\partial}{\partial \varepsilon} \mathcal{L}_{\mathrm{VM}_{v,\kappa}}(\phi + \varepsilon \eta) \rvert_{\varepsilon = 0}
    \\ &\quad = \mathbb{E} \big[ \bigl( \! - \! \int_0^{T} \big(\langle \nabla_x \phi(X^{v,\kappa}_t,t), b(X^{v,\kappa}_t,t) \! - \! \frac{1}{2} \sigma \sigma^{\top} \nabla_x \phi(X^{v,\kappa}_t,t) \rangle \! + \! \partial_t \phi(X^{v,\kappa}_t,t) \! + \! \frac{\kappa}{2} \mathrm{Tr}\big( \sigma \sigma^{\top} \nabla^2_x \phi(X^{v,\kappa}_t,t) \big) \big) \, \mathrm{d}t \\ &\qquad\quad \ + \phi(X^{v,\kappa}_T,T) - \int_0^T f(X^{v,\kappa}_t,t) \, \mathrm{d}t - g(X^{v,\kappa}_T)
    + (1 - \frac{1}{\kappa})
    \frac{\kappa}{2} \int_0^{T} \mathrm{Tr}\big( \sigma \sigma^{\top} \nabla^2_x \phi(X^{v,\kappa}_t,t) \big) \, \mathrm{d}t
    \bigr) \\ &\qquad\quad \times
    \bigl( \int_0^{T}
    \langle \nabla_x \eta(X^{v,\kappa}_t,t), \mathrm{d}X^{v,\kappa}_t - ( b(X^{v,\kappa}_t,t) - \sigma \sigma^{\top} \nabla_x \phi(X^{v,\kappa}_t,t) ) \, \mathrm{d}t \rangle + \eta(X^{v,\kappa}_0,0)
    \\ &\qquad\qquad \ + (1 - \frac{1}{\kappa}) \frac{\kappa}{2} \int_0^{T} \mathrm{Tr}\big( \sigma \sigma^{\top} \nabla^2_x \eta(X^{v,\kappa}_t,t) \big) \, \mathrm{d}t \bigr) \big].
    \\ &\quad = \mathbb{E} \big[ \bigl( - \int_0^{T} \big(\langle \nabla_x \phi(X^{v,\kappa}_t,t), b(X^{v,\kappa}_t,t) - \frac{1}{2} \sigma \sigma^{\top} \nabla_x \phi(X^{v,\kappa}_t,t) \rangle + \partial_t \phi(X^{v,\kappa}_t,t) \big) \, \mathrm{d}t \\ &\qquad\quad \ + \phi(X^{v,\kappa}_T,T) - \int_0^T f(X^{v,\kappa}_t,t) \, \mathrm{d}t - g(X^{v,\kappa}_T)
    - \frac{1}{2} \int_0^{T} \mathrm{Tr}\big( \sigma \sigma^{\top} \nabla^2_x \phi(X^{v,\kappa}_t,t) \big) \, \mathrm{d}t
    \bigr) \\ &\qquad\quad \times
    \bigl( \int_0^{T}
    \langle \nabla_x \eta(X^{v,\kappa}_t,t), \sqrt{\kappa} \sigma(X^{v,\kappa}_t,t) \mathrm{d}W_t \! + \! (v(X^{v,\kappa}_t,t) \! - \! b(X^{v,\kappa}_t,t) \! + \! \sigma \sigma^{\top} \nabla_x \phi(X^{v,\kappa}_t,t) ) \, \mathrm{d}t \rangle \\ &\qquad\qquad \ + \eta(X^{v,\kappa}_0,0) - \frac{1}{2} \big( \frac{1}{\kappa} - 1 \big) \int_0^{T} \mathrm{Tr}\big( \kappa \sigma \sigma^{\top} \nabla^2_x \eta(X^{v,\kappa}_t,t) \big) \, \mathrm{d}t \bigr) \big].
\end{split}
\end{talign*}

We apply the Girsanov theorem (\Cref{thm:girsanov}\emph{(iv)}) to perform the change of variable from $W$ to $\hat{B}$:
\begin{talign*}
    \hat{B}_t = W_t - \kappa^{-\frac{1}{2}} \int_0^t \sigma^{-1} (v(X^{v,\kappa}_s,s) - b(X^{v,\kappa}_s,s) +
    \sigma \sigma^{\top} \nabla_x \phi(X^{v,\kappa}_s,s) ) \, \mathrm{d}s,
\end{talign*}
and we define $\bar{X}^{\kappa}$ in analogy with $X^{v,\kappa}$, replacing $W$ by $\hat{B}$ (and using that $X^{v,\kappa}$ satisfies $\mathrm{d}X^{v,\kappa}_t = v(X^{v,\kappa}_t,t) \, \mathrm{d}t + \sqrt{\kappa} \sigma(X^{v,\kappa}_t,t) \, \mathrm{d}W_t$):
\begin{align}
\begin{split} \label{eq:bar_X_kappa}
    \mathrm{d}\bar{X}^{\kappa}_t &= v(\bar{X}^{\kappa}_t,t) \, \mathrm{d}t + \sqrt{\kappa} \sigma(\bar{X}^{\kappa}_t,t) \, \mathrm{d}\hat{B}_t \\ &= \big( b(\bar{X}^{\kappa}_s,s) \! - \! \sigma \sigma^{\top} \nabla_x \phi(\bar{X}^{\kappa}_s,s)
    \big) \, \mathrm{d}t \! + \! \sqrt{\kappa} \sigma(\bar{X}^{\kappa}_t,t) \, \mathrm{d}W_t,
\end{split}
\end{align}
That is, we re-express the right-hand side of \eqref{eq:partial_epsilon_hat_L_VM} in terms of
$\hat{B}$ and $\bar{X}^{\kappa}$ using the importance weight, and subsequently write $\hat{B}$ in terms of $W$:
\begin{talign}
\begin{split} \label{eq:L_VM_F}
    &\frac{\partial}{\partial \varepsilon} \mathcal{L}_{\mathrm{VM}_{v,\kappa}}(\phi + \varepsilon \eta) \rvert_{\varepsilon = 0} \\ &= \mathbb{E} \big[
    \bigl(
    - \int_0^{T} \big(\langle \nabla_x \phi(\bar{X}^{\kappa}_t,t), b(\bar{X}^{\kappa}_t,t) - \frac{1}{2} \sigma \sigma^{\top} \nabla_x \phi(\bar{X}^{\kappa}_t,t) \rangle + \partial_t \hat{\phi}(\bar{X}^{\kappa}_t,t) \big) \, \mathrm{d}t \\ &\qquad\quad \ + \phi(\bar{X}^{\kappa}_T,T) - \int_0^T f(\bar{X}^{\kappa}_t,t) \, \mathrm{d}t - g(\bar{X}^{\kappa}_T)
    - \frac{1}{2} \int_0^{T} \mathrm{Tr}\big( \sigma \sigma^{\top} \nabla^2_x \hat{\phi}(\bar{X}^{\kappa}_t,t) \big) \, \mathrm{d}t
    \bigr)
    \\ &\qquad \times
    \bigl( \int_0^{T} \langle
    \nabla_x \eta(\bar{X}^{\kappa}_t,t),
    (v(\bar{X}^{\kappa}_t,t) - b(\bar{X}^{\kappa}_t,t) +
    \sigma \sigma^{\top} \nabla_x \phi(\bar{X}^{\kappa}_t,t) ) \, \mathrm{d}t + \sqrt{\kappa} \sigma(\bar{X}^{\kappa}_t,t) \, \mathrm{d}\hat{B}_t
    \rangle \\ &\qquad\qquad + \eta(\bar{X}^{\kappa}_0,0)
    - \frac{1}{2} \big( \frac{1}{\kappa} - 1 \big) \int_0^{T} \mathrm{Tr}\big( \kappa \sigma \sigma^{\top} \nabla^2_x \eta(\bar{X}^{\kappa}_t,t) \big) \, \mathrm{d}t \bigr)
    \bigr) \\ &\qquad \times \exp\big( \int_0^{T} \langle (\sqrt{\kappa} \sigma)^{-1} (v(\bar{X}^{\kappa}_s,s) - b(\bar{X}^{\kappa}_s,s) +
    \sigma \sigma^{\top} \nabla_x \phi(\bar{X}^{\kappa}_s,s) ), \mathrm{d}\hat{B}_s \rangle \\ &\qquad\qquad\quad + \frac{1}{2} \int_0^{T} \| (\sqrt{\kappa} \sigma)^{-1} (v(\bar{X}^{\kappa}_s,s) - b(\bar{X}^{\kappa}_s,s) +
    \sigma \sigma^{\top} \nabla_x \phi(\bar{X}^{\kappa}_s,s) ) \|^2 \, \mathrm{d}s \big)
    \big]
    \\ &= \mathbb{E} \big[
    \bigl(
    - \int_0^{T} \big(\langle \nabla_x \phi(\bar{X}^{\kappa}_t,t), b(\bar{X}^{\kappa}_t,t) - \frac{1}{2} \sigma \sigma^{\top} \nabla_x \phi(\bar{X}^{\kappa}_t,t) \rangle + \partial_t \hat{\phi}(\bar{X}^{\kappa}_t,t) \big) \, \mathrm{d}t \\ &\qquad\quad \ + \phi(\bar{X}^{\kappa}_T,T) - \int_0^T f(\bar{X}^{\kappa}_t,t) \, \mathrm{d}t - g(\bar{X}^{\kappa}_T)
    - \frac{1}{2} \int_0^{T} \mathrm{Tr}\big( \sigma \sigma^{\top} \nabla^2_x \hat{\phi}(\bar{X}^{\kappa}_t,t) \big) \, \mathrm{d}t
    \bigr) \\ &\qquad \times \bigl( \int_0^{T} \langle \sqrt{\kappa} \sigma^{\top} \nabla_x \eta(\bar{X}^{\kappa}_t,t), \mathrm{d}W_t \rangle + \eta(\bar{X}^{\kappa}_0,0)
    - \frac{1}{2} \big( \frac{1}{\kappa} - 1 \big) \int_0^{T} \mathrm{Tr}\big( \kappa \sigma \sigma^{\top} \nabla^2_x \eta(\bar{X}^{\kappa}_t,t) \big) \, \mathrm{d}t
    \bigr) \\ &\qquad \times \mathrm{IW}(W,0,T)
    \big],
\end{split}
\end{talign}
where we used the definition
\begin{talign*}
\begin{split}
    \mathrm{IW}(W,t,t') &= \exp\big( \int_t^{t'} \langle (\sqrt{\kappa} \sigma)^{-1} (v(\bar{X}^{\kappa}_s,s) - b(\bar{X}^{\kappa}_s,s) +
    \sigma \sigma^{\top} \nabla_x \phi(\bar{X}^{\kappa}_s,s) ), \mathrm{d}W_s \rangle \\ &\qquad\quad - \frac{1}{2} \int_t^{t'} \| (\sqrt{\kappa} \sigma)^{-1} (v(\bar{X}^{\kappa}_s,s) - b(\bar{X}^{\kappa}_s,s) +
    \sigma \sigma^{\top} \nabla_x \phi(\bar{X}^{\kappa}_s,s) ) \|^2 \, \mathrm{d}s \big).
\end{split}
\end{talign*}

\paragraph{Step 2: Applying the Dynkin representation theorem}
Next, we define the functionals $F : C([0,T];\mathbb{R}^d) \to \mathbb{R}$:
\begin{talign}
\begin{split} \label{eq:def_F}
    F(\bar{X}^{\kappa}) &= - \int_0^{T} \big(\langle \nabla_x \phi(\bar{X}^{\kappa}_t,t), b(\bar{X}^{\kappa}_t,t) - \frac{1}{2} \sigma \sigma^{\top} \nabla_x \phi(\bar{X}^{\kappa}_t,t) \rangle + \partial_t \hat{\phi}(\bar{X}^{\kappa}_t,t) \big) \, \mathrm{d}t \\ &\quad + \phi(\bar{X}^{\kappa}_T,T) - \int_0^T f(\bar{X}^{\kappa}_t,t) \, \mathrm{d}t - g(\bar{X}^{\kappa}_T)
    - \frac{1}{2} \int_0^{T} \mathrm{Tr}\big( \sigma \sigma^{\top} \nabla^2_x \hat{\phi}(\bar{X}^{\kappa}_t,t) \big) \, \mathrm{d}t
    \\ &=: \int_0^{T} h_1(\bar{X}^{\kappa}_t) \, \mathrm{d}t + h_2(\bar{X}^{\kappa}_{T}).
\end{split}
\end{talign}
We can apply the Dynkin representation theorem (\Cref{thm:parabolic-dynkin}) to the functional $F(\bar{X}^{\kappa})$. That is, setting $\sigma \gets \sqrt{\kappa} \sigma$, $\chi \gets \frac{1}{\kappa}-1 \in (-1,+\infty)$, $X \gets \bar{X}^{\kappa}$, $b \gets b -
\sigma \sigma^{\top} \nabla_x \phi$, we can write
\begin{talign} \label{eq:F_eta_star}
    F(\bar{X}^{\kappa}) = \eta^{\star}(\bar{X}^{\kappa}_0,0) + \int_0^{T} \langle \sqrt{\kappa} \sigma^{\top} \nabla_x \eta^{\star}(\bar{X}^{\kappa}_t,t), \mathrm{d}W_t \rangle - \frac{1}{2}(\frac{1}{\kappa} \! - \! 1) \int_0^{T} \mathrm{Tr}\big( \kappa \sigma \sigma^{\top} \nabla^2_x \eta^{\star}(\bar{X}^{\kappa}_t,t) \big) \, \mathrm{d}t,
\end{talign}
where $\eta^{\star}$ is the solution to the
parabolic PDE
\begin{talign}
\begin{split}
\label{eq:parabolic_pde_2}
&\partial_t \eta^{\star}(x,t) + \mathcal{L} \eta^{\star}(x,t)+h_1(x,t)=0\quad\text{for }x\in \mathbb{R}^d, t \in [0,T],
\qquad
\eta^{\star}(x,T)=h_2(x)\quad\text{for }x\in\mathbb{R}^d, \\
&\mathcal{L} \eta^{\star}= (b  -
\sigma \sigma^{\top} \nabla \phi)\!\cdot\!\nabla \eta^{\star}+\tfrac{1}{2}\mathrm{Tr}(\sigma \sigma^{\top}\nabla^2 \eta^{\star}).
\end{split}
\end{talign}
A solution $\eta^{\star} \in C^{2+\alpha,1+\alpha/2}(\mathbb{R}^d \times [0,T))\cap C(\mathbb{R}^d \times [0,T])$ to this problem such that
\begin{talign*}
\mathbb{E}\!\left[\int_0^{T}\!\!\|\sigma(\bar{X}^{\kappa}_t,t)^\top\nabla \eta^{\star}(\bar{X}^{\kappa}_t,t)\|^2\,dt\right]<\infty
\end{talign*}
is guaranteed to exist by \Cref{rem:regularity_PDP}. The conditions stated in this remark are met because there exists $\alpha \in (0,1)$ such that the following facts hold:
\begin{itemize}[left=3pt]
\item $b, \sigma \in \mathcal{C}^{\alpha,\alpha/2}(\mathbb{R}^d \times [0,T])$ and $\sigma$ is uniformly elliptic.
\item $h_1\in C^{\alpha,\alpha/2}(\mathbb{R}^d \times [0,T])$ and $h_2\in C^{2+\alpha}(\mathbb{R}^d)$, because for all $\phi \in \mathcal{F}$, $\phi \in C^{2+\alpha,1+\alpha/2}(\mathbb{R}^d \times [0,T])$, and because $f \in C^{\alpha,\alpha/2}(\mathbb{R}^d \times [0,T])$, $g \in C^{2+\alpha}(\mathbb{R}^d)$, and $b, \sigma \in \mathcal{C}^{\alpha,\alpha/2}(\mathbb{R}^d \times [0,T])$.
\item $h_1$ and $h_2$ have at most polynomial growth because $\phi$, $\nabla_x \phi$, $\nabla_x^2 \phi$, $f$ and $g$ all have at most polynomial growth.
\end{itemize}
If we plug \eqref{eq:def_F}-\eqref{eq:F_eta_star} into equation \eqref{eq:L_VM_F}, we obtain that
\begin{talign}
\begin{split} \label{eq:L_VM_F_strat_2}
    &\frac{\partial}{\partial \varepsilon} \mathcal{L}_{\mathrm{VM}_{v,\kappa}}(\phi + \varepsilon \eta) \rvert_{\varepsilon = 0} \\ &= \mathbb{E} \big[
    \bigl(
    \int_0^{T} \langle \sqrt{\kappa} \sigma^{\top} \nabla_x \eta^{\star}(\bar{X}^{\kappa}_t,t), \mathrm{d}W_t \rangle + \eta^{\star}(\bar{X}^{\kappa}_0,0)
    - \frac{1}{2} \big( \frac{1}{\kappa} - 1 \big) \int_0^{T} \mathrm{Tr}\big( \kappa \sigma \sigma^{\top} \nabla^2_x \eta^{\star}(\bar{X}^{\kappa}_t,t) \big) \, \mathrm{d}t
    \bigr) \\ &\qquad \times \bigl( \int_0^{T} \langle \sqrt{\kappa} \sigma^{\top} \nabla_x \eta(\bar{X}^{\kappa}_t,t), \mathrm{d}W_t \rangle + \eta(\bar{X}^{\kappa}_0,0)
    - \frac{1}{2} \big( \frac{1}{\kappa} - 1 \big) \int_0^{T} \mathrm{Tr}\big( \kappa \sigma \sigma^{\top} \nabla^2_x \eta(\bar{X}^{\kappa}_t,t) \big) \, \mathrm{d}t
    \bigr) \\ &\qquad \times \mathrm{IW}(W,0,T)
    \big]
\end{split}
\end{talign}

\paragraph{Step 3: Choosing the right perturbation $\eta$ to obtain an almost sure equality}
Thus, if we set
$\eta(x,t) = - \eta^{\star}(x,t)$,
we have that
\begin{talign*}
\begin{split}
    &\frac{\partial}{\partial \varepsilon} \mathcal{L}_{\mathrm{VM}_{v,\kappa}}(\phi + \varepsilon \eta) \rvert_{\varepsilon = 0}
    \\ &= - \mathbb{E} \big[
    \bigl(
    \int_0^{T} \langle \sqrt{\kappa} \sigma^{\top} \nabla_x \eta^{\star}(\bar{X}^{\kappa}_t,t), \mathrm{d}W_t \rangle + \eta^{\star}(\bar{X}^{\kappa}_0,0)
    - \frac{1}{2} \big( \frac{1}{\kappa} - 1 \big) \int_0^{T} \mathrm{Tr}\big( \kappa \sigma \sigma^{\top} \nabla^2_x \eta^{\star}(\bar{X}^{\kappa}_t,t) \big) \, \mathrm{d}t
    \bigr)^2 \\ &\qquad\quad \times \mathrm{IW}(W,0,T) \big]
    \\ &= - \mathbb{E}[F(\bar{X}^{\kappa})^2 \times \mathrm{IW}(W,0,T)]
    \\ &= - \mathbb{E}[F(X^{v,\kappa})^2].
\end{split}
\end{talign*}
Hence, if $\hat{\phi}$ satisfies that $\frac{\partial}{\partial \varepsilon} \mathcal{L}_{\mathrm{VM}_{v,\kappa}}(\hat{\phi} + \varepsilon \eta) \rvert_{\varepsilon = 0} = 0$ for all perturbations $\eta$, we need that for $\hat{\phi}$, $F(X^{v,\kappa}) = 0$ almost surely, which by the definition of $F$ in \eqref{eq:def_F}
means that almost surely,
\begin{talign}
\begin{split} \label{eq:tilde_F_zero}
    &F(X^{v,\kappa}) = - \int_0^{T} \big(\langle \nabla_x \hat{\phi}(X^{v,\kappa}_t,t), b(X^{v,\kappa}_t,t) - \frac{1}{2} \sigma \sigma^{\top} \nabla_x \hat{\phi}(X^{v,\kappa}_t,t) \rangle + \partial_t \hat{\phi}(X^{v,\kappa}_t,t) \big) \, \mathrm{d}t \\ &\quad + \hat{\phi}(X^{v,\kappa}_T,T) - \int_0^T f(X^{v,\kappa}_t,t) \, \mathrm{d}t - g(X^{v,\kappa}_T)
    - \frac{1}{2} \int_0^{T} \mathrm{Tr}\big( \sigma \sigma^{\top} \nabla^2_x \hat{\phi}(X^{v,\kappa}_t,t) \big) \, \mathrm{d}t = 0.
\end{split}
\end{talign}

\paragraph{Step 4: Writing down the Radon-Nikodym derivative $\frac{\mathrm{d}\mathbb{P}^{u^{\star}}}{\mathrm{d}\mathbb{P}^{\hat{\phi}}}$}
Let $\mathbb{P}^{\hat{\phi}}$ be the path probability measure for the process $X^{\hat{\phi}}$ controlled by $- \sigma^{\top} \nabla \hat{\phi}$ with initial distribution $\mu_0$, which satisfies the SDE
\begin{talign*}
    \mathrm{d}X^{\hat{\phi}}_t = \big( b(X^{\hat{\phi}}_t,t) - \sigma \sigma^{\top} \nabla \hat{\phi}(X^{\hat{\phi}}_t,t) \big) \, \mathrm{d}t + \sigma(X^{\hat{\phi}}_t,t) \, \mathrm{d}W_t.
\end{talign*}
If $\mathbb{P}$ is the measure of the uncontrolled process, Corollary \ref{cor:Girsanov_dX_t} of the Girsanov theorem implies that for any process $Y$ such that its probability measure is absolutely continuous with respect to $\mathbb{P}$ and/or $\mathbb{P}^{\hat{\phi}}$\footnote{According to the Girsanov theorem, the processes $Y$ of this form are those which have the same diffusion coefficient as $X$ and $X^{\hat{\phi}}$, that is, diffusion coefficient $\sigma$.},
\begin{talign}
\begin{split} \label{eq:dP_dP_hatphi}
    \frac{\mathrm{d}\mathbb{P}}{\mathrm{d}\mathbb{P}^{\hat{\phi}}}(Y) &= \frac{\mathrm{d}\mathbb{P}^{\hat{\phi}}}{\mathrm{d}\mathbb{P}}(Y)^{-1} \\ &= \exp \big(-\!\int_0^{T}\big\langle - \sigma^{\top} \nabla \hat{\phi}(Y_s,s),\,\sigma^{-1} \big( \mathrm{d}Y_s - b(Y_s,s) \, \mathrm{d}s \big) - \tfrac12 \big( - \sigma^{\top} \nabla \hat{\phi}(Y_s,s) \big) \, \mathrm{d}s \big\rangle \big)
    \\ &= \exp \big(\!\int_0^{T}\big\langle
    \nabla \hat{\phi}(Y_s,s), \mathrm{d}Y_s - (b(Y_s,s) - \frac{1}{2} \sigma \sigma^{\top} \nabla \hat{\phi}(Y_s,s)) \, \mathrm{d}s \big\rangle \big).
\end{split}
\end{talign}
Also, for any process $Y$ whose probability measure is absolutely continuous with respect to $\mathbb{P}$ and/or $\mathbb{P}^{u^{\star}}$, the path integral characterization of the optimal path measure (equation \eqref{eq:path_integral_characterization} in \Cref{thm:value_functions}) reads
\begin{talign}
\begin{split}
\frac{\mathrm{d}\mathbb{P}^{u^{\star}}}{\mathrm{d}\mathbb{P}}(Y) &= \exp \big( V(Y_0,0) - \int_0^T f(Y_t,t) \, \mathrm{d}t - g(Y_T) \big).
\label{eq:dP_u_star_dP}
\end{split}
\end{talign}
Hence, for any such process $Y$,
\begin{talign}
\begin{split} \label{eq:P_u_star_P_v_Y}
    \frac{\mathrm{d}\mathbb{P}^{u^{\star}}}{\mathrm{d}\mathbb{P}^{\hat{\phi}}}(Y) &= \frac{\mathrm{d}\mathbb{P}^{u^{\star}}}{\mathrm{d}\mathbb{P}}(Y) \times \frac{\mathrm{d}\mathbb{P}}{\mathrm{d}\mathbb{P}^{\hat{\phi}}}(Y) \\ &= \exp \big( \int_0^T \langle \nabla_x \hat{\phi}(Y_t,t), \mathrm{d}Y_t - ( b(Y_t,t) - \frac{1}{2} \sigma \sigma^{\top} \nabla_x \hat{\phi}(Y_t,t) ) \, \mathrm{d}t
    \rangle  \\ &\qquad\qquad\qquad + V(Y_0,0) - \int_0^T f(Y_t,t) \, \mathrm{d}t - g(Y_T) \big),
\end{split}
\end{talign}

At this point, we have to treat separately the case \emph{(i)} in which $\kappa = 1$ and $\tilde{\mu}_0$ is mutually absolutely continuous with respect to $\mu_0$, which is simpler, and the general case \emph{(ii)} in which $\kappa \in (0,\infty)$ and/or we only have that $\mathrm{supp}(\mu_0) \subseteq \mathrm{supp}(\tilde{\mu}_0)$.

\paragraph{Step 5(i): Concluding the proof for the case $\kappa = 1$, $\mu_0$ and $\tilde{\mu}_0$ mutually absolutely continuous}
Observe that if we undo the Itô's lemma change \eqref{eq:phi_ito} in equation \eqref{eq:tilde_F_zero} and set $\kappa = 1$, we obtain
\begin{talign*}
\begin{split}
    &\int_0^{T} \langle \nabla_x \hat{\phi}(X^{v,1}_t,t), \mathrm{d}X^{v,1}_t - ( b(X^{v,1}_t,t) - \frac{1}{2} \sigma \sigma^{\top} \nabla_x \hat{\phi}(X^{v,1}_t,t) ) \, \mathrm{d}t \rangle \\ &\qquad\quad \ + \hat{\phi}(X^{v,1}_0,0) - \int_0^T f(X^{v,1}_t,t) \, \mathrm{d}t - g(X^{v,1}_T) = 0.
\end{split}
\end{talign*}
Let $\mathbb{P}^{v,1}$ be the path probability measure of $X^{v,1}$.
$\mathbb{P}^{v,1}$ is absolutely continuous with respect to $\mathbb{P}$, which means that almost surely,
\begin{talign}
\begin{split} \label{eq:P_u_star_P_v}
    \frac{\mathrm{d}\mathbb{P}^{u^{\star}}}{\mathrm{d}\mathbb{P}^{\hat{\phi}}}(X^{v,1}) &= \frac{\mathrm{d}\mathbb{P}^{u^{\star}}}{\mathrm{d}\mathbb{P}}(X^{v,1}) \times \frac{\mathrm{d}\mathbb{P}}{\mathrm{d}\mathbb{P}^{\hat{\phi}}}(X^{v,1}) \\ &= \exp \big( \int_0^T \langle \nabla_x \hat{\phi}(X^{v,1}_t,t), \mathrm{d}X^{v,1}_t - ( b(X^{v,1}_t,t) - \frac{1}{2} \sigma \sigma^{\top} \nabla_x \hat{\phi}(X^{v,1}_t,t) ) \, \mathrm{d}t
    \rangle  \\ &\qquad\qquad\qquad + V(X^{v,1}_0,0) - \int_0^T f(X^{v,1}_t,t) \, \mathrm{d}t - g(X^{v,1}_T) \big)
    \\ &= \exp \big( V(X^{v,1}_0,0) - \hat{\phi}(X^{v,1}_0,0) \big),
\end{split}
\end{talign}
where the last equality holds almost surely by equation
\eqref{eq:tilde_F_zero}.

Observe that since both $\mathbb{P}^{u^{\star}}$ and $\mathbb{P}^{\hat{\phi}}$ have marginal $\mu_0$ at time $t=0$, for almost every $x \in \mathbb{R}^d$ we have that
\begin{talign*}
\begin{split}
    1 = \frac{\mu_0(x)}{\mu_0(x)}
    = \mathbb{E}[\frac{\mathrm{d}\mathbb{P}^{u^{\star}}}{\mathrm{d}\mathbb{P}^{\hat{\phi}}}(X^{v,1})|X^{v,1}_0 = x] = \exp \big( V(x,0) - \hat{\phi}(x,0) \big).
\end{split}
\end{talign*}
which proves that $\hat{\phi}(x,0) = V(x,0)$ for almost every $x \in
\mathrm{supp}(\mu_0)$, and by continuity of $\hat{\phi}$ and $V$, the equality holds for all $x \in \mathrm{supp}(\mu_0)$.

Next, we prove that the control $-\sigma^{\top} \nabla \hat{\phi}$ is the optimal control.
Plugging the equality $\hat{\phi}(x,0) = V(x,0)$ into equation \eqref{eq:P_u_star_P_v} implies that trajectories $X^{v,1}$ satisfy almost surely
\begin{talign*}
\begin{split}
    \frac{\mathrm{d}\mathbb{P}^{u^{\star}}}{\mathrm{d}\mathbb{P}^{\hat{\phi}}}(X^{v,1}) &= 1,
\end{split}
\end{talign*}
And since $\mathbb{P}^{u^{\star}}$ is absolutely continuous with respect to $\mathbb{P}^{v,1}$ because $X^{u^{\star}}$ and $X^{v,1}$ share the same diffusion coefficient and initial distribution, the trajectories $X^{u^{\star}}$ also satisfy almost surely that
\begin{talign*}
\begin{split}
    \frac{\mathrm{d}\mathbb{P}^{u^{\star}}}{\mathrm{d}\mathbb{P}^{\hat{\phi}}}(X^{u^{\star}}) &= 1,
\end{split}
\end{talign*}
Applying Girsanov once more,
\begin{talign} \label{eq:girsanov_app_KL}
    \mathbb{E}\big[ \frac{1}{2} \int_0^T \| u^{\star}(X^{u^{\star}}_t,t) + \sigma^{\top} \nabla \hat{\phi}(X^{u^{\star}}_t,t) \|^2 \, \mathrm{d}t \big] = \mathrm{KL}(\mathbb{P}^{u^{\star}}\,\|\,\mathbb{P}^{\hat{\phi}}) = \mathbb{E}\big[ \log \frac{\mathrm{d}\mathbb{P}^{u^{\star}}}{\mathrm{d}\mathbb{P}^{\hat{\phi}}}(X^{u^{\star}}) \big] = 0.
\end{talign}
This means that $u^{\star}(X^{u^{\star}}_t,t) = - \sigma^{\top} \nabla V(X^{u^{\star}}_t,t)$ is equal to $- \sigma^{\top} \nabla \hat{\phi}(X^{u^{\star}}_t,t)$ almost surely, and by (i) continuity of $\sigma$, $\nabla V$ and $\nabla \hat{\phi}$, and (ii) the fact that the support of the measure of $X^{u^{\star}}_t$ is all of $\mathbb{R}^d$ for $t \in (0,T]$, we have that they are equal for all $(x,t) \in \mathbb{R}^d \times [0,T]$.

\paragraph{Step 5(ii): Concluding the proof for the general case}
When $\kappa \neq 1$,
the probabilistic interpretation in terms of $\frac{\mathrm{d}\mathbb{P}^{u^{\star}}}{\mathrm{d}\mathbb{P}^{\hat{\phi}}}$ as in equation \eqref{eq:P_u_star_P_v} does not directly hold because the support of the probability measure of $X^{v,\kappa}$ does not overlap with the support of $\mathbb{P}^{u^{\star}}$ and/or $\mathbb{P}^{\hat{\phi}}$, since the corresponding processes have different quadratic variations as their diffusion coefficients are different.

Instead, we will use an argument based on a certain set of trajectories $X^{v,\kappa}$ being dense in a certain space, and the continuity of a certain functional.

Namely, define the sets
\begin{talign*}
\begin{split}
    &C_{\mathrm{supp}(\tilde{\mu}_0)}([0,T];\mathbb{R}^d) = \{ w \in C([0,T];\mathbb{R}^d) \, | \, w(0) \in \mathrm{supp}(\tilde{\mu}_0) \}, \\
    &\mathcal{C}^{v,\kappa}_{\mathrm{supp}(\tilde{\mu}_0),F} = \{ X^{v,\kappa} \in C([0,T];\mathbb{R}^d) \, | \, X^{v,\kappa}_0 \in \mathrm{supp}(\tilde{\mu}_0), \, X^{v,\kappa} \text{ solves } \eqref{eq:X_v_kappa_SDE}, \, F(X^{v,\kappa}) = 0 \},
\end{split}
\end{talign*}
both equipped with the uniform norm $\|w\|_\infty=\sup_{t\in[0,T]}|w(t)|$.

Let $\mathbb{P}^{v,\kappa}$ be the law of the solution $X^{v,\kappa}$ to the SDE \eqref{eq:X_v_kappa_SDE} with initial distribution $\tilde{\mu}_0$.
In Step 3 of this proof, we have shown that $\mathbb{P}^{v,\kappa}$-almost surely, $F(X^{v,\kappa}) = 0$, or equivalently, that $\mathcal{C}^{v,\kappa}_{\mathrm{supp}(\tilde{\mu}_0),F}$ is dense in $\mathrm{supp}(\mathbb{P}^{v,\kappa})$ in the uniform norm.

By \cref{thm:full_supp} and using the regularity, growth and uniform ellipticity assumptions
on $v$, we have that $\mathrm{supp}(\mathbb{P}^{v,\kappa})$ is equal to $C_{\mathrm{supp}(\tilde{\mu}_0)}([0,T];\mathbb{R}^d)$ in the uniform norm. Hence, we conclude that $\mathcal{C}^{v,\kappa}_{\mathrm{supp}(\tilde{\mu}_0),F}$ is dense in $C_{\mathrm{supp}(\tilde{\mu}_0)}([0,T];\mathbb{R}^d)$ in the uniform norm.

By Lemma \ref{lem:exp_D_t_F}, the functional $F$ defined in \eqref{eq:def_F} over $C_{\mathrm{supp}(\tilde{\mu}_0)}([0,T];\mathbb{R}^d)$ is continuous in the uniform norm.
Since $F$ takes value zero on $\mathcal{C}^{v,\kappa}_{\mathrm{supp}(\tilde{\mu}_0),F}$ by construction, and $\mathcal{C}^{v,\kappa}_{\mathrm{supp}(\tilde{\mu}_0),F}$ is dense in $C_{\mathrm{supp}(\tilde{\mu}_0)}([0,T];\mathbb{R}^d)$, we conclude that $F$ is the zero functional on $C_{\mathrm{supp}(\tilde{\mu}_0)}([0,T];\mathbb{R}^d)$.

In particular, since $C_{\mathrm{supp}(\mu_0)}([0,T];\mathbb{R}^d) \subseteq C_{\mathrm{supp}(\tilde{\mu}_0)}([0,T];\mathbb{R}^d)$ because $\mathrm{supp}(\mu_0) \subseteq \mathrm{supp}(\tilde{\mu}_0)$ by assumption, for all trajectories $X^{u^{\star}}$ we obtain that $F(X^{u^{\star}}) = 0$. By Itô's lemma,
\begin{talign*}
\begin{split}
    &\int_0^{T} \langle \nabla_x \hat{\phi}(X^{u^{\star}}_t,t), \mathrm{d}X^{u^{\star}}_t \rangle \\ &= \hat{\phi}(X^{u^{\star}}_T,T) - \hat{\phi}(X^{u^{\star}}_0,0) - \int_0^{T} \big( \partial_t \hat{\phi}(X^{u^{\star}}_t,t) + \frac{1}{2} \mathrm{Tr}\big( \sigma \sigma^{\top} \nabla^2_x \hat{\phi}(X^{u^{\star}}_t,t) \big) \big) \, \mathrm{d}t,
\end{split}
\end{talign*}
which allows us to re-express $F(X^{u^{\star}})$ as follows:
\begin{talign*}
\begin{split}
    0 &= F(X^{u^{\star}}) = \int_0^{T} \langle \nabla_x \hat{\phi}(X^{u^{\star}}_t,t), \mathrm{d}X^{u^{\star}}_t - ( b(X^{u^{\star}}_t,t) - \frac{1}{2} \sigma \sigma^{\top} \nabla_x \hat{\phi}(X^{u^{\star}}_t,t) ) \, \mathrm{d}t \rangle \\ &\qquad\qquad\qquad + \hat{\phi}(X^{u^{\star}}_0,0) - \int_0^T f(X^{u^{\star}}_t,t) \, \mathrm{d}t - g(X^{u^{\star}}_T).
\end{split}
\end{talign*}
As in Step 5(i), this equality now implies that
\begin{talign*}
\frac{\mathrm{d}\mathbb{P}^{u^{\star}}}{\mathrm{d}\mathbb{P}^{\hat{\phi}}}(X^{u^{\star}}) = \frac{\mathrm{d}\mathbb{P}^{u^{\star}}}{\mathrm{d}\mathbb{P}}(X^{u^{\star}}) \times \frac{\mathrm{d}\mathbb{P}}{\mathrm{d}\mathbb{P}^{\hat{\phi}}}(X^{u^{\star}}) = \exp \big( V(X^{u^{\star}}_0,0) - \hat{\phi}(X^{u^{\star}}_0,0) \big),
\end{talign*}
and the remainder of the proof follows analogously to Step 5(i).
\end{proof}

\begin{lemma} \label{lem:exp_D_t_F}
    Let $F : C([0,T];\mathbb{R}^d) \to \mathbb{R}$ be the functional
    \begin{talign*}
    \begin{split}
        F(X) &= - \int_0^{T} \big(\langle \nabla_x \hat{\phi}(X_t,t), b(X_t,t) - \frac{1}{2} \sigma \sigma^{\top} \nabla_x \hat{\phi}(X_t,t) \rangle + \partial_t \hat{\phi}(X_t,t) \big) \, \mathrm{d}t \\ &\quad + \hat{\phi}(X_T,T) - \int_0^T f(X_t,t) \, \mathrm{d}t - g(X_T)
        - \frac{1}{2} \int_0^{T} \mathrm{Tr}\big( \sigma \sigma^{\top} \nabla^2_x \hat{\phi}(X_t,t) \big) \, \mathrm{d}t \\ &:= \int_0^{T} h_1(X_t,t) \, \mathrm{d}t + h_2(X_T),
    \end{split}
    \end{talign*}
    where $h_1 : \mathbb{R}^d \times [0,T] \to \mathbb{R}$ and $h_2 : \mathbb{R}^d \to \mathbb{R}$. $F$ is continuous in the uniform norm of $C([0,T];\mathbb{R}^d)$.
\end{lemma}
\begin{proof}
    If $(X^n)_{n \geq 0} \subset C([0,T];\mathbb{R}^d)$ is such that $X^n \to X$ in the uniform norm, then
    \begin{talign}
    \begin{split} \label{eq:H_n}
        |F(X)-F(X^n)| \leq \int_0^{T} \big| h_1(X_t,t) - h_1(X^n_t,t) \big| \, \mathrm{d}t + \big| h_2(X_T) - h_2(X^n_T)\big|.
    \end{split}
    \end{talign}
    Since there exists $N$ such that for all $n \geq N$, $\|X_n - X\|_{\infty} \leq 1$, we have that
    \begin{talign*}
        \sup_{n} \| X_n \|_{\infty} \leq \max\{ \max_{i \in [n-1]}\{ \| X_n \|_{\infty} \}, \| X \|_{\infty} + 1 \},
    \end{talign*}
    which means that the set $(X^n_t)_{n \geq 0, t \in [0,T]}$ is contained in a compact set of $\mathbb{R}^d$. Since $h_1$ and $h_2$ are continuous, they are uniformly continuous over compact sets, which implies that for any $\varepsilon > 0$, the right-hand side of \eqref{eq:H_n} is smaller than $\varepsilon$ for all $n$ large enough. This proves that $F$ is continuous.
\end{proof}

\begin{remark}[Value Matching loss function with separate architectures for the actor and critic] \label{rem:separate_architectures}
    Observe that the argument in \Cref{thm:vm_finite} also works if we parameterize the control $u$ as a separate neural network instead of obtaining it from the approximate value function as $- \sigma^{\top} \nabla \phi$. The resulting loss function is:
    \begin{talign}
    \begin{split} \label{eq:L_VM_finite_separate}
        \mathcal{L}_{\mathrm{VM}_{v,\kappa}}(\phi,u) &=
        \mathbb{E} \bigl[ \frac{1}{2} \bigl( \int_0^{T} \langle -(\sigma^{\top})^{-1} u(X^{v,\kappa}_t,t), \frac{1}{\kappa} \mathrm{d}X^{v,\kappa}_t - ( b(X^{v,\kappa}_t,t) + \frac{1}{2} \sigma u(X^{v,\kappa}_t,t) ) \, \mathrm{d}t \rangle \\ &\qquad\quad + \frac{1}{\kappa} \phi(X^{v,\kappa}_0,0) + (1 - \frac{1}{\kappa}) \big( \phi(X^{v,\kappa}_T,T) - \int_0^{T} \partial_t \phi(X^{v,\kappa}_t,t) \, \mathrm{d}t \big) \\ &\qquad\quad - \int_0^T f(X^{v,\kappa}_t,t) \, \mathrm{d}t - g(X^{v,\kappa}_T)
        \bigr)^2 \bigr]
    \end{split}
    \end{talign}
    Several deep learning works on SOC \citep{nusken2021solving,domingo2024stochastic,domingoadjoint2025} parameterize the control $u$. Thus, Value Matching can also be applied to their settings by considering an additional network $\phi$, or by adding a head to the $u$ architecture whose output is $\phi$. Similarly, the Value Matching loss function for hitting time problems can also be written with separate architectures for the actor and critic.
\end{remark}

\subsection{The Value Matching loss for hitting time problems ($T = \infty$): definition and optimality} \label{subsec:vm_loss}

\begin{theorem}[First-order optimality of the Value Matching (VM) loss with uncapped hitting times] \label{thm:vm_loss_hitting}
    Consider the following SOC problem, whose horizon is given by the hitting time $\tau$ of the set $\mathcal{S}$:
    \begin{talign} \label{eq:hitting_time_SOC_1}
        &\min_{u \in \mathcal{U}} \mathbb{E}\big[ \int_0^{\tau} \big( \frac{1}{2}\|u(X^u_t)\|^2 + f(X^u_t) \big) \, \mathrm{d}t + g(X^u_{\tau}) \big], \\
        \mathrm{s.t.} \quad &\mathrm{d}X^u_t = \big( b(X^u_t) + \sigma(X^u_t) u(X^u_t) \big) \, \mathrm{d}t + \sigma(X^u_t) \, \mathrm{d}W_t, \quad X^u_0 \sim \mu_0.
        \label{eq:hitting_time_SOC_2}
    \end{talign}
    Fix $\alpha \in (0,1)$. Assume that $b$ and $\sigma$ are locally Lipschitz with linear growth and belong to $C^{\alpha}(\mathbb{R}^d)$, that $\sigma\sigma^{\top}$ is uniformly elliptic, that the boundary $\partial\mathcal{S}$ is $C^{2+\alpha}$, and that $f \in C^{\alpha}(\mathbb{R}^d)$ and $g \in C^{2+\alpha}(\mathbb{R}^d)$ have at most polynomial growth.
    Given an arbitrary $\kappa \in (0,+\infty)$
    and an arbitrary drift $v \in C^1(\mathbb{R}^d; \mathbb{R}^d)$ that is locally Lipschitz with linear growth, consider trajectories $X^{v,\kappa}$ defined as
    \begin{talign} \label{eq:X_v_kappa}
        \mathrm{d}X^{v,\kappa}_t = v(X^{v,\kappa}_t) \, \mathrm{d}t + \sqrt{\kappa} \sigma(X^{v,\kappa}_t) \, \mathrm{d}W_t,
    \end{talign}
    with $X^{v,\kappa}_0$ distributed according to $\tilde{\mu}_0$ such that $\mu_0$ is absolutely continuous with respect to $\tilde{\mu}_0$.

    Consider the space of functions
    \begin{talign*}
    \mathcal{F} = \{ \phi : \mathbb{R}^d \to \mathbb{R} \, | \, \phi \in C^{2+\alpha}(\mathbb{R}^d), \, \phi, \nabla \phi, \nabla^2 \phi \text{ grow at most polynomially} \},
    \end{talign*}
    where $C^{2+\alpha}(\mathbb{R}^d)$ is a Hölder space (\Cref{def:holder_space}).
    Define the Value Matching (VM) loss as
    \begin{talign}
    \begin{split} \label{eq:VM_uncapped}
        \mathcal{L}_{\mathrm{VM}_{v,\kappa}}(\phi) &=
        \mathbb{E} \bigl[ \frac{1}{2} \bigl( \int_0^{\tau} \langle \nabla \phi(X^{v,\kappa}_t), \frac{1}{\kappa}\mathrm{d}X^{v,\kappa}_t - ( b(X^{v,\kappa}_t) - \frac{1}{2} \sigma \sigma^{\top} \nabla \phi(X^{v,\kappa}_t) ) \, \mathrm{d}t \rangle \\ &\qquad \ + \frac{1}{\kappa} \phi(X^{v,\kappa}_0) + (1 - \frac{1}{\kappa}) \phi(X^{v,\kappa}_\tau) - \int_0^{\tau} f(X^{v,\kappa}_t) \, \mathrm{d}t - g(X^{v,\kappa}_\tau)
        \bigr)^2 \bigr]
    \end{split}
    \end{talign}
    Suppose that $\hat{\phi}$ is a first-order critical point of $\mathcal{L}_{\mathrm{VM}_{v,\kappa}}$.
    That is, for any perturbation $\eta \in \mathcal{F}$, we have that $\frac{\partial}{\partial \varepsilon} \mathcal{L}_{\mathrm{VM}_{v,\kappa}}(\hat{\phi} + \varepsilon \eta) \rvert_{\varepsilon = 0} = 0$. Then, $\hat{\phi}$ satisfies $\hat{\phi}(x) = V(x)$ and $u^{\star}(x) = - \sigma^{\top}(x) \nabla \hat{\phi}(x)$ for all $x \in \mathbb{R}^d$, where $V$ and $u^{\star}$ are the value function and the optimal control of the problem \eqref{eq:hitting_time_SOC_1}-\eqref{eq:hitting_time_SOC_2}.
\end{theorem}

The proof uses the structure of the proof of \Cref{thm:vm_finite}. First, we rewrite the loss function to isolate the dependence on $\kappa$:
\begin{talign*}
\begin{split}
    \mathcal{L}_{\mathrm{VM}_{v,\kappa}}(\phi) &=
    \mathbb{E} \bigl[ \frac{1}{2} \bigl( \int_0^{\tau} \langle \nabla \phi(X^{v,\kappa}_t), \mathrm{d}X^{v,\kappa}_t - ( b(X^{v,\kappa}_t) - \frac{1}{2} \sigma \sigma^{\top} \nabla \phi(X^{v,\kappa}_t) ) \, \mathrm{d}t \rangle \\ &\qquad \ + \phi(X^{v,\kappa}_0) - \int_0^{\tau} f(X^{v,\kappa}_t) \, \mathrm{d}t - g(X^{v,\kappa}_\tau)
    \\ &\qquad \ + (1 - \frac{1}{\kappa})\big( \phi(X^{v,\kappa}_\tau) - \phi(X^{v,\kappa}_0) - \int_0^{\tau} \langle \nabla \phi(X^{v,\kappa}_t), \mathrm{d}X^{v,\kappa}_t \rangle \big)
    \bigr)^2 \bigr]
\end{split}
\end{talign*}

\paragraph{Step 1: Computing the directional derivative of $\mathcal{L}_{\mathrm{VM}_{v,\kappa}}$ along a perturbation $\eta$}
Observe that for an arbitrary perturbation $\eta$,
\begin{talign}
\begin{split}
    &\mathcal{L}_{\mathrm{VM}_{v,\kappa}}(\phi + \varepsilon \eta)
    \\ &=
    \mathbb{E} \bigl[ \frac{1}{2} \bigl( \int_0^{\tau} \langle \nabla (\phi+\varepsilon \eta)(X^{v,\kappa}_t), \mathrm{d}X^{v,\kappa}_t - ( b(X^{v,\kappa}_t) - \frac{1}{2} \sigma \sigma^{\top} \nabla (\phi+\varepsilon \eta)(X^{v,\kappa}_t) ) \, \mathrm{d}t \rangle \\ &\qquad\quad + (\phi+\varepsilon \eta)(X^{v,\kappa}_0) - \int_0^{\tau} f(X^{v,\kappa}_t) \, \mathrm{d}t - g(X^{v,\kappa}_{\tau})
    \\ &\qquad\quad + \! (1 \! - \! \frac{1}{\kappa})
    \big( (\phi \! + \! \varepsilon \eta)(X^{v,\kappa}_{\tau}) - (\phi \! + \! \varepsilon \eta)(X^{v,\kappa}_0)
    - \int_0^{\tau} \langle \nabla (\phi \! + \! \varepsilon \eta)(X^{v,\kappa}_t), \mathrm{d}X^{v,\kappa}_t \rangle
    \big)
    \bigr)^2 \bigr],
\end{split} \\
\begin{split} \label{eq:partial_epsilon_hat_L_stopping}
    &\implies \frac{\partial}{\partial \varepsilon} \mathcal{L}_{\mathrm{VM}_{v,\kappa}}(\phi + \varepsilon \eta) \rvert_{\varepsilon = 0}
    \\ &\quad = \mathbb{E} \big[ \bigl( \int_0^{\tau} \langle \nabla \phi(X^{v,\kappa}_t), \mathrm{d}X^{v,\kappa}_t - ( b(X^{v,\kappa}_t) - \frac{1}{2} \sigma \sigma^{\top} \nabla \phi(X^{v,\kappa}_t) ) \, \mathrm{d}t \rangle \\ &\qquad\quad \ + \phi(X^{v,\kappa}_0) - \int_0^{\tau} f(X^{v,\kappa}_t) \, \mathrm{d}t - g(X^{v,\kappa}_{\tau})
    \\ &\qquad\quad \ + (1 - \frac{1}{\kappa}) \big( \phi(X^{v,\kappa}_{\tau}) - \phi(X^{v,\kappa}_0)
    - \int_0^{\tau} \langle \nabla \phi(X^{v,\kappa}_t), \mathrm{d}X^{v,\kappa}_t \rangle
    \big)
    \bigr) \\ &\qquad\quad \times
    \bigl( \int_0^{\tau}
    \langle \nabla \eta(X^{v,\kappa}_t), \mathrm{d}X^{v,\kappa}_t - ( b(X^{v,\kappa}_t) - \sigma \sigma^{\top} \nabla \phi(X^{v,\kappa}_t) ) \, \mathrm{d}t \rangle
    + \eta(X^{v,\kappa}_0) \\ &\qquad\qquad + (1 - \frac{1}{\kappa})\big( \eta(X^{v,\kappa}_\tau) - \eta(X^{v,\kappa}_0) - \int_0^{\tau} \langle \nabla \eta(X^{v,\kappa}_t), \mathrm{d}X^{v,\kappa}_t \rangle \big) \bigr) \big].
\end{split}
\end{talign}
Applying Itô's lemma on $\phi$ and $\eta$ yields:
\begin{talign}
\begin{split} \label{eq:phi_eta_ito}
    \int_0^{\tau} \langle \nabla_x \phi(X^{v,\kappa}_t), \mathrm{d}X^{v,\kappa}_t \rangle &= \phi(X^{v,\kappa}_{\tau}) - \phi(X^{v,\kappa}_0) - \frac{\kappa}{2} \int_0^{\tau} \mathrm{Tr}\big( \sigma \sigma^{\top} \nabla^2_x \phi(X^{v,\kappa}_t) \big) \, \mathrm{d}t,
    \\ \int_0^{\tau} \langle \nabla_x \eta(X^{v,\kappa}_t), \mathrm{d}X^{v,\kappa}_t \rangle &= \eta(X^{v,\kappa}_{\tau}) - \eta(X^{v,\kappa}_0) - \frac{\kappa}{2} \int_0^{\tau} \mathrm{Tr}\big( \sigma \sigma^{\top} \nabla^2_x \eta(X^{v,\kappa}_t) \big) \, \mathrm{d}t.
\end{split}
\end{talign}
When we plug this into the right-hand side of \eqref{eq:partial_epsilon_hat_L_stopping}, we obtain:
\begin{talign*}
\begin{split}
    &\frac{\partial}{\partial \varepsilon} \mathcal{L}_{\mathrm{VM}_{v,\kappa}}(\phi + \varepsilon \eta) \rvert_{\varepsilon = 0}
    \\ &\quad = \mathbb{E} \big[ \bigl( - \int_0^{\tau} \big(\langle \nabla \phi(X^{v,\kappa}_t), b(X^{v,\kappa}_t) - \frac{1}{2} \sigma \sigma^{\top} \nabla \phi(X^{v,\kappa}_t) \rangle \big) \, \mathrm{d}t \\ &\qquad\quad \ + \phi(X^{v,\kappa}_{\tau}) - \int_0^{\tau} f(X^{v,\kappa}_t) \, \mathrm{d}t - g(X^{v,\kappa}_{\tau})
    - \frac{1}{2} \int_0^{\tau} \mathrm{Tr}\big( \sigma \sigma^{\top} \nabla^2_x \phi(X^{v,\kappa}_t) \big) \, \mathrm{d}t
    \bigr) \\ &\qquad\quad \times
    \bigl( \int_0^{\tau}
    \langle \nabla \eta(X^{v,\kappa}_t), \mathrm{d}X^{v,\kappa}_t - ( b(X^{v,\kappa}_t) - \sigma \sigma^{\top} \nabla \phi(X^{v,\kappa}_t) ) \, \mathrm{d}t \rangle
    + \eta(X^{v,\kappa}_0) \\ &\qquad\qquad \ - \frac{1}{2}(\frac{1}{\kappa} - 1) \int_0^{\tau} \mathrm{Tr}\big( \kappa \sigma \sigma^{\top} \nabla^2_x \eta(X^{v,\kappa}_t) \big) \, \mathrm{d}t \bigr) \big].
\end{split}
\end{talign*}
Applying the Girsanov change of variables from $X^{v,\kappa}$ to $\bar{X}^{\kappa}$, the analog of \eqref{eq:L_VM_F} is:
\begin{talign}
\begin{split} \label{eq:L_VM_F_uncapped}
    &\frac{\partial}{\partial \varepsilon} \mathcal{L}_{\mathrm{VM}_{v,\kappa}}(\phi + \varepsilon \eta) \rvert_{\varepsilon = 0}
    \\ &= \mathbb{E} \big[
    \bigl(
    - \int_0^{\tau} \big(\langle \nabla \phi(\bar{X}^{\kappa}_t), b(\bar{X}^{\kappa}_t) - \frac{1}{2} \sigma \sigma^{\top} \nabla \phi(\bar{X}^{\kappa}_t) \rangle \big) \, \mathrm{d}t \\ &\qquad\quad \ + \phi(\bar{X}^{\kappa}_{\tau}) - \int_0^{\tau} f(\bar{X}^{\kappa}_t) \, \mathrm{d}t - g(\bar{X}^{\kappa}_{\tau})
    - \frac{1}{2} \int_0^{\tau} \mathrm{Tr}\big( \sigma \sigma^{\top} \nabla^2_x \phi(\bar{X}^{\kappa}_t) \big) \, \mathrm{d}t
    \bigr) \\ &\qquad \times \bigl( \int_0^{\tau} \langle \sqrt{\kappa} \sigma^{\top} \nabla \eta(\bar{X}^{\kappa}_t), \mathrm{d}W_t \rangle \! + \! \eta(\bar{X}^{\kappa}_0) \! - \! \frac{1}{2}(\frac{1}{\kappa} \! - \! 1) \int_0^{\tau} \mathrm{Tr}\big( \kappa \sigma \sigma^{\top} \nabla^2_x \eta(\bar{X}^{\kappa}_t) \big) \, \mathrm{d}t \bigr) \mathrm{IW}(W,0,\tau)
    \big],
\end{split}
\end{talign}
where $\bar{X}^{\kappa}$ is a solution of $\mathrm{d}\bar{X}^{\kappa}_t \! = \! \big(b(\bar{X}^{\kappa}_s) \! - \!
\sigma \sigma^{\top} \nabla \phi(\bar{X}^{\kappa}_s) \big) \, \mathrm{d}t \! + \! \sqrt{\kappa} \sigma(\bar{X}^{\kappa}_t) \, \mathrm{d}W_t$.

\paragraph{Step 2: Applying the Dynkin representation theorem} The analog of the functional $F$
defined in \eqref{eq:def_F} is
\begin{talign}
\begin{split} \label{eq:def_F_2}
    F(\bar{X}^{\kappa}) &= - \int_0^{\tau} \langle \nabla \phi(\bar{X}^{\kappa}_t), b(\bar{X}^{\kappa}_t) - \frac{1}{2} \sigma \sigma^{\top} \nabla \phi(\bar{X}^{\kappa}_t) \rangle \, \mathrm{d}t \\ &\quad + \phi(\bar{X}^{\kappa}_{\tau}) - \int_0^{\tau} f(\bar{X}^{\kappa}_t) \, \mathrm{d}t - g(\bar{X}^{\kappa}_{\tau})
    - \frac{1}{2} \int_0^{\tau} \mathrm{Tr}\big( \sigma \sigma^{\top} \nabla^2_x \phi(\bar{X}^{\kappa}_t) \big) \, \mathrm{d}t \\ &=: \int_0^{\tau} h_1(\bar{X}^{\kappa}_t) \, \mathrm{d}t + h_2(\bar{X}^{\kappa}_{\tau}).
\end{split}
\end{talign}
We can apply the Dynkin representation theorem (\Cref{thm:pde-dynkin}) to the functional $F(\bar{X}^{\kappa})$. That is, setting $\sigma \gets \sqrt{\kappa} \sigma$, $\chi \gets \frac{1}{\kappa}-1 \in (-1,+\infty)$, $X \gets \bar{X}^{\kappa}$, we can write
\begin{talign} \label{eq:F_eta_star_hitting}
    F(\bar{X}^{\kappa}) = \eta^{\star}(\bar{X}^{\kappa}_0) + \int_0^{\tau} \langle \sqrt{\kappa} \sigma^{\top} \nabla \eta^{\star}(\bar{X}^{\kappa}_t), \mathrm{d}W_t \rangle - \frac{1}{2}(\frac{1}{\kappa} \! - \! 1) \int_0^{\tau} \mathrm{Tr}\big( \kappa \sigma \sigma^{\top} \nabla^2_x \eta^{\star}(\bar{X}^{\kappa}_t) \big) \, \mathrm{d}t,
\end{talign}
where $\eta^{\star}$ is the solution to the Dirichlet-Poisson problem
\begin{talign}
\begin{split} \label{eq:Dirichlet-Poisson_2}
&\mathcal{L} \eta^{\star}(x)+h_1(x)=0\quad\text{for }x\in \mathcal{S}^c,
\qquad
\eta^{\star}(x)=h_2(x)\quad\text{for }x\in\partial \mathcal{S}, \\
&\mathcal{L} \eta^{\star}= (b  -
\sigma \sigma^{\top} \nabla \phi)\!\cdot\!\nabla \eta^{\star}+\tfrac{1}{2}\mathrm{Tr}(\sigma \sigma^{\top}\nabla^2 \eta^{\star}).
\end{split}
\end{talign}
A solution $\eta^{\star} \in C^{2+\alpha}(\mathcal{S}^c)\cap C(\overline{\mathcal{S}^c})$ to this problem such that $\mathbb{E}\!\left[\int_0^{\tau}\!\!\|\sigma(\bar{X}^{\kappa}_t)^\top\nabla \eta^{\star}(\bar{X}^{\kappa}_t)\|^2\,dt\right]<\infty$ is guaranteed to exist by \Cref{rem:regularity_DP}.
The conditions stated in this remark are met because there exists $\alpha \in (0,1)$ such that the following facts hold:
\begin{itemize}[left=3pt]
\item $b, \sigma \in \mathcal{C}^{\alpha}(\overline{\mathcal{S}^c})$ (in fact they are $\mathcal{C}^{\alpha}(\mathbb{R}^d)$) and $\sigma$ is uniformly elliptic.
\item The boundary $\partial \mathcal{S}$ is $C^{2+\alpha}$.
\item $h_1\in C^{\alpha}(\overline{\mathcal{S}^c})$ and $h_2\in C^{2+\alpha}(\partial{\mathcal{S}})$, because for all $\phi \in \mathcal{F}$, $\phi \in C^{2+\alpha}(\mathbb{R}^d)$, and because $f \in C^{\alpha}(\mathbb{R}^d)$, $g \in C^{2+\alpha}(\mathbb{R}^d)$, and $b, \sigma \in \mathcal{C}^{\alpha}(\overline{\mathcal{S}^c})$.
\item $h_1$ and $h_2$ have at most polynomial growth because $\phi$, $\nabla \phi$, $\nabla^2 \phi$, $f$ and $g$ all have at most polynomial growth.
\end{itemize}

Hence, if we plug equations \eqref{eq:def_F_2}-\eqref{eq:F_eta_star_hitting} into \eqref{eq:L_VM_F_uncapped}, we obtain that
\begin{talign*}
\begin{split}
    &\frac{\partial}{\partial \varepsilon} \mathcal{L}_{\mathrm{VM}_{v,\kappa}}(\phi + \varepsilon \eta) \rvert_{\varepsilon = 0} \\ &= \mathbb{E} \big[ \big( \eta^{\star}(\bar{X}^{\kappa}_0) + \int_0^{\tau} \langle \sqrt{\kappa} \sigma^{\top} \nabla \eta^{\star}(\bar{X}^{\kappa}_t), \mathrm{d}W_t \rangle - \frac{1}{2}(\frac{1}{\kappa} \! - \! 1) \int_0^{\tau} \mathrm{Tr}\big( \kappa \sigma \sigma^{\top} \nabla^2_x \eta^{\star}(\bar{X}^{\kappa}_t) \big) \, \mathrm{d}t \big) \\ &\qquad \times \bigl( \eta(\bar{X}^{\kappa}_0) + \int_0^{\tau} \langle \sqrt{\kappa} \sigma^{\top} \nabla \eta(\bar{X}^{\kappa}_t), \mathrm{d}W_t \rangle - \frac{1}{2}(\frac{1}{\kappa} \! - \! 1) \int_0^{\tau} \mathrm{Tr}\big( \kappa \sigma \sigma^{\top} \nabla^2_x \eta(\bar{X}^{\kappa}_t) \big) \, \mathrm{d}t \bigr) \mathrm{IW}(W,0,\tau) \big]
\end{split}
\end{talign*}

\paragraph{Step 3: Choosing the right perturbation $\eta$ to obtain an almost sure equality}
And as in the proof of \Cref{thm:vm_finite}, if we pick $\eta(x) = - \eta^{\star}(x)$, we obtain that \begin{talign*}
\begin{split}
    \frac{\partial}{\partial \varepsilon} \mathcal{L}_{\mathrm{VM}_{v,\kappa}}(\phi + \varepsilon \eta) \rvert_{\varepsilon = 0} = - \mathbb{E}[F(\bar{X}^{\kappa})^2 \mathrm{IW}(W,0,\tau)] = - \mathbb{E}[F(X^{v,\kappa})^2],
\end{split}
\end{talign*}
In this case, it is convenient to further define the diffusion bridge process $X^{v,\kappa,x,y}$ as the solution of the SDE \eqref{eq:X_v_kappa} conditioned on the fact that $X^{v,\kappa,x,y}_0 = x$ and $X^{v,\kappa,x,y}_{\tau} = y$. Let $\mathbb{P}^{v,\kappa,x,y}$ be the path probability measure of the diffusion bridge process $X^{v,\kappa,x,y}$. If we let $p_{v,\kappa}$ be the joint distribution of $(X^{v,\kappa}_0,X^{v,\kappa}_\tau)$ observe that through the tower property of conditional expectation, we can rewrite
\begin{talign*}
    \frac{\partial}{\partial \varepsilon} \mathcal{L}_{\mathrm{VM}_{v,\kappa}}(\hat{\phi} + \varepsilon \eta) \rvert_{\varepsilon = 0} = - \mathbb{E}[F(X^{v,\kappa})^2] = - \int_{(\mathbb{R}^d)^2} \mathbb{E}[F(X^{v,\kappa,x,y})^2] \, \mathrm{d}p_{v,\kappa}(x,y),
\end{talign*}
Thus, if $\hat{\phi}$ satisfies that $\frac{\partial}{\partial \varepsilon} \mathcal{L}_{\mathrm{VM}_{v,\kappa}}(\hat{\phi} + \varepsilon \eta) \rvert_{\varepsilon = 0} = 0$ for all perturbations $\eta$, we need that when $\phi = \hat{\phi}$, for $p_{v,\kappa}$-almost every $(x,y)$, $\mathbb{E}[F(X^{v,\kappa,x,y})^2] = 0$. Hence, for $p_{v,\kappa}$-almost every $(x,y)$, we must have that for $\mathbb{P}^{v,\kappa,x,y}$-almost every path $X^{v,\kappa,x,y}$, $F(X^{v,\kappa,x,y}) = 0$.

\paragraph{Step 4: Writing down the Radon-Nikodym derivative $\frac{\mathrm{d}\mathbb{P}^{u^{\star}}}{\mathrm{d}\mathbb{P}^{\hat{\phi}}}$} The analog of \eqref{eq:P_u_star_P_v_Y} is as follows: for any process $Y$ whose probability measure is absolutely continuous with respect to $\mathbb{P}$ and/or $\mathbb{P}^{u^{\star}}$,
\begin{talign}
\begin{split} \label{eq:P_u_star_P_v_Y_uncapped}
    &\frac{\mathrm{d}\mathbb{P}^{u^{\star}}}{\mathrm{d}\mathbb{P}^{\hat{\phi}}}(Y) = \frac{\mathrm{d}\mathbb{P}^{u^{\star}}}{\mathrm{d}\mathbb{P}}(Y) \times \frac{\mathrm{d}\mathbb{P}}{\mathrm{d}\mathbb{P}^{\hat{\phi}}}(Y) \\ &= \exp \big( \int_0^{\tau} \langle \nabla \hat{\phi}(Y_t), \mathrm{d}Y_t - ( b(Y_t) - \frac{1}{2} \sigma \sigma^{\top} \nabla \hat{\phi}(Y_t) ) \, \mathrm{d}t
    \rangle
    + V(Y_0) - \int_0^{\tau} f(Y_t) \, \mathrm{d}t - g(Y_{\tau}) \big).
\end{split}
\end{talign}
\paragraph{Step 5(i): Concluding the proof for the case $\kappa = 1$, $\mu_0$ and $\tilde{\mu}_0$ mutually absolutely continuous} The argument is the same as in Step 5(i) of the proof of \Cref{thm:vm_finite}: we apply Itô's lemma and use the resulting equation to conclude that $\mathbb{P}^{v,1}$-almost surely,
\begin{talign}
\begin{split} \label{eq:dP_u_star_dP_hatphi_uncapped}
\frac{\mathrm{d}\mathbb{P}^{u^{\star}}}{\mathrm{d}\mathbb{P}^{\hat{\phi}}}(X^{v,1}) &= \frac{\mathrm{d}\mathbb{P}^{u^{\star}}}{\mathrm{d}\mathbb{P}}(X^{v,1}) \times \frac{\mathrm{d}\mathbb{P}}{\mathrm{d}\mathbb{P}^{\hat{\phi}}}(X^{v,1}) \\ &= \exp \big( \int_0^{\tau} \langle \nabla \hat{\phi}(X^{v,1}_t), \mathrm{d}X^{v,1}_t - ( b(X^{v,1}_t) - \frac{1}{2} \sigma \sigma^{\top} \nabla \hat{\phi}(X^{v,1}_t) ) \, \mathrm{d}t
    \rangle  \\ &\qquad\qquad\qquad + V(X^{v,1}_0) - \int_0^{\tau} f(X^{v,1}_t) \, \mathrm{d}t - g(X^{v,1}_{\tau}) \big)
    \\ &= \exp \big( V(X^{v,1}_0) - \hat{\phi}(X^{v,1}_0) \big).
\end{split}
\end{talign}
Then we use that $\mathbb{P}^{u^{\star}}$ and $\mathbb{P}^{\hat{\phi}}$ both have marginal $\mu_0$ at time $t=0$ to conclude that $V(x) = \hat{\phi}(x)$ for all $x \in \mathrm{supp}(\mu_0)$, which then implies that $\frac{\mathrm{d}\mathbb{P}^{u^{\star}}}{\mathrm{d}\mathbb{P}^{\hat{\phi}}}(X^{v,1}) = 1$, and by absolute continuity of $\mathbb{P}^{u^{\star}}$ with respect to the probability measure of $X^{v,1}$, $\frac{\mathrm{d}\mathbb{P}^{u^{\star}}}{\mathrm{d}\mathbb{P}^{\hat{\phi}}}(X^{u^{\star}}) = 1$. And applying the Girsanov theorem in analogy with \eqref{eq:girsanov_app_KL}, and using that $u^{\star} = - \sigma^{\top} \nabla V$ and $- \sigma^{\top} \nabla \hat{\phi}$ are continuous, we obtain that $u^{\star}(x) = - \sigma^{\top} \nabla V(x) = - \sigma^{\top} \nabla \hat{\phi}(x)$ for all $x \in \mathbb{R}^d$, which by the non-degeneracy of $\sigma$ implies that $\nabla V(x) = \nabla \hat{\phi}(x)$ for all $x \in \mathbb{R}^d$. Thus, using the fundamental theorem of calculus, we have that for any $x \in \mathbb{R}^d$ and any $y \in \mathrm{supp}(\mu_0)$,
\begin{talign*}
    V(x) - \hat{\phi}(x) = V(y) - \hat{\phi}(y) + \int_0^1 \langle \nabla V(t x + (1-t)y) - \nabla \hat{\phi}(t x + (1-t)y), x - y \rangle \, \mathrm{d}t = 0,
\end{talign*}
which concludes the proof.

\paragraph{Step 5(ii): Concluding the proof for the general case}  The argument is similar to the one in Step 5(ii) of the proof of \Cref{thm:vm_finite}, with some changes.

For any $s \geq 0$ and any $x, y \in \mathbb{R}^d$, define
\begin{talign*}
\begin{split}
    &\mathcal{C}^{\mathrm{stop}}_{x,y,\infty,s}
    \! = \! \{w\in C([0,\infty);\mathbb{R}^d) \, | \, w(t)=w(\tau_w) = y\ \text{for all }t\ge \tau_w, w(0) = x, \tau_w = s \}, \\
    &\mathcal{C}^{\mathrm{stop},v,\kappa}_{x,y,F} = \{ X^{v,\kappa} \in
    \mathcal{C}^{\mathrm{stop}}_{x,y,\infty, s} \, | \, X^{v,\kappa} \text{ solves } \eqref{eq:X_v_kappa} \text{ on } [0,\tau], \, F(X^{v,\kappa}) = 0 \},
\end{split}
\end{talign*}
where $\tau_w$ denotes the hitting time of the path $w$ on the set $\mathcal{S}$.
Observe that when restricted on $\mathcal{C}^{\mathrm{stop}}_{x,y,\infty, s}$, $F$ takes the form
\begin{talign*}
\begin{split}
    F(X) &= - \int_0^{s} \langle \nabla \hat{\phi}(X_t), b(X_t) - \frac{1}{2} \sigma \sigma^{\top} \nabla \hat{\phi}(X_t) \rangle \, \mathrm{d}t \\ &\quad + \hat{\phi}(y) - \int_0^{s} f(X_t) \, \mathrm{d}t - g(y)
    - \frac{1}{2} \int_0^{s} \mathrm{Tr}\big( \sigma \sigma^{\top} \nabla^2_x \hat{\phi}(X_t) \big) \, \mathrm{d}t \\ &= \int_0^{s} h_1(X_t) \, \mathrm{d}t + h_2(y),
\end{split}
\end{talign*}
which means that the restriction of $F$ to $\mathcal{C}^{\mathrm{stop}}_{x,y,\infty, s}$ is continuous in the uniform norm over $[0,s]$, by the argument of Lemma \ref{lem:exp_D_t_F}.

Recall the notation $\mathbb{P}^{v,\kappa,x,y}$ and $p_{v,\kappa}$ defined in Step 3 of this proof.
In Step 3, we obtained that for $p_{v,\kappa}$-almost every $(x,y) \in (\mathbb{R}^d)^2$, we have that for $\mathbb{P}^{v,\kappa,x,y}$-almost every path $X^{v,\kappa,x,y}$, $F(X^{v,\kappa,x,y}) = 0$. Equivalently, we have that for $p_{v,\kappa}$-almost every $(x,y) \in (\mathbb{R}^d)^2$, $\mathcal{C}^{\mathrm{stop},v,\kappa}_{x,y,F}$ is dense in $\mathrm{supp}(\mathbb{P}^{v,\kappa,x,y})$ in the uniform norm over $[0,s]$.

By \Cref{thm:full_supp_bridge} and using the regularity, growth and uniform ellipticity assumptions
on $v$, we have that $\mathrm{supp}(\mathbb{P}^{v,\kappa,x,y})$ is equal to $\mathcal{C}^{\mathrm{stop}}_{x,y,\infty,s}$ in the uniform norm. Hence, $\mathcal{C}^{\mathrm{stop},v,\kappa}_{x,y,F}$ is dense in $\mathcal{C}^{\mathrm{stop}}_{x,y,\infty,s}$ in the uniform norm.

Since $F$ takes value zero on $\mathcal{C}^{\mathrm{stop},v,\kappa}_{x,y,F}$ by construction, and $\mathcal{C}^{\mathrm{stop},v,\kappa}_{x,y,F}$ is dense in $\mathcal{C}^{\mathrm{stop}}_{x,y,\infty,s}$ for a.e. $(x,y) \in (\mathbb{R}^d)^2$, and $F$ is continuous when restricted to $\mathcal{C}^{\mathrm{stop}}_{x,y,\infty,s}$, we conclude that $F$ is the zero functional on $\mathcal{C}^{\mathrm{stop}}_{x,y,\infty,s}$ for a.e. $(x,y) \in (\mathbb{R}^d)^2$.

Now, observe that if we define
\begin{talign*}
    \mathcal{C}^{\mathrm{stop}}_{\mathrm{supp}(\tilde{\mu}_0),\infty}
    \! = \! \{w\in C([0,\infty);\mathbb{R}^d) \, | \, w(t)=w(\tau_w)\ \text{for all }t\ge \tau_w, w(0) \in \mathrm{supp}(\tilde{\mu}_0) \},
\end{talign*}
we have that
\begin{talign*}
    \mathcal{C}^{\mathrm{stop}}_{\mathrm{supp}(\tilde{\mu}_0),\infty} = \bigcup_{s \in (0,\infty), x \in \mathrm{supp}(\tilde{\mu}_0), y \in \partial \mathcal{S}} \mathcal{C}^{\mathrm{stop}}_{x,y,\infty,s}
\end{talign*}
Hence, $F(X) = 0$ for almost every $X \in \mathcal{C}^{\mathrm{stop}}_{\mathrm{supp}(\tilde{\mu}_0),\infty}$. Since $\mu_0$ is absolutely continuous with respect to $\tilde{\mu}_0$, we obtain that almost surely,
\begin{talign*}
    F(X^{u^{\star}}) = 0.
\end{talign*}
Then, by Itô's lemma,
\begin{talign*}
\begin{split}
    \int_0^{\tau} \langle \nabla \hat{\phi}(X^{u^{\star}}_t), \mathrm{d}X^{u^{\star}}_t \rangle = \hat{\phi}(X^{u^{\star}}_{\tau}) - \hat{\phi}(X^{u^{\star}}_0) - \int_0^{\tau} \frac{1}{2} \mathrm{Tr}\big( \sigma \sigma^{\top} \nabla^2_x \hat{\phi}(X^{u^{\star}}_t) \big) \, \mathrm{d}t,
\end{split}
\end{talign*}
which allows us to re-express $F(X^{u^{\star}})$ almost surely as follows:
\begin{talign*}
\begin{split}
    0 &= F(X^{u^{\star}}) = \int_0^{\tau} \langle \nabla \hat{\phi}(X^{u^{\star}}_t), \mathrm{d}X^{u^{\star}}_t - ( b(X^{u^{\star}}_t) - \frac{1}{2} \sigma \sigma^{\top} \nabla \hat{\phi}(X^{u^{\star}}_t) ) \, \mathrm{d}t \rangle \\ &\qquad\qquad\qquad + \hat{\phi}(X^{u^{\star}}_0) - \int_0^{\tau} f(X^{u^{\star}}_t) \, \mathrm{d}t - g(X^{u^{\star}}_{\tau}).
\end{split}
\end{talign*}
As in Step 5(i), this equality now implies that
\begin{talign*}
\frac{\mathrm{d}\mathbb{P}^{u^{\star}}}{\mathrm{d}\mathbb{P}^{\hat{\phi}}}(X^{u^{\star}}) = \frac{\mathrm{d}\mathbb{P}^{u^{\star}}}{\mathrm{d}\mathbb{P}}(X^{u^{\star}}) \times \frac{\mathrm{d}\mathbb{P}}{\mathrm{d}\mathbb{P}^{\hat{\phi}}}(X^{u^{\star}}) = \exp \big( V(X^{u^{\star}}_0) - \hat{\phi}(X^{u^{\star}}_0) \big),
\end{talign*}
and the remainder of the proof follows analogously to Step 5(i).

\subsection{Proof of Theorem \ref{thm:capped_vm_loss}: The Value Matching loss for hitting time problems} \label{subsec:pf_VM_hitting}
We prove the result for the SOC problem including a state cost $f$, which means that the VM loss function reads:
\begin{talign}
\begin{split} \label{eq:VM_loss_f}
    \mathcal{L}_{\mathrm{VM}_{v,\kappa}}(\phi) &=
    \mathbb{E} \bigl[ \frac{1}{2} \bigl( \int_0^{\tau(T)} \langle \nabla \phi(X^{v,\kappa}_t), \frac{1}{\kappa} \mathrm{d}X^{v,\kappa}_t - ( b(X^{v,\kappa}_t) - \frac{1}{2} \sigma \sigma^{\top} \nabla \phi(X^{v,\kappa}_t) ) \, \mathrm{d}t \rangle \\ &\qquad + \! \frac{1}{\kappa} \phi(X^{v,\kappa}_0) \! - \! \int_0^{\tau(T)} f(X^{v,\kappa}_t) \, \mathrm{d}t
    \! + \! (\mathbf{1}_{\mathcal{S}}(X^{v,\kappa}_{\tau(T)}) \! - \! \frac{1}{\kappa})
    \phi(X^{v,\kappa}_{\tau(T)})
    \! - \! \mathbf{1}_{\mathcal{S}}(X^{v,\kappa}_{\tau(T)}) g(X^{v,\kappa}_{\tau(T)})
    \bigr)^2 \bigr].
\end{split}
\end{talign}
We follow the same steps as in the proof of \Cref{thm:vm_loss_hitting}. In this case, Steps 2, 4 and 5 present additional complexity. We start by rewriting the loss function to flesh out the dependence on $\kappa$:
\begin{talign*}
\begin{split}
    \mathcal{L}_{\mathrm{VM}_{v,\kappa}}(\phi) &=
    \mathbb{E} \bigl[ \frac{1}{2} \bigl( \int_0^{\tau(T)} \langle \nabla_x \phi(X^{v,\kappa}_t), \mathrm{d}X^{v,\kappa}_t - ( b(X^{v,\kappa}_t) - \frac{1}{2} \sigma \sigma^{\top} \nabla_x \phi(X^{v,\kappa}_t) ) \, \mathrm{d}t \rangle \\ &\qquad \ + \phi(X^{v,\kappa}_0) - \int_0^{\tau(T)} f(X^{v,\kappa}_t) \, \mathrm{d}t
    - \mathbf{1}_{\mathcal{S}^c}(X^{v,\kappa}_{\tau(T)}) \phi(X^{v,\kappa}_{\tau(T)}) - \mathbf{1}_{\mathcal{S}}(X^{v,\kappa}_{\tau(T)}) g(X^{v,\kappa}_{\tau(T)})
    \\ &\qquad \ + (1 - \frac{1}{\kappa}) \big( \phi(X^{v,\kappa}_{\tau(T)}) - \phi(X^{v,\kappa}_0)
    - \int_0^{\tau(T)} \langle \nabla_x \phi(X^{v,\kappa}_t), \mathrm{d}X^{v,\kappa}_t \rangle
    \big)
    \bigr)^2 \bigr]
\end{split}
\end{talign*}

\paragraph{Step 1: Computing the directional derivative of $\mathcal{L}_{\mathrm{VM}_{v,\kappa}}$ along a perturbation $\eta$} For an arbitrary perturbation $\eta \in \mathcal{F}$, we can write
\begin{talign}
\begin{split}
    &\mathcal{L}_{\mathrm{VM}_{v,\kappa}}(\phi + \varepsilon \eta) \\ &= \!
    \mathbb{E} \bigl[ \frac{1}{2} \bigl( \int_0^{\tau(T)} \langle \nabla (\phi+\varepsilon \eta)(X^{v,\kappa}_t), \mathrm{d}X^{v,\kappa}_t - ( b(X^{v,\kappa}_t) - \frac{1}{2} \sigma \sigma^{\top} \nabla (\phi+\varepsilon \eta)(X^{v,\kappa}_t) ) \, \mathrm{d}t \rangle \\ &\qquad + \! (\phi+\varepsilon \eta)(X^{v,\kappa}_0) \! - \! \int_0^{\tau(T)} f(X^{v,\kappa}_t) \, \mathrm{d}t
    \! - \! \mathbf{1}_{\mathcal{S}^c}(X^{v,\kappa}_{\tau(T)}) (\phi\!+ \!\varepsilon \eta)(X^{v,\kappa}_{\tau(T)}) \! - \! \mathbf{1}_{\mathcal{S}}(X^{v,\kappa}_{\tau(T)}) g(X^{v,\kappa}_{\tau(T)})
    \\ &\qquad + \! (1 \! - \! \frac{1}{\kappa}) \big( (\phi \! + \! \varepsilon \eta)(X^{v,\kappa}_{\tau(T)}) \! - \! (\phi \! + \! \varepsilon \eta)(X^{v,\kappa}_0)
    \! - \! \int_0^{\tau(T)} \langle \nabla (\phi \! + \! \varepsilon \eta)(X^{v,\kappa}_t), \mathrm{d}X^{v,\kappa}_t \rangle \big) \bigr)^2 \bigr]
\end{split}
\end{talign}
Hence,
\begin{talign}
\begin{split} \label{eq:grad_VM_hitting_time}
    &\frac{\partial}{\partial \varepsilon} \mathcal{L}_{\mathrm{VM}_{v,\kappa}}(\phi + \varepsilon \eta) \rvert_{\varepsilon = 0} \\ &\qquad = \mathbb{E} \bigl[ \bigl( \int_0^{\tau(T)} \langle \nabla \phi(X^{v,\kappa}_t), \mathrm{d}X^{v,\kappa}_t - ( b(X^{v,\kappa}_t) - \frac{1}{2} \sigma \sigma^{\top} \nabla \phi(X^{v,\kappa}_t) ) \, \mathrm{d}t \rangle \\ &\qquad\qquad \ + \phi(X^{v,\kappa}_0) - \int_0^{\tau(T)} f(X^{v,\kappa}_t) \, \mathrm{d}t
    - \mathbf{1}_{\mathcal{S}^c}(X^{v,\kappa}_{\tau(T)}) \phi(X^{v,\kappa}_{\tau(T)}) - \mathbf{1}_{\mathcal{S}}(X^{v,\kappa}_{\tau(T)}) g(X^{v,\kappa}_{\tau(T)})
    \\ &\qquad\qquad \ + (1 - \frac{1}{\kappa}) \big( \phi(X^{v,\kappa}_{\tau(T)}) - \phi(X^{v,\kappa}_0)
    - \int_0^{\tau(T)} \langle \nabla \phi(X^{v,\kappa}_t), \mathrm{d}X^{v,\kappa}_t \rangle \big) \bigr) \\ &\qquad\qquad \times \big( \int_0^{\tau(T)}
    \langle \nabla \eta(X^{v,\kappa}_t), \mathrm{d}X^{v,\kappa}_t - ( b(X^{v,\kappa}_t) - \sigma \sigma^{\top} \nabla \phi(X^{v,\kappa}_t) ) \, \mathrm{d}t \rangle
    \\ &\qquad\qquad\quad \ + \eta(X^{v,\kappa}_0) - \mathbf{1}_{\mathcal{S}^c}(X^{v,\kappa}_{\tau(T)}) \eta(X^{v,\kappa}_{\tau(T)}) \\ &\qquad\qquad\quad \ + (1 - \frac{1}{\kappa}) \big( \eta(X^{v,\kappa}_{\tau(T)}) - \eta(X^{v,\kappa}_0)
    - \int_0^{\tau(T)} \langle \nabla \eta(X^{v,\kappa}_t), \mathrm{d}X^{v,\kappa}_t \rangle \big)\big) \bigr]
\end{split}
\end{talign}
Applying Itô's lemma to $\phi$ and $\eta$ yields
\begin{talign}
\begin{split} \label{eq:phi_eta_ito_stopping}
    &\int_0^{\tau(T)} \langle \nabla_x \phi(X^{v,\kappa}_t), \mathrm{d}X^{v,\kappa}_t \rangle = \phi(X^{v,\kappa}_{\tau(T)}) - \phi(X^{v,\kappa}_0) - \frac{\kappa}{2} \int_0^{\tau(T)} \mathrm{Tr}\big( \sigma \sigma^{\top} \nabla^2_x \phi(X^{v,\kappa}_t) \big) \, \mathrm{d}t, \\
    &\int_0^{\tau(T)} \langle \nabla_x \eta(X^{v,\kappa}_t), \mathrm{d}X^{v,\kappa}_t \rangle = \eta(X^{v,\kappa}_{\tau(T)}) - \eta(X^{v,\kappa}_0) - \frac{\kappa}{2} \int_0^{\tau(T)} \mathrm{Tr}\big( \sigma \sigma^{\top} \nabla^2_x \eta(X^{v,\kappa}_t) \big) \, \mathrm{d}t.
\end{split}
\end{talign}
Plugging this into the right-hand side of \eqref{eq:grad_VM_hitting_time} yields
\begin{talign}
\begin{split} \label{eq:grad_VM_hitting_time_2}
    &\frac{\partial}{\partial \varepsilon} \mathcal{L}_{\mathrm{VM}_{v,\kappa}}(\phi + \varepsilon \eta) \rvert_{\varepsilon = 0} \\ &= \mathbb{E} \bigl[
    \bigl( - \int_0^{\tau(T)} \langle \nabla \phi(X^{v,\kappa}_t), b(X^{v,\kappa}_t) - \frac{1}{2} \sigma \sigma^{\top} \nabla \phi(X^{v,\kappa}_t) \rangle \, \mathrm{d}t - \frac{1}{2} \int_0^{\tau(T)} \mathrm{Tr}\big( \sigma \sigma^{\top} \nabla^2_x \phi(X^{v,\kappa}_t) \big) \, \mathrm{d}t \\ &\qquad\quad - \int_0^{\tau(T)} f(X^{v,\kappa}_t) \, \mathrm{d}t
    + \mathbf{1}_{\mathcal{S}}(X^{v,\kappa}_{\tau(T)}) \big( \phi(X^{v,\kappa}_{\tau(T)}) - g(X^{v,\kappa}_{\tau(T)}) \big) \bigr)
    \\ &\qquad \times \big( \int_0^{\tau(T)}
    \langle \nabla \eta(X^{v,\kappa}_t), \mathrm{d}X^{v,\kappa}_t - ( b(X^{v,\kappa}_t) - \sigma \sigma^{\top} \nabla \phi(X^{v,\kappa}_t) ) \, \mathrm{d}t \rangle
    \\ &\qquad\qquad + \eta(X^{v,\kappa}_0) - \mathbf{1}_{\mathcal{S}^c}(X^{v,\kappa}_{\tau(T)}) \eta(X^{v,\kappa}_{\tau(T)}) + \frac{1}{2}(1 - \frac{1}{\kappa}) \int_0^{\tau(T)} \mathrm{Tr}\big( \kappa \sigma \sigma^{\top} \nabla^2_x \eta(X^{v,\kappa}_t) \big) \, \mathrm{d}t \big) \bigr]
\end{split}
\end{talign}
And after applying a change of variables through the Girsanov theorem in analogy with \eqref{eq:L_VM_F}, we obtain that:
\begin{talign}
\begin{split} \label{eq:L_VM_F_capped}
    &\frac{\partial}{\partial \varepsilon} \mathcal{L}_{\mathrm{VM}_{v,\kappa}}(\phi + \varepsilon \eta) \rvert_{\varepsilon = 0}
    \\ &= \mathbb{E} \big[
    \bigl( - \int_0^{\tau(T)} \langle \nabla \phi(\bar{X}^{\kappa}_t), b(\bar{X}^{\kappa}_t) - \frac{1}{2} \sigma \sigma^{\top} \nabla \phi(\bar{X}^{\kappa}_t) \rangle \, \mathrm{d}t - \frac{1}{2} \int_0^{\tau(T)} \mathrm{Tr}\big( \sigma \sigma^{\top} \nabla^2_x \phi(\bar{X}^{\kappa}_t) \big) \, \mathrm{d}t \\ &\qquad\quad - \int_0^{\tau(T)} f(\bar{X}^{\kappa}_t) \, \mathrm{d}t
    + \mathbf{1}_{\mathcal{S}}(\bar{X}^{\kappa}_{\tau(T)}) \big( \phi(\bar{X}^{\kappa}_{\tau(T)}) - g(\bar{X}^{\kappa}_{\tau(T)}) \big) \bigr)
    \\ &\qquad \times \bigl( \eta(\bar{X}^{\kappa}_0) + \int_0^{\tau(T)} \langle \sqrt{\kappa} \sigma^{\top} \nabla \eta(\bar{X}^{\kappa}_t), \mathrm{d}W_t \rangle \! + \! \frac{1}{2}(1 \! - \! \frac{1}{\kappa}) \int_0^{\tau(T)} \mathrm{Tr}\big( \kappa \sigma \sigma^{\top} \nabla^2_x \eta(X^{v,\kappa}_t) \big) \, \mathrm{d}t \\ &\qquad\qquad - \mathbf{1}_{\mathcal{S}^c}(\bar{X}^{\kappa}_{\tau(T)}) \eta(\bar{X}^{\kappa}_{\tau(T)}) \bigr) \times \mathrm{IW}(W,0,\tau(T))
    \big],
\end{split}
\end{talign}
where $\bar{X}^{\kappa}$ is a solution of $\mathrm{d}\bar{X}^{\kappa}_t \! = \! \big(b(\bar{X}^{\kappa}_t) \! - \! \sigma \sigma^{\top} \nabla \phi(\bar{X}^{\kappa}_t) \big) \, \mathrm{d}t \! + \! \sqrt{\kappa} \sigma(\bar{X}^{\kappa}_t) \, \mathrm{d}W_t$.
We define the truncated functional $F^{(T)}$ as the first factor in the expectation in equation \eqref{eq:grad_VM_hitting_time_2}:
\begin{talign}
\begin{split} \label{eq:F_T_def}
    F^{(T)}(\bar{X}^{\kappa}) &= - \int_0^{\tau(T)} \langle \nabla \phi(\bar{X}^{\kappa}_t), b(\bar{X}^{\kappa}_t) - \frac{1}{2} \sigma \sigma^{\top} \nabla \phi(\bar{X}^{\kappa}_t) \rangle \, \mathrm{d}t - \frac{1}{2} \int_0^{\tau(T)} \mathrm{Tr}\big( \sigma \sigma^{\top} \nabla^2_x \phi(\bar{X}^{\kappa}_t) \big) \, \mathrm{d}t \\ &\quad - \int_0^{\tau(T)} f(\bar{X}^{\kappa}_t) \, \mathrm{d}t
    + \mathbf{1}_{\mathcal{S}}(\bar{X}^{\kappa}_{\tau(T)}) \big( \phi(\bar{X}^{\kappa}_{\tau(T)}) - g(\bar{X}^{\kappa}_{\tau(T)}) \big).
\end{split}
\end{talign}
\paragraph{Step 2(i): Applying the Dynkin representation theorem to $F$} This step is exactly as Step 2 of \Cref{thm:vm_loss_hitting}. $F(\bar{X}^{\kappa})$ is defined as in equation \eqref{eq:def_F_2} and $\eta^{\star}$
is defined as the solution of the Dirichlet-Poisson problem \eqref{eq:Dirichlet-Poisson_2}, and through the same argument we conclude that
\begin{talign} \label{eq:F_eta_star_2}
F(\bar{X}^{\kappa}) = \eta^{\star}(\bar{X}^{\kappa}_0) + \int_0^{\tau} \langle \sqrt{\kappa} \sigma^{\top} \nabla \eta^{\star}(\bar{X}^{\kappa}_t), \mathrm{d}W_t \rangle
- \frac{1}{2}(\frac{1}{\kappa} \! - \! 1) \int_0^{\tau} \mathrm{Tr}\big( \kappa \sigma \sigma^{\top} \nabla^2_x \eta^{\star}(\bar{X}^{\kappa}_t) \big) \, \mathrm{d}t.
\end{talign}

\paragraph{Step 2(ii): Obtaining a representation of the truncated functional $F^{(T)}$ in terms of $\eta^{\star}$} And observe that if $(\mathcal{F}_t)_{t \geq 0}$ is the filtration associated to the Brownian motion $W$,
\begin{talign}
\begin{split} \label{eq:cond_exp_F}
    &\mathbb{E}[F(\bar{X}^{\kappa})|\mathcal{F}_T] \\ &= \mathbb{E} \big[ - \int_0^{\tau} \langle \nabla \phi(\bar{X}^{\kappa}_t), b(\bar{X}^{\kappa}_t) - \frac{1}{2} \sigma \sigma^{\top} \nabla \phi(\bar{X}^{\kappa}_t) \rangle \, \mathrm{d}t \\ &\quad + \phi(\bar{X}^{\kappa}_{\tau}) - \int_0^{\tau} f(\bar{X}^{\kappa}_t) \, \mathrm{d}t - g(\bar{X}^{\kappa}_{\tau})
    - \frac{1}{2} \int_0^{\tau} \mathrm{Tr}\big( \sigma \sigma^{\top} \nabla^2_x \phi(\bar{X}^{\kappa}_t) \big) \, \mathrm{d}t \big| \mathcal{F}_T \big] \\ &=
    - \int_0^{\tau(T)} \langle \nabla \phi(\bar{X}^{\kappa}_t), b(\bar{X}^{\kappa}_t) - \frac{1}{2} \sigma \sigma^{\top} \nabla \phi(\bar{X}^{\kappa}_t) \rangle \, \mathrm{d}t \\ &\quad - \int_0^{\tau(T)} f(\bar{X}^{\kappa}_t) \, \mathrm{d}t + \mathbf{1}_{\mathcal{S}}(\bar{X}^{\kappa}_{\tau(T)}) \big( \phi(\bar{X}^{\kappa}_{\tau(T)}) - g(\bar{X}^{\kappa}_{\tau(T)}) \big)
    - \frac{1}{2} \int_0^{\tau(T)} \mathrm{Tr}\big( \sigma \sigma^{\top} \nabla^2_x \phi(\bar{X}^{\kappa}_t) \big) \, \mathrm{d}t \\ &\quad + \mathbf{1}_{\mathcal{S}^c}(\bar{X}^{\kappa}_{\tau(T)}) \mathbb{E}[ - \int_T^{\tau} \langle \nabla \phi(\bar{X}^{\kappa}_t), b(\bar{X}^{\kappa}_t) - \frac{1}{2} \sigma \sigma^{\top} \nabla \phi(\bar{X}^{\kappa}_t) \rangle \, \mathrm{d}t \\ &\qquad\qquad\qquad\quad \ + \! \phi(\bar{X}^{\kappa}_{\tau}) \! - \! \int_T^{\tau} f(\bar{X}^{\kappa}_t) \, \mathrm{d}t \! - \! g(\bar{X}^{\kappa}_{\tau})
    \! - \! \frac{1}{2} \int_T^{\tau} \mathrm{Tr}\big( \sigma \sigma^{\top} \nabla^2_x \phi(\bar{X}^{\kappa}_t) \big) \, \mathrm{d}t | \mathcal{F}_T ].
\end{split}
\end{talign}
Next, we have that
\begin{talign}
\begin{split} \label{eq:cond_exp_eta_star}
    &\mathbb{E}[ - \int_T^{\tau} \langle \nabla \phi(\bar{X}^{\kappa}_t), b(\bar{X}^{\kappa}_t) - \frac{1}{2} \sigma \sigma^{\top} \nabla \phi(\bar{X}^{\kappa}_t) \rangle \, \mathrm{d}t \\ &\quad + \! \phi(\bar{X}^{\kappa}_{\tau}) \! - \! \int_T^{\tau} f(\bar{X}^{\kappa}_t) \, \mathrm{d}t \! - \! g(\bar{X}^{\kappa}_{\tau})
    \! - \! \frac{1}{2} \int_T^{\tau} \mathrm{Tr}\big( \sigma \sigma^{\top} \nabla^2_x \phi(\bar{X}^{\kappa}_t) \big) \, \mathrm{d}t | \mathcal{F}_T ] \\ &= \mathbb{E}[ - \int_T^{\tau} \langle \nabla \phi(\bar{X}^{\kappa}_t), b(\bar{X}^{\kappa}_t) - \frac{1}{2} \sigma \sigma^{\top} \nabla \phi(\bar{X}^{\kappa}_t) \rangle \, \mathrm{d}t \\ &\qquad + \! \phi(\bar{X}^{\kappa}_{\tau}) \! - \! \int_T^{\tau} f(\bar{X}^{\kappa}_t) \, \mathrm{d}t \! - \! g(\bar{X}^{\kappa}_{\tau})
    \! - \! \frac{1}{2} \int_T^{\tau} \mathrm{Tr}\big( \sigma \sigma^{\top} \nabla^2_x \phi(\bar{X}^{\kappa}_t) \big) \, \mathrm{d}t | \bar{X}^{\kappa}_{T} ] \\ &= \eta^{\star}(\bar{X}^{\kappa}_{T})
    - \mathbb{E} \big[ \frac{1}{2}(\frac{1}{\kappa} \! - \! 1) \int_0^{\tau} \mathrm{Tr}\big( \kappa \sigma \sigma^{\top} \nabla^2_x \eta^{\star}(\tilde{X}^{\kappa}_t) \big) \, \mathrm{d}t  | \tilde{X}^{\kappa}_{0} = \bar{X}^{\kappa}_{T} \big],
\end{split}
\end{talign}
where the process $\tilde{X}^{\kappa}$ satisfies the same SDE as $\bar{X}^{\kappa}$.
In this equation, the first equality holds by the Markov property of SDEs, and the second equality holds by the following equation, which is true for all $x \in \mathcal{S}^c$ by the time symmetry of the problem, and the relation in \eqref{eq:F_eta_star_2}:
\begin{talign*}
\begin{split}
&\mathbb{E}[ - \int_T^{\tau} \langle \nabla \phi(\bar{X}^{\kappa}_t), b(\bar{X}^{\kappa}_t) - \frac{1}{2} \sigma \sigma^{\top} \nabla \phi(\bar{X}^{\kappa}_t) \rangle \, \mathrm{d}t \\ &\qquad + \! \phi(\bar{X}^{\kappa}_{\tau}) \! - \! \int_T^{\tau} f(\bar{X}^{\kappa}_t) \, \mathrm{d}t \! - \! g(\bar{X}^{\kappa}_{\tau})
\! - \! \frac{1}{2} \int_T^{\tau} \mathrm{Tr}\big( \sigma \sigma^{\top} \nabla^2_x \phi(\bar{X}^{\kappa}_t) \big) \, \mathrm{d}t | \bar{X}^{\kappa}_{T} = x] \\
&= \mathbb{E}[ - \int_0^{\tau} \langle \nabla \phi(\bar{X}^{\kappa}_t), b(\bar{X}^{\kappa}_t) - \frac{1}{2} \sigma \sigma^{\top} \nabla \phi(\bar{X}^{\kappa}_t) \rangle \, \mathrm{d}t \\ &\qquad + \! \phi(\bar{X}^{\kappa}_{\tau}) \! - \! \int_0^{\tau} f(\bar{X}^{\kappa}_t) \, \mathrm{d}t \! - \! g(\bar{X}^{\kappa}_{\tau})
\! - \! \frac{1}{2} \int_0^{\tau} \mathrm{Tr}\big( \sigma \sigma^{\top} \nabla^2_x \phi(\bar{X}^{\kappa}_t) \big) \, \mathrm{d}t | \bar{X}^{\kappa}_{0} = x] \\ &= \mathbb{E}[F(\bar{X}^{\kappa})\,|\,\bar{X}^{\kappa}_0=x] = \eta^{\star}(x) - \mathbb{E} \big[ \frac{1}{2}(\frac{1}{\kappa} \! - \! 1) \int_0^{\tau} \mathrm{Tr}\big( \kappa \sigma \sigma^{\top} \nabla^2_x \eta^{\star}(\bar{X}^{\kappa}_t) \big) \, \mathrm{d}t  | \bar{X}^{\kappa}_{0} = x \big].
\end{split}
\end{talign*}
Plugging the definition of $F^{(T)}(\bar{X}^{\kappa})$ (equation \eqref{eq:F_T_def}) and equation \eqref{eq:cond_exp_eta_star} into the right-hand side of \eqref{eq:cond_exp_F} yields:
\begin{talign}
\begin{split} \label{eq:exp_F_T}
    \mathbb{E}[F(\bar{X}^{\kappa})|\mathcal{F}_T] &= F^{(T)}(\bar{X}^{\kappa}) + \mathbf{1}_{\mathcal{S}^c}(\bar{X}^{\kappa}_{\tau(T)}) \Big( \eta^{\star}(\bar{X}^{\kappa}_{T}) \\ &\qquad - \mathbb{E} \big[ \frac{1}{2}(\frac{1}{\kappa} \! - \! 1) \int_0^{\tau} \mathrm{Tr}\big( \kappa \sigma \sigma^{\top} \nabla^2_x \eta^{\star}(\tilde{X}^{\kappa}_t) \big) \, \mathrm{d}t \,\big|\, \tilde{X}^{\kappa}_{0} = \bar{X}^{\kappa}_{T} \big] \Big).
\end{split}
\end{talign}
From equation \eqref{eq:F_eta_star_2}, we also obtain that
\begin{talign}
\begin{split} \label{eq:E_F_F_T_2}
    \mathbb{E}[F(\bar{X}^{\kappa})|\mathcal{F}_T] &= \eta^{\star}(\bar{X}^{\kappa}_0) + \int_0^{\tau(T)} \langle \sqrt{\kappa} \sigma^{\top} \nabla \eta^{\star}(\bar{X}^{\kappa}_t), \mathrm{d}W_t \rangle
    - \frac{1}{2} \big( \frac{1}{\kappa} - 1 \big) \int_0^{\tau(T)} \mathrm{Tr}\big( \kappa \sigma \sigma^{\top} \nabla^2_x \eta^{\star}(\bar{X}^{\kappa}_t) \big) \, \mathrm{d}t \\ &\quad
    -\mathbf{1}_{\mathcal{S}^c}(\bar{X}^{\kappa}_{\tau(T)}) \frac{1}{2} \big( \frac{1}{\kappa} - 1 \big) \mathbb{E}[ \int_{T}^{\tau} \mathrm{Tr}\big( \kappa \sigma \sigma^{\top} \nabla^2_x \eta^{\star}(\bar{X}^{\kappa}_t) \big) \, \mathrm{d}t ].
\end{split}
\end{talign}
Equations \eqref{eq:exp_F_T} and \eqref{eq:E_F_F_T_2} together yield
\begin{talign}
\begin{split} \label{eq:F_T_eta_star}
    F^{(T)}(\bar{X}^{\kappa}) &= \eta^{\star}(\bar{X}^{\kappa}_0) + \int_0^{\tau(T)} \langle \sqrt{\kappa} \sigma^{\top} \nabla \eta^{\star}(\bar{X}^{\kappa}_t), \mathrm{d}W_t \rangle \\ &\quad
    - \frac{1}{2} \big( \frac{1}{\kappa} - 1 \big) \int_0^{\tau(T)} \mathrm{Tr}\big( \kappa \sigma \sigma^{\top} \nabla^2_x \eta^{\star}(\bar{X}^{\kappa}_t) \big) \, \mathrm{d}t - \mathbf{1}_{\mathcal{S}^c}(\bar{X}^{\kappa}_{\tau(T)}) \eta^{\star}(\bar{X}^{\kappa}_{\tau(T)}).
\end{split}
\end{talign}
Hence, if we plug the definition of $F^{(T)}(\bar{X}^{\kappa})$ (equation \eqref{eq:F_T_def}) and equation \eqref{eq:F_T_eta_star} into the right-hand side of \eqref{eq:grad_VM_hitting_time_2}, we obtain that
\begin{talign*}
\begin{split}
    &\frac{\partial}{\partial \varepsilon} \mathcal{L}_{\mathrm{VM}_{v,\kappa}}(\phi + \varepsilon \eta) \rvert_{\varepsilon = 0} \\ &= \mathbb{E} \big[ \big(
    \eta^{\star}(\bar{X}^{\kappa}_0) + \int_0^{\tau(T)} \langle \sqrt{\kappa} \sigma^{\top} \nabla \eta^{\star}(\bar{X}^{\kappa}_t), \mathrm{d}W_t \rangle \\ &\qquad\quad
    - \frac{1}{2} \big( \frac{1}{\kappa} - 1 \big) \int_0^{\tau(T)} \mathrm{Tr}\big( \kappa \sigma \sigma^{\top} \nabla^2_x \eta^{\star}(\bar{X}^{\kappa}_t) \big) \, \mathrm{d}t - \mathbf{1}_{\mathcal{S}^c}(\bar{X}^{\kappa}_{\tau(T)}) \eta^{\star}(\bar{X}^{\kappa}_{\tau(T)})
    \big) \\ &\qquad \times \bigl(
    \eta(\bar{X}^{\kappa}_0) + \int_0^{\tau(T)} \langle \sqrt{\kappa} \sigma^{\top} \nabla \eta(\bar{X}^{\kappa}_t), \mathrm{d}W_t \rangle \\ &\qquad\qquad
    - \frac{1}{2} \big( \frac{1}{\kappa} - 1 \big) \int_0^{\tau(T)} \mathrm{Tr}\big( \kappa \sigma \sigma^{\top} \nabla^2_x \eta(\bar{X}^{\kappa}_t) \big) \, \mathrm{d}t - \mathbf{1}_{\mathcal{S}^c}(\bar{X}^{\kappa}_{\tau(T)}) \eta(\bar{X}^{\kappa}_{\tau(T)})
    \bigr) \times \mathrm{IW}(W,0,\tau(T)) \big]
\end{split}
\end{talign*}

\paragraph{Step 3: Choosing the right perturbation $\eta$ to obtain an almost sure equality} If we pick $\eta(x) = - \eta^{\star}(x)$,
we obtain that
\begin{talign*}
\begin{split}
    \frac{\partial}{\partial \varepsilon} \mathcal{L}_{\mathrm{VM}_{v,\kappa}}(\phi + \varepsilon \eta) \rvert_{\varepsilon = 0} = - \mathbb{E}[F^{(T)}(\bar{X}^{\kappa})^2 \cdot \mathrm{IW}(W,0,\tau(T))] = - \mathbb{E}[F^{(T)}(X^{v,\kappa})^2].
\end{split}
\end{talign*}
As in Step 3 of the proof of \Cref{thm:vm_loss_hitting}, it is convenient to further define the diffusion bridge process $X^{v,\kappa,x,y,T}$ as the solution of the SDE \eqref{eq:X_v_kappa} conditioned on the fact that $X^{v,\kappa,x,y,T}_0 = x$ and $X^{v,\kappa,x,y,T}_{\tau(T)} = y$. Let $\mathbb{P}^{v,\kappa,x,y,T}$ be the path probability measure of the diffusion bridge process $X^{v,\kappa,x,y,T}$. If we let $p_{v,\kappa,T}$ be the joint distribution of $(X^{v,\kappa}_0,X^{v,\kappa}_{\tau(T)})$ observe that through the tower property of conditional expectation, we can rewrite
\begin{talign*}
    \frac{\partial}{\partial \varepsilon} \mathcal{L}_{\mathrm{VM}_{v,\kappa}}(\phi + \varepsilon \eta) \rvert_{\varepsilon = 0} = - \mathbb{E}[F^{(T)}(X^{v,\kappa})^2] = - \int_{(\mathbb{R}^d)^2} \mathbb{E}[F^{(T)}(X^{v,\kappa,x,y,T})^2] \, \mathrm{d}p_{v,\kappa,T}(x,y),
\end{talign*}
Thus, if $\hat{\phi}$ satisfies that $\frac{\partial}{\partial \varepsilon} \mathcal{L}_{\mathrm{VM}_{v,\kappa}}(\hat{\phi} + \varepsilon \eta) \rvert_{\varepsilon = 0} = 0$ for all perturbations $\eta$, we need that when $\phi = \hat{\phi}$, for $p_{v,\kappa,T}$-almost every $(x,y)$, $\mathbb{E}[F^{(T)}(X^{v,\kappa,x,y,T})^2] = 0$. Hence, for $p_{v,\kappa,T}$-almost every $(x,y)$, we must have that for $\mathbb{P}^{v,\kappa,x,y,T}$-almost every path $X^{v,\kappa,x,y,T}$, $F^{(T)}(X^{v,\kappa,x,y,T}) = 0$.

\paragraph{Step 4: Writing down the Radon-Nikodym derivative $\frac{\mathrm{d}\mathbb{P}^{u^{\star}}}{\mathrm{d}\mathbb{P}^{\hat{\phi}}}$}
The next step of the proof involves an argument similar to the one in Step 4 of the proofs of \Cref{thm:vm_finite} and \Cref{thm:vm_loss_hitting}. Let $Y_{[0,\tau(T)]}$ be an arbitrary process defined over $[0,\tau(T)]$ whose probability measure is absolutely continuous with respect to $\mathbb{P}$ and/or $\mathbb{P}^{u^{\star}}$ over $[0,\tau(T)]$.
In analogy with \eqref{eq:dP_dP_hatphi}, by Corollary \ref{cor:Girsanov_dX_t} of the Girsanov theorem,
\begin{talign}
\begin{split} \label{eq:dP_u_star_dP_capped_1}
\frac{\mathrm{d}\mathbb{P}}{\mathrm{d}\mathbb{P}^{\hat{\phi}}}(Y_{[0,\tau(T)]}) &=
\exp \big(\!\int_0^{\tau(T)}\big\langle \sigma^{\top} \nabla \hat{\phi}(Y_s), \mathrm{d}Y_s - (b(Y_s) - \frac{1}{2} \sigma \sigma^{\top} \nabla \hat{\phi}(Y_s)) \, \mathrm{d}s \big\rangle \big).
\end{split}
\end{talign}

Next, observe that if $X, \tilde{X} \sim \mathbb{P}$ over $[0,\tau]$, and we let $X_{[0,\tau(T)]}$ be the restriction of $X$ to $[0,\tau(T)]$,
\begin{talign*}
\begin{split}
    \frac{\mathrm{d}\mathbb{P}^{u^{\star}}}{\mathrm{d}\mathbb{P}}(X_{[0,\tau(T)]}) &= \mathbb{E}[ \frac{\mathrm{d}\mathbb{P}^{u^{\star}}}{\mathrm{d}\mathbb{P}}(X) | \mathcal{F}_{T}] = \mathbb{E}[ \exp \big( V(X_0) - \int_0^{\tau} f(X_t) \, \mathrm{d}t - g(X_{\tau}) \big) | \mathcal{F}_{T}] \\ &= \exp \big( V(X_0) - \int_0^{\tau(T)} f(X_t) \, \mathrm{d}t \big) \mathbb{E}[ \exp \big( - \int_{\tau(T)}^{\tau} f(X_t) \, \mathrm{d}t - g(X_{\tau}) \big) | \mathcal{F}_{T}] \\ &= \exp \big( V(X_0) - \int_0^{\tau(T)} f(X_t) \, \mathrm{d}t \big) \mathbb{E}[ \exp \big( - \int_{0}^{\tau} f(\tilde{X}_t) \, \mathrm{d}t - g(\tilde{X}_{\tau}) \big) | \tilde{X}_0 = X_{\tau(T)} ] \\ &= \exp \big( V(X_0) - \int_0^{\tau(T)} f(X_t) \, \mathrm{d}t - V(X_{\tau(T)}) \big).
\end{split}
\end{talign*}
Thus, for any process $Y_{[0,\tau(T)]}$ whose probability measure is absolutely continuous with respect to $\mathbb{P}$, $\frac{\mathrm{d}\mathbb{P}^{u^{\star}}}{\mathrm{d}\mathbb{P}}(Y_{[0,\tau(T)]})$, we have that $\frac{\mathrm{d}\mathbb{P}^{u^{\star}}}{\mathrm{d}\mathbb{P}}(Y_{[0,\tau(T)]}) = \exp \big( V(Y_0) - \int_0^{\tau(T)} f(Y_t) \, \mathrm{d}t - V(Y_{\tau(T)}) \big)$.
In analogy with equation \eqref{eq:P_u_star_P_v},
we obtain that
\begin{talign}
\begin{split} \label{eq:difference_V_phi}
    \frac{\mathrm{d}\mathbb{P}^{u^{\star}}}{\mathrm{d}\mathbb{P}^{\hat{\phi}}}(Y_{[0,\tau(T)]}) &= \frac{\mathrm{d}\mathbb{P}^{u^{\star}}}{\mathrm{d}\mathbb{P}}(Y_{[0,\tau(T)]}) \times \frac{\mathrm{d}\mathbb{P}}{\mathrm{d}\mathbb{P}^{\hat{\phi}}}(Y_{[0,\tau(T)]}) \\ &= \exp \big(
    \int_0^{\tau(T)}\big\langle \sigma^{\top} \nabla \hat{\phi}(Y_s), \mathrm{d}Y_s - (b(Y_s) - \frac{1}{2} \sigma \sigma^{\top} \nabla \hat{\phi}(Y_s)) \, \mathrm{d}s \big\rangle
    \\ &\qquad\qquad\qquad + V(Y_0) - \int_0^{\tau(T)} f(Y_t) \, \mathrm{d}t - V(Y_{\tau(T)}) \big),
\end{split}
\end{talign}
and more generally, for any $k \geq 0$,
\begin{talign}
\begin{split} \label{eq:difference_V_phi_k}
    \frac{\mathrm{d}\mathbb{P}^{u^{\star}}}{\mathrm{d}\mathbb{P}^{\hat{\phi}}}(Y_{[\tau(kT),\tau((k+1)T)]})
    &= \exp \big(
    \int_{\tau(kT)}^{\tau((k+1)T)}\big\langle \sigma^{\top} \nabla \hat{\phi}(Y_s,s), \mathrm{d}Y_s - (b(Y_s) - \frac{1}{2} \sigma \sigma^{\top} \nabla \hat{\phi}(Y_s,s)) \, \mathrm{d}s \big\rangle
    \\ &\qquad\qquad\qquad + V(Y_{\tau(kT)}) - \int_{\tau(kT)}^{\tau((k+1)T)} f(Y_t,t) \, \mathrm{d}t - V(Y_{\tau((k+1)T)}) \big).
\end{split}
\end{talign}
\paragraph{Step 5(i): Concluding the proof for the case $\kappa = 1$, $\mu_0$ and $\tilde{\mu}_0$ mutually absolutely continuous}
The argument is the same as in Step 5(i) of the proof of \Cref{thm:vm_finite} and \Cref{thm:vm_loss_hitting}. We apply Itô's lemma (equation \eqref{eq:phi_eta_ito_stopping}) into the definition of $F^{(T)}$, and set $\kappa=1$, to obtain that almost surely,
\begin{talign}
\begin{split} \label{eq:F_T_zero}
    0 &= F^{(T)}(X^{v,1}) \\ &= - \int_0^{\tau(T)} \langle \nabla \hat{\phi}(X^{v,1}_t), b(X^{v,1}_t) - \frac{1}{2} \sigma \sigma^{\top} \nabla \hat{\phi}(X^{v,1}_t) \rangle \, \mathrm{d}t - \frac{1}{2} \int_0^{\tau(T)} \mathrm{Tr}\big( \sigma \sigma^{\top} \nabla^2_x \hat{\phi}(X^{v,1}_t) \big) \, \mathrm{d}t \\ &\quad - \int_0^{\tau(T)} f(X^{v,1}_t) \, \mathrm{d}t
    + \mathbf{1}_{\mathcal{S}}(X^{v,1}_{\tau(T)}) \big( \hat{\phi}(X^{v,1}_{\tau(T)}) - g(X^{v,1}_{\tau(T)}) \big) \\ &= \int_0^{\tau(T)} \langle \nabla \hat{\phi}(X^{v,1}_t), \mathrm{d}X^{v,1}_t - ( b(X^{v,1}_t) - \frac{1}{2} \sigma \sigma^{\top} \nabla \hat{\phi}(X^{v,1}_t) ) \, \mathrm{d}t \rangle \\ &\qquad + \hat{\phi}(X^{v,1}_0) - \int_0^{\tau(T)} f(X^{v,1}_t) \, \mathrm{d}t
    - \mathbf{1}_{\mathcal{S}^c}(X^{v,1}_{\tau(T)}) \hat{\phi}(X^{v,1}_{\tau(T)}) - \mathbf{1}_{\mathcal{S}}(X^{v,1}_{\tau(T)}) g(X^{v,1}_{\tau(T)}).
\end{split}
\end{talign}
Let $\mathbb{P}^{v,1}$ be the probability measure of $X^{v,1}$ with initial distribution $\tilde{\mu}_0$. Since $\mathbb{P}^{v,1}$ and $\mathbb{P}$ are mutually absolutely continuous because they share the same diffusion coefficient and their initial marginals are mutually absolutely continuous, equation \eqref{eq:difference_V_phi} holds for $X^{v,1}_{[0,\tau(T)]}$. Plugging \eqref{eq:F_T_zero} into \eqref{eq:difference_V_phi} we can write
\begin{talign*}
\begin{split}
    \frac{\mathrm{d}\mathbb{P}^{u^{\star}}}{\mathrm{d}\mathbb{P}^{\hat{\phi}}}(X^{v,1}_{[0,\tau(T)]}) &= \exp \big( V(X^{v,1}_0) \! - \! \hat{\phi}(X^{v,1}_0) \! - \! V(X^{v,1}_{\tau(T)}) \! + \! \mathbf{1}_{\mathcal{S}^c}(X^{v,1}_{\tau(T)}) \hat{\phi}(X^{v,1}_{\tau(T)}) \! + \! \mathbf{1}_{\mathcal{S}}(X^{v,1}_{\tau(T)}) g(X^{v,1}_{\tau(T)}) \big) \\ &=
    \exp \big( \mathbf{1}_{\mathcal{S}^c}(X^{v,1}_{0}) \big( V(X^{v,1}_0) \! - \! \hat{\phi}(X^{v,1}_0) \big) \! - \! \mathbf{1}_{\mathcal{S}^c}(X^{v,1}_{\tau(T)}) \big( V(X^{v,1}_{\tau(T)}) \! - \! \hat{\phi}(X^{v,1}_{\tau(T)}) \big)
\end{split}
\end{talign*}
Here, the second equality holds because $V(x) = g(x)$ for $x \in \mathcal{S}$ by definition of the value function $V$, and because $V(x) = \hat{\phi}(x)$ for $x \in \mathcal{S}$, by virtue of equation \eqref{eq:F_T_zero}, since when $X^{v,1}_0 \in \mathcal{S}$, $\tau(T) = 0$.
More generally, this argument shows that for any $k \geq 0$,
\begin{talign}
\begin{split} \label{eq:RND_telescope}
\frac{\mathrm{d}\mathbb{P}^{u^{\star}}}{\mathrm{d}\mathbb{P}^{\hat{\phi}}}(X_{[\tau(kT),\tau((k+1)T)]}^{v,1}) &= \exp \big( \mathbf{1}_{\mathcal{S}^c}(X^{v,1}_{\tau(kT)}) \big( V(X^{v,1}_{\tau(kT)}) - \hat{\phi}(X^{v,1}_{\tau(kT)}) \big) \\ &\qquad\quad - \mathbf{1}_{\mathcal{S}^c}(X^{v,1}_{\tau((k+1)T)}) \big( V(X^{v,1}_{\tau((k+1)T)}) - \hat{\phi}(X^{v,1}_{\tau((k+1)T)}) \big) \big).
\end{split}
\end{talign}
Let $X^{v,1}_{[0,\tau]} := X^{v,1}$ be an uncapped trajectory. We apply equation
\eqref{eq:dP_u_star_dP_hatphi_uncapped} from the proof of \Cref{thm:vm_loss_hitting}, and use a telescoping argument as follows:
\begin{talign}
\begin{split} \label{eq:telescoping_arg}
    \frac{\mathrm{d}\mathbb{P}^{u^{\star}}}{\mathrm{d}\mathbb{P}^{\hat{\phi}}}(X_{[0,\tau]}^{v,1}) &=
    \exp \big( \int_0^{\tau} \langle \nabla \hat{\phi}(X^{v,1}_t), \mathrm{d}X^{v,1}_t - ( b(X^{v,1}_t) - \frac{1}{2} \sigma \sigma^{\top} \nabla \hat{\phi}(X^{v,1}_t) ) \, \mathrm{d}t
    \rangle  \\ &\qquad\qquad\qquad + V(X^{v,1}_0) - \int_0^{\tau} f(X^{v,1}_t) \, \mathrm{d}t - g(X^{v,1}_{\tau}) \big)
    \\ &= \prod_{k=0}^{+\infty} \exp \big( \int_{\tau(kT)}^{\tau((k+1)T)} \langle \nabla \hat{\phi}(X^{v,1}_t), \mathrm{d}X^{v,1}_t - ( b(X^{v,1}_t) - \frac{1}{2} \sigma \sigma^{\top} \nabla \hat{\phi}(X^{v,1}_t) ) \, \mathrm{d}t
    \rangle  \\ &\qquad\qquad\qquad + V(X^{v,1}_{\tau(kT)}) - \int_{\tau(kT)}^{\tau((k+1)T)} f(X^{v,1}_t) \, \mathrm{d}t - V(X^{v,1}_{\tau((k+1)T)}) \big) \\
    &= \prod_{k=0}^{+\infty} \frac{\mathrm{d}\mathbb{P}^{u^{\star}}}{\mathrm{d}\mathbb{P}^{\hat{\phi}}}(X_{[\tau(kT),\tau((k+1)T)]}^{v,1})
    \\ &= \prod_{k=0}^{+\infty} \exp \big(
    \mathbf{1}_{\mathcal{S}^c}(X^{v,1}_{\tau(kT)}) \big( V(X^{v,1}_{\tau(kT)}) - \hat{\phi}(X^{v,1}_{\tau(kT)}) \big) \\ &\qquad\qquad\qquad - \mathbf{1}_{\mathcal{S}^c}(X^{v,1}_{\tau((k+1)T)}) \big( V(X^{v,1}_{\tau((k+1)T)}) - \hat{\phi}(X^{v,1}_{\tau((k+1)T)}) \big)
    \big)
    \\ &= \exp \big( \mathbf{1}_{\mathcal{S}^c}(X^{v,1}_{0}) \big( V(X^{v,1}_0) - \hat{\phi}(X^{v,1}_0) \big) \big).
\end{split}
\end{talign}
In the equation above, infinite products are well defined because beyond a certain integer determined by the stopping time $\tau$, all the factors are 1. The third equality holds by \eqref{eq:difference_V_phi_k}, and the fourth equality holds by \eqref{eq:RND_telescope}.

Like in Steps 5(i) of the proofs of \Cref{thm:vm_finite} and \Cref{thm:vm_loss_hitting}, since both $\mathbb{P}^{u^{\star}}$ and $\mathbb{P}^{\hat{\phi}}$ have marginal $\mu_0$ at time $t=0$, for all $x \in \mathbb{R}^d$ we have that
\begin{talign*}
\begin{split}
    1 = \frac{\mu_0(x)}{\mu_0(x)} = \frac{\mathrm{d}\mathbb{P}^{u^{\star}}_0(x)}{\mathrm{d}\mathbb{P}^{\hat{\phi}}_0(x)} = \mathbb{E}[\frac{\mathrm{d}\mathbb{P}^{u^{\star}}}{\mathrm{d}\mathbb{P}^{\hat{\phi}}}(X^{v,\kappa})|X^{v,\kappa}_0] = \exp \big( V(X^{v,\kappa}_0) - \hat{\phi}(X^{v,\kappa}_0) \big).
\end{split}
\end{talign*}
which proves that $\hat{\phi}(x) = V(x)$ for all $x \in \mathbb{R}^d$. Plugging this into equation \eqref{eq:telescoping_arg} implies that for all trajectories $X^{v,1}$,
\begin{talign*}
\begin{split}
    \frac{\mathrm{d}\mathbb{P}^{u^{\star}}}{\mathrm{d}\mathbb{P}^{\hat{\phi}}}(X_{[0,\tau(T)]}^{v,1}) &= 1,
\end{split}
\end{talign*}
By the absolute continuity of $\mathbb{P}^{u^{\star}}$ with respect to the probability measure of $X^{v,1}$, we obtain that $\frac{\mathrm{d}\mathbb{P}^{u^{\star}}}{\mathrm{d}\mathbb{P}^{\hat{\phi}}}(X_{[0,\tau(T)]}^{u^{\star}}) = 1$. And applying Girsanov once more,
\begin{talign} \label{eq:girsanov_app_KL_capped}
    \mathbb{E}\big[ \frac{1}{2} \int_0^{\tau(T)} \| u^{\star}(X^{u^{\star}}_t) + \sigma^{\top} \nabla \hat{\phi}(X^{u^{\star}}_t) \|^2 \, \mathrm{d}t \big] = \mathrm{KL}(\mathbb{P}^{u^{\star}}\,\|\,\mathbb{P}^{\hat{\phi}}) = \mathbb{E}\big[ \log \frac{\mathrm{d}\mathbb{P}^{u^{\star}}}{\mathrm{d}\mathbb{P}^{\hat{\phi}}}(X^{u^{\star}}) \big] = 0.
\end{talign}
This means that $u^{\star} = - \sigma^{\top} \nabla \hat{\phi}$ almost surely, and by continuity, we have that they are equal everywhere.

\paragraph{Step 5(ii): Concluding the proof for the general case} The argument follows the structure of Step 5(ii) of the proof of \Cref{thm:vm_loss_hitting}.

For any $s \in [0,T]$ and any $x, y \in \mathbb{R}^d$, define
\begin{talign*}
\begin{split}
    &\mathcal{C}^{\mathrm{stop}(T)}_{x,y,\infty,s}
    \! = \! \{w\in C([0,T];\mathbb{R}^d) \, | \, w(t)=w(\tau(T)_w) = y\ \text{for all }t\ge \tau(T)_w, w(0) = x, \tau(T)_w = s \}, \\
    &\mathcal{C}^{\mathrm{stop}(T),v,\kappa}_{x,y,F} = \{ X^{v,\kappa} \in
    \mathcal{C}^{\mathrm{stop}(T)}_{x,y,\infty, s} \, | \, X^{v,\kappa} \text{ solves } \eqref{eq:X_v_kappa} \text{ on } [0,\tau(T)], \, F(X^{v,\kappa}) = 0 \},
\end{split}
\end{talign*}
where $\tau(T)_w$ denotes the capped hitting time of the path $w$ on the set $\mathcal{S}$. Observe that when restricted on $\mathcal{C}^{\mathrm{stop}(T)}_{x,y,\infty, s}$, the truncated functional $F^{(T)}$ takes the form
\begin{talign*}
\begin{split}
    F^{(T)}(X) &= - \int_0^{s} \langle \nabla \hat{\phi}(X_t), b(X_t) - \frac{1}{2} \sigma \sigma^{\top} \nabla \hat{\phi}(X_t) \rangle \, \mathrm{d}t - \frac{1}{2} \int_0^{s} \mathrm{Tr}\big( \sigma \sigma^{\top} \nabla^2_x \hat{\phi}(X_t) \big) \, \mathrm{d}t \\ &\quad - \int_0^{s} f(X_t) \, \mathrm{d}t
    + \mathbf{1}_{\mathcal{S}}(y) \big( \hat{\phi}(y) - g(y) \big).
\end{split}
\end{talign*}
which means that the restriction of $F^{(T)}$ to $\mathcal{C}^{\mathrm{stop}(T)}_{x,y,\infty, s}$ is continuous in the uniform norm over $[0,s]$, by the argument of Lemma \ref{lem:exp_D_t_F}.

Recall the notation $\mathbb{P}^{v,\kappa,x,y,T}$ and $p_{v,\kappa,T}$ defined in Step 3 of this proof.
In Step 3, we obtained that for $p_{v,\kappa,T}$-almost every $(x,y) \in (\mathbb{R}^d)^2$, we have that for $\mathbb{P}^{v,\kappa,x,y,T}$-almost every path $X^{v,\kappa,x,y,T}$, $F^{(T)}(X^{v,\kappa,x,y,T}) = 0$. Equivalently, we have that for $p_{v,\kappa,T}$-almost every $(x,y) \in (\mathbb{R}^d)^2$, $\mathcal{C}^{\mathrm{stop}(T),v,\kappa}_{x,y,F}$ is dense in $\mathrm{supp}(\mathbb{P}^{v,\kappa,x,y,T})$ in the uniform norm over $[0,s]$.

By \Cref{thm:full_supp_bridge} and using the regularity, growth and uniform ellipticity assumptions
on $v$, we have that $\mathrm{supp}(\mathbb{P}^{v,\kappa,x,y,T})$ is equal to $\mathcal{C}^{\mathrm{stop}(T)}_{x,y,\infty,s}$ in the uniform norm. Hence, $\mathcal{C}^{\mathrm{stop}(T),v,\kappa}_{x,y,F}$ is dense in $\mathcal{C}^{\mathrm{stop}(T)}_{x,y,\infty,s}$ in the uniform norm.

Since $F^{(T)}$ takes value zero on $\mathcal{C}^{\mathrm{stop}(T),v,\kappa}_{x,y,F}$ by construction, and $\mathcal{C}^{\mathrm{stop}(T),v,\kappa}_{x,y,F}$ is dense in $\mathcal{C}^{\mathrm{stop}(T)}_{x,y,\infty,s}$ for a.e. $(x,y) \in (\mathbb{R}^d)^2$, and $F^{(T)}$ is continuous when restricted to $\mathcal{C}^{\mathrm{stop}(T)}_{x,y,\infty,s}$, we conclude that $F^{(T)}$ is the zero functional on $\mathcal{C}^{\mathrm{stop}(T)}_{x,y,\infty,s}$ for a.e. $(x,y) \in (\mathbb{R}^d)^2$.

Now, observe that if we define
\begin{talign*}
    \mathcal{C}^{\mathrm{stop}(T)}_{\mathrm{supp}(\tilde{\mu}_0),\infty}
    \! = \! \{w\in C([0,\infty);\mathbb{R}^d) \, | \, w(t)=w(\tau(T)_w)\ \text{for all }t\ge \tau(T)_w, w(0) \in \mathrm{supp}(\tilde{\mu}_0) \},
\end{talign*}
we have that
\begin{talign*}
    \mathcal{C}^{\mathrm{stop}(T)}_{\mathrm{supp}(\tilde{\mu}_0),\infty} = \bigcup_{s \in (0,\infty), x \in \mathrm{supp}(\tilde{\mu}_0), y \in \partial \mathcal{S}} \mathcal{C}^{\mathrm{stop}(T)}_{x,y,\infty,s}
\end{talign*}
Hence, $F^{(T)}(X) = 0$ for almost every $X \in \mathcal{C}^{\mathrm{stop}(T)}_{\mathrm{supp}(\tilde{\mu}_0),\infty}$. Since $\mu_0$ is absolutely continuous with respect to $\tilde{\mu}_0$, we obtain that almost surely,
\begin{talign*}
    F^{(T)}(X^{u^{\star}}) = 0.
\end{talign*}
Then, by Itô's lemma,
\begin{talign*}
\begin{split}
    \int_0^{\tau(T)} \langle \nabla \hat{\phi}(X^{u^{\star}}_t), \mathrm{d}X^{u^{\star}}_t \rangle = \hat{\phi}(X^{u^{\star}}_{\tau(T)}) - \hat{\phi}(X^{u^{\star}}_0) - \int_0^{\tau(T)} \frac{1}{2} \mathrm{Tr}\big( \sigma \sigma^{\top} \nabla^2_x \hat{\phi}(X^{u^{\star}}_t) \big) \, \mathrm{d}t,
\end{split}
\end{talign*}
which allows us to re-express $F^{(T)}(X^{u^{\star}})$ almost surely as follows:
\begin{talign*}
\begin{split}
    0 &= F^{(T)}(X^{u^{\star}}) =
    \int_0^{\tau(T)} \langle \nabla \hat{\phi}(X^{u^{\star}}_t), \mathrm{d}X^{u^{\star}}_t - ( b(X^{u^{\star}}_t) - \frac{1}{2} \sigma \sigma^{\top} \nabla \hat{\phi}(X^{u^{\star}}_t) ) \, \mathrm{d}t \rangle \\ &\qquad\qquad\qquad + \hat{\phi}(X^{u^{\star}}_0) - \int_0^{\tau(T)} f(X^{u^{\star}}_t) \, \mathrm{d}t
    - \mathbf{1}_{\mathcal{S}^c}(X^{u^{\star}}_{\tau(T)}) \hat{\phi}(X^{u^{\star}}_{\tau(T)}) - \mathbf{1}_{\mathcal{S}}(X^{u^{\star}}_{\tau(T)}) g(X^{u^{\star}}_{\tau(T)}).
\end{split}
\end{talign*}
As in Step 5(i), this equality now implies that
\begin{talign*}
\begin{split}
\frac{\mathrm{d}\mathbb{P}^{u^{\star}}}{\mathrm{d}\mathbb{P}^{\hat{\phi}}}(X^{u^{\star}}_{[0,\tau(T)]}) &= \frac{\mathrm{d}\mathbb{P}^{u^{\star}}}{\mathrm{d}\mathbb{P}}(X^{u^{\star}}_{[0,\tau(T)]}) \times \frac{\mathrm{d}\mathbb{P}}{\mathrm{d}\mathbb{P}^{\hat{\phi}}}(X^{u^{\star}}_{[0,\tau(T)]}) \\ &=
\exp \big( \mathbf{1}_{\mathcal{S}^c}(X^{u^{\star}}_{0}) \big( V(X^{u^{\star}}_0) \! - \! \hat{\phi}(X^{u^{\star}}_0) \big) \! - \! \mathbf{1}_{\mathcal{S}^c}(X^{u^{\star}}_{\tau(T)}) \big( V(X^{u^{\star}}_{\tau(T)}) \! - \! \hat{\phi}(X^{u^{\star}}_{\tau(T)}) \big) \big),
\end{split}
\end{talign*}
and the remainder of the proof follows analogously to Step 5(i), with $X^{u^{\star}}$ taking the role of $X^{v,1}$.

\begin{corollary} \label{cor:particular_VM}
    When
    \begin{talign*}
        v(x) =
        \tfrac{1+\kappa}{2}\big(b(x) - 2D\nabla \bar \phi(x)\big) + \tfrac{\kappa-1}{2}\big(\tilde b(x) - 2D\nabla \bar{\tilde{\phi}}(x) \big),
    \end{talign*}
    where $\bar{\phi} = \texttt{sg}(\phi)$ is the stopped gradient version of $\phi$,
    the Value Matching loss with state cost $f$ reads
    \begin{talign*}
    \begin{split}
        &\mathcal{L}_{\mathrm{VM}_{v,\kappa}}(\phi) \\ &= \mathbb{E} \Bigl[ \frac{1}{2} \Bigl( \int_0^{\tau(T)} \big\langle \nabla \phi(X^{v,\kappa}_t), \frac{1}{\sqrt{\kappa}} \sigma \, \mathrm{d}W_t
        + \big( \frac{1}{\kappa} - 1 \big) \big( b(X^{v,\kappa}_t) - D \nabla \log \rho(X^{v,\kappa}_t) - \nabla \cdot D(X^{v,\kappa}_t) \big) \, \mathrm{d}t \\ &\qquad\qquad\qquad\qquad + D \big(
        \nabla \phi(X^{v,\kappa}_t) \!-\!(\frac{1}{\kappa}\!+\!1)\!\nabla \bar{\phi}(X^{v,\kappa}_t)
        \!+ \!(\tfrac{1}{\kappa}\!-\!1) \nabla \bar{\tilde{\phi}}(X^{v,\kappa}_t)
        \big) \, \mathrm{d}t
        \big\rangle \\ &\qquad\quad + \! \frac{1}{\kappa} \phi(X^{v,\kappa}_0) \! - \! \int_0^{\tau(T)} f(X^{v,\kappa}_t) \, \mathrm{d}t
        \! + \! (\mathbf{1}_{\mathcal{S}}(X^{v,\kappa}_{\tau(T)}) \! - \! \frac{1}{\kappa})
        \phi(X^{v,\kappa}_{\tau(T)})
        \! - \! \mathbf{1}_{\mathcal{S}}(X^{v,\kappa}_{\tau(T)}) g(X^{v,\kappa}_{\tau(T)})
        \Bigr)^2 \Bigr].
    \end{split}
    \end{talign*}
\end{corollary}
\begin{proof}
    Observe that
    \begin{talign*}
    \begin{split}
        &\frac{1}{\kappa} \mathrm{d}X^{v,\kappa}_t - ( b(X^{v,\kappa}_t) - D \nabla \phi(X^{v,\kappa}_t) ) \, \mathrm{d}t
        \\ &= \!\big( \frac{1+\kappa}{2\kappa}\big(b(X^{v,\kappa}_t) - 2D\nabla \bar{\phi}(X^{v,\kappa}_t) \big) + \frac{\kappa-1}{2\kappa} \big(\tilde{b}(X^{v,\kappa}_t)
        - 2 D \nabla \bar{\tilde{\phi}}(X^{v,\kappa}_t)
        \big) \big) \, \mathrm{d}t \!+ \!\frac{1}{\sqrt{\kappa}} \sigma \, \mathrm{d}W_t \\ &\quad
        - ( b(X^{v,\kappa}_t) - D \nabla \phi(X^{v,\kappa}_t) ) \, \mathrm{d}t
        \\ &=\big( \frac{1-\kappa}{2\kappa} \big( b(X^{v,\kappa}_t) - \tilde{b}(X^{v,\kappa}_t) \big) + D \big(
        \nabla \phi(X^{v,\kappa}_t) \!-\!(\frac{1}{\kappa}\!+\!1)\!\nabla \bar{\phi}(X^{v,\kappa}_t)
        \!+ \!(\frac{1}{\kappa}\!-\!1) \nabla \bar{\tilde{\phi}}(X^{v,\kappa}_t)
        \big) \big) \, \mathrm{d}t \\ &\quad + \!\frac{1}{\sqrt{\kappa}} \sigma \, \mathrm{d}W_t
        \\ &=\big( \frac{1-\kappa}{\kappa} \big( b(X^{v,\kappa}_t) - D \nabla \log \rho(X^{v,\kappa}_t) - \nabla \cdot D(X^{v,\kappa}_t) \big) \\ &\qquad + D \big(
        \nabla \phi(X^{v,\kappa}_t) \!-\!(\frac{1}{\kappa}\!+\!1)\!\nabla \bar{\phi}(X^{v,\kappa}_t)
        \!+ \!(\frac{1}{\kappa}\!-\!1) \nabla \bar{\tilde{\phi}}(X^{v,\kappa}_t)
        \big) \big) \, \mathrm{d}t + \!\frac{1}{\sqrt{\kappa}} \sigma \, \mathrm{d}W_t,
    \end{split}
    \end{talign*}
    where we used that $\tilde b = - b + 2 D \nabla \log \rho + 2 \nabla \cdot D$.
    Plugging this into equation \eqref{eq:VM_loss_f} concludes the proof.
\end{proof}

\section{Value Matching with RKHS parameterizations: quartic optimization landscapes} \label{sec:convergence_proof}

In this section, we examine the optimization landscapes of the Value Matching loss functions introduced in \Cref{sec:first_order_optimality}, under the assumption that the class $\mathcal{F}$ that we use to learn the value function $V$ is a ball of a reproducing kernel Hilbert space (RKHS) $\mathcal{H}$. Equivalently, if $\mathcal{G}$ is the feature space associated to $\mathcal{H}$, and $\psi : \mathbb{R}^d \times [0,T] \to \mathcal{G}$ is the feature map of $\mathcal{H}$, by construction
\begin{talign}
    \mathcal{H} = \big\{ \phi : \mathbb{R}^d \times [0,T] \to \mathbb{R} \, \big| \, \exists \theta \in \mathcal{G}, \ \mathrm{s.t.} \ \forall (x,t) \in \mathbb{R}^d \times [0,T], \ \phi(x,t) = \langle \theta, \psi(x,t) \rangle_{\mathcal{G}} \big\},
\end{talign}
and given $\phi \in \mathcal{H}$, the RKHS norm of $\phi$ is defined as
\begin{talign}
    \|\phi\|_{\mathcal{H}} = \inf \big\{ \|\theta\|_{\mathcal{G}} \, | \, \forall (x,t) \in \mathbb{R}^d \times [0,T], \ \phi(x,t) = \langle \theta, \psi(x,t) \rangle_{\mathcal{G}} \big\}.
\end{talign}
Thus, for some $R > 0$, we define the class $\mathcal{F}$ as follows:
\begin{talign}
    \mathcal{F} = \big\{ \phi : \mathbb{R}^d \times [0,T] \to \mathbb{R} \, \big| \, \exists \theta \in \mathcal{G}, \ \mathrm{s.t.} \ \phi(x,t) = \langle \theta, \psi(x,t) \rangle_{\mathcal{G}}, \ \| \theta\|_{\mathcal{G}} \leq R \big\}.
\end{talign}
Here are two common choices for $\mathcal{G}$:
\begin{enumerate}[label=(\roman*),left=0pt]
\item Finite-dimensional RKHS: If we set $\mathcal{G} = \mathbb{R}^{m}$, we have that
\begin{talign}
    \mathcal{F} = \big\{ \phi : \mathbb{R}^d \times [0,T] \to \mathbb{R} \, \big| \, \exists \theta \in \mathbb{R}^{m}, \ \mathrm{s.t.} \ \phi(x,t) = \langle \theta, \psi(x,t) \rangle_{\mathcal{G}}, \ \| \theta\|_2 \leq R \big\}.
\end{talign}
\item Gaussian features: If we set $\mathcal{G} = L^2(\mathbb{R}^d \times [0,T])$ and for all $(x,t) \in \mathbb{R}^d \times [0,T]$, $\psi(x,t) \in L^2(\mathbb{R}^d \times [0,T])$ is defined such that for some $\varepsilon > 0$,
\begin{talign}
[\psi(x,t)](x',t') = \frac{1}{(2\pi \varepsilon^2)^{(d+1)/2}} \exp\big( - \frac{\|(x',t') - (x,t)\|^2}{2 \varepsilon^2} \big),
\end{talign}
we have that
\begin{talign}
\begin{split}
    \mathcal{F} = \big\{ \phi : \mathbb{R}^d \times [0,T] \to \mathbb{R} \, &\big| \, \exists \theta \in L^2(\mathbb{R}^d \times [0,T]), \\ &\qquad \mathrm{s.t.} \ \phi(x,t) = \int_{\mathbb{R}^d \times [0,T]}  \theta(x',t') [\psi(x,t)](x',t') \, \mathrm{d}(x',t'), \\ &\qquad\qquad \int_{\mathbb{R}^d \times [0,T]} \theta(x,t)^2 \, \mathrm{d}(x,t) \leq R \big\}.
\end{split}
\end{talign}
\end{enumerate}

\paragraph{Value Matching loss for finite-horizon problems with RKHS parameterization} If we replace $\phi(x,t) = \langle \theta, \psi(x,t) \rangle_{\mathcal{G}}$ into \eqref{eq:L_VM_finite}, we obtain that

\begin{talign}
\begin{split}
    &\mathcal{L}_{\mathrm{VM}_{v,\kappa}}(\theta) \\ &=
    \mathbb{E} \bigl[ \frac{1}{2} \bigl( \int_0^{T} \langle \nabla_x \langle \theta, \psi(X^{v,\kappa}_t,t) \rangle_{\mathcal{G}},  \frac{1}{\kappa} \mathrm{d}X^{v,\kappa}_t - ( b(X^{v,\kappa}_t,t) - \frac{1}{2} \sigma \sigma^{\top} \nabla_x \langle \theta, \psi(X^{v,\kappa}_t,t) \rangle_{\mathcal{G}} ) \, \mathrm{d}t \rangle \\ &\qquad\quad + \frac{1}{\kappa} \langle \theta, \psi(X^{v,\kappa}_0,0) \rangle_{\mathcal{G}} + (1 - \frac{1}{\kappa})\big( \langle \theta, \psi(X^{v,\kappa}_T,T) \rangle_{\mathcal{G}} - \int_0^{T} \partial_t \langle \theta, \psi(X^{v,\kappa}_t,t) \rangle_{\mathcal{G}} \, \mathrm{d}t \big) \\ &\qquad\quad - \int_0^T f(X^{v,\kappa}_t,t) \, \mathrm{d}t - g(X^{v,\kappa}_T)
    \bigr)^2 \bigr]
    \\ &=
    \mathbb{E} \bigl[ \frac{1}{2} \bigl( \langle \theta, \big( \int_0^{T} \langle \nabla_x \psi(X^{v,\kappa}_t,t), \frac{1}{2} \sigma \sigma^{\top} \nabla_x \psi(X^{v,\kappa}_t,t) \rangle \, \mathrm{d}t \big) \theta \rangle_{\mathcal{G}}
    \\ &\qquad\quad + \big\langle \theta, \int_0^{T} \langle \nabla_x \psi(X^{v,\kappa}_t,t), \frac{1}{\kappa} \mathrm{d}X^{v,\kappa}_t - b(X^{v,\kappa}_t,t) \, \mathrm{d}t \rangle + \frac{1}{\kappa} \psi(X^{v,\kappa}_0,0) \\ &\qquad\qquad\quad \ + (1 - \frac{1}{\kappa})\big(\psi(X^{v,\kappa}_T,T)
    - \int_0^{T} \partial_t \psi(X^{v,\kappa}_t,t) \, \mathrm{d}t \big) \big\rangle_{\mathcal{G}} \\ &\qquad\quad
    - \int_0^T f(X^{v,\kappa}_t,t) \, \mathrm{d}t - g(X^{v,\kappa}_T)
    \bigr)^2 \bigr].
\end{split}
\end{talign}
Observe that this is a quartic loss function in $\theta$. More work is required to understand the optimization landscape of this loss function.

\paragraph{Value Matching loss for hitting time problems ($T=+\infty$) with RKHS parameterization} If we replace $\phi(x,t) = \langle \theta, \psi(x,t) \rangle_{\mathcal{G}}$ into \eqref{eq:VM_uncapped}, we obtain that
\begin{talign}
\begin{split}
    &\mathcal{L}_{\mathrm{VM}_{v,\kappa}}(\theta) \\ &=
    \mathbb{E} \bigl[ \frac{1}{2} \bigl( \int_0^{\tau} \langle \nabla \langle \theta, \psi(X^{v,\kappa}_t) \rangle_{\mathcal{G}}, \frac{1}{\kappa}\mathrm{d}X^{v,\kappa}_t - ( b(X^{v,\kappa}_t) - \frac{1}{2} \sigma \sigma^{\top} \nabla \langle \theta, \psi(X^{v,\kappa}_t) \rangle_{\mathcal{G}} ) \, \mathrm{d}t \rangle \\ &\qquad \ + \frac{1}{\kappa} \langle \theta, \psi(X^{v,\kappa}_0) \rangle_{\mathcal{G}} + (1 - \frac{1}{\kappa}) \langle \theta, \psi(X^{v,\kappa}_{\tau}) \rangle_{\mathcal{G}} - \int_0^{\tau} f(X^{v,\kappa}_t) \, \mathrm{d}t - g(X^{v,\kappa}_\tau)
    \bigr)^2 \bigr]
    \\ &=
    \mathbb{E} \bigl[ \frac{1}{2} \bigl( \big\langle \theta, \big(\int_0^{\tau} \langle \nabla \psi(X^{v,\kappa}_t), \frac{1}{2} \sigma \sigma^{\top} \nabla \psi(X^{v,\kappa}_t) \rangle \, \mathrm{d}t \big) \theta \rangle_{\mathcal{G}}
    \\ &\qquad\quad + \big\langle \theta, \int_0^{\tau} \langle \nabla_x \psi(X^{v,\kappa}_t), \frac{1}{\kappa} \mathrm{d}X^{v,\kappa}_t - b(X^{v,\kappa}_t) \, \mathrm{d}t \rangle + \frac{1}{\kappa} \psi(X^{v,\kappa}_0)
    + (1 - \frac{1}{\kappa})\psi(X^{v,\kappa}_{\tau}) \big\rangle_{\mathcal{G}} \\
    &\qquad\quad
    - \int_0^{\tau} f(X^{v,\kappa}_t) \, \mathrm{d}t - g(X^{v,\kappa}_\tau)
    \bigr)^2 \bigr].
\end{split}
\end{talign}
This is also a quartic loss function in $\theta$.

\paragraph{Value Matching loss for hitting time problems (finite $T$) with RKHS parameterization}
If we replace $\phi(x,t) = \langle \theta, \psi(x,t) \rangle_{\mathcal{G}}$ into \eqref{eq:VM_loss_f}, we obtain that

\begin{talign}
\begin{split}
    &\mathcal{L}_{\mathrm{VM}_{v,\kappa}}(\theta) \\ &=
    \mathbb{E} \bigl[ \frac{1}{2} \bigl( \int_0^{\tau(T)} \langle \nabla \langle \theta, \psi(X^{v,\kappa}_t) \rangle_{\mathcal{G}}, \frac{1}{\kappa} \mathrm{d}X^{v,\kappa}_t - ( b(X^{v,\kappa}_t) - \frac{1}{2} \sigma \sigma^{\top} \nabla \langle \theta, \psi(X^{v,\kappa}_t) \rangle_{\mathcal{G}} ) \, \mathrm{d}t \rangle \\ &\qquad + \! \frac{1}{\kappa} \langle \theta, \psi(X^{v,\kappa}_0) \rangle_{\mathcal{G}} \! - \! \int_0^{\tau(T)} f(X^{v,\kappa}_t) \, \mathrm{d}t
    \! + \! (\mathbf{1}_{\mathcal{S}}(X^{v,\kappa}_{\tau(T)}) \! - \! \frac{1}{\kappa})
    \langle \theta, \psi(X^{v,\kappa}_{\tau(T)}) \rangle_{\mathcal{G}}
    \! - \! \mathbf{1}_{\mathcal{S}}(X^{v,\kappa}_{\tau(T)}) g(X^{v,\kappa}_{\tau(T)})
    \bigr)^2 \bigr] \\
     &= \mathbb{E} \bigl[ \frac{1}{2} \bigl( \big\langle \theta, \big(\int_0^{\tau(T)} \langle \nabla \psi(X^{v,\kappa}_t), \frac{1}{2} \sigma \sigma^{\top} \nabla \psi(X^{v,\kappa}_t) \rangle \, \mathrm{d}t \big) \theta \rangle_{\mathcal{G}}
    \\ &\qquad\quad + \big\langle \theta, \int_0^{\tau(T)} \langle \nabla \psi(X^{v,\kappa}_t), \frac{1}{\kappa} \mathrm{d}X^{v,\kappa}_t - b(X^{v,\kappa}_t) \, \mathrm{d}t \rangle + \frac{1}{\kappa} \psi(X^{v,\kappa}_0)
    + (\mathbf{1}_{\mathcal{S}}(X^{v,\kappa}_{\tau(T)}) - \frac{1}{\kappa})\psi(X^{v,\kappa}_{\tau(T)}) \big\rangle_{\mathcal{G}} \\
    &\qquad\quad
    - \int_0^{\tau(T)} f(X^{v,\kappa}_t) \, \mathrm{d}t - \mathbf{1}_{\mathcal{S}}(X^{v,\kappa}_{\tau(T)}) g(X^{v,\kappa}_{\tau(T)})
    \bigr)^2 \bigr].
\end{split}
\end{talign}
This is also a quartic loss function in $\theta$.

\section{First-order guarantees for existing losses}

\subsection{Guarantees for the Feynman-Kac loss function by  \cite{strahan2023predicting}} \label{subsec:strahan}
\citet{strahan2023predicting} parameterize the committor using a neural network $\mathfrak{q}$ and consider the following loss
\begin{talign}
\begin{split} \label{eq:strahan_0}
    \mathcal{L}_{\mathrm{FKE}}(\mathfrak{q}) &= \mathbb{E}_{X_0 \sim \mu_0} \big[ \big(\mathbb{E}[\mathfrak{q}(X_{\tau(T)})| X_0\big] - \mathfrak{q}(X_{0}) \big)^2 \mathbf{1}_{(A \cup B)^{c}}(X_{0}) \\ &\qquad\qquad\qquad + \lambda \big( \mathfrak{q}(X_{0}) - \mathbf{1}_{B}(X_0)\big)^2 \mathbf{1}_{A \cup B}(X_{0}) \big],
\end{split}
\end{talign}
where $X$ solves the uncontrolled SDE \eqref{eq:SDE_bg_state_dependent}, and $\mu_0$ is an arbitrary distribution.
Note that this loss is not scalable, because it requires estimating the conditional expectation $\mathbb{E}[\mathfrak{q}(X_{\tau(T)})| X_0\big]$ for every single $X_0$, which requires sampling multiple trajectories for every single starting point $X_0$. We also consider the related losses, which are similar but scalable:
\begin{talign}
\begin{split} \label{eq:strahan_1}
    \bar{\mathcal{L}}_{\mathrm{FKE}}(\mathfrak{q}) &= \mathbb{E}_{X_0 \sim \mu_0} \big[ \mathbb{E}\big[ \big( \mathfrak{q}(X_{\tau(T)}) - \mathfrak{q}(X_{0}) \big)^2 | X_0\big] \mathbf{1}_{(A \cup B)^{c}}(X_{0}) \\ &\qquad\qquad\qquad + \lambda \big( \mathfrak{q}(X_{0}) - \mathbf{1}_{B}(X_0)\big)^2 \mathbf{1}_{A \cup B}(X_{0}) \big],
\end{split}
\end{talign}
and
\begin{talign}
\begin{split} \label{eq:strahan_2}
    \widehat{\mathcal{L}}_{\mathrm{FKE}}(\mathfrak{q}) &= \mathbb{E}_{X_0 \sim \mu_0} \big[ \mathbb{E}\big[ \big( \bar{\mathfrak{q}}(X_{\tau(T)}) - \mathfrak{q}(X_{0}) \big)^2 | X_0\big] \mathbf{1}_{(A \cup B)^{c}}(X_{0}) \\ &\qquad\qquad\qquad + \lambda \big(\mathfrak{q}(X_{0}) - \mathbf{1}_{B}(X_0)\big)^2 \mathbf{1}_{A \cup B}(X_{0}) \big],
\end{split}
\end{talign}
where $\bar{\mathfrak{q}} = \texttt{sg}(\mathfrak{q})$. Next, we prove that the losses $\mathcal{L}_{\mathrm{FKE}}$ and $\widehat{\mathcal{L}}_{\mathrm{FKE}}$ are first-order optimal, but $\bar{\mathcal{L}}_{\mathrm{FKE}}$ is not.
\begin{proposition}
    For any $\lambda > 0$, the loss $\mathcal{L}_{\mathrm{FKE}}$ in equation \eqref{eq:strahan_0} is first-order optimal, i.e. if $\hat{\mathfrak{q}}$ satisfies that $\frac{\partial}{\partial \varepsilon} \mathcal{L}_{\mathrm{FKE}}(\hat{\mathfrak{q}} + \varepsilon \eta) \rvert_{\varepsilon = 0} = 0$ for all perturbations $\eta$, then $\hat{\mathfrak{q}}$ is equal to the committor function $q$.
\end{proposition}
\begin{proof}
    Applying \Cref{thm:V_1_V_2}(ii), the committor function $q$ is characterized as the only function $q$ such that
    \begin{talign} \label{eq:char_committor}
    \begin{cases}
        \forall x \in \mathbb{R}^d, \qquad &q(x) = \mathbb{E}\big[ q(X_{\tau(T)}) | X_0 = x\big], \\
        \forall x \in A, \qquad &q(x) = 0, \\
        \forall x \in B, \qquad &q(x) = 1.
    \end{cases}
    \end{talign}
    Observe that for an arbitrary $\mathfrak{q}$ and an arbitrary perturbation $\eta$,
    \begin{talign}
    \begin{split}
        &\mathcal{L}_{\mathrm{FKE}}(\mathfrak{q} + \varepsilon \eta) \\ &= \mathbb{E}_{X_0 \sim \mu_0} \big[ \big(\mathbb{E}\big[ (\mathfrak{q} + \varepsilon \eta)(X_{\tau(T)}) | X_0 = x\big] - (\mathfrak{q} + \varepsilon \eta)(X_{0}) \big)^2 \mathbf{1}_{(A \cup B)^{c}}(X_{0}) \\ &\qquad\qquad\qquad\qquad\qquad\qquad\quad + \lambda \big((\mathfrak{q} + \varepsilon \eta)(X_{0}) - \mathbf{1}_{B}(X_0)\big)^2 \mathbf{1}_{A \cup B}(X_{0}) \big],
    \end{split} \\
    \begin{split} \label{eq:partial_epsilon_L_FKE}
        &\implies \frac{\partial}{\partial \varepsilon} \mathcal{L}_{\mathrm{FKE}}(\mathfrak{q} + \varepsilon \eta) \rvert_{\varepsilon = 0} \\ &\qquad\qquad\qquad = \mathbb{E}_{X_0 \sim \mu_0} \big[ \big( \mathbb{E}\big[\mathfrak{q}(X_{\tau(T)}) | X_0\big] - \mathfrak{q}(X_{0}) \big) \big( \mathbb{E}\big[\eta(X_{\tau(T)}) | X_0\big] - \eta(X_{0}) \big) \mathbf{1}_{(A \cup B)^{c}}(X_{0}) \\ &\qquad\qquad\qquad\qquad\qquad + \lambda \big( \mathfrak{q}(X_{0}) - \mathbf{1}_{B}(X_0)\big) \eta(X_{0}) \mathbf{1}_{A \cup B}(X_{0}) \big]
    \end{split}
    \end{talign}
    By the characterization \eqref{eq:char_committor}, observe that the committor function $q$ satisfies that
    $\mathcal{L}_{\mathrm{FKE}}(q) = 0$. Let $\hat{\mathfrak{q}}$ be a function that satisfies the first-order optimality condition. Then, by convexity and smoothness, we obtain that
    \begin{talign}
    \begin{split}
        0 = \mathcal{L}_{\mathrm{FKE}}(q) \geq \mathcal{L}_{\mathrm{FKE}}(\hat{\mathfrak{q}}) + \frac{\partial}{\partial \varepsilon} \mathcal{L}_{\mathrm{FKE}}\big(\hat{\mathfrak{q}} + \varepsilon (q-\hat{\mathfrak{q}}) \big) \rvert_{\varepsilon = 0} = \mathcal{L}_{\mathrm{FKE}}(\hat{\mathfrak{q}}) \geq 0.
    \end{split}
    \end{talign}
    Hence, we have that $\mathcal{L}_{\mathrm{FKE}}(\hat{\mathfrak{q}}) = 0$ as well, and this implies the equalities $\mathbb{E}_{X_0 \sim \mu_0} \big[ \big(\mathbb{E}[\hat{\mathfrak{q}}(X_{\tau(T)})| X_0\big] - \hat{\mathfrak{q}}(X_{0}) \big)^2 \mathbf{1}_{(A \cup B)^{c}}(X_{0}) \big] = 0$ and $\mathbb{E}_{X_0 \sim \mu_0} \big[ \big(\hat{\mathfrak{q}}(X_{0}) - \mathbf{1}_{B}(X_0)\big)^2 \mathbf{1}_{A \cup B}(X_{0}) \big] = 0$. Hence, we have that $\hat{\mathfrak{q}}$ satisfies the conditions in equation \eqref{eq:char_committor} almost surely, and thus for all $x \in \mathbb{R}^d$ by continuity. Thus, by the uniqueness of the characterization \eqref{eq:char_committor}, we obtain that $\hat{\mathfrak{q}} = q$.
\end{proof}
\begin{proposition}
    For any $\lambda > 0$, the loss $\widehat{\mathcal{L}}_{\mathrm{FKE}}$ in equation \eqref{eq:strahan_2}
    is first-order optimal, i.e. if $\hat{\mathfrak{q}}$ satisfies that $\frac{\partial}{\partial \varepsilon} \widehat{\mathcal{L}}_{\mathrm{FKE}}(\hat{\mathfrak{q}} + \varepsilon \eta) \rvert_{\varepsilon = 0} = 0$ for all perturbations $\eta$, then $\hat{\mathfrak{q}}$ is equal to the committor function $q$.
\end{proposition}
\begin{proof}
Observe that for an arbitrary $\mathfrak{q}$ and an arbitrary perturbation $\eta$,
\begin{talign}
\begin{split}
    &\widehat{\mathcal{L}}_{\mathrm{FKE}}(\mathfrak{q} + \varepsilon \eta) = \mathbb{E}_{X_0 \sim \mu_0} \big[ \mathbb{E}\big[ \big( \mathfrak{q}(X_{\tau(T)}) - (\mathfrak{q} + \varepsilon \eta)(X_{0}) \big)^2 | X_0\big] \mathbf{1}_{(A \cup B)^{c}}(X_{0}) \\ &\qquad\qquad\qquad\qquad\qquad\qquad\quad + \lambda \big( (\mathfrak{q} + \varepsilon \eta)(X_{0}) - \mathbf{1}_{B}(X_0)\big)^2 \mathbf{1}_{A \cup B}(X_{0}) \big],
\end{split} \\
\begin{split} \label{eq:partial_epsilon_hat_L_FKE}
    &\implies \frac{\partial}{\partial \varepsilon} \widehat{\mathcal{L}}_{\mathrm{FKE}}(\mathfrak{q} + \varepsilon \eta) \rvert_{\varepsilon = 0} \\ &\qquad\qquad\qquad = \mathbb{E}_{X_0 \sim \mu_0} \big[ \mathbb{E}\big[\langle \mathfrak{q}(X_{\tau(T)}) - \mathfrak{q}(X_{0}), - \eta(X_{0}) \rangle | X_0\big] \mathbf{1}_{(A \cup B)^{c}}(X_{0}) \\ &\qquad\qquad\qquad\qquad\qquad + \lambda \langle \mathfrak{q}(X_{0}) - \mathbf{1}_{B}(X_0), \eta(X_{0}) \rangle \mathbf{1}_{A \cup B}(X_{0}) \big]
    \\ &\qquad\qquad\qquad = \mathbb{E}_{X_0 \sim \mu_0} \big[ \langle \mathbb{E}\big[ \mathfrak{q}(X_{\tau(T)}) | X_0\big] - \mathfrak{q}(X_{0}), - \eta(X_{0}) \rangle \mathbf{1}_{(A \cup B)^{c}}(X_{0}) \\ &\qquad\qquad\qquad\qquad\qquad + \lambda \langle \mathfrak{q}(X_{0}) - \mathbf{1}_{B}(X_0), \eta(X_{0}) \rangle \mathbf{1}_{A \cup B}(X_{0}) \big],
\end{split}
\end{talign}
Now, if we set $\eta(X_{0}) = \mathbb{E}\big[ \mathfrak{q}(X_{\tau(T)}) | X_0\big] - \mathfrak{q}(X_{0})$, the right-hand side becomes
\begin{talign}
\begin{split}
    &- \mathbb{E}_{X_0 \sim \mu_0} \big[ \big( \mathbb{E}\big[ \mathfrak{q}(X_{\tau(T)}) | X_0\big] - \mathfrak{q}(X_{0}) \big)^2 \mathbf{1}_{(A \cup B)^{c}}(X_{0}) \\ &\qquad\qquad + \lambda \langle \mathfrak{q}(X_{0}) - \mathbf{1}_{B}(X_0), \mathbb{E}\big[ \mathfrak{q}(X_{\tau(T)}) | X_0\big] - \mathfrak{q}(X_{0}) \rangle \mathbf{1}_{A \cup B}(X_{0}) \big] \\ &= - \mathbb{E}_{X_0 \sim \mu_0} \big[ \big( \mathbb{E}\big[ \mathfrak{q}(X_{\tau(T)}) | X_0\big] - \mathfrak{q}(X_{0}) \big)^2 \mathbf{1}_{(A \cup B)^{c}}(X_{0}) \big],
\end{split}
\end{talign}
where the last equality holds because when $X_0 \in A \cup B$, the capped hitting time satisfies $\tau(T) = 0$, and thus, $\mathbb{E}\big[ \mathfrak{q}(X_{\tau(T)}) | X_0\big] = \mathfrak{q}(X_{0})$. We conclude that any minimizer of $\widehat{\mathcal{L}}_{\mathrm{FKE}}$ must satisfy for almost all $x \in \mathbb{R}^d$ that
\begin{talign} \label{eq:first_cond_char}
\mathbb{E}\big[ \mathfrak{q}(X_{\tau(T)}) | X_0 = x\big] = \mathfrak{q}(x).
\end{talign}
Next, we set $\eta$ such that for $x \in A \cup B$, $\eta(x) = \mathfrak{q}(x) - \mathbf{1}_{B}(x)$. Plugging this into the right-hand side of \eqref{eq:partial_epsilon_hat_L_FKE}, and using that equation \eqref{eq:first_cond_char} must hold for any minimizer, we have that for any minimizer,
\begin{talign}
    \mathbb{E}_{X_0 \sim \mu_0} \big[ \big( \mathfrak{q}(X_{0}) - \mathbf{1}_{B}(X_0) \big)^2 \mathbf{1}_{A \cup B}(X_{0}) \big] = 0,
\end{talign}
which means that $\mathfrak{q}(x) = \mathbf{1}_{B}(x)$ for almost all $x \in A \cup B$. This concludes the proof because we have shown that any minimizer $\hat{\mathfrak{q}}$ must satisfy all the conditions in the characterization \eqref{eq:char_committor} almost surely, and thus for all $x$ by continuity.
\end{proof}
\begin{theorem}
    The loss $\bar{\mathcal{L}}_{\mathrm{FKE}}$ in equation \eqref{eq:strahan_1} is not first-order optimal. In particular, the committor function $q$ is not a critical point of $\bar{\mathcal{L}}_{\mathrm{FKE}}$, i.e. there exists a perturbation $\eta$ such that $\frac{\partial}{\partial \varepsilon} \bar{\mathcal{L}}_{\mathrm{FKE}}(q + \varepsilon \eta) \rvert_{\varepsilon = 0} < 0$.
\end{theorem}
\begin{proof}
Observe that for an arbitrary $\mathfrak{q}$ and an arbitrary perturbation $\eta$,
\begin{talign}
\begin{split}
    &\bar{\mathcal{L}}_{\mathrm{FKE}}(\mathfrak{q} + \varepsilon \eta) = \mathbb{E}_{X_0 \sim \mu_0} \big[ \mathbb{E}\big[ \big( (\mathfrak{q} + \varepsilon \eta)(X_{\tau(T)}) - (\mathfrak{q} + \varepsilon \eta)(X_{0}) \big)^2 | X_0\big] \mathbf{1}_{(A \cup B)^{c}}(X_{0}) \\ &\qquad\qquad\qquad\qquad\qquad\qquad\quad + \lambda \big( (\mathfrak{q} + \varepsilon \eta)(X_{0}) - \mathbf{1}_{B}(X_0)\big)^2 \mathbf{1}_{A \cup B}(X_{0}) \big],
\end{split} \\
\begin{split} \label{eq:partial_epsilon_2}
    &\implies \frac{\partial}{\partial \varepsilon} \bar{\mathcal{L}}_{\mathrm{FKE}}(\mathfrak{q} + \varepsilon \eta) \rvert_{\varepsilon = 0} = \\ &\qquad\qquad\qquad \mathbb{E}_{X_0 \sim \mu_0} \big[ \mathbb{E}\big[\big( \mathfrak{q}(X_{\tau(T)}) - \mathfrak{q}(X_{0})\big) \big( \eta(X_{\tau(T)}) - \eta(X_{0}) \big) | X_0\big] \mathbf{1}_{(A \cup B)^{c}}(X_{0}) \\ &\qquad\qquad\qquad\qquad\qquad + \lambda \big( \mathfrak{q}(X_{0}) - \mathbf{1}_{B}(X_0) \big) \eta(X_{0}) \mathbf{1}_{A \cup B}(X_{0}) \big],
\end{split}
\end{talign}
The committor function $q$ satisfies the characterization \eqref{eq:char_committor}, which means that $\mathbb{E}_{X_0 \sim \mu_0} \big[ \big( q(X_{0}) - \mathbf{1}_{B}(X_0) \big) \eta(X_{0}) \mathbf{1}_{A \cup B}(X_{0}) \big] = 0$ for any perturbation $\eta$. Now, if we set $\eta(x) = - q(x)$ for any $x \in (A \cup B)^c$, we obtain that the right-hand side of \eqref{eq:partial_epsilon_2} becomes:
\begin{talign}
    \frac{\partial}{\partial \varepsilon} \bar{\mathcal{L}}_{\mathrm{FKE}}(q + \varepsilon \eta) \rvert_{\varepsilon = 0} = - \mathbb{E}_{X_0 \sim \mu_0} \big[ \mathbb{E}\big[\big( q(X_{\tau(T)}) - q(X_{0}) \big)^2 | X_0\big] \mathbf{1}_{(A \cup B)^{c}}(X_{0}) \big].
\end{talign}
Note that the right-hand side is negative, which means that this $\eta$ is a descent direction of $\bar{\mathcal{L}}_{\mathrm{FKE}}$ on the committor function $q$, and thus, the committor function is not a critical point of the loss.
\end{proof}

\subsection{Guarantees of the loss function by \cite{mitchell2024committor}} \label{subsec:committor}

\citet{mitchell2024committor} parameterize the committor using a neural network $\mathfrak{q}$ and consider the following loss, which we denote by logarithmic committor loss:
\begin{talign}
    \mathcal{L}_{\mathrm{LOG}}(\mathfrak{q}) &= \mathbb{E}_{X_0 \sim \mu_0} \big[ \big( \log \mathbb{E}[\bar{\mathfrak{q}}(X_{\tau(T)}) \mathbf{1}_{(A \cup B)^{c}}(X_{\tau(T)}) + \mathbf{1}_{B}(X_{\tau(T)}) | X_0\big] - \log \mathfrak{q}(X_{0}) \big)^2 \big],
\end{talign}
where $\bar{\mathfrak{q}} = \texttt{sg}(\mathfrak{q})$\footnote{Even though the loss they write in their equation 25 does not make the stopped gradient explicit, they comment on it in the subsequent paragraph.}. Note that this loss is hard to scale, because it requires estimating a conditional expectation with respect to $X_0$ at each point $X_0$.

\begin{theorem}
    The loss $\mathcal{L}_{\mathrm{LOG}}$ is first-order optimal, i.e. if $\hat{\mathfrak{q}}$ satisfies that $\frac{\partial}{\partial \varepsilon} \mathcal{L}_{\mathrm{LOG}}(\hat{\mathfrak{q}} + \varepsilon \eta) \rvert_{\varepsilon = 0} = 0$ for all perturbations $\eta$, then $\hat{\mathfrak{q}}$ is equal to the committor function $q$.
\end{theorem}
\begin{proof}
    Observe that for an arbitrary $\mathfrak{q}$ and an arbitrary perturbation $\eta$,
    \begin{talign}
    \begin{split}
        &\mathcal{L}_{\mathrm{LOG}}(\mathfrak{q} + \varepsilon \eta) \\ &= \mathbb{E}_{X_0 \sim \mu_0} \big[ \big( \log \mathbb{E}[\bar{\mathfrak{q}}(X_{\tau(T)}) \mathbf{1}_{(A \cup B)^{c}}(X_{\tau(T)}) + \mathbf{1}_{B}(X_{\tau(T)}) \mathbf{1}_{A \cup B}(X_{\tau(T)}) | X_0\big] - \log (\mathfrak{q}+\varepsilon \eta)(X_{0}) \big)^2 \big],
    \end{split} \\
    \begin{split} \label{eq:partial_epsilon_L_log}
        &\implies \frac{\partial}{\partial \varepsilon} \mathcal{L}_{\mathrm{LOG}}(\mathfrak{q} + \varepsilon \eta) \rvert_{\varepsilon = 0} \\ &\quad = \mathbb{E}_{X_0 \sim \mu_0} \big[ \big( \log \mathfrak{q}(X_{0}) - \log \mathbb{E}[\bar{\mathfrak{q}}(X_{\tau(T)}) \mathbf{1}_{(A \cup B)^{c}}(X_{\tau(T)}) +
        \mathbf{1}_{B}
        (X_{\tau(T)}) \mathbf{1}_{A \cup B}(X_{\tau(T)}) | X_0\big] \big) \frac{\eta(X_0)}{\mathfrak{q}(X_{0})} \big].
    \end{split}
    \end{talign}
    Thus, the condition $\frac{\partial}{\partial \varepsilon} \mathcal{L}_{\mathrm{LOG}}(\hat{\mathfrak{q}} + \varepsilon \eta) \rvert_{\varepsilon = 0} = 0$ implies that:
    \begin{talign}
    \begin{split}
        0 = \mathbb{E}_{X_0 \sim \mu_0} \big[ \big( \log \hat{\mathfrak{q}}(X_{0}) - \log \mathbb{E}[\hat{\mathfrak{q}}(X_{\tau(T)}) \mathbf{1}_{(A \cup B)^{c}}(X_{\tau(T)}) +
        \mathbf{1}_{B}
        (X_{\tau(T)}) \mathbf{1}_{A \cup B}(X_{\tau(T)}) | X_0\big] \big) \frac{\eta(X_0)}{\hat{\mathfrak{q}}(X_{0})} \big].
    \end{split}
    \end{talign}
    For this to hold true for any perturbation $\eta$, we need that for almost any $x \in \mathbb{R}^d$,
    \begin{talign}
    \begin{split}
        \log \hat{\mathfrak{q}}(x) = \log \mathbb{E}[\hat{\mathfrak{q}}(X_{\tau(T)}) \mathbf{1}_{(A \cup B)^{c}}(X_{\tau(T)}) +
        \mathbf{1}_{B}
        (X_{\tau(T)}) \mathbf{1}_{A \cup B}(X_{\tau(T)}) | X_0 = x\big] \\
        \iff \hat{\mathfrak{q}}(x) = \mathbb{E}[\hat{\mathfrak{q}}(X_{\tau(T)}) \mathbf{1}_{(A \cup B)^{c}}(X_{\tau(T)}) +
        \mathbf{1}_{B}
        (X_{\tau(T)}) \mathbf{1}_{A \cup B}(X_{\tau(T)}) | X_0 = x\big].
    \end{split}
    \end{talign}
    Since when $x \in \mathcal{S}$, $\tau(T) = 0$, this equation is equivalent to equation \eqref{eq:char_committor}, which implies that its only solution is the committor function $q$.
\end{proof}

\section{Related work}

\subsection{Stochastic optimal control for sampling and inference} Our work builds on and extends prior works on deep learning methods for general SOC problems, as well as works that study rare event simulation from the SOC perspective. The Value Matching loss is related to Soft Actor-Critic. We discuss relevant references in these three subareas below.

\textbf{Solving SOC Problems with Deep Learning}. \citet{weinan2017deep,han2018solving} pioneered deep learning SOC algorithms based on solving forward-backward SDEs (FBSDEs). Building on those, \citet{nusken2021solving} instead introduced the moment loss and the log-variance loss. All these losses are intimately connected to the Value Matching loss that we introduce (see \Cref{subsec:comparison_SOC_losses} for a thorough comparison). \citet{kappen2016adaptive,hua2024efficient} proposed importance sampling-based approaches for adaptive and off-policy learning. Stochastic optimal control matching~\citep{domingo2024stochastic} rewrote the cross-entropy method as a regression loss with variance reduction. Adjoint matching~\citep{domingoadjoint2025} solved the SOC problem by explicitly integrating the adjoint ODE while proposing a more computationally efficient adjoint-matching loss. \citet{pidstrigach2025conditioning} reformulate and extend the framework of \citep{domingo2024stochastic} from a Malliavin calculus perspective. For a recent review, we refer to \citet{domingo2024stochastic,domingo2024taxonomy}.

\textbf{Soft Actor-Critic methods}. The Soft Actor-Critic (SAC) method was introduced in RL by \citet{haarnoja2018soft}, and has become the most prominent approach to off-policy RL in many areas. Since SOC is KL-regularized RL with SDE trajectories, SAC can be adapted to SOC problems, and \cite{zhou2025solving} studies its convergence in the episodic setting, i.e. in the case where the loss function is constructed at a trajectory level. We compare the VM loss to actor-critic methods in detail in \Cref{subsec:SAC}.

\textbf{SOC in Rare Events}. The closest works to us are~\cite{hartmann2012efficient,hartmann2013characterization}, which first introduce SOC to characterize rare events. In particular, \citet{hartmann2013characterization} contains a control formulation with hitting times which is very related to ours; we devote \Cref{subsec:hartmann} to a thorough comparison with~\citep{hartmann2013characterization}. Subsequent works have been extended using neural network parameterization and improved losses~\citep{yan2022learning,holdijk2024stochastic}, although their settings are different as they do not consider hitting time problems.
A related rare event SOC problem with fixed horizon as well is estimating Doob's $h$-transforms.
\citet{singh2023variational,singh2024splitting,du2024doob} proposed deep learning approaches to learn Doob's $h$-transforms.
Very few works in SOC or RL have proposed principled algorithms for stopping time horizon problems, one of them being \cite{ribera2025random_horizons} which shows that the REINFORCE algorithm can be extended to stopping time horizons.

\subsection{Machine learning for rare event sampling} Below, we summarize the main existing methods for rare event sampling, which are not based on SOC.

\textbf{Estimating Committor Function}. \citet{khoo2019solving} first proposed to parameterize the committor function by neural networks and solve the variational formulation of backward Kolmogorov equation (BKE)~\citep{vanden2006towards}. To mitigate the sampling difficulty in reactive region, \citet{li2019computing} proposed to use metadynamics and temperature annealing to accelerate the sampling and training procedure. Later on, the idea of active importance sampling was introduced by \citet{rotskoff2022active} to use learned committor function with umbrella sampling to augment the sampling process. \citet{hasyim2022supervised} further improved the approach by including supervised learning and better umbrella initialization. \citet{kang2024computing} extended this idea by using a different design of bias potential related to the minimum variance estimator of the reaction rate. In addition to solving the variational formulation for BKE,  \citet{li2022semigroup,strahan2023predicting} learn to match the committor function through the Feynman-Kac formulation in a self-consistent manner. \citet{mitchell2024committor} instead matches the logarithm of the committor function to emphasize regions near one side.

\textbf{Transition Path Sampling}. Transition path sampling methods were established by ~\citet{bolhuis2002transition} and provided a way to estimate reaction rate. \citet{peters2006obtaining} formulated learning committor function as a maximum-likelihood loss (the neural network parameterization was introduced in \citet{jung2019artificial}) and further rewritten by \citet{sun2022multitask} as a classification problem when transition paths from each configuration are available. \citet{hummer2004transition} built the connection between the posterior probability of transition path given a configuration and the committor function. The idea was leveraged in \citet{jung2023machine} in combination with a path sampling method to learn the committor function.

\subsection{\texorpdfstring{Comparison with related stochastic optimal control-based committor approaches}{Comparison with related stochastic optimal control-based committor approaches}}
\label{subsec:related_soc_committor}

The connection between committor learning, stochastic optimal control, and conditioned stochastic dynamics has been studied in prior work. \citet{hartmann2013characterization} was the first work to formulate committor estimation as a hitting-time control problem, without providing algorithms or empirical results.
Subsequently, \citet{hartmann2019variational} developed a least-squares off-policy loss based on forward-backward SDEs, allowing changes to the sampling drift through a related backward process correction.
More recently, \citet{hartmann2026importance} proposed to learn the committor by parameterizing the value function through a policy iteration algorithm with simulation from the trajectory based on the currently learned control.
Nevertheless, both \citet{hartmann2019variational,hartmann2026importance} truncated the hitting-time horizon in practice, without providing guarantees for it.
\citet{singh2023variational,holdijk2024stochastic,du2024doob} developed variational approaches to learn Doob's $h$-transform or the optimal control of the time-dependent stochastic optimal control problem, while \citet{singh2024splitting} related the $h$-transform to the committor function through spectral analysis and a bistability assumption. \citet{yuan2023optimal} similarly established the optimality of committor-based control for sampling transition paths, while the committor was learned through the BKE variational loss. REACT bridges the gap to enable learning committor functions via a principled stochastic optimal control framework: (1) we proved that solving the capped hitting-time control problem leads to the optimal solution of the committor optimal control problem; (2) we introduced an off-policy Value Matching loss with a first-order optimality guarantee, accommodating flexible sampling drifts and noise levels, thereby enabling learning with the $\kappa$-rescaled process, which can accelerate exploration.

\subsection{Comparison with the SOC committor functions framework of \cite{hartmann2013characterization}} \label{subsec:hartmann}
The beginning of our
\Cref{sec:methodology}
mirrors closely the SOC formulation of the committor function problem exposed in \citet[Sec.~6.1,~7.2]{hartmann2013characterization}, but beyond that, our work differs substantially from theirs. These are the biggest differences:
\begin{enumerate}[label=(\roman*),left=3pt]
    \item \emph{Uncapped vs. capped hitting times}: The framework of \cite{hartmann2013characterization} only considers uncapped hitting times $\tau$, which as we discuss in \Cref{subsec:SOC_hitting_times} can make algorithms potentially very expensive to run when both $A$ and $B$ are very hard to reach. In \Cref{subsec:SOC_hitting_times}, we show that the SOC problem with uncapped hitting times is equivalent to two fixed-point problems that involve hitting times $\tau(T)$ capped at an arbitrary horizon $T \in (0,+\infty) \cup \{+\infty\}$. In \Cref{sec:solving_SOC}, the algorithms that we propose use capped hitting times.
    \item \emph{Algorithmic improvements}: Up to the fact that we use capped hitting times, the algorithmic realization proposed by \citet[Sec.~7.3]{hartmann2013characterization} is essentially the direct backpropagation approach that we describe in \Cref{subsec:kl_method}. However, this loss function is not entirely principled, because we are not able to backpropagate through the hitting time, and it does not handle off-policy sampling, all of which motivates our development of the Value Matching loss.
    \item \emph{Choosing the terminal costs appropriately}: \citet{hartmann2013characterization} considers terminal costs of the form $g(x) = - \log \mathbf{1}_{B}(x)$, which are ill-posed numerically due to $-\log 0 = +\infty$.
    As we describe in \Cref{sec:methodology}, we reformulate the problem to learning the rescaled committor $q_{\xi}$, which does have well-behaved terminal costs.
    \item \emph{Architectural differences}: \citet{hartmann2013characterization} parameterize the learned committor using a finite mixture of feature functions, which may not scale well to high-dimensional problems. The architecture that we introduce in  \Cref{subsec:parameterizing_committor}
    is scalable and is critical to the empirical success of our method.
    \item \emph{reactive current-preserving simulation}: One major limitation of the SOC formulation is that it requires sampling transition paths. We derive the stationary distribution of the controlled process and a reactive current-preserving process for sampling transition paths, which can accelerate exploration.
\end{enumerate}

\subsection{Comparison with the FBSDE and moment log-variance SOC losses} \label{subsec:comparison_SOC_losses}

The Value Matching loss function builds on existing works on deep learning SOC loss functions. In this subsection we compare the time-dependent, finite-horizon Value Matching loss introduced in \Cref{thm:vm_finite} to these loss functions, which were also developed in the time-dependent, finite-horizon setting.

Given the control problem \eqref{eq:finite_horizon_SOC_1}-\eqref{eq:finite_horizon_SOC_2}, the forward and backward SDEs (FBSDEs) are the coupled equations:
\begin{talign} \label{eq:forward_SDE}
    \mathrm{d}X_s &= b(X_s,s) \, \mathrm{d}s + \sigma(X_s,s) \, \mathrm{d}W_s, \qquad X_0 \sim \mu_0, \\
    \mathrm{d}Y_s &= \big( - f(X_s,s) + \frac{1}{2} \|Z_s\|^2 \big) \, \mathrm{d}s + \langle Z_s, \mathrm{d}W_s \rangle, \qquad Y_T = g(X_T)
    \label{eq:backward_SDE}
\end{talign}
Setting $Y_s = V(X_s,s)$ and $Z_s = - u^{\star}(X_s,s) = \sigma^{\top} \nabla V(X_s,s)$ solves \eqref{eq:forward_SDE}-\eqref{eq:backward_SDE}.

FBSDEs were used to design Monte Carlo methods for SOC in the early nineties, a notable mention being \emph{least squares Monte Carlo}, in which $(Z_s)_{0 \leq s \leq T}$ is approximated backwards in time by iteratively solving least squares regression problems (see \cite[Chapter~3]{pham2009continuous}). \cite{weinan2017deep,han2018solving} developed pioneering deep learning SOC methods based on FBSDEs. They parameterized the control function using a neural network $u = u_{\theta}$, and considered the following object:
\begin{talign} \label{eq:Y_T}
    Y_T(y_0,u) = y_0 + \int_0^T \big( - f(X_s,s) + \frac{1}{2}\|u(X_s,s)\|^2 \big) \, \mathrm{d}s - \int_0^T \langle u(X_s,s), \mathrm{d}W_s \rangle,
\end{talign}
which is a solution of the backward SDE \eqref{eq:backward_SDE} with $Z_s = - u(X_s,s)$, and with the initial condition $Y_0 = y_0$ instead of the final condition $Y_T = g(X_T)$. The loss function they consider is precisely the $L^2$ discrepancy between $Y_T(y_0,u)$ and its target value $g(X_T)$:
\begin{talign} \label{eq:L_FBSDE}
    \mathcal{L}_{\mathrm{FBSDE}}(u,y_0) = \mathbb{E}\big[\frac{1}{2}\big( Y_T(y_0,u) - g(X_T) \big)^2 \big].
\end{talign}
That is, $y_0 \in \mathbb{R}$ and $u=u_{\theta}$ are optimized jointly.
In the setting that they consider, the initial distribution is a Dirac delta: $\mu_0 = \delta_{x_0}$, which means that at optimality $y^{\star}_0 = V(x_0,0)$.

A shortcoming of the loss $\mathcal{L}_{\mathrm{FBSDE}}$ is that the uncontrolled trajectories $X$ are usually quite different from the optimally controlled ones, which means that the control may not be well learned at regions of interest. To remedy this, \citet[Sec.~III.B]{hartmann2019variational} introduced a generalized FBSDE system:
\begin{talign} \label{eq:general_forward_SDE}
    \mathrm{d}X^v_s &= \big( b(X^v_s,s) + \sigma(X^v_s,s) v(X^v_s,s) \big) \, \mathrm{d}s + \sigma(X^v_s,s) \, \mathrm{d}W_s, \qquad X_0 \sim \mu_0, \\
    \mathrm{d}Y_s &= \big( - f(X_s,s) + \langle v(X^v_s,s), Z_s \rangle + \frac{1}{2} \|Z_s\|^2 \big) \, \mathrm{d}s + \langle Z_s, \mathrm{d}W_s \rangle, \qquad Y_T = g(X_T)
    \label{eq:general_backward_SDE}
\end{talign}
where $v$ is an arbitrary vector-valued function with the appropriate regularity conditions. In analogy with \eqref{eq:Y_T}, \cite{nusken2021solving} considered the object
\begin{talign}
\begin{split} \label{eq:Y_T_nusken}
    Y^v_T(y_0,u) &= y_0 + \int_0^T \big( - f(X^v_s,s) - \langle u(X^v_s,s), v(X^v_s,s) \rangle + \frac{1}{2} \|u(X^v_s,s)\|^2 \big) \, \mathrm{d}s \\ &\quad - \int_0^T \langle u(X^v_s,s), \mathrm{d}W_s \rangle,
\end{split}
\end{talign}
and in analogy with \eqref{eq:L_FBSDE}, they introduced the moment loss function\footnote{Together with the moment loss, \cite{nusken2021solving} also introduced the log-variance loss, in which the parameter $y_0$ is replaced with an empirical average: $\mathcal{L}_{\mathrm{LogVar}_{v}}(u) \! = \! \mathrm{Var}[ Y^v_T(0,u) \! - \! g(X^v_T) ] \! = \! \mathbb{E} [ ( Y^v_T(0,u) \! - \! g(X^v_T) )^2 ] \! - \! \mathbb{E} [ Y^v_T(0,u) \! - \! g(X^v_T) ]^2$. Although the log-variance loss has been comparatively more popular than the moment loss, it is not as closely related to the Value Matching loss.}, which generalizes the FBSDE loss as follows:
\begin{talign} \label{eq:Mom_v}
    \mathcal{L}_{\mathrm{Mom}_v}(u,y_0) = \mathbb{E}\big[ \frac{1}{2} \big( Y^v_T(y_0,u) - g(X^v_T) \big)^2 \big].
\end{talign}
\cite{nusken2021solving} also consider the setting $\mu_0 = \delta_{x_0}$, under which $y^{\star}_0 = V(x_0,0)$ as well.

Next, we compare $\mathcal{L}_{\mathrm{Mom}_v}$ with the Value Matching loss function. If we set $\kappa = 1$ and we substitute equation \eqref{eq:general_forward_SDE} into equation \eqref{eq:L_VM_finite_rewritten}, the VM loss reads
\begin{talign}
\begin{split} \label{eq:L_VM_finite_rewritten_2}
    \mathcal{L}_{\mathrm{VM}_{v,1}}(\phi) &=
    \mathbb{E} \bigl[ \frac{1}{2} \bigl( \int_0^{T} \langle \sigma^{\top} \nabla_x \phi(X^{v}_t,t), (v(X^v_s,s) + \frac{1}{2} \sigma^{\top} \nabla_x \phi(X^{v}_t,t) ) \, \mathrm{d}t + \mathrm{d}W_t \rangle \\ &\qquad\quad + \phi(X^{v}_0,0) - \int_0^T f(X^{v}_t,t) \, \mathrm{d}t - g(X^{v}_T) \bigr)^2 \bigr],
\end{split}
\end{talign}
And identifying $-\sigma^{\top} \nabla_x \phi(x,t)$ with $u(x,t)$, we can further rewrite this as
\begin{talign}
    \mathcal{L}_{\mathrm{VM}_{v,1}}(\phi) = \mathbb{E}\big[ \frac{1}{2} \big( Y^v_T(\phi(X^v_0,0),-\sigma^{\top} \nabla_x \phi) - g(X^v_T) \big)^2 \big].
\end{talign}
That is, besides the control parameterization $u = -\sigma^{\top} \nabla_x \phi$, the main difference between the moment loss and the VM loss is that instead of having an additional parameter $y_0 \in \mathbb{R}$ to fit the initial condition of the backward SDE, we use the learned value function $\phi(X^v_0,0)$ itself. This replacement is intuitive, since in the case $\mu_0 = \delta_{x_0}$ the optimal $y^{\star}_0$ is the value function $V(x_0,0)$. Moreover, since $\phi(\cdot,0)$ is a function instead of a scalar, the VM loss can handle arbitrary initial distributions $\mu_0$.

\begin{remark}[Failure mode of the moment and log-variance loss function with arbitrary $\mu_0$]
    Running the moment loss function with an arbitrary initial distribution $\mu_0$ yields undesired behavior. Namely, if we decompose $Y^v_T(y_0,u) = y_0 + Y^v_T(u)$, we have that
    \begin{talign}
    \begin{split}
        \mathcal{L}_{\mathrm{Mom}_v}(u,y_0) &= \mathbb{E}\big[ \frac{1}{2} \big( y_0 + Y^v_T(u) - g(X^v_T) \big)^2 \big] \\ &= \mathbb{E}\big[ \frac{1}{2} \big( y_0 + \mathbb{E}\big[ Y^v_T(u) - g(X^v_T) \big] - \mathbb{E}\big[ Y^v_T(u) - g(X^v_T) \big] \\ &\qquad\quad + \mathbb{E}\big[ Y^v_T(u) - g(X^v_T) \big| X^v_0 \big] - \mathbb{E}\big[ Y^v_T(u) - g(X^v_T) \big| X^v_0 \big] + Y^v_T(u) - g(X^v_T) \big)^2 \big] \\ &= \frac{1}{2} \big( y_0 + \mathbb{E}\big[ Y^v_T(u) - g(X^v_T) \big] \big)^2 \\ &\quad + \frac{1}{2} \mathbb{E} \big[ \big( \mathbb{E}\big[ Y^v_T(u) - g(X^v_T) \big| X^v_0 \big] - \mathbb{E}\big[ Y^v_T(u) - g(X^v_T) \big] \big)^2 \big] \\ &\quad + \frac{1}{2} \mathbb{E} \big[ \big( Y^v_T(u) - g(X^v_T) - \mathbb{E}\big[ Y^v_T(u) - g(X^v_T) \big| X^v_0 \big] \big)^2 \big]
    \end{split}
    \end{talign}
    Thus,
    \begin{talign}
    \begin{split}
        \min_{y_0} \mathcal{L}_{\mathrm{Mom}_v}(u,y_0) &= \frac{1}{2} \mathbb{E} \big[ \big( \mathbb{E}\big[ Y^v_T(u) - g(X^v_T) \big| X^v_0 \big] - \mathbb{E}\big[ Y^v_T(u) - g(X^v_T) \big] \big)^2 \big] \\ &\quad + \frac{1}{2} \mathbb{E} \big[ \big( Y^v_T(u) - g(X^v_T) - \mathbb{E}\big[ Y^v_T(u) - g(X^v_T) \big| X^v_0 \big] \big)^2 \big]
    \end{split}
    \end{talign}
    Thus, when $\mu_0$ is not a Dirac delta, the moment loss is imposing a trade-off between making $Y^v_T(u) - g(X^v_T) - \mathbb{E}\big[ Y^v_T(u) - g(X^v_T) \big| X^v_0 \big]$ small, which is the $L^2$ discrepancy of the generalized backward SDE \eqref{eq:general_backward_SDE} with the choice $y_0 = \mathbb{E}\big[ Y^v_T(u) - g(X^v_T) \big| X^v_0 \big]$, and making $\mathbb{E}\big[ Y^v_T(u) - g(X^v_T) \big| X^v_0 \big] - \mathbb{E}\big[ Y^v_T(u) - g(X^v_T) \big]$ small, which is an undesirable term that should not be there as it makes the loss function biased. The log-variance loss is also biased for arbitrary $\mu_0$, by an analogous argument. Switching to the VM loss is the fix to correct this bias.
\end{remark}

\begin{remark}[Setting $v = \texttt{sg}(u_{\theta})$]
    \cite{nusken2021solving} propose to choose $v = \texttt{sg}(u)$ as the stopped gradient version of the learned control $u_{\theta}$, and refer to all the methods that sample using the detached version of $u_{\theta}$ as Iterative Diffusion Optimization (IDO) algorithms. Our REACT framework also leverages the learned rescaled committor to sample the trajectories, and hence falls within the IDO setup, but we have an additional knob, $\kappa$, which lets us sample at arbitrary noise level. Sampling trajectories at different noise levels had not been considered by previous works.
\end{remark}

\begin{remark}[Comparison with the BSDE loss from \cite{nusken2023interpolating}]
    The work \cite{nusken2023interpolating} by the same authors introduces the BSDE loss, which in our setting matches the VM loss with $\kappa = 1$ as written in \eqref{eq:L_VM_finite_rewritten_2}. That is, instead of using an additional parameter $y_0$, the BSDE loss does use $\phi(X^{v}_0,0)$ as an approximation to the value function $V(X^{v}_0,0)$. In fact, their formulation is more general in that they consider elliptic and parabolic boundary value problems, and in their section 6.2, they study committor problems in a toy (non-standard) setting. Interestingly, they use the BSDE loss for the backward Kolmogorov equation \eqref{eq:BKE_bg}, instead of the HJB equation \eqref{eq:HJB_setup} like we do. Using the HJB is the superior approach because it yields a control problem that naturally allows us to sample rollouts with controlled SDEs. Besides this, other significant technical improvements of our work over \cite{nusken2023interpolating} include being able to handle different noise levels $\kappa > 0$, and our first-order optimality analysis.
\end{remark}

\subsection{Comparison with REINFORCE-style loss functions}
\label{subsec:girsanov_gradient}

REINFORCE-style loss functions like the one introduced by~\cite{hua2024efficient} exploit Girsanov's theorem to express the gradient of the control objective as a path-space expectation.
The starting point is the control objective for the capped problem:
\begin{equation} \label{eq:J_theta}
J(\theta) = \mathbb{E}\Big[\frac{1}{2}\int_0^{\tau(T)} |u_\theta(X^\theta_t)|^2\,\mathrm{d}t + \bar g_\theta(X^\theta_{\tau(T)})\Big],
\end{equation}
where $u_\theta = -\sigma^\top\nabla\phi_\theta$ is the feedback control induced by the current value estimate, $X^\theta$ solves~\eqref{eq:DBP_SDE}, and $\bar g_\theta$ denotes the terminal cost (either $g$ on $\partial\mathcal{S}$ or $\texttt{sg}(\phi_\theta)$ otherwise). By applying Girsanov's theorem, one can show that the gradient $\nabla_\theta J(\theta)$ can be written as an expectation over the controlled process \emph{without differentiating through the SDE}:
\begin{equation} \label{eq:girsanov_grad}
\nabla_\theta J(\theta) = \mathbb{E}\Big[\Big(\frac{1}{2}\int_0^{\tau(T)} |u_\theta(X^\theta_t)|^2\,\mathrm{d}t + \bar g_\theta(X^\theta_{\tau(T)})\Big) \int_0^{\tau(T)} \big\langle \nabla_\theta u_\theta(X^\theta_t),\,\mathrm{d}W_t\big\rangle\Big].
\end{equation}
This is a REINFORCE-like estimator: the SOC cost multiplies a stochastic integral that depends on the parameter gradient of the control, but not on the gradient of the trajectory itself. This estimator does not require backpropagating through the dynamics and yields an unbiased gradient,
but its main limitation is its variance: as with any REINFORCE-type estimator, the stochastic integral factor can lead to high variance, particularly for long trajectories or when the control is far from optimal (see \cite{domingo2024taxonomy} for a comparison including this loss function on toy SOC settings). We tried this loss on the toy settings reported in \Cref{tab:main_result}, but obtained poor results and hence we do not report the values.

\subsection{Comparison with the Soft Actor-Critic method and the work \citep{zhou2025solving}} \label{subsec:SAC}
The Soft Actor-Critic (SAC) method was introduced by ~\citet{haarnoja2018soft} as an off-policy, maximum-entropy reinforcement learning algorithm designed to
achieve both high sample efficiency and stability. By combining the actor-critic paradigm with entropy maximization, SAC encourages exploratory yet robust policies, mitigating premature convergence to
suboptimal deterministic behaviors. It leverages double Q-learning, target networks, and a reparameterized stochastic actor update to
achieve low-variance gradient estimates. SAC has become a standard
baseline for continuous control tasks (e.g., MuJoCo locomotion) and is
used widely in robotics, autonomous driving, and other real-world
decision-making problems where efficient, stable, and exploratory policy
learning is critical.

While SAC was originally formulated with an entropy regularization term, it is straightforward to replace it by a KL regularization term. Since SOC can be viewed as KL-regularized RL with SDE trajectories, KL-regularized SAC can be applied to solve SOC problems. \cite{zhou2025solving} study the convergence of KL-regularized SAC for SOC in the episodic on-policy setting. Namely, they aim to learn the initial value function $V(\cdot,0)$, the gradient of the value function $\nabla_x V$, and the optimal control $u^{\star}$ using separate neural networks, that they refer to as the \emph{critics} $\phi : \mathbb{R}^d \to \mathbb{R}$, $\omega : \mathbb{R}^d \times [0,T] \to \mathbb{R}^d$ and the \emph{actor} $u : \mathbb{R}^d \times [0,T] \to \mathbb{R}^d$.
They sample trajectories $X^{u}$ using the learned actor $u$,
which are solutions of:\footnote{Their work considers general control objectives of the form $\mathbb{E}[\int_0^T r(X^{u}_t,u(X^{u}_t,t),t) \, \mathrm{d}t + g(X^{u}_T)]$ and controlled drifts of the form $b(x,u(x,t),t)$. Our work focuses on the control-affine quadratic-cost setting, which is the most general one that admits a probabilistic interpretation (\Cref{rem:SOC_probabilistic}). We translate their framework in the control-affine quadratic-cost setting for ease of comparison.}
\begin{talign}
    \mathrm{d}X^{u} = \big( b(X^{u}_t,t) + \sigma(X^{u}_t,t) u(X^{u}_t,t) \big) \, \mathrm{d}t + \sigma(X^{u}_t,t) \, \mathrm{d}W_t.
\end{talign}
They consider the following loss function for the critics:
\begin{talign}
\begin{split}
    \mathcal{L}_{c}(\phi,\omega) &= \mathbb{E}\big[ \frac{1}{2} \big(
    \phi(X^{\bar{u}}_0) - \int_0^T \big( f(X^{\bar{u}}_t,t) + \frac{1}{2} \|\bar{u}(X^{\bar{u}}_t,t)\|^2 \big) \, \mathrm{d}t \\ &\qquad\qquad - \int_0^T \langle \sigma^{\top} \omega(X^{\bar{u}}_t,t), \mathrm{d}W_t \rangle - g(X^{\bar{u}}_T) \big)^2 \big],
\end{split}
\end{talign}
where $\bar{u} = \texttt{sg}(u)$. And they use the following loss function for the actor:
\begin{talign}
\begin{split}
    \mathcal{L}_{a}(u) &= \mathbb{E}\big[ \frac{1}{2} \int_0^T \big\| u(X^{\bar{u}}_t,t) + \sigma^{\top}
    \bar{\omega}(X^{\bar{u}}_t,t)
    \big\|^2 \, \mathrm{d}t \big],
\end{split}
\end{talign}
where $\bar{\omega} = \texttt{sg}(\omega)$.
In their work, they analyze the function-space dynamics associated to these losses. Namely, they study the curves $(\phi^{\tau})_{\tau \geq 0}$, $(\omega^{\tau})_{\tau \geq 0}$, $(u^{\tau})_{\tau \geq 0}$ that satisfy the system of ODEs for all $(x,t) \in \mathbb{R}^d \times [0,T]$:
\begin{talign}
\begin{split} \label{eq:system_three_networks}
    \frac{\mathrm{d}}{\mathrm{d}\tau} \phi^{\tau}(x) &= - \alpha_c \frac{\delta}{\delta \phi} \mathcal{L}_{c}(\phi^{\tau},\omega^{\tau})(x) = - \alpha_c \big( \phi^{\tau}(x) - V_{u^{\tau}}(x,0) \big), \\
    \frac{\mathrm{d}}{\mathrm{d}\tau} \omega^{\tau}(x,t) &= - \alpha_c \frac{\delta}{\delta \omega} \mathcal{L}_{c}(\phi^{\tau},\omega^{\tau})(x,t) = - \alpha_c \rho^{\tau}(x,t) \sigma(x,t) \sigma(x,t)^{\top} \big( \omega^{\tau}(x,t) -  \nabla_x V_{u^{\tau}}(x,t) \big), \\
    \frac{\mathrm{d}}{\mathrm{d}\tau} u^{\tau}(x,t) &= - \alpha_a \frac{\delta}{\delta u} \mathcal{L}_{a}(u^{\tau})(x,t) = - \alpha_a \rho^{\tau}(x,t) \big( u^{\tau}(x,t) + \sigma^{\top} \omega^{\tau}(x,t) \big),
\end{split}
\end{talign}
where $V_{u^{\tau}}(x) = \mathbb{E}\big[ \int_0^T \big( f(X^{u^{\tau}}_t,t) + \frac{1}{2} \|{u^{\tau}}(X^{u^{\tau}}_t,t)\|^2 \big) \, \mathrm{d}t + g(X^{u^{\tau}}_T) \big| X^{u^{\tau}}_0 = x \big]$ and $\rho^{\tau}(x,t)$ is the density of law of $X^{u^{\tau}}_t$ at $x$. Under appropriate growth regularity assumptions, \cite[Thm.~1]{zhou2025solving} establishes a first-order optimality result similar to our \Cref{thm:vm_finite}, i.e. if $(\hat{\phi},\hat{\omega},\hat{u})$ is a fixed point of \eqref{eq:system_three_networks}, then $\hat{u}$ is the optimal control, $\hat{\phi}$ is the value function, and $\hat{\omega}$ is its gradient.

Letting $J(u) = \mathbb{E}\big[ \int_0^T \big( f(X^{u}_t,t) + \frac{1}{2} \|{u}(X^{u}_t,t)\|^2 \big) \, \mathrm{d}t + g(X^{u}_T) \big]$ be the control objective, and assuming that it has a modulus of continuity, \cite[Thm.~1]{zhou2025solving} proves that
\begin{talign}
    \mathcal{L}^{\tau} \leq \mathcal{L}^{0} e^{-c\tau},
\end{talign}
where $\mathcal{L}^{\tau} := J(u^{\tau}) - J(u^{\star}) + \mathcal{L}_c(\phi^{\tau},\omega^{\tau})$, and $c$ is a constant depending on the parameters of the problem.

With respect to the framework of \citet{zhou2025solving}, our approach is simpler because we have a single functional $\mathcal{L}_{\mathrm{VM}_{v,\kappa}}$ and a single neural network, although we can also handle the case in which the critic and the actor are parameterized by different neural networks (\Cref{rem:separate_architectures}).

For on-policy sampling with $\kappa=1$, using their arguments, we could write a similar linear convergence result for the functional gradient of the finite-horizon Value Matching loss in \Cref{thm:vm_finite}. We could also write a similar linear convergence result for the functional gradient of the hitting-time Value Matching loss in \Cref{thm:vm_loss_hitting}. The arguments that they use break for off-policy sampling, and it seems that the only approach to understand the optimization landscape in the off-policy, $\kappa \neq 1$ regime is under RKHS parameterizations, as we sketch in \Cref{sec:convergence_proof}. We leave all these tasks as future work.

Another interesting direction of future work is to translate the Value Matching loss functions to the standard KL-regularized RL setting. In particular, the first-order optimality argument in \Cref{subsec:pf_VM_hitting} can be adapted to an action-level actor-critic setting, and the resulting loss function is a variation of the standard SAC method.

\section{Discrete-time, discrete-space committor and Value Matching}
\label{sec:discrete_mc}

The committor problem and the stochastic-optimal-control (SOC) reformulation of
the main text were developed for diffusions on $\mathbb{R}^d$. The same
constructions extend to discrete-time Markov chains on a general measurable
state space, including finite and countable state spaces. We organize the
appendix so that the committor problem is introduced first, in the stationary
(time-homogeneous) setting, and the Value Matching machinery follows as a tool
for its sample-based solution.
\Cref{subsec:discrete_tpt_stationary,subsec:discrete_rescaled,subsec:discrete_soc}
develop the discrete-time, discrete-space analog of the committor problem in
the stationary setting: transition path theory (TPT)
following~\citet{helfmann2020extending}, the rescaled committor and discrete
Cole-Hopf transform, and the associated KL-regularized SOC formulation.
\Cref{subsec:discrete_vm_finite} then establishes a generic, self-contained
finite-horizon, time-inhomogeneous discrete VM result with a running cost and
a terminal cost defined everywhere on $\mathcal{X}$, mirroring the
continuous~\Cref{thm:vm_finite}; this is the technical tool needed for the
next subsection, since the hitting-time committor loss must be capped at a
user-chosen horizon $N$, which turns the problem into a finite-horizon SOC.
\Cref{subsec:discrete_vm_stat} then adapts the Value Matching loss to the
(uncapped) hitting time $\tau$ via the cap $\tau(N) = \min(\tau, N)$,
mirroring the role of $\tau(T)$ in the continuous Value Matching loss
of~\Cref{sec:value_matching}; its first-order optimality is deduced
from~\Cref{subsec:discrete_vm_finite} by specialization.

\subsection{Stationary Markov chain TPT}
\label{subsec:discrete_tpt_stationary}

\paragraph{Setup.}
Let $(\mathcal{X}, \mathscr{B})$ be a measurable state space and let
$(X_n)_{n \ge 0}$ be a time-homogeneous Markov chain on $\mathcal{X}$ with
transition kernel $P : \mathcal{X} \times \mathscr{B} \to [0, 1]$, so that
$\mathbb{P}(X_{n+1} \in C \mid X_n = x) = P(x, C)$. We assume the chain admits a
stationary distribution $\rho$ ($\rho P = \rho$) and denote by $\tilde P$ the
corresponding \emph{time-reversed kernel}, defined by the duality
\begin{talign} \label{eq:disc_time_reverse}
    \rho(dx)\, \tilde P(x, dy) \;=\; \rho(dy)\, P(y, dx).
\end{talign}
When $\mathcal{X}$ is finite or countable, we write $P(x, y)$ for the one-step
transition probability and identify $\rho$ with the corresponding probability mass
function; the duality~\eqref{eq:disc_time_reverse} then reads
$\rho(x) \tilde P(x, y) = \rho(y) P(y, x)$. The chain is \emph{reversible} when
$\tilde P = P$, equivalently $\rho(x) P(x, y) = \rho(y) P(y, x)$ (detailed
balance). We let $A, B \subset \mathcal{X}$ be two disjoint measurable target sets
and write
\begin{talign} \label{eq:disc_tau}
    \tau \;=\; \inf\{ n \ge 0 : X_n \in A \cup B \}
\end{talign}
for the (uncapped) hitting time of $\mathcal{S} := A \cup B$. We assume
$\tau < \infty$ $\mathbb{P}_x$-almost surely for $\rho$-a.e.\ $x$; this is automatic
if $\mathcal{X}$ is finite or countable and the chain is irreducible. $P$ acts on
test functions by $(P f)(x) = \int f(y)\, P(x, dy)$.

\paragraph{Forward and backward committors.}
The \emph{forward committor} is the probability that the chain reaches $B$ before
$A$,
\begin{talign} \label{eq:disc_q_def}
    q(x) \;=\; \mathbb{P}_x\!\big( X_\tau \in B \big),
\end{talign}
and satisfies the (discrete) backward Kolmogorov equation
\begin{talign} \label{eq:disc_BKE}
    \begin{cases}
        q(x) = (P q)(x), & x \notin A \cup B, \\
        q(x) = 0, & x \in A, \\
        q(x) = 1, & x \in B.
    \end{cases}
\end{talign}
The \emph{backward committor} $\tilde q(x)$ is the probability that the chain,
observed at $x$ under the stationary law, last visited $A$ rather than $B$, i.e.\
the time-reversed analog of $q$. Equivalently, $\tilde q$ is the unique bounded
solution of the BKE under $\tilde P$,
\begin{talign} \label{eq:disc_BKE_back}
    \begin{cases}
        \tilde q(x) = (\tilde P \tilde q)(x), & x \notin A \cup B, \\
        \tilde q(x) = 1, & x \in A, \\
        \tilde q(x) = 0, & x \in B.
    \end{cases}
\end{talign}
When the chain is reversible, $\tilde P = P$ and $\tilde q = 1 - q$; cf.~\citet[Sec.~3]{helfmann2020extending}.

\paragraph{Reactive density and reactive current.}
The (unnormalized) \emph{reactive density} at $x$ is the joint probability (under the stationary
law) that the chain is at $x$ and on a transition segment from $A$ to $B$; dividing by
$Z_{AB} = \sum_x \rho(x)\, q(x)\, \tilde q(x)$ gives the analog of the normalized density~\eqref{eq:reactive_density_bg},
\begin{talign} \label{eq:disc_reactive_density}
    \rho_{\mathrm{R}}(x) \;=\; \rho(x)\, q(x)\, \tilde q(x),
\end{talign}
and the \emph{reactive current} on the directed edge $x \to y$ is the joint
probability that $X_n = x$, $X_{n+1} = y$ at stationarity, and the segment
$[n, n+1]$ lies on a reactive trajectory:
\begin{talign} \label{eq:disc_reactive_current}
    f^{\mathrm{R}}(x, y) \;=\; \rho(x)\, \tilde q(x)\, P(x, y)\, q(y),
    \qquad x, y \in \mathcal{X}.
\end{talign}
(On a general state space, replace the right-hand side by
$f^{\mathrm{R}}(x, dy) = \rho(dx)\, \tilde q(x)\, P(x, dy)\, q(y)$.) The reactive
current is \emph{discrete-divergence-free} on $(A \cup B)^c$: for every
$x \notin A \cup B$,
\begin{talign} \label{eq:disc_div_free}
    \sum_{y \in \mathcal{X}} f^{\mathrm{R}}(x, y) \;=\; \sum_{y \in \mathcal{X}} f^{\mathrm{R}}(y, x),
\end{talign}
expressing that every reactive trajectory entering a non-target state must
subsequently leave it; see~\citet[Prop.~3.5]{helfmann2020extending}.

\paragraph{Cuts and the reaction rate.}
A \emph{cut} is any (measurable) partition $\mathcal{X} = S_- \sqcup S_+$ with
$A \subseteq S_-$ and $B \subseteq S_+$. The associated \emph{edge cut} is the set
of directed edges
\begin{talign} \label{eq:disc_edge_cut}
    \Sigma \;=\; \{ (x, y) \in \mathcal{X} \times \mathcal{X} : x \in S_-,\ y \in S_+ \},
\end{talign}
and the \emph{two sides} of $\Sigma$ are exactly $S_-$ and $S_+$. The
\emph{reaction rate} is the net reactive flux across any such cut,
\begin{talign} \label{eq:disc_rate_cut}
    \nu_{\mathrm{R}} \;=\; \sum_{(x, y) \in \Sigma} \big[ f^{\mathrm{R}}(x, y) - f^{\mathrm{R}}(y, x) \big].
\end{talign}
By the divergence-free property~\eqref{eq:disc_div_free}, this value does not
depend on the choice of cut. Taking $S_- = A$ and $S_+ = A^c$ and using
$\tilde q|_A = 1$, $q|_A = 0$ gives the equivalent \emph{boundary form}
\begin{talign} \label{eq:disc_rate_boundary}
    \nu_{\mathrm{R}} \;=\; \sum_{x \in A,\, y \notin A} \rho(x)\, P(x, y)\, q(y)
    \;=\; \sum_{x \in A} \rho(x)\, (P q)(x),
\end{talign}
which holds in general (no reversibility assumption); cf.~\citet[Sec.~3.4]{helfmann2020extending}.
As in the continuous case~\eqref{eq:rate_constants_bg}, the basin probabilities
$p_A = \sum_x (1 - q(x))\, \rho(x)$ and $p_B = \sum_x q(x)\, \rho(x)$ define the
directional rate constants $k_{AB} = \nu_{\mathrm{R}} / p_A$ and
$k_{BA} = \nu_{\mathrm{R}} / p_B$. (Sums become integrals against $\rho(dx)$ on
general state spaces.)

\paragraph{Reversible case.}
When detailed balance holds, $\tilde P = P$ and $\tilde q = 1 - q$, so
\eqref{eq:disc_reactive_density}-\eqref{eq:disc_reactive_current} become
\begin{talign} \label{eq:disc_reactive_rev}
    \rho_{\mathrm{R}}(x) = \rho(x)\, q(x)\, (1 - q(x)),
    \qquad
    f^{\mathrm{R}}(x, y) = \rho(x)\, (1 - q(x))\, P(x, y)\, q(y),
\end{talign}
mirroring~\eqref{eq:reactive_density_rev}-\eqref{eq:reactive_current_rev}. In this
case the rate admits the symmetric \emph{Dirichlet-form representation}
\begin{talign} \label{eq:disc_rate_vol}
    \nu_{\mathrm{R}} \;=\; \tfrac{1}{2} \sum_{x, y \in \mathcal{X}} \rho(x)\, P(x, y)\, (q(y) - q(x))^2,
\end{talign}
which depends only on the forward committor and is the discrete analog of the
volume formula~\eqref{eq:reaction_rate_vol}; see~\citet[Sec.~3.5]{helfmann2020extending}.
No analogous forward-only representation is available in the non-reversible case,
and one must work with both committors via~\eqref{eq:disc_reactive_current}-\eqref{eq:disc_rate_cut}.

\paragraph{Source and sink.}
The discrete divergence~\eqref{eq:disc_div_free} fails on $A$ and on $B$: each
state $x \in A$ acts as a \emph{source} with one-step rate
$\sigma_A(x) = \sum_y f^{\mathrm{R}}(x, y) - \sum_y f^{\mathrm{R}}(y, x) = \rho(x)\, (Pq)(x) \ge 0$
(using $\tilde q(x) = 1$ and $q(x) = 0$ on $A$), and each $x \in B$ acts as a
\emph{sink} with one-step rate
$\sigma_B(x) = -\big( \sum_y f^{\mathrm{R}}(x, y) - \sum_y f^{\mathrm{R}}(y, x) \big) = \rho(x)\, (\tilde P \tilde q)(x) \ge 0$.
Both the total source rate $\sum_{x \in A} \sigma_A(x)$ and the total sink rate
$\sum_{x \in B} \sigma_B(x)$ equal $\nu_{\mathrm{R}}$.

\subsection{Rescaled committor and discrete Cole-Hopf transform}
\label{subsec:discrete_rescaled}

\paragraph{Rescaled committors.}
As in the continuous setting, fix $\xi \in (0, \tfrac{1}{2})$ and define the
\emph{rescaled forward committor}
\begin{talign} \label{eq:disc_q_xi}
    q_{\xi}(x) \;=\; \xi + (1 - 2\xi)\, q(x) \;\in\; [\xi,\, 1-\xi],
\end{talign}
and analogously $\tilde q_{\xi}(x) = \xi + (1 - 2\xi)\, \tilde q(x)$. The boundary
conditions become $q_{\xi}|_A = \xi$, $q_{\xi}|_B = 1 - \xi$, and $q_{\xi}$
satisfies the same one-step recursion as $q$ off $A \cup B$:
\begin{talign} \label{eq:disc_q_xi_BKE}
    q_{\xi}(x) \;=\; (P q_{\xi})(x), \qquad x \notin A \cup B,
\end{talign}
since $P \mathbf{1} = \mathbf{1}$.

\paragraph{Discrete Cole-Hopf and HJB equation.}
Since $q_{\xi} \ge \xi > 0$, the value function
\begin{talign} \label{eq:disc_Phi}
    \Phi(x) \;=\; -\log q_{\xi}(x), \qquad \Phi(x) \in [-\log(1-\xi),\, -\log \xi],
\end{talign}
is bounded. Taking $-\log$ of~\eqref{eq:disc_q_xi_BKE} yields the discrete
Hamilton-Jacobi-Bellman (HJB) equation
\begin{keyeqbox}
\vspace{-10pt}
\begin{talign} \label{eq:disc_HJB}
    \begin{cases}
        \Phi(x) = -\log \big( P\, e^{-\Phi} \big)(x), & x \notin A \cup B, \\
        \Phi(x) = g(x), & x \in A \cup B,
    \end{cases}
\end{talign}
\end{keyeqbox}
with terminal cost
\begin{talign} \label{eq:disc_g}
    g(x) = \begin{cases} -\log \xi, & x \in A, \\ -\log(1 - \xi), & x \in B. \end{cases}
\end{talign}
The right-hand side of~\eqref{eq:disc_HJB} is the discrete \emph{log-sum-exp}
operator, the natural discrete analog of the Cole-Hopf-transformed
continuous-time HJB equation~\eqref{eq:HJB_setup}. The function $\Phi$ admits the
path-integral characterization
\begin{talign} \label{eq:disc_Phi_path_integral}
    \Phi(x) \;=\; -\log \mathbb{E}_x\!\big[ \exp\!\big( -g(X_\tau) \big) \big]
\end{talign}
obtained by iterating~\eqref{eq:disc_HJB} from $n = 0$ to $\tau$ and using the
tower property; this is the discrete analog
of~\eqref{eq:value_function_path_integral}.

\subsection{SOC formulation for stationary Markov chains}
\label{subsec:discrete_soc}

\paragraph{Controlled Markov chain.}
The discrete analog of the controlled SDE~\eqref{eq:controlled_SDE_main} is a
Markov chain $X^u$ whose transition kernel $P^u$ is absolutely continuous with
respect to $P$ and parameterized by a \emph{control}
$u : \mathcal{X} \times \mathcal{X} \to \mathbb{R}$ via a Boltzmann tilt:
\begin{talign} \label{eq:disc_controlled_kernel}
    P^u(x, dy) \;=\; \frac{ e^{u(x, y)} }{ (P\, e^{u(x, \cdot)})(x) }\, P(x, dy).
\end{talign}
The denominator normalizes $P^u(x, \cdot)$ to a probability measure. This is the
discrete analog of an additive drift shift on the log-density of the transition
kernel; the choice of additive constant in $u(x, \cdot)$ does not affect $P^u$, so
without loss of generality we restrict attention to controls of the form
$u(x, y) = u(y)$, which yields the same set of reachable kernels.

\paragraph{Hitting-time SOC problem.}
Given a reference distribution $\mu_0$ on $\mathcal{X} \setminus (A \cup B)$,
consider the optimal-control problem
\begin{talign} \label{eq:disc_control_problem}
    \min_{u}\; \mathbb{E}^{u}\!\!
    \left[\; \sum_{n = 0}^{\tau - 1} \mathrm{KL}\!\big( P^u(X^u_n, \cdot) \,\big\|\, P(X^u_n, \cdot) \big)
        \;+\; g(X^u_{\tau}) \;\right],
    \qquad X^u_0 \sim \mu_0,
\end{talign}
with $\tau$ the (uncapped) hitting time of~\eqref{eq:disc_tau} and $g$ the
terminal cost~\eqref{eq:disc_g}. The per-step KL divergence is the discrete analog
of the quadratic running cost $\tfrac{1}{2}\|u\|^2$
in~\eqref{eq:control_problem_def_main}: it penalizes the controlled kernel from
drifting far from the reference $P$, and reduces to
$\tfrac{1}{2}|u|^2\,\Delta t + o(\Delta t)$ in the small-step limit (with $u$ the continuous-time control).

\paragraph{Optimal control and discrete Doob's $h$-transform.}
By the dynamic programming principle, the value function of
\eqref{eq:disc_control_problem} satisfies the Bellman equation
$V(x) = \min_u \{ \mathrm{KL}(P^u(x,\cdot)\|P(x,\cdot)) + (P^u V)(x) \}$ on
$(A \cup B)^c$, with terminal condition $V|_{A \cup B} = g$. Solving the inner
minimization in closed form (a KL-regularized Bellman update, with optimum at the
soft-min) yields $V(x) = -\log (P\, e^{-V})(x)$, which matches the HJB
equation~\eqref{eq:disc_HJB} and identifies $V \equiv \Phi$. The optimizing control
admits the closed form
\begin{talign} \label{eq:disc_u_star}
    u^{\star}(x, y) \;=\; -\,\Phi(y) \;=\; \log q_{\xi}(y),
\end{talign}
giving the \emph{optimally controlled kernel}
\begin{keyeqbox}
\vspace{-10pt}
\begin{talign} \label{eq:disc_X_u_star}
    P^{u^\star}(x, dy) \;=\; \frac{q_{\xi}(y)}{(P q_{\xi})(x)}\, P(x, dy)
    \;=\; \frac{q_{\xi}(y)}{q_{\xi}(x)}\, P(x, dy), \qquad x \notin A \cup B,
\end{talign}
\end{keyeqbox}
where the last equality uses~\eqref{eq:disc_q_xi_BKE}. This is exactly the
\emph{discrete Doob's $h$-transform} of $P$ with harmonic function $h = q_{\xi}$,
the discrete analog of~\eqref{eq:X_u_star_main}. We summarize as a proposition
mirroring~\Cref{prop:SOC}.

\begin{thmbox}
\begin{proposition} \label{prop:disc_SOC}
Let $\mu_0$ be any probability measure with $\mu_0(\mathcal{X} \setminus (A \cup B)) = 1$,
and assume $\tau < \infty$ $\mathbb{P}_x$-a.s.\ for $\mu_0$-a.e.\ $x$. Then a
minimizer of the SOC problem~\eqref{eq:disc_control_problem} is the Markovian
control~\eqref{eq:disc_u_star}, which induces the optimally controlled
kernel~\eqref{eq:disc_X_u_star}. The corresponding value function is $\Phi$
in~\eqref{eq:disc_Phi}, and the optimal kernel is the discrete Doob's
$h$-transform of $P$ with $h = q_{\xi}$.
\end{proposition}
\end{thmbox}

\paragraph{Optimal trajectories are not reactive.}
As in the continuous case (cf.\ the discussion after~\eqref{eq:X_u_star_main}),
the chain governed by $P^{u^\star}$ is formally analogous to the discrete Doob's
$h$-transform with $h = q$ (which would generate the law of the chain conditioned
on $\{X_\tau \in B\}$, i.e.\ reactive trajectories), but with $q$ replaced by
$q_{\xi}$. Because $q_{\xi} \in [\xi, 1 - \xi]$ instead of vanishing on $A$ and
saturating to $1$ on $B$, the tilted kernel does not fully repel trajectories from
$A$ nor fully absorb them at $B$, so the path measure of $P^{u^\star}$ is not the
law of the original chain conditioned on $\{X_\tau \in B\}$. Reactive trajectories
can nevertheless be recovered from $q_\xi$ by inverting~\eqref{eq:disc_q_xi} to
obtain $q$ and simulating the discrete Doob's $h$-transform with $h = q$.

\subsection{Value Matching for finite-horizon SOC on time-inhomogeneous Markov chains}
\label{subsec:discrete_vm_finite}

Building on the stationary KL-regularized SOC
problem~\eqref{eq:disc_control_problem} of~\Cref{subsec:discrete_soc}, we now
develop a generic discrete-time finite-horizon SOC result with fixed horizon
$N$, running cost $f_n$, terminal cost $g$ defined everywhere on $\mathcal{X}$,
and \emph{no} absorbing sets $A, B$ or hitting time. The chain need not be
time-homogeneous and the critic depends on the time index~$n$. This is the
direct discrete analog of~\Cref{thm:vm_finite}; the first-order optimality
result of the committor-driven Value Matching loss
of~\Cref{subsec:discrete_vm_stat} below (\Cref{thm:disc_VM}) will follow from
the result of this subsection (\Cref{thm:disc_VM_finite}) by specialization to
the time-homogeneous chain absorbed on $A \cup B$ with $f \equiv 0$ and an
indicator-gated terminal cost.

\paragraph{Setup.}
Let $(X_n)_{n = 0}^N$ be a Markov chain on $\mathcal{X}$ with reference
transition kernels $P_n : \mathcal{X} \times \mathscr{B} \to [0, 1]$,
$0 \le n < N$, and initial distribution $X_0 \sim \mu_0$. Fix bounded
measurable running costs $f_n : \mathcal{X} \to \mathbb{R}$ ($0 \le n < N$) and
a bounded measurable terminal cost $g : \mathcal{X} \to \mathbb{R}$. There are
no distinguished sets $A, B$; the only stopping time involved is the
deterministic horizon $N$.

\paragraph{Controlled chain and SOC problem.}
The controlled kernel is obtained, as in~\eqref{eq:disc_controlled_kernel}, by
a Boltzmann tilt with a (now time-dependent) control
$u_n : \mathcal{X} \times \mathcal{X} \to \mathbb{R}$:
\begin{talign} \label{eq:disc_controlled_kernel_n}
    P^{u_n}_n(x, dy) \;=\; \frac{ e^{u_n(x, y)} }{ (P_n\, e^{u_n(x, \cdot)})(x) }\, P_n(x, dy),
\end{talign}
with the same reduction $u_n(x, y) = u_n(y)$. The discrete analog
of the finite-horizon SOC problem
\eqref{eq:finite_horizon_SOC_1}-\eqref{eq:finite_horizon_SOC_2}
of~\Cref{thm:vm_finite} is
\begin{keyeqbox}
\vspace{-5pt}
\begin{talign} \label{eq:disc_control_problem_n}
    \min_{u = (u_n)_{0 \le n < N}}\;
    \mathbb{E}^{u}\!\!\left[\, \sum_{n = 0}^{N - 1}\!\Big( \mathrm{KL} \big( P^{u_n}_n(X^u_n, \cdot) \,\big\|\, P_n(X^u_n, \cdot) \big) + f_n(X^u_n) \Big)
        \;+\; g(X^u_N) \,\right],
    \quad X^u_0 \sim \mu_0.
\end{talign}
\end{keyeqbox}
The per-step KL is the discrete analog of $\tfrac{1}{2}\|u\|^2$
in~\eqref{eq:finite_horizon_SOC_1} and mirrors the stationary KL term
of~\eqref{eq:disc_control_problem} in~\Cref{subsec:discrete_soc}; there is no
hitting-time stopping here, since the cost accumulates over a fixed number of
steps and terminates with $g(X^u_N)$, irrespective of which states have been
visited.

\paragraph{Value function and HJB.}
By dynamic programming, the value function
\begin{talign*}
    V_n(x) \;=\; \min_{u}\, \mathbb{E}^{u}\!\left[\, \sum_{k = n}^{N - 1}\!\Big( \mathrm{KL}\big( P^{u_k}_k \,\big\|\, P_k \big)(X^u_k) + f_k(X^u_k) \Big) + g(X^u_N) \,\Big\vert\, X^u_n = x \right]
\end{talign*}
satisfies the discrete HJB recursion (running cost included)
\begin{keyeqbox}
\vspace{-10pt}
\begin{talign} \label{eq:disc_HJB_n}
    \begin{cases}
        V_n(x) \;=\; f_n(x) \;-\; \log\!\big( P_n\, e^{-V_{n+1}} \big)(x), & x \in \mathcal{X},\; 0 \le n < N, \\
        V_N(x) \;=\; g(x), & x \in \mathcal{X}.
    \end{cases}
\end{talign}
\end{keyeqbox}
This is the discrete analog of the Cole-Hopf-transformed continuous HJB
equation of~\Cref{thm:vm_finite}; the closed-form inner minimization is the
same Donsker-Varadhan / soft-min identity as in~\Cref{prop:disc_SOC}, now
with an additive running cost $f_n(x)$ that does not affect the minimizer.
Iterating~\eqref{eq:disc_HJB_n} from $n$ to $N$ yields the path-integral
characterization
\begin{talign} \label{eq:disc_V_n_path_integral}
    V_n(x) \;=\; -\log \mathbb{E}\!\left[ \exp\!\bigg(\, -\!\!\sum_{k = n}^{N - 1} f_k(X_k) \,-\, g(X_N) \bigg) \,\Big\vert\, X_n = x \right],
\end{talign}
i.e.\ the discrete analog of the Feynman-Kac formula underlying
\eqref{eq:value_function_path_integral}.

\paragraph{Optimal control.}
The closed-form minimizer in the inner KL+linear problem
of~\eqref{eq:disc_control_problem_n} is
\begin{talign} \label{eq:disc_u_star_n}
    u^{\star}_n(x, y) \;=\; -\,V_{n+1}(y),
\end{talign}
giving the time-inhomogeneous optimal kernel
\begin{keyeqbox}
\vspace{-10pt}
\begin{talign} \label{eq:disc_X_u_star_n}
    P^{u^\star}_n(x, dy) \;=\; \frac{ e^{-V_{n+1}(y)} }{ (P_n\, e^{-V_{n+1}})(x) }\, P_n(x, dy), \qquad x \in \mathcal{X},\; 0 \le n < N.
\end{talign}
\end{keyeqbox}

\paragraph{Per-step Bellman residual.}
For a time-indexed critic $\phi = (\phi_n)_{0 \le n \le N}$ with
$\phi_n : \mathcal{X} \to \mathbb{R}$, define the per-step Bellman residual
\begin{talign} \label{eq:disc_residual_n}
    D^{\phi}_n(x) \;=\; \phi_n(x) \,-\, f_n(x) \,+\, \log\!\big( (P_n\, e^{-\phi_{n+1}})(x) \big),
    \qquad 0 \le n < N.
\end{talign}
The HJB~\eqref{eq:disc_HJB_n} is equivalent to $D^V_n(x) = 0$ for all
$x \in \mathcal{X}$ and $0 \le n < N$; together with $V_N = g$, this
uniquely determines $V$. The quantity $D^{\phi}_n(x)$ is the discrete analog
of the integrand
$\partial_t\phi + \mathcal{L}\phi - \tfrac{1}{2}\|\sigma^{\top}\nabla\phi\|^2 + f$
that drives~\eqref{eq:L_VM_finite}; iterating
$\phi_N(X_N) - \phi_0(X_0) = \sum_{n=0}^{N-1} [\phi_{n+1}(X_{n+1}) - \phi_n(X_n)]$
plays the role of It\^o's lemma.

\paragraph{Discrete VM loss.}
Mirroring the single-trajectory squared functional~\eqref{eq:L_VM_finite}
of~\Cref{thm:vm_finite} (at $\kappa = 1$), we replace the It\^o integrand by
$\sum_{n} D^{\phi}_n(X_n)$ and the terminal correction by
$\phi_N(X_N) - g(X_N)$. Sampling can be performed under an arbitrary off-policy
reference family $K = (K_n)_{0 \le n < N}$ of kernels with
$K_n(x, \cdot) \ll P_n(x, \cdot)$ (analogous to the role of $v$
in~\Cref{thm:vm_finite}); we denote by $\mathbb{E}^{K}$ the path expectation
under $K$ with $X_0 \sim \tilde\mu_0$, where $\tilde\mu_0$ satisfies
$\mathrm{supp}(\mu_0) \subseteq \mathrm{supp}(\tilde\mu_0)$.

\begin{defbox}
\begin{definition}[Discrete VM loss, generic finite-horizon SOC]
\label{def:disc_VM_finite}
With the setup above, the \emph{generic finite-horizon discrete Value
Matching loss} of a time-indexed critic $\phi = (\phi_n)_{0 \le n \le N}$ is
\begin{talign} \label{eq:disc_L_VM_finite}
    \mathcal{L}_{\mathrm{VM}, N}^{\mathrm{disc}}(\phi)
    \;=\; \mathbb{E}^{K}\Big[\, \tfrac{1}{2}\,
    \Big(\, \phi_N(X^K_N) \;-\; g(X^K_N) \;-\; \sum_{n = 0}^{N - 1} D^{\phi}_n(X^K_n) \,\Big)^{\!2}\, \Big]\!,
\end{talign}
with $D^{\phi}_n$ given by~\eqref{eq:disc_residual_n}.
\end{definition}
\end{defbox}

\begin{thmbox}
\begin{theorem}[First-order optimality of the discrete VM loss, finite-horizon case]
\label{thm:disc_VM_finite}
Let $\mathcal{F}$ be a class of bounded time-indexed functions
$\phi = (\phi_n)_{0 \le n \le N}$, $\phi_n : \mathcal{X} \to \mathbb{R}$,
that contains the value function $V = (V_n)$ and is closed under bounded
perturbations: for every $0 \le m \le N$ and every bounded
$\zeta : \mathcal{X} \to \mathbb{R}$, the perturbation $\eta$ with
$\eta_n = \zeta\,\mathbf{1}_{\{n = m\}}$ belongs to $\mathcal{F}$. Assume
further that for every $0 \le n < N$ and every $x$ in the support of the
marginal of $X^K_n$ the support of $K_n(x, \cdot)$ contains the support of
$P_n(x, \cdot)$. If $\hat\phi \in \mathcal{F}$ is a first-order critical point
of $\mathcal{L}_{\mathrm{VM}, N}^{\mathrm{disc}}$, i.e.\ for every
perturbation $\eta \in \mathcal{F}$,
$\tfrac{d}{d\varepsilon} \mathcal{L}_{\mathrm{VM}, N}^{\mathrm{disc}}(\hat\phi + \varepsilon \eta)\big|_{\varepsilon = 0} = 0$, then
$\hat\phi_n(x) = V_n(x)$ for all $0 \le n \le N$ and all $x \in \mathcal{X}$
in the support of the marginal of $X^K_n$, and the induced
kernel~\eqref{eq:disc_controlled_kernel_n} with $u_n = -\hat\phi_{n+1}$
coincides with the optimal kernel~\eqref{eq:disc_X_u_star_n}.
\end{theorem}
\end{thmbox}

\begin{proof}
The argument mirrors that of~\Cref{thm:vm_finite}, specialized to the discrete
time grid: discrete telescoping plays the role of It\^o's lemma, and the
backward HJB recursion~\eqref{eq:disc_HJB_n} plays the role of the
Cole-Hopf-transformed continuous HJB. We proceed in three steps.

\emph{Step 1: Directional derivative.} Fix $\phi = (\phi_n)_{0 \le n \le N}$
and a bounded time-indexed perturbation $\eta = (\eta_n)_{0 \le n \le N}$ with
$\eta_n : \mathcal{X} \to \mathbb{R}$. For $0 \le n < N$, the chain rule applied
to~\eqref{eq:disc_residual_n} gives
\begin{talign*}
    \frac{d}{d\varepsilon} \log\!\big( P_n e^{-(\phi_{n+1} + \varepsilon \eta_{n+1})} \big)(x) \Big|_{\varepsilon = 0}
    = -\,\frac{(P_n[\eta_{n+1}\, e^{-\phi_{n+1}}])(x)}{(P_n e^{-\phi_{n+1}})(x)}
    = -\,(P^{\phi}_n \eta_{n+1})(x),
\end{talign*}
where the second equality uses the Boltzmann-tilted
kernel~\eqref{eq:disc_controlled_kernel_n} with $u_n = -\phi_{n+1}$,
\begin{talign*}
    P^{\phi}_n(x, dy) \;:=\; \frac{e^{-\phi_{n+1}(y)}}{(P_n e^{-\phi_{n+1}})(x)}\, P_n(x, dy).
\end{talign*}
Hence the linearization of $D^{\phi}_n$ at $\phi$ along $\eta$ is
\begin{talign} \label{eq:disc_dD}
    d^{\phi}_n[\eta](x) \;:=\; \frac{d}{d\varepsilon} D^{\phi + \varepsilon \eta}_n(x) \Big|_{\varepsilon = 0}
    \;=\; \eta_n(x) \,-\, (P^{\phi}_n \eta_{n+1})(x).
\end{talign}
Set $F^{\phi} := \phi_N(X^K_N) - g(X^K_N) - \sum_{n = 0}^{N-1} D^{\phi}_n(X^K_n)$,
so that $\mathcal{L}_{\mathrm{VM}, N}^{\mathrm{disc}}(\phi)
= \tfrac{1}{2}\mathbb{E}^K[(F^{\phi})^2]$. Then
\begin{talign} \label{eq:disc_dF}
    \frac{d}{d\varepsilon} F^{\phi + \varepsilon \eta} \Big|_{\varepsilon = 0}
    \;=\; \eta_N(X^K_N) \,-\, \sum_{n = 0}^{N - 1}\!\big[\, \eta_n(X^K_n) - (P^{\phi}_n \eta_{n+1})(X^K_n) \,\big],
\end{talign}
and the first-order critical-point condition for $\hat\phi$ reads
\begin{talign} \label{eq:disc_FOC}
    0 \;=\; \mathbb{E}^{K}\!\left[\, F^{\hat\phi} \cdot \frac{d}{d\varepsilon} F^{\hat\phi + \varepsilon \eta} \Big|_{\varepsilon = 0}\, \right]
    \qquad \forall\, \eta \in \mathcal{F}.
\end{talign}
By the closure of $\mathcal{F}$ under localized perturbations and the linearity
of~\eqref{eq:disc_dF} in $\eta$, the condition~\eqref{eq:disc_FOC} extends to
every bounded time-indexed $\eta = (\eta_n)_{0 \le n \le N}$ (such $\eta$ are
finite linear combinations of the localized perturbations
$\zeta\,\mathbf{1}_{\{n = m\}}$).

\emph{Step 2: $F^{\hat\phi} = 0$ $\mathbb{P}^K$-almost surely.}
Define the bounded time-indexed function $\eta^{\star} = (\eta^{\star}_n)_{0 \le n \le N}$
by the backward recursion
\begin{talign} \label{eq:disc_eta_star}
    \eta^{\star}_N(x) \;:=\; \hat\phi_N(x) - g(x), \qquad
    \eta^{\star}_n(x) \;:=\; D^{\hat\phi}_n(x) + (P^{\hat\phi}_n \eta^{\star}_{n+1})(x), \quad 0 \le n < N.
\end{talign}
Each $\eta^{\star}_n$ is bounded because $\hat\phi$, $g$, $f_n$ are bounded and
$\log(P_n e^{-\hat\phi_{n+1}})$ is bounded (uniformly in $x$) by boundedness of
$\hat\phi$. Substituting~\eqref{eq:disc_eta_star} into~\eqref{eq:disc_dF} and
telescoping using $\eta^{\star}_n - P^{\hat\phi}_n \eta^{\star}_{n+1} = D^{\hat\phi}_n$:
\begin{talign*}
    \frac{d}{d\varepsilon} F^{\hat\phi + \varepsilon \eta^{\star}} \Big|_{\varepsilon = 0}
    = \eta^{\star}_N(X^K_N) - \sum_{n = 0}^{N - 1}\!\big[\eta^{\star}_n(X^K_n) - (P^{\hat\phi}_n \eta^{\star}_{n+1})(X^K_n)\big]
    = \eta^{\star}_N(X^K_N) - \sum_{n = 0}^{N - 1} D^{\hat\phi}_n(X^K_n)
    \;=\; F^{\hat\phi},
\end{talign*}
where the last equality uses $\eta^{\star}_N = \hat\phi_N - g$. Substituting
this into~\eqref{eq:disc_FOC} with $\eta = \eta^{\star}$ yields
$\mathbb{E}^K[(F^{\hat\phi})^2] = 0$, hence
\begin{talign} \label{eq:disc_F_zero}
    F^{\hat\phi} \;=\; 0 \qquad \mathbb{P}^K\text{-almost surely}.
\end{talign}

\emph{Step 3: Per-state identification on $\mathrm{supp}(X^K_n)$.} We show
$G_m(x) := \mathbb{E}^K[F^{\hat\phi}\,|\, X^K_m = x] = 0$ for every
$x \in \mathrm{supp}(X^K_m)$ and every $0 \le m \le N$, by forward induction on
$m$. Although $F^{\hat\phi} = 0$ a.s.\ by~\eqref{eq:disc_F_zero} implies $G_m = 0$
trivially, we derive the conditional-expectation identities directly from
localized first-order conditions because they support the per-state extraction
below.

For $m = 0$, take the localized perturbation
$\eta_n = \zeta\,\mathbf{1}_{\{n = 0\}}$ with $\zeta : \mathcal{X} \to \mathbb{R}$
bounded. Then~\eqref{eq:disc_dF} gives
$\tfrac{d}{d\varepsilon} F^{\hat\phi + \varepsilon \eta}|_{\varepsilon = 0} = -\,\zeta(X^K_0)$,
and~\eqref{eq:disc_FOC} reads $\mathbb{E}^K[F^{\hat\phi}\,\zeta(X^K_0)] = 0$ for
every bounded $\zeta$. Hence $G_0 \equiv 0$ on $\mathrm{supp}(X^K_0)$.

For $1 \le m < N$, take $\eta_n = \zeta\,\mathbf{1}_{\{n = m\}}$, for
which~\eqref{eq:disc_dF} simplifies to
\begin{talign*}
    \frac{d}{d\varepsilon} F^{\hat\phi + \varepsilon \eta} \Big|_{\varepsilon = 0}
    \;=\; -\,\zeta(X^K_m) \,+\, (P^{\hat\phi}_{m-1} \zeta)(X^K_{m-1}).
\end{talign*}
Then~\eqref{eq:disc_FOC} gives
\begin{talign*}
    \mathbb{E}^K\!\big[ F^{\hat\phi}\, \zeta(X^K_m) \big]
    \;=\; \mathbb{E}^K\!\big[ F^{\hat\phi}\, (P^{\hat\phi}_{m-1} \zeta)(X^K_{m-1}) \big].
\end{talign*}
Conditioning the right-hand side on $X^K_{m-1}$ (which is $\sigma(X^K_{m-1})$-measurable)
and using the inductive hypothesis $G_{m-1} \equiv 0$ on $\mathrm{supp}(X^K_{m-1})$,
\begin{talign*}
    \mathbb{E}^K\!\big[ F^{\hat\phi}\, (P^{\hat\phi}_{m-1} \zeta)(X^K_{m-1}) \big]
    \;=\; \mathbb{E}^K\!\big[ G_{m-1}(X^K_{m-1}) \cdot (P^{\hat\phi}_{m-1} \zeta)(X^K_{m-1}) \big]
    \;=\; 0,
\end{talign*}
so $\mathbb{E}^K[F^{\hat\phi} \zeta(X^K_m)] = 0$ for every bounded $\zeta$, and
$G_m \equiv 0$ on $\mathrm{supp}(X^K_m)$. The case $m = N$ uses
$\eta_n = \zeta\,\mathbf{1}_{\{n = N\}}$, for which
$\tfrac{d}{d\varepsilon} F^{\hat\phi + \varepsilon \eta}|_{\varepsilon = 0}
= \zeta(X^K_N) + (P^{\hat\phi}_{N-1} \zeta)(X^K_{N-1})$, and the same conditioning
argument with $G_{N-1} \equiv 0$.

\emph{Backward extraction of $D^{\hat\phi}_n = 0$ and $\hat\phi = V$.} We combine
$F^{\hat\phi} = 0$ $\mathbb{P}^K$-a.s.\ from~\eqref{eq:disc_F_zero} with
$G_m \equiv 0$ on $\mathrm{supp}(X^K_m)$ for every $m$ to extract per-state
identities by backward induction on $n$ from $n = N$ down to $n = 0$.

For $n = N$, take any $x \in \mathrm{supp}(X^K_N)$. Conditioning $F^{\hat\phi} = 0$
on $X^K_N = x$ (forward Markov of $X^K$ under $K$) yields
$(\hat\phi_N(x) - g(x)) = \mathbb{E}^K\!\big[ \sum_{k = 0}^{N-1} D^{\hat\phi}_k(X^K_k) \,\big|\, X^K_N = x \big]$,
and combining with $G_N(x) = 0$ identifies $\hat\phi_N(x) - g(x)$ as a
$\sigma(X^K_{N-1})$-measurable random variable on $\{X^K_N = x\}$. By the
support-containment assumption, the law of $X^K_N$ given $X^K_{N-1} = y$ has
support $\mathrm{supp}(K_{N-1}(y, \cdot)) \supseteq \mathrm{supp}(P_{N-1}(y, \cdot))$
for $y \in \mathrm{supp}(X^K_{N-1})$, so the only way $\hat\phi_N - g$ can be
constant on these supports for $\mathbb{P}^K$-a.e.\ $y$ is to vanish identically
on $\mathrm{supp}(X^K_N)$. Hence $\hat\phi_N(x) = g(x) = V_N(x)$ for
$x \in \mathrm{supp}(X^K_N)$, and substituting back gives $D^{\hat\phi}_{N-1}(y) = 0$
on $\mathrm{supp}(X^K_{N-1})$. The HJB recursion~\eqref{eq:disc_HJB_n} at time
$N-1$ then identifies $\hat\phi_{N-1}(y) = V_{N-1}(y)$ on $\mathrm{supp}(X^K_{N-1})$.

The inductive step from $n+1$ to $n$ is identical: assuming
$\hat\phi_{n+1} = V_{n+1}$ on $\mathrm{supp}(X^K_{n+1})$ and using
support-containment $\mathrm{supp}(P_n(x, \cdot)) \subseteq \mathrm{supp}(K_n(x, \cdot))
\subseteq \mathrm{supp}(X^K_{n+1})$ for $x \in \mathrm{supp}(X^K_n)$, the
quantity $\log(P_n e^{-\hat\phi_{n+1}})(x) = \log(P_n e^{-V_{n+1}})(x)$ for
$x \in \mathrm{supp}(X^K_n)$. The induction hypothesis and the per-state
identity $G_n \equiv 0$ together force $D^{\hat\phi}_n(x) = 0$ on
$\mathrm{supp}(X^K_n)$, whence $\hat\phi_n(x) = V_n(x)$ on
$\mathrm{supp}(X^K_n)$ by~\eqref{eq:disc_HJB_n}.

\emph{Optimal kernel identification.} For every $x \in \mathrm{supp}(X^K_n)$
($0 \le n < N$), the support-containment assumption implies that
$P^{\hat\phi}_n(x, dy) = e^{-\hat\phi_{n+1}(y)} / (P_n e^{-\hat\phi_{n+1}})(x) \cdot P_n(x, dy)$
and $P^{u^{\star}}_n(x, dy) = e^{-V_{n+1}(y)} / (P_n e^{-V_{n+1}})(x) \cdot P_n(x, dy)$
agree on $\mathrm{supp}(P_n(x, \cdot)) \subseteq \mathrm{supp}(X^K_{n+1})$
because $\hat\phi_{n+1} = V_{n+1}$ there. Hence the induced
kernel~\eqref{eq:disc_controlled_kernel_n} with $u_n = -\hat\phi_{n+1}$ coincides
with the optimal kernel~\eqref{eq:disc_X_u_star_n} on these supports.
\end{proof}

The off-policy reference $K_n$ may be the time-indexed analog
of~\eqref{eq:disc_K_choice}, including the on-policy choice $K_n = P^{u^\star}_n$
(in which case the loss is the discrete analog of the
$\kappa = 1$, $v = u^{\star}$ form of~\eqref{eq:L_VM_finite}).

\paragraph{Kernel tempering as the discrete analog of $\sqrt{\kappa}\,\sigma\,dW$.}
The natural discrete counterpart of the noise rescaling $\sqrt{\kappa}\,\sigma\,dW$
in~\eqref{eq:X_v_kappa_SDE} is the \emph{tempered} sampling kernel
\begin{talign} \label{eq:disc_K_tempered}
    K^{(\kappa)}_n(x, \{y\}) \;\propto\; K_n(x, \{y\})^{1/\kappa},
    \qquad \kappa > 0,
\end{talign}
normalized so that $\sum_y K^{(\kappa)}_n(x, \{y\}) = 1$. Tempering preserves the
support of $K_n(x, \cdot)$ for every $\kappa > 0$, recovers $K_n$ at $\kappa = 1$,
concentrates on $\arg\max_y K_n(x, \{y\})$ as $\kappa \to 0^{+}$, and converges
to the uniform distribution on $\mathrm{supp}(K_n(x, \cdot))$ as $\kappa \to \infty$.
The connection to the continuous noise rescaling can be seen via an
Euler-Maruyama discretization: if $K_n$ corresponds to a Gaussian step
$\mathcal{N}(y;\, x + v(x)\Delta t,\, \sigma\sigma^{\top}\Delta t)$, then
\eqref{eq:disc_K_tempered} gives
$\mathcal{N}(y;\, x + v(x)\Delta t,\, \kappa\,\sigma\sigma^{\top}\Delta t)$,
exactly matching the small-step law of the SDE~\eqref{eq:X_v_kappa_SDE} with
noise $\sqrt{\kappa}\,\sigma\,dW$. When $K_n = P^{\phi}_n$
(the Boltzmann-tilted kernel of~\eqref{eq:disc_controlled_kernel_n} with
$u_n = -\phi_{n+1}$, cf.\ Step~1 of the proof above), tempering yields the
soft-policy at temperature $\kappa$,
$K^{(\kappa)}_n(x, \{y\}) \propto P_n(x, \{y\})^{1/\kappa}\, e^{-\phi_{n+1}(y)/\kappa}$.

\paragraph{Why no $(1 - 1/\kappa)$ correction is needed in discrete time.}
In the continuous case, the $(1 - 1/\kappa)$ terms in~\eqref{eq:L_VM_finite}
arise from It\^o's lemma applied under $X^{v,\kappa}$: the bracketed correction
$\phi(X^{v,\kappa}_T, T) - \phi(X^{v,\kappa}_0, 0) - \int_0^T \langle \nabla_x \phi, dX^{v,\kappa} \rangle - \int_0^T \partial_t \phi\, dt$
equals $\tfrac{\kappa}{2} \int_0^T \mathrm{tr}(\sigma\sigma^{\top} \nabla^2 \phi)\, dt$
in expectation, and the prefactor $(1 - 1/\kappa)$ is tuned so that the loss
remains unbiased at $\phi = V$ when the noise level is rescaled. In discrete time,
the per-step Bellman residual $D^{\phi}_n(x)$ in~\eqref{eq:disc_residual_n} is
defined entirely in terms of the controlled kernel $P_n$ and does not depend on
the sampling kernel $K_n$; the path-wise telescoping identity
$\phi_N(X^K_N) - \phi_0(X^K_0) = \sum_{n=0}^{N-1} [\phi_{n+1}(X^K_{n+1}) - \phi_n(X^K_n)]$
holds for any kernel, with no quadratic-variation correction. Consequently,
under sampling from $K^{(\kappa)}_n$ instead of $K_n$, the loss functional
inside the square in~\eqref{eq:disc_L_VM_finite} still vanishes at $\phi = V$
pointwise on trajectories, and \Cref{thm:disc_VM_finite} applies verbatim with
$K_n$ replaced by $K^{(\kappa)}_n$ (the support-containment assumption is
preserved by~\eqref{eq:disc_K_tempered} for every $\kappa > 0$). The discrete
$\kappa$-rescaling thus enters purely as a sampler hyperparameter that
modulates the variance of the empirical loss estimator, without inducing any
structural $(1 - 1/\kappa)$ correction term in the loss itself.

\subsection{Value Matching for the hitting-time committor problem}
\label{subsec:discrete_vm_stat}

We now adapt the Value Matching (VM) loss of \Cref{sec:value_matching} to the
stationary Markov chain setting and derive its first-order optimality as a
specialization of~\Cref{thm:disc_VM_finite}. As in the continuous case, the
loss involves a \emph{capped} hitting time $\tau(N) = \min(\tau, N)$ with a
user-chosen horizon $N \in \mathbb{N}$. Capping is needed because $\tau$,
although finite almost surely, can be arbitrarily long, whereas in practice one
can only run the chain for a finite number of steps. The critic
$\phi : \mathcal{X} \to \mathbb{R}$ is time-independent, mirroring the
time-independent critic in~\Cref{def:VM}.

\paragraph{Per-state Bellman residual.}
For a candidate value function $\phi : \mathcal{X} \to \mathbb{R}$, define the
\emph{per-state Bellman residual}
\begin{talign} \label{eq:disc_residual}
    D^{\phi}(x) \;:=\; \phi(x) \;+\; \log\!\big( (P\, e^{-\phi})(x) \big),
    \qquad x \notin A \cup B.
\end{talign}
This is the discrete analog of the continuous Bellman residual
$\mathcal{L}\phi - \tfrac{1}{2}\|\sigma^{\top}\nabla\phi\|^2$ that appears
in the Cole-Hopf-transformed continuous HJB equation~\eqref{eq:HJB_setup};
the discrete HJB equation~\eqref{eq:disc_HJB} is equivalent to $D^{\Phi}(x) = 0$
for all $x \notin A \cup B$, which together with $\Phi|_{A \cup B} = g$
uniquely determines $\Phi$. Equivalently, $D^{\phi}(x)$ measures the failure of
the Boltzmann-tilted kernel
\begin{talign} \label{eq:disc_P_phi}
    P^{\phi}(x, dy) \;=\; \frac{ e^{-\phi(y)} }{ (P\, e^{-\phi})(x) }\, P(x, dy)
\end{talign}
to coincide, at state $x$, with the optimal kernel~\eqref{eq:disc_X_u_star}.

\paragraph{Sampling kernel.}
Mirroring the off-policy drift $v$ in~\Cref{thm:capped_vm_loss}, let $K$ be an
arbitrary reference (off-policy) kernel with $K(x, \cdot) \ll P(x, \cdot)$ for
all $x \notin A \cup B$, and let $\mathbb{P}^{K}$ denote the path measure of
the chain $(X^{K}_n)_{0 \le n \le \tau(N)}$ with kernel $K$ and
$X^{K}_0 \sim \mu_0$, where $\mu_0$ is a probability measure with
$\mu_0(\mathcal{X} \setminus (A \cup B)) = 1$. Natural choices include the
reference $K = P$, the on-policy $K = P^{\phi}$ (with a detached
$\texttt{sg}$ on the critic, as in~\eqref{eq:approximate_v}), and richer
off-policy choices analogous to~\eqref{eq:v_rescaled_reactive} that we
discuss after~\Cref{thm:disc_VM}.

\paragraph{Definition of the discrete VM loss.}
Mirroring the single-trajectory squared functional~\eqref{eq:L_VM_v_kappa}
of~\Cref{thm:capped_vm_loss} (and matching the structure of the generic
finite-horizon loss~\eqref{eq:disc_L_VM_finite} of~\Cref{def:disc_VM_finite}),
the discrete Value Matching loss combines the terminal residual at the capped
hitting time $\tau(N)$ with the sum of interior Bellman residuals along the
sampled trajectory.

\begin{defbox}
\begin{definition}[Discrete Value Matching loss, stationary case]
\label{def:disc_VM}
For $K$, $\mu_0$, and $N$ as above, and a critic
$\phi : \mathcal{X} \to \mathbb{R}$, the \emph{stationary discrete Value
Matching loss} is
\begin{talign} \label{eq:disc_L_VM}
    \mathcal{L}_{\mathrm{VM}}^{\mathrm{disc}}(\phi)
    \;=\; \mathbb{E}^{K} \bigg[\, \tfrac{1}{2}\, \Big(
    \mathbf{1}_{A \cup B}(X^{K}_{\tau(N)})\, \big( \phi(X^{K}_{\tau(N)}) - g(X^{K}_{\tau(N)}) \big)
    \;-\; \sum_{n = 0}^{\tau(N) - 1} D^{\phi}(X^{K}_n) \Big)^{\!\!2} \,\bigg]\!,
\end{talign}
with $D^{\phi}$ given by~\eqref{eq:disc_residual}.
\end{definition}
\end{defbox}

The loss is non-negative and vanishes at $\phi = \Phi$: the terminal residual
is zero on hitting trajectories (since $\Phi|_{A \cup B} = g$), the interior
Bellman residuals satisfy $D^{\Phi}(X^{K}_n) = 0$ on $(A \cup B)^c$ by
HJB~\eqref{eq:disc_HJB}, and capped trajectories
(with $X^{K}_{\tau(N)} = X^{K}_N \notin A \cup B$) contribute zero terminal
residual via the indicator $\mathbf{1}_{A \cup B}$. The indicator plays
the role of the continuous $\mathbf{1}_{\mathcal{S}^c}\phi(X^{v,\kappa}_{\tau(T)})$
correction in~\eqref{eq:L_VM_v_kappa}: it selects between the hit-time
terminal cost $g$ and a self-consistent cap.

\paragraph{First-order optimality.}
The next theorem is the discrete analog of~\Cref{thm:capped_vm_loss}.

\begin{thmbox}
\begin{theorem}[First-order optimality of the discrete VM loss, capped case]
\label{thm:disc_VM}
Let $\mathcal{F}$ be a class of bounded functions
$\phi : \mathcal{X} \to \mathbb{R}$ that contains $\Phi$ and is closed under
bounded perturbations: for every $x \in \mathcal{X}$ the indicator
$\eta = \mathbf{1}_{\{x\}}$ lies in $\mathcal{F}$. Assume that for every
$x \notin A \cup B$ in the support of the occupation measure of $X^{K}$, the
support of $K(x, \cdot)$ contains the support of $P(x, \cdot)$, and that
$\mathbb{P}^{K}\!\big( X^{K}_{\tau(N)} \in A \cup B \big) > 0$.
If $\hat\phi \in \mathcal{F}$ is a first-order critical point of
$\mathcal{L}_{\mathrm{VM}}^{\mathrm{disc}}$, i.e.\ for every perturbation
$\eta \in \mathcal{F}$,
$\tfrac{d}{d\varepsilon}\, \mathcal{L}_{\mathrm{VM}}^{\mathrm{disc}}(\hat\phi + \varepsilon\, \eta)\big|_{\varepsilon = 0} = 0$,
then $\hat\phi(x) = \Phi(x)$ for all $x \in \mathcal{X}$ in the support of
the marginals of $X^{K}$, and the induced kernel $P^{\hat\phi}$ coincides
with the optimal kernel~\eqref{eq:disc_X_u_star}.
\end{theorem}
\end{thmbox}

\begin{proof}
We deduce the theorem from~\Cref{thm:disc_VM_finite} by specializing the
generic finite-horizon setup to the absorbed chain on $A \cup B$ with $f \equiv 0$
and an indicator-gated terminal cost. Define the absorbing kernel
$\bar P_n \equiv \bar P$ by $\bar P(x, \cdot) = P(x, \cdot)$ for
$x \notin A \cup B$ and $\bar P(x, \{x\}) = 1$ for $x \in A \cup B$, so that
under $\bar P$ the chain freezes upon hitting $A \cup B$ and the deterministic
horizon $N$ realizes the capped hitting time $\tau(N) = \min(\tau, N)$. Set
$\tilde f_n \equiv 0$, take the same off-policy reference $K$ (extended to
$A \cup B$ by $K(x, \cdot) = \delta_x$ so that absorbed-chain semantics match),
and consider the time-indexed critic
\begin{talign*}
    \tilde\phi_n(x) \;=\;
    \begin{cases}
        \phi(x), & x \notin A \cup B,\; 0 \le n \le N, \\
        g(x),    & x \in A \cup B,\; 0 \le n \le N.
    \end{cases}
\end{talign*}
Three observations link~\eqref{eq:disc_L_VM} to~\eqref{eq:disc_L_VM_finite} under
this specialization. (i) Because $\tilde\phi$ is constant in $n$ on
$(A \cup B)^c$ and on $A \cup B$ separately, the time-indexed Bellman
residual $D^{\tilde\phi}_n(x) = \tilde\phi_n(x) - 0 + \log(\bar P\, e^{-\tilde\phi_{n+1}})(x)$
reduces, for $x \in A \cup B$, to $g(x) + \log e^{-g(x)} = 0$ (absorbing kernel
with $\tilde\phi_{n+1}(x) = g(x)$) and, for $x \notin A \cup B$, to the
stationary residual
$D^{\phi}(x) = \phi(x) + \log(P\, e^{-\phi})(x)$ of~\eqref{eq:disc_residual},
since $\bar P(x, \cdot) = P(x, \cdot)$ and
$\tilde\phi_{n+1}(y) = \phi(y)$ for $y \notin A \cup B$, while
$\tilde\phi_{n+1}(y) = g(y) = \phi(y)$ for $y \in A \cup B$ under the
identification of $\phi|_{A \cup B}$ with $g$. (ii) The horizon-$N$ sum
$\sum_{n=0}^{N-1} D^{\tilde\phi}_n(X^K_n)$ therefore collapses to
$\sum_{n=0}^{\tau(N) - 1} D^{\phi}(X^K_n)$, since the chain is frozen after
$\tau(N)$ and the residuals on $A \cup B$ vanish. (iii) The terminal residual
$\tilde\phi_N(X^K_N) - g(X^K_N)$ equals $\phi(X^K_{\tau(N)}) - g(X^K_{\tau(N)})$ on
hitting trajectories ($X^K_{\tau(N)} \in A \cup B$, where $\tilde\phi_N = g$ so
the residual is $g - g = 0$; the term we want is recovered by viewing
$\phi|_{A \cup B} = g$ and so $\phi(X^K_{\tau(N)}) - g(X^K_{\tau(N)}) = 0$ on
hits as well) and $\phi(X^K_N) - g(X^K_N)$ on capped trajectories. The
indicator $\mathbf{1}_{A \cup B}(X^K_{\tau(N)})$ in~\eqref{eq:disc_L_VM}
exactly encodes the gating: on capped paths the indicator vanishes and we
recover zero terminal residual, matching the generic-loss expression
$\tilde\phi_N - g$ evaluated under the natural absorbing-kernel extension with
$g$ understood as $g|_{A \cup B}$. With this identification,
\eqref{eq:disc_L_VM} equals~\eqref{eq:disc_L_VM_finite} for the absorbed-chain
specialization (modulo a measure-preserving change of variables for the chain
beyond $\tau$).

Steps 1-3 of the proof of~\Cref{thm:disc_VM_finite} now apply verbatim, with
two simplifications: the chain is time-homogeneous so all
$D^{\tilde\phi}_n \equiv D^{\phi}$ (and likewise $P^{\tilde\phi}_n \equiv P^{\phi}$
of~\eqref{eq:disc_P_phi}) on $(A \cup B)^c$, and the deterministic horizon
$N$ in the sums and the terminal residual is realized as the random capped
hitting time $\tau(N)$ via absorption. The hitting-positive-probability
assumption $\mathbb{P}^K(X^K_{\tau(N)} \in A \cup B) > 0$ plays the role of
$X^K_N$ having support on the boundary states in the generic theorem, which
enters in the per-state extraction of $\hat\phi = g$ on
$\mathrm{supp}(X^K_{\tau(N)}) \cap (A \cup B)$. Concretely:

\emph{Step 2 (recursive perturbation).} Define $\eta^{\star}$ on the absorbed
chain by $\eta^{\star}_{\tau(N)}(x) = \mathbf{1}_{A \cup B}(x)\,(\hat\phi(x) - g(x))$
and $\eta^{\star}_n(x) = D^{\hat\phi}(x) + (P^{\hat\phi} \eta^{\star}_{n+1})(x)$
for $0 \le n < \tau(N)$, extended by $0$ on $A \cup B$. The same telescoping
identity as in~\eqref{eq:disc_eta_star} gives
$\tfrac{d}{d\varepsilon} F^{\hat\phi + \varepsilon \eta^{\star}}|_{\varepsilon = 0} = F^{\hat\phi}$,
whence $\mathbb{E}^K[(F^{\hat\phi})^2] = 0$ and $F^{\hat\phi} = 0$
$\mathbb{P}^K$-almost surely, where
$F^{\hat\phi} := \mathbf{1}_{A \cup B}(X^K_{\tau(N)})\,(\hat\phi(X^K_{\tau(N)}) - g(X^K_{\tau(N)})) - \sum_{n = 0}^{\tau(N) - 1} D^{\hat\phi}(X^K_n)$.

\emph{Step 3 (per-state identification).} Localized perturbations
$\eta = \mathbf{1}_{\{x\}}$ in $\mathcal{F}$ correspond, in the
time-indexed setup, to perturbing $\tilde\phi_n$ at a single state $x$ for all
$n$; this is a finite linear combination of the
$\zeta\,\mathbf{1}_{\{n = m\}}$-perturbations of~\Cref{thm:disc_VM_finite},
so the same conditional-expectation identities $G_m \equiv 0$ on
$\mathrm{supp}(X^K_m)$ hold. The backward extraction yields $\hat\phi = g$ on
$\mathrm{supp}(X^K_{\tau(N)}) \cap (A \cup B)$ (using the
hitting-positive-probability assumption) and $D^{\hat\phi}(x) = 0$ for
$x \in (A \cup B)^c$ in the support of the occupation measure of $X^K$.
By~\eqref{eq:disc_residual} and the discrete HJB~\eqref{eq:disc_HJB}, the
latter is equivalent to $\hat\phi$ satisfying the HJB equation on the support,
whose unique bounded solution with boundary data $\hat\phi|_{A \cup B} = g$ is
$\Phi$. The induced kernel identity $P^{\hat\phi} = P^{u^{\star}}$ then follows
from~\eqref{eq:disc_X_u_star}.
\end{proof}

\paragraph{Remark: uncapped variant.}
If $\mathbb{E}^{K}[\tau] < \infty$ and $\phi$ is bounded, the cap can be
removed: every trajectory hits $A \cup B$ in finitely many steps almost surely,
and the indicator in~\eqref{eq:disc_L_VM} becomes identically $1$. The
resulting \emph{uncapped} discrete VM loss
\begin{talign} \label{eq:disc_L_VM_uncapped}
    \mathcal{L}_{\mathrm{VM}, \infty}^{\mathrm{disc}}(\phi)
    \;=\; \mathbb{E}^{K}\!\!\left[\, \tfrac{1}{2} \Big(\,
    \phi(X^{K}_{\tau}) \;-\; g(X^{K}_{\tau}) \;-\; \sum_{n = 0}^{\tau - 1} D^{\phi}(X^{K}_n) \,\Big)^{\!2} \,\right]\!,
\end{talign}
is the discrete analog of the uncapped continuous loss of
\Cref{thm:vm_loss_hitting}; its first-order critical points coincide with
$\Phi$ on the support of the occupation measure of $X^{K}$ by the same
argument as above. As in the continuous case, the proof for the
capped loss~\eqref{eq:disc_L_VM} builds on the uncapped argument, with the
indicator $\mathbf{1}_{A \cup B}(X^{K}_{\tau(N)})$ controlling the
contribution of capped trajectories. In practice the capped form is preferred
because $\tau$ can be prohibitively long even when finite a.s.

\paragraph{Off-policy / off-kernel choices.}
A natural choice for $K$ mirroring~\eqref{eq:v_rescaled_reactive} uses the
current estimate of $\Phi$ to form a discrete analog of the $\kappa$-rescaled
reactive process, e.g.\
\begin{talign} \label{eq:disc_K_choice}
    K(x, dy) \;\propto\; q_{\xi}(y)^{(1 + \kappa)/2}\, \tilde q_{\xi}(y)^{(\kappa - 1)/2}\, P(x, dy),
\end{talign}
which at $\kappa = 1$ reduces to the optimal kernel~\eqref{eq:disc_X_u_star} and
at $\kappa < 1$ pushes mass more aggressively from $A$ toward $B$. During
training one uses detached approximations of $q_{\xi}$, $\tilde q_{\xi}$ in $K$
(the discrete analog of the stop-gradient $\texttt{sg}(\phi)$ trick
of~\eqref{eq:approximate_v}).

\subsection{Remarks and limitations}
\label{subsec:discrete_remarks}

\paragraph{Connection to \citet{helfmann2020extending}.}
The TPT objects of~\Cref{subsec:discrete_tpt_stationary} are the
discrete-time, discrete-space specialization of the (time-homogeneous)
framework of \citet{helfmann2020extending}; the time-inhomogeneous,
finite-horizon TPT formulation of~\citet[Sec.~3]{helfmann2020extending} is
not reproduced here, since~\Cref{subsec:discrete_vm_finite} addresses
instead the generic finite-horizon SOC problem of~\Cref{thm:vm_finite}, in
which absorbing sets $A, B$ play no role. The novelty of this appendix lies
in (i) the discrete Cole-Hopf
transform~\eqref{eq:disc_Phi}-\eqref{eq:disc_HJB}, which casts the rescaled
committor problem as a discrete HJB equation; (ii) the SOC
formulation~\eqref{eq:disc_control_problem} with KL-regularized one-step
cost, recovering the discrete Doob's $h$-transform with $h = q_\xi$ as the
optimal kernel; and (iii) the discrete Value Matching
losses~\eqref{eq:disc_L_VM} and~\eqref{eq:disc_L_VM_finite}, together with
their first-order optimality results~\Cref{thm:disc_VM,thm:disc_VM_finite},
which extend the off-policy training scheme of~\Cref{sec:value_matching} to
Markov chains in both the committor-driven hitting-time setting and the
generic finite-horizon SOC setting.

\paragraph{Bridge to the continuous formulation.}
Substituting $P_n = I + \Delta t\, \mathcal{L}_n + o(\Delta t)$ for a
continuous-time generator $\mathcal{L}_n$ and rescaling time by $\Delta t$, the
BKE $q = P q$ becomes $\mathcal{L} q = 0$, recovering~\eqref{eq:BKE_bg} in the
$\Delta t \to 0$ limit. The discrete reactive
current~\eqref{eq:disc_reactive_current} correspondingly converges (up to
division by $\Delta t$) to the continuous reactive
current $J_{\mathrm{R}}$ of~\eqref{eq:reactive_current_bg}. To leading order in
$\Delta t$, the per-state Bellman residual~\eqref{eq:disc_residual} expands as
$D^{\phi}(x) = \Delta t\, \big( \mathcal{L}\phi(x) - \tfrac{1}{2}\|\sigma^{\top}\nabla\phi(x)\|^{2} \big) + o(\Delta t)$,
which is the continuous Bellman residual appearing in
the Cole-Hopf-transformed HJB equation~\eqref{eq:HJB_setup}, and the discrete
VM loss~\eqref{eq:disc_L_VM} formally reproduces the continuous VM
loss~\eqref{eq:L_VM_v_kappa} in the $\Delta t \to 0$ limit.

\paragraph{Practical considerations.}
On finite or countable state spaces the critic $\phi(x)$ (or $\phi_n(x)$) can
be parameterized in tabular form or by a neural network; in the
time-inhomogeneous case one adds a positional encoding of the time index $n$.
On general state spaces one inherits the architectural choices
of~\Cref{subsec:parameterizing_committor}. Evaluating the Bellman
residual~\eqref{eq:disc_residual} requires the log-sum-exp
$\log (P\, e^{-\phi})(x)$; when $|\mathcal{X}|$ is large this can be
approximated by Monte Carlo with samples $Y \sim P(x, \cdot)$.

\paragraph{Out of scope.}
This appendix does not address (a) continuous-time, discrete-space Markov chains
(jump processes), which require their own infinitesimal generator framework
(see~\citet{zhu2025mdns} for a recent line of work using SOC on continuous-time Markov chains in
discrete state spaces, applied to neural sampling from $\pi \propto e^{-U}$ via
masked-diffusion path-measure alignment)
nor (b) periodically forced dynamics, which is also covered
by~\citet{helfmann2020extending}. Both are natural follow-ups that fit within
the same SOC + Value Matching pattern, and we leave their treatment to future
work.

\section{Additional experiment details}

\subsection{Reaction rate estimation}
\label{appendix:rate}

For reversible overdamped Langevin systems, the reaction rate $\nu_\mathrm{R}$ can be interpreted as the total probability flux of reactive trajectories from $A$ to $B$ through any dividing surface $\Sigma$:
\begin{talign}
\nu_\mathrm{R} = Z_{AB}\int_{\Sigma} J_{\mathrm{R}}(x)\cdot n_{\Sigma}(x)\,\mathrm{d}S(x).
\end{talign}
Equivalently, this surface integral admits the volume representation~\citep{vanden2006towards}
\begin{talign}
\nu_\mathrm{R} = \beta^{-1}\int_{\mathcal{X}} |\nabla q(x)|^2\,\rho(x)\,\mathrm{d}x,
\end{talign}
which is the form we evaluate. In the two-dimensional experiments we evaluate this integral numerically, and the integrals for $p_A$ and $p_B$ below, with the trapezoid rule on an $800\times800$ grid over the domain, with $\rho$ normalized by the partition function computed on the finite-element mesh and with $q=0$ on $A$ and $q=1$ on $B$.

The same rate can also be interpreted dynamically as the long-time frequency of reactive transitions,
\begin{talign}
\nu_\mathrm{R} = \lim_{T\rightarrow \infty} \frac{N_T}{T},
\end{talign}
where $N_T$ is the number of reactive trajectories observed over a trajectory of length $T$. The directional rate constants are
\begin{talign}
k_{AB} = \lim_{T\rightarrow \infty} \frac{N_T}{T_A}, \qquad
k_{BA} = \lim_{T\rightarrow \infty} \frac{N_T}{T_B},
\end{talign}
where $T_A$ and $T_B$ are the total times for which the last visited metastable set was $A$ and $B$, respectively, with $T_A+T_B=T$. These quantities satisfy
\begin{talign}
k_{AB} = \frac{\nu_\mathrm{R}}{p_A}, \qquad
k_{BA} = \frac{\nu_\mathrm{R}}{p_B}, \qquad
K_{\mathrm{eq}} = \frac{k_{AB}}{k_{BA}} = \frac{p_B}{p_A},
\end{talign}
where
\begin{talign}
p_A = \lim_{T\rightarrow \infty} \frac{T_A}{T}
    = \int_{\mathcal{X}} \rho(x)\,(1-q(x))\,\mathrm{d}x, \qquad
p_B = \lim_{T\rightarrow \infty} \frac{T_B}{T}
    = \int_{\mathcal{X}} \rho(x)\,q(x)\,\mathrm{d}x.
\end{talign}
Finally, the reactive weight is
\begin{talign}
Z_{AB} = \int_{\mathcal{X}} \rho(x)\,q(x)(1-q(x))\,\mathrm{d}x,
\end{talign}
and the normalized reactive density is $\rho_{\mathrm{R}}(x)=Z_{AB}^{-1}\rho(x)q(x)(1-q(x))$.

\subsection{Experiment setups}
\label{appendix:system_details}

This section summarizes the six low-dimensional systems used in the experiments. The first three are reversible overdamped Langevin benchmarks; the remaining three test generalized reversible or non-reversible dynamics. The two molecular systems, alanine dipeptide and chignolin, are described in~\Cref{appendix:aldp,appendix:chignolin}. The overdamped references are computed by finite-element methods, while the underdamped and non-reversible references are computed by finite-difference methods.

\subsubsection{Dynamics and systems}

\paragraph{Overdamped Langevin benchmarks.}
For the triple-well, Müller-Brown, and rugged Müller-Brown systems we simulate overdamped Langevin dynamics with potentials scaled by the reported temperature, $U=V/(k_BT)$. The triple-well potential is
\begin{equation}
\label{eq:triplewell}
\begin{aligned}
V_{\mathrm{TW}}(x,y) &= 3e^{-x^2-(y-1/3)^2}-3e^{-x^2-(y-5/3)^2}
-5e^{-(x-1)^2-y^2}-5e^{-(x+1)^2-y^2} \\
&\quad +0.2x^4+0.2(y-1/3)^4.
\end{aligned}
\end{equation}
The Müller-Brown potential is
\begin{equation}
\begin{aligned}
V_{\mathrm{MB}}(x,y) &= -200e^{-(x-1)^2-10y^2}-100e^{-x^2-10(y-1/2)^2} \\
&\quad -170e^{-6.5(x+1/2)^2+11(x+1/2)(y-3/2)-6.5(y-3/2)^2} \\
&\quad +15e^{0.7(x+1)^2+0.6(x+1)(y-1)+0.7(y-1)^2},
\end{aligned}
\end{equation}
and the rugged Müller-Brown potential is $V_{\mathrm{RMB}}(x,y)=V_{\mathrm{MB}}(x,y)+9\sin(10\pi x)\sin(10\pi y)$. We use $k_BT=0.5,10.0,15.0$ for the triple-well, Müller-Brown, and rugged Müller-Brown systems, respectively. The metastable states $A$ and $B$ are Euclidean balls of radius $\varepsilon$ (\Cref{tab:hparams}) centered at
\begin{itemize}[leftmargin=*]
    \item triple-well: $(-1.0,0)$ and $(1.0,0)$;
    \item Müller-Brown: $(-0.558,1.442)$ and $(0.623,0.028)$;
    \item rugged Müller-Brown: $(-0.57,1.43)$ and $(0.56,0.044)$.
\end{itemize}

\paragraph{Underdamped double-well.}
For the underdamped benchmark we use the phase-space dynamics
\begin{equation}
\mathrm{d}r_t=p_t\,\mathrm{d}t,\qquad
\mathrm{d}p_t=\left(-4r_t(r_t^2-1)-\gamma p_t\right)\mathrm{d}t+\sqrt{2\gamma k_BT}\,\mathrm{d}W_t,
\end{equation}
with $\gamma=1.0$ and $k_BT=1.0$. The states are the Hamiltonian sublevel sets $A=\{H(r,p)<0.3,\,r<0\}$ and $B=\{H(r,p)<0.3,\,r>0\}$, where $H(r,p)=\frac12p^2+(r^2-1)^2$.

\paragraph{Non-reversible systems.}
The Maier-Stein dynamics are
\begin{equation}
\label{eq:maier-stein-sde}
\mathrm{d}X_t=
\begin{pmatrix}
x_t-x_t^3-\alpha x_t y_t^2\\-(1+x_t^2)y_t\end{pmatrix}\mathrm{d}t+\sqrt{0.1}\,\mathrm{d}W_t, \qquad X_t = (x_t, y_t),
\end{equation}
with $\alpha=10.0$, and $A$, $B$ the balls of radius $\varepsilon$ centered at $(-1.0,0)$ and $(1.0,0)$.

The periodically-driven system uses the static potential $V_{\mathrm{TW}}/4$ and the time-dependent drift
\begin{equation}
\mathrm{d}X_t=\left(-\nabla (V_{\mathrm{TW}}/4)(X_t)+1.4\cos(2\pi t/T_{\mathrm{per}})(-y_t,x_t)\right)\mathrm{d}t+\sqrt{0.1}\,\mathrm{d}W_t,
\end{equation}
with forcing period $T_{\mathrm{per}}=1.8$ and the same $A$ and $B$ as for Maier-Stein. Its committor is time-dependent and satisfies the time-dependent BKE.

\subsubsection{Hyperparameters}
\label{appendix:hyperparameters}

We use the same MLP architecture for all experiments, including the baselines, with the same optimizer and budget of energy evaluations. The initial distribution $\mu_0$ mixes stored trajectory frames, points next to $A$ and $B$, and parallel tempering (PT) samples. In the tables, ``iterate'' denotes the biased density proportional to $\exp(-U)|\nabla q|^2$, and ``reactive'' denotes the rescaled reactive density~\eqref{eq:rescaled_reactive_density}, proportional to $\exp(-U)\,q_{\xi}(1-q_{\xi})$. On the rugged Müller-Brown potential, REACT-VM and the semigroup take a quarter of their PT start points from the higher temperature replicas of the PT chains, so that the start points spread over the transition region.

\paragraph{Start points, simulation and replay buffer.}\label{appendix:start_points}
In the overdamped experiments of~\Cref{sec:experiment}, and with the same rules in the underdamped and non-reversible experiments, $\mu_0$ mixes three sources of start points, drawn at every epoch: (a) frames of stored trajectories whose current estimate $\hat q$ is closest to $1/2$, which return to the parts of the transition region that earlier rollouts reached; (b) points next to $A$ and $B$; and (c) samples of $\rho^{(\xi)}_{\mathrm{R}}$ from parallel tempering chains (on the rugged Müller-Brown potential, a quarter of these start points are taken from the higher temperature replicas of the chains, which spreads them over the whole transition region). The target is $B$ with probability $1/2$ in each iteration of the start points, and a particle that has reached neither set at the horizon $T$ is continued for up to four further times. Because the VM loss is off-policy, rollouts simulated with earlier iterates of $\phi$ remain valid training data: REACT-VM stores them in a first-in-first-out replay buffer and takes several optimizer steps on minibatches from it per batch of new rollouts, evaluating each stored rollout with the drift that simulated it. REACT-DBP is on-policy and uses every rollout once; its start points are (b) and parallel tempering samples of $\rho|\nabla\hat q|^2$. \Cref{tab:hparams,tab:hparams_additional} list the settings.

\begin{table}[t]
\centering
\caption{Hyperparameters for the overdamped Langevin system experiments.}
\label{tab:hparams}
\begin{tabular}{lccc}
\toprule
\textbf{Hyperparameter} & \textbf{Triple-well} & \textbf{Müller} & \textbf{Rugged Müller} \\
\midrule
\multicolumn{4}{c}{\textbf{Model}} \\
\midrule
Network architecture & \multicolumn{3}{c}{MLP} \\
Network size & \multicolumn{3}{c}{(32, 64, 128, 64, 32)} \\
\midrule
\multicolumn{4}{c}{\textbf{Training}} \\
\midrule
Energy evaluations & $3\times10^{7}$ & $10^{8}$ & $10^{9}$ \\
Learning rate & \multicolumn{3}{c}{0.001 (cosine decay)} \\
Batch size & \multicolumn{3}{c}{256} \\
Replay buffer & \multicolumn{3}{c}{65,536 trajectories} \\
Updates per epoch & 100 & 800 & 800 \\
$\varepsilon$ & \multicolumn{3}{c}{0.025} \\
$\xi$ & \multicolumn{3}{c}{$\exp(-7.0)$} \\
$\mu_0$ (VM, semigroup) & \multicolumn{3}{c}{stored frames 30\%, near $A$, $B$ 25\%, reactive 45\%} \\
$\mu_0$ (DBP) & \multicolumn{3}{c}{near $A$, $B$ 25\%, iterate 75\%} \\
PT starts from hotter replicas & 0 & 0 & 25\% \\
Shooting & \multicolumn{3}{c}{committor} \\
$T$ &  \multicolumn{3}{c}{0.1}\\
$\text{d}t$ & 0.001 & 0.0000625 & 0.0000625   \\
\midrule
\multicolumn{4}{c}{\textbf{Inference}} \\
\midrule
Evaluation & \multicolumn{3}{c}{quadrature, $800\times800$ grid} \\
\bottomrule
\end{tabular}
\end{table}

\begin{table}[t]
\centering
\caption{Hyperparameters for the underdamped and non-reversible experiments.}
\label{tab:hparams_additional}
\begin{tabular}{lccc}
\toprule
\textbf{Hyperparameter} & \textbf{Underdamped DW} & \textbf{Maier-Stein} & \textbf{Periodic TW} \\
\midrule
\multicolumn{4}{c}{\textbf{Model}} \\
\midrule
Network architecture & \multicolumn{3}{c}{MLP} \\
Network size & \multicolumn{3}{c}{(32, 64, 128, 64, 32)} \\
\midrule
\multicolumn{4}{c}{\textbf{Training}} \\
\midrule
Energy evaluations & $10^{7}$ & $10^{8}$ & $10^{8}$ \\
Learning rate & \multicolumn{3}{c}{$0.001$ (cosine decay over $10^{8}$ energy evaluations)} \\
Batch size & \multicolumn{3}{c}{256} \\
Replay buffer & \multicolumn{3}{c}{65,536 trajectories} \\
Updates per epoch & \multicolumn{3}{c}{800} \\
$\varepsilon$ & $H<0.3$ & $0.025$ & $0.025$ \\
$\xi$ & \multicolumn{3}{c}{$\exp(-7.0)$} \\
$\mu_0$ (VM, semigroup) & \multicolumn{3}{c}{stored frames 30\%, near $A$, $B$ 25\%, reactive 45\%} \\
$\mu_0$ (DBP) & \multicolumn{3}{c}{near $A$, $B$ 25\%, iterate 75\%} \\
Shooting & \multicolumn{3}{c}{committor} \\
$T$ & $2.0$ & $1.0$ & $3.6$ \\
$\text{d}t$ & $0.01$ & $0.01$ & $0.05$ \\
\bottomrule
\end{tabular}
\end{table}

\paragraph{System-specific notes.}
For the underdamped double-well, $\varepsilon$ denotes the Hamiltonian threshold rather than a Euclidean radius, and PT is performed in phase space; the start points next to $A$ and $B$ lie in the same well with $0.3<H\le0.9$, below the barrier at $H=1$. For Maier-Stein, the stationary density is not available in closed form at $\alpha=10$, so PT uses the equilibrium ($\alpha=1$, gradient) potential only as a proposal. For the periodically-driven system, the network input is augmented with $(\cos(2\pi t/T_{\mathrm{per}}),\sin(2\pi t/T_{\mathrm{per}}))$, and the start points from PT and next to $A$ and $B$ get a start time drawn uniformly over one period. For the underdamped double-well and the periodically-driven system, REACT-DBP and the semigroup baseline are trained with an additional basin penalty (weight $1$): the mean squared deviation of the learned committor from its boundary values. The semigroup baseline uses half of the horizon $T$ of REACT, as for the overdamped systems. On the underdamped double-well, all methods are evaluated after $10^{7}$ energy evaluations.

\subsection{Baseline methods}
\label{appendix:baseline}
We implement three baselines for the reversible overdamped Langevin systems: two based on the variational formulation~\eqref{eq:kbe_variational} and one based on the Feynman-Kac formulation~\eqref{eq:fk}. All three use the network, the hard boundary condition (without $\xi$), the optimizer and the budget of energy evaluations of REACT (\Cref{tab:hparams}). The semigroup baseline is also used on the underdamped double-well, the two non-reversible systems (\Cref{tab:hparams_additional}), alanine dipeptide (\Cref{appendix:aldp}) and chignolin (\Cref{appendix:chignolin}), where the variational baselines do not apply.
For the first baseline, we directly optimize \Cref{eq:kbe_variational} over the equilibrium ensemble.
For the second baseline, we apply the importance sampling scheme of \citet{kang2024computing} to enhance the sampling of reactive regions. For the third baseline, we fit the committor to short unbiased trajectories through the semigroup form of~\eqref{eq:fk}. We obtain the ground truth by the finite-element method implemented in~\cite{baratta2023dolfinx}.

All methods share the network, the hard boundary condition~\eqref{eq:q_parameterization} (with $\xi$ for REACT only), the optimizer and the budget of energy evaluations ($3\times10^{7}$ on the triple-well, $10^{8}$ on the Müller-Brown and $10^{9}$ on the rugged Müller-Brown potential), which counts every energy evaluation, including parallel tempering (\Cref{appendix:hyperparameters}). REACT-VM and the semigroup baseline also share the start points $\mu_0$ and the replay buffer of trajectories (\Cref{appendix:start_points}). REACT-DBP is on-policy, uses every trajectory once and draws its start points next to $A$ and $B$ and from $\rho|\nabla\hat q|^2$. We use $\kappa=1$ and $\xi=e^{-7}$; \Cref{appendix:ablations} varies $\kappa$, $\xi$ and the boundary treatment of the baselines.

\subsubsection{Baseline 1 (BKE)} The first baseline we consider is to solve the variational formulation of BKE via neural network parameterization.

\textbf{Equilibrium sample generation}.
We generate samples using PT, following the same procedure as we sample our initial distribution.
The samples are drawn before training and use the whole budget of energy evaluations: $2.94\times10^{6}$ on the triple-well, $9.8\times10^{6}$ on the Müller-Brown and $9.8\times10^{7}$ on the rugged Müller-Brown potentials.

\textbf{Training details}. The loss is the Dirichlet energy~\eqref{eq:kbe_variational} on these samples; the boundary conditions are imposed by the parameterization~\eqref{eq:q_parameterization}. We train for 100{,}000 iterations with batch size 8{,}192, using Adam with learning rate $10^{-3}$, cosine decay over the iterations, gradient-norm clipping at 10 and an exponential moving average of the weights (decay 0.999), which is the evaluated network.

\textbf{Evaluation details}. All methods are evaluated in the same way (\Cref{appendix:rate}).

\subsubsection{Baseline 2 (BKE-IS)}

The second baseline combines the variational formulation with importance sampling to ensure optimization over transition regions (as they contribute most errors).
We first draw $2\times10^{6}$ samples from the original target with parallel tempering, each with weight 1. After every 1{,}000 optimization steps, we draw $4\times10^{5}$ new samples by parallel tempering from the biased distribution $\exp(-U)|\nabla q|^2$, where $q$ is the current committor estimate, with unnormalized log-importance weights $-\log|\nabla q|^2$; we normalize these weights with a softmax, multiply them by the number of new samples and append the samples and their weights to the dataset. On the rugged Müller-Brown potential all sample counts are ten times larger, and on the triple-well 0.3 times as large. We train for 20{,}000 steps (19 rounds of new samples), with batches of 512 samples and the importance-weighted variational loss of~\citet{kang2024computing}.
\par
For boundary conditions, network parameterization, other training setups, and evaluation details, we use the same setups as in Baseline 1.

\subsubsection{Baseline 3 (semigroup)}
The third baseline minimizes the Feynman-Kac residual~\eqref{eq:fk} in its semigroup form~\citep{li2022semigroup,strahan2023predicting}. From each start point we simulate five independent trajectories of the unbiased dynamics with horizon $T=0.05$, stopped when they enter $A\cup B$. Each trajectory gives one sample of the quantity inside the conditional expectation in~\eqref{eq:fk}: the boundary value if it entered $A$ or $B$, and the network's committor at its end point otherwise. The squared residual in~\eqref{eq:fk} is estimated by the product of the residuals of two distinct trajectories from the same start point, averaged over all such pairs and over the start points; since the trajectories are independent given the start point, this estimate is unbiased. The boundary conditions are imposed by the hard parameterization, and gradients flow through the committor at the start point and at the end points. The start points, their continuation, the replay buffer, the time step and the optimizer are those of REACT-VM (\Cref{tab:hparams}); the parallel tempering chains sample the reactive density~\eqref{eq:reactive_density_bg} of the current estimate (on the rugged Müller-Brown potential a quarter of the start points come from the higher-temperature replicas of the chains). There is no control, no target set and no $\xi$. On the underdamped double-well and the two non-reversible systems it uses the settings of REACT-VM in~\Cref{tab:hparams_additional} with half of the horizon $T$, and the basin penalty of REACT-DBP on the underdamped double-well and the periodically-driven system.

\paragraph{Soft boundary condition.} For the boundary ablation (\Cref{appendix:ablations}) the three baselines are also trained with a soft boundary condition: without the distance factors in~\eqref{eq:q_parameterization} and with a penalty (weight 1) added to the loss: the mean squared deviation of the network's committor from its boundary values, $0$ and $1$, at points drawn around the centres of $A$ and $B$ with Gaussian noise of standard deviation 0.025.

\subsubsection{Transition path sampling}

To obtain hitting-time transition path ensemble for the reversible overdamped Langevin systems, we use the transition path sampling (TPS) algorithm~\citep{dellago1998transition} with two-way shooting. The settings for three systems are exactly the same. The states are defined within a radius of 0.025 to the boundary condition $A$ and $B$. We use the implementation of the algorithm from \citep{du2024doob} at \href{https://github.com/plainerman/Variational-Doob}{GitHub Repo}.

\subsection{Ablations of the two-dimensional experiments}
\label{appendix:ablations}
Each ablation changes one setting and keeps everything else, including the budget and the three runs; the figures have the format of~\Cref{fig:final_errors}.

\paragraph{Noise level $\kappa$.} \Cref{fig:abl_kappa_errors} varies $\kappa\in\{0.5, 0.8, 1, 1.5, 2, 3\}$ for REACT-VM; $\kappa=1$ is the default setting. Whether $\kappa\neq1$ helps is case by case. Every $\kappa$ gives accurate rates, and different values of $\kappa$ give the best results on different systems, which makes $\kappa$ a hyperparameter and suggests to train with sampled $\kappa$ in practice.

\paragraph{Rescaling parameter $\xi$.} \Cref{fig:abl_xi_errors} varies $\xi\in\{e^{-1}, e^{-3}, e^{-5}, e^{-7}\}$ for REACT-VM; $\xi=e^{-7}$ is the default setting. Empirically, $\xi$ does not matter much: every value gives accurate rates on all three systems. The default $\xi=e^{-7}$ is the most accurate on the triple-well and Müller-Brown potentials, while larger values do slightly better on the rugged Müller-Brown potential.

\paragraph{Boundary conditions for baselines.} \Cref{fig:abl_bc_errors} compares the three baselines trained with hard and soft boundary conditions (the parameterization~\eqref{eq:q_parameterization}, or the boundary-condition loss of~\Cref{appendix:baseline}). We always use the hard boundary condition as the default, as it is generally more robust: the soft boundary condition causes large rate errors in some cases (BKE on the triple-well), while elsewhere the two are close.

\begin{figure}[p]
\centering
\includegraphics[width=\textwidth]{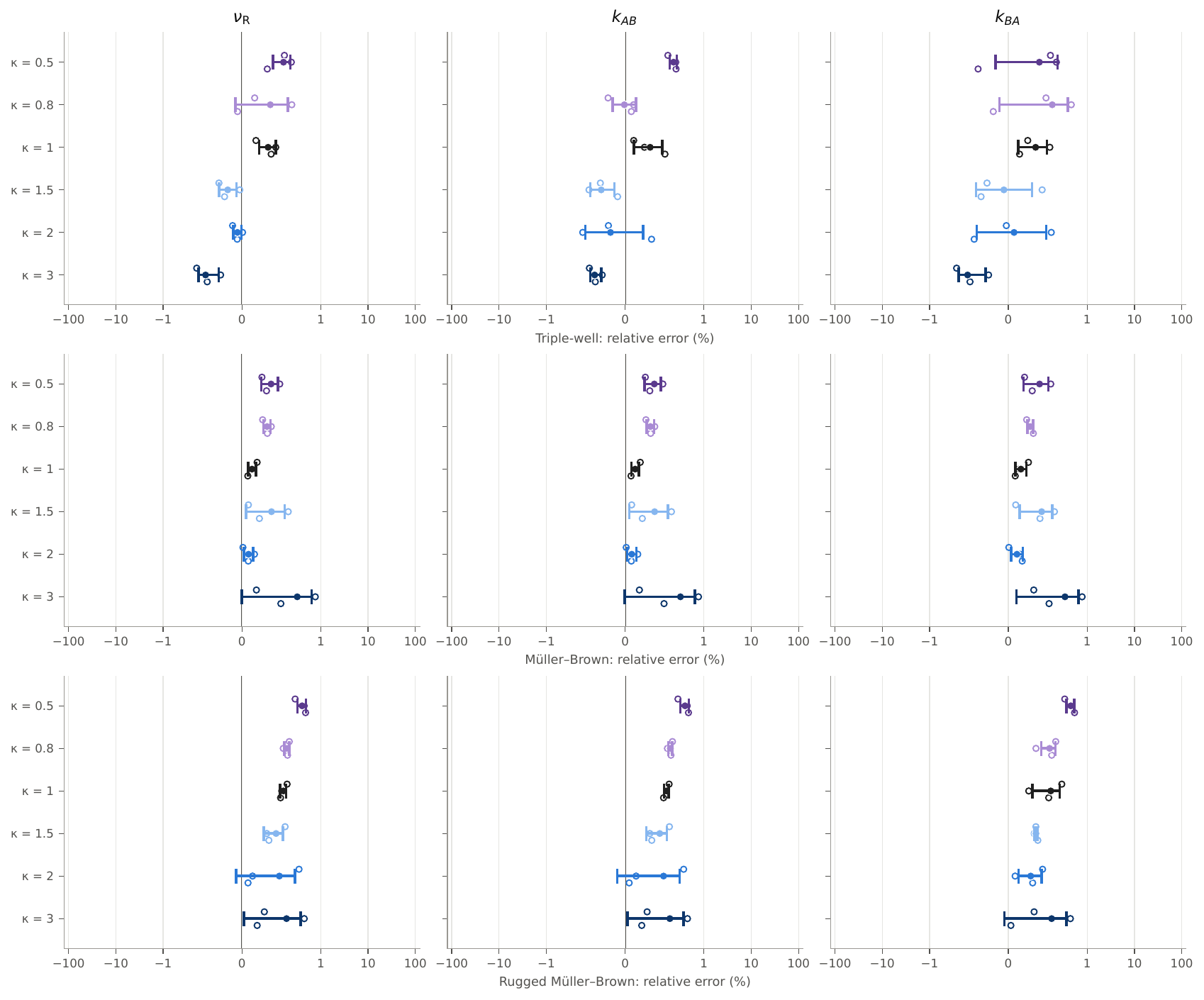}
\caption{REACT-VM with noise levels $\kappa\in\{0.5,0.8,1,1.5,2,3\}$ of the sampling dynamics: relative errors at the end of training, as in \Cref{fig:final_errors}. $\kappa=1$ is the default setting.}
\label{fig:abl_kappa_errors}
\end{figure}

\begin{figure}[p]
\centering
\includegraphics[width=\textwidth]{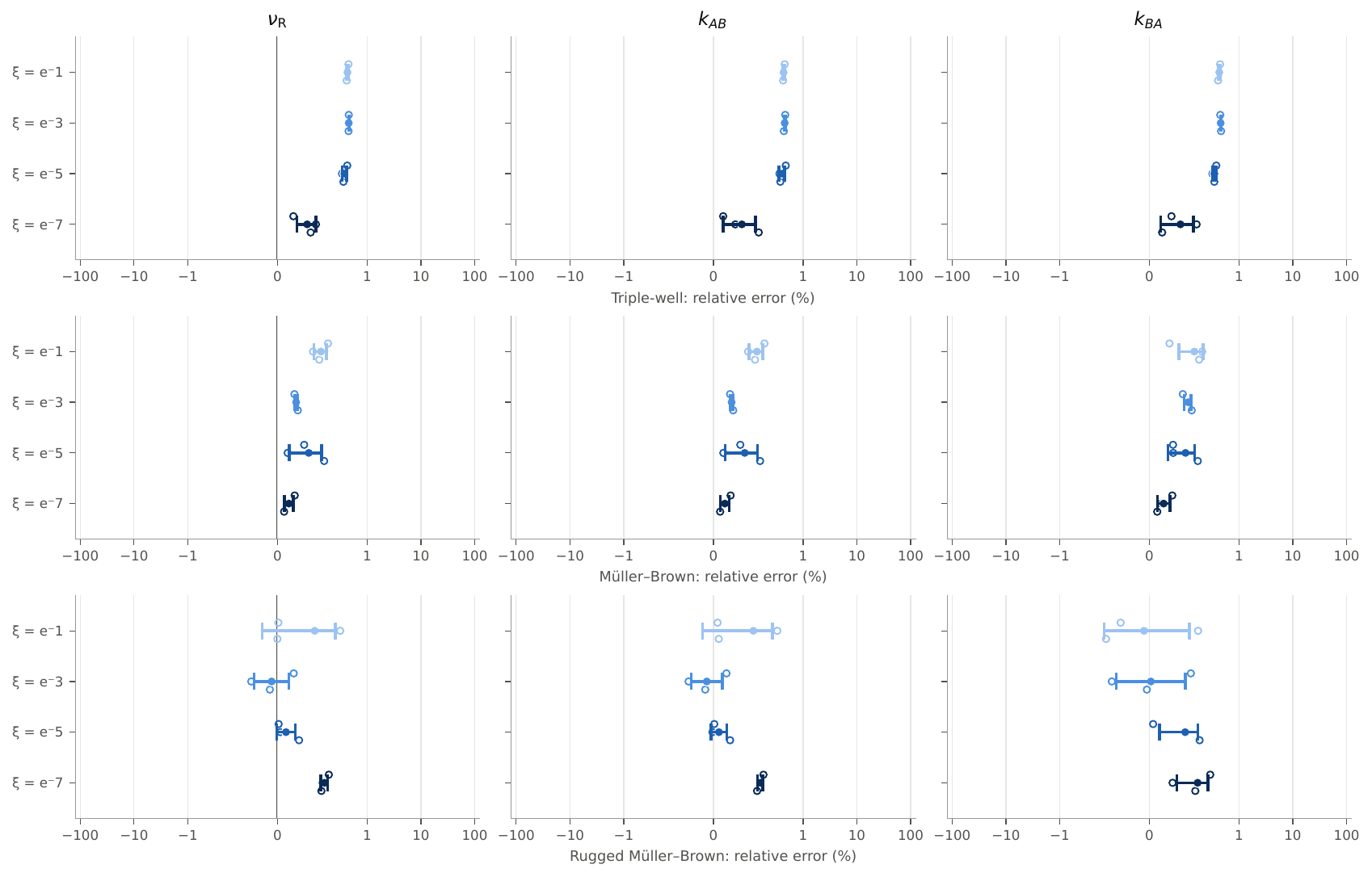}
\caption{REACT-VM with different values of $\xi$: relative errors at the end of training, as in \Cref{fig:final_errors}. $\xi=e^{-7}$ is the default setting.}
\label{fig:abl_xi_errors}
\end{figure}

\begin{figure}[p]
\centering
\includegraphics[width=\textwidth]{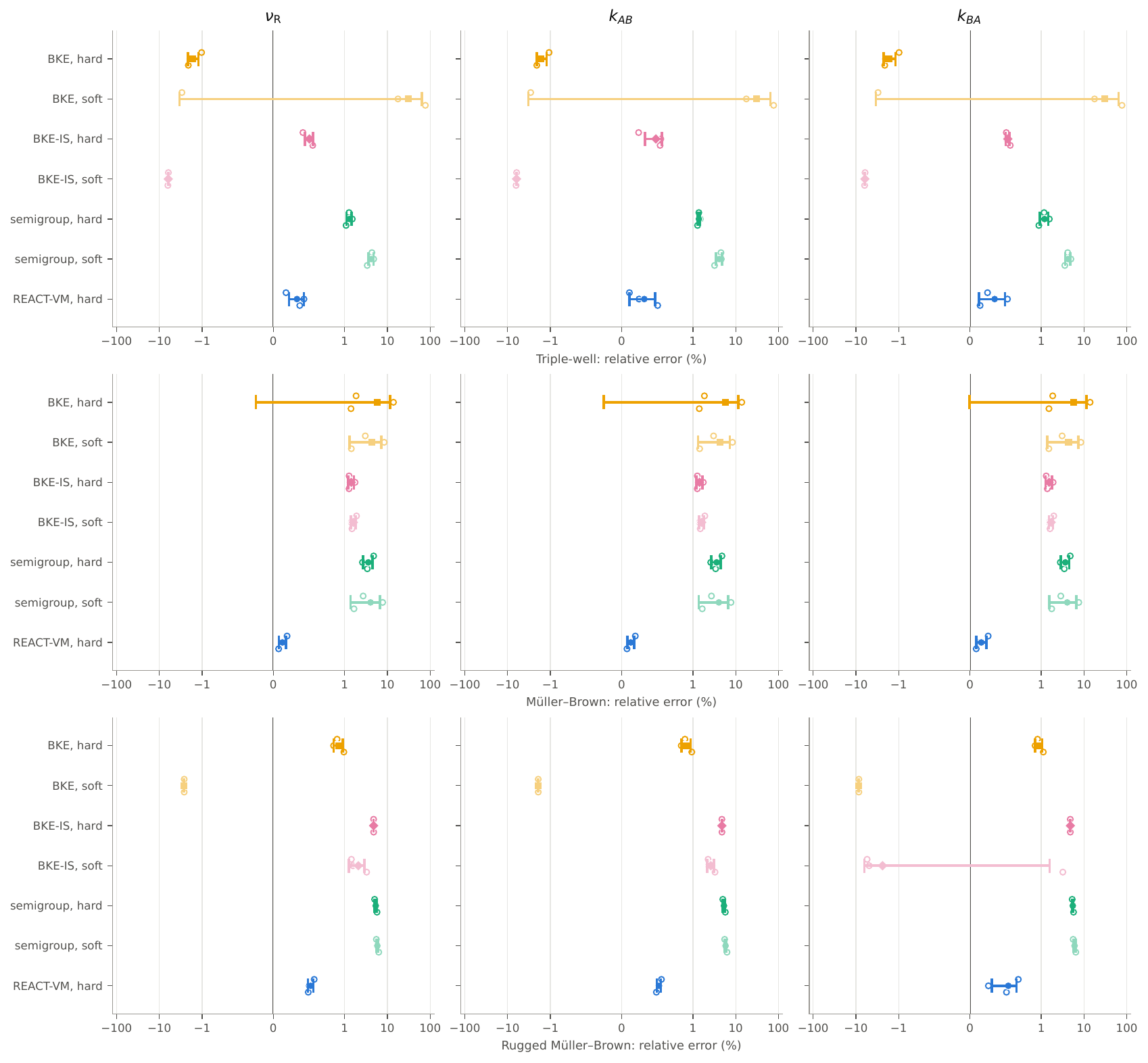}
\caption{Hard versus soft boundary conditions for the baselines, REACT-VM (hard) as reference: relative errors at the end of training, as in \Cref{fig:final_errors}. The hard boundary condition is the default, as it is generally more robust.}
\label{fig:abl_bc_errors}
\end{figure}

\subsection{Additional results}
\label{appendix:additional_viz}

\Cref{tab:main_result} gives the full results of the comparison of the three overdamped systems shown in \Cref{fig:final_errors}.
\begin{table*}[t]
  \centering
  \caption{Quantitative evaluation on the three overdamped systems at an equal budget of energy evaluations ($3\times10^{7}$ on the triple-well, $10^{8}$ on the Müller-Brown and $10^{9}$ on the rugged Müller-Brown potential), compared with the BKE, BKE-IS and semigroup baselines. All methods use the same network, hard boundary condition and optimizer (\Cref{appendix:hyperparameters}). We report the mean absolute error of the committor $q$, the reaction rate $\nu_\mathrm{R}$, the rate constants $k_{AB}$ and $k_{BA}$, the basin weights $p_A$ and $p_B$, the reactive weight $Z_{AB}$ and the equilibrium constant $K_{\mathrm{eq}}$: mean $\pm$ standard deviation over three runs (seeds 100, 101, 102). Bold: the value closest to the finite-element reference in each column except the committor error. All quantities are unitless.}
  \label{tab:main_result}
  \setlength{\tabcolsep}{4pt}
  \scriptsize
  \textbf{Triple-well potential}\\[2pt]
  \resizebox{\textwidth}{!}{%
  \begin{tabular}{l c c c c c c c c}
    \toprule
    Method & Mean $q$ error & $\nu_\mathrm{R}$ & $k_{AB}$ & $k_{BA}$ & $p_A$ & $p_B$ & $Z_{AB}$ & $K_{\mathrm{eq}}$ \\
    \midrule
    FEM & N/A & $1.74\!\times\!10^{-2}$ & $3.47\!\times\!10^{-2}$ & $3.47\!\times\!10^{-2}$ & $5.00\!\times\!10^{-1}$ & $5.00\!\times\!10^{-1}$ & $1.50\!\times\!10^{-2}$ & $1.00\!\times\!10^{0}$ \\
    \midrule
    BKE & $4.40 \!\pm\! 0.13 \!\times\! 10^{-3}$ & $1.71 \!\pm\! 0.01 \!\times\! 10^{-2}$ & $3.41 \!\pm\! 0.01 \!\times\! 10^{-2}$ & $3.41 \!\pm\! 0.02 \!\times\! 10^{-2}$ & {\boldmath$5.00 \!\pm\! 0.00 \!\times\! 10^{-1}$} & {\boldmath$5.00 \!\pm\! 0.00 \!\times\! 10^{-1}$} & $1.98 \!\pm\! 0.03 \!\times\! 10^{-2}$ & {\boldmath$1.00 \!\pm\! 0.00 \!\times\! 10^{0}$} \\
    BKE-IS & $2.18 \!\pm\! 0.16 \!\times\! 10^{-3}$ & {\boldmath$1.74 \!\pm\! 0.00 \!\times\! 10^{-2}$} & $3.48 \!\pm\! 0.00 \!\times\! 10^{-2}$ & {\boldmath$3.48 \!\pm\! 0.00 \!\times\! 10^{-2}$} & {\boldmath$5.00 \!\pm\! 0.00 \!\times\! 10^{-1}$} & {\boldmath$5.00 \!\pm\! 0.00 \!\times\! 10^{-1}$} & {\boldmath$1.45 \!\pm\! 0.01 \!\times\! 10^{-2}$} & {\boldmath$1.00 \!\pm\! 0.00 \!\times\! 10^{0}$} \\
    Semigroup & $8.20 \!\pm\! 0.64 \!\times\! 10^{-3}$ & $1.76 \!\pm\! 0.00 \!\times\! 10^{-2}$ & $3.52 \!\pm\! 0.00 \!\times\! 10^{-2}$ & $3.51 \!\pm\! 0.01 \!\times\! 10^{-2}$ & {\boldmath$5.00 \!\pm\! 0.00 \!\times\! 10^{-1}$} & {\boldmath$5.00 \!\pm\! 0.00 \!\times\! 10^{-1}$} & $1.40 \!\pm\! 0.04 \!\times\! 10^{-2}$ & {\boldmath$1.00 \!\pm\! 0.00 \!\times\! 10^{0}$} \\
    REACT-DBP & $4.48 \!\pm\! 0.61 \!\times\! 10^{-2}$ & $2.36 \!\pm\! 0.06 \!\times\! 10^{-2}$ & $4.70 \!\pm\! 0.11 \!\times\! 10^{-2}$ & $4.73 \!\pm\! 0.15 \!\times\! 10^{-2}$ & $5.01 \!\pm\! 0.03 \!\times\! 10^{-1}$ & $4.99 \!\pm\! 0.03 \!\times\! 10^{-1}$ & $2.84 \!\pm\! 0.21 \!\times\! 10^{-2}$ & $9.95 \!\pm\! 0.10 \!\times\! 10^{-1}$ \\
    REACT-VM & $2.61 \!\pm\! 0.12 \!\times\! 10^{-3}$ & {\boldmath$1.74 \!\pm\! 0.00 \!\times\! 10^{-2}$} & {\boldmath$3.47 \!\pm\! 0.00 \!\times\! 10^{-2}$} & {\boldmath$3.48 \!\pm\! 0.00 \!\times\! 10^{-2}$} & {\boldmath$5.00 \!\pm\! 0.00 \!\times\! 10^{-1}$} & {\boldmath$5.00 \!\pm\! 0.00 \!\times\! 10^{-1}$} & $1.38 \!\pm\! 0.00 \!\times\! 10^{-2}$ & {\boldmath$1.00 \!\pm\! 0.00 \!\times\! 10^{0}$} \\
    \bottomrule
  \end{tabular}%
  }
  \textbf{Müller-Brown potential}\\[2pt]
  \resizebox{\textwidth}{!}{%
  \begin{tabular}{l c c c c c c c c}
    \toprule
    Method & Mean $q$ error & $\nu_\mathrm{R}$ & $k_{AB}$ & $k_{BA}$ & $p_A$ & $p_B$ & $Z_{AB}$ & $K_{\mathrm{eq}}$ \\
    \midrule
    FEM & N/A & $4.96\!\times\!10^{-4}$ & $5.09\!\times\!10^{-4}$ & $1.93\!\times\!10^{-2}$ & $9.74\!\times\!10^{-1}$ & $2.57\!\times\!10^{-2}$ & $2.57\!\times\!10^{-4}$ & $2.64\!\times\!10^{-2}$ \\
    \midrule
    BKE & $6.61 \!\pm\! 1.23 \!\times\! 10^{-2}$ & $5.24 \!\pm\! 0.29 \!\times\! 10^{-4}$ & $5.38 \!\pm\! 0.29 \!\times\! 10^{-4}$ & $2.04 \!\pm\! 0.11 \!\times\! 10^{-2}$ & {\boldmath$9.74 \!\pm\! 0.00 \!\times\! 10^{-1}$} & {\boldmath$2.57 \!\pm\! 0.00 \!\times\! 10^{-2}$} & $2.40 \!\pm\! 0.14 \!\times\! 10^{-4}$ & {\boldmath$2.64 \!\pm\! 0.00 \!\times\! 10^{-2}$} \\
    BKE-IS & $6.16 \!\pm\! 0.89 \!\times\! 10^{-2}$ & $5.03 \!\pm\! 0.01 \!\times\! 10^{-4}$ & $5.16 \!\pm\! 0.01 \!\times\! 10^{-4}$ & $1.96 \!\pm\! 0.01 \!\times\! 10^{-2}$ & {\boldmath$9.74 \!\pm\! 0.00 \!\times\! 10^{-1}$} & {\boldmath$2.57 \!\pm\! 0.00 \!\times\! 10^{-2}$} & $2.26 \!\pm\! 0.04 \!\times\! 10^{-4}$ & $2.63 \!\pm\! 0.00 \!\times\! 10^{-2}$ \\
    Semigroup & $1.27 \!\pm\! 0.07 \!\times\! 10^{-1}$ & $5.14 \!\pm\! 0.04 \!\times\! 10^{-4}$ & $5.27 \!\pm\! 0.05 \!\times\! 10^{-4}$ & $2.00 \!\pm\! 0.02 \!\times\! 10^{-2}$ & {\boldmath$9.74 \!\pm\! 0.00 \!\times\! 10^{-1}$} & {\boldmath$2.57 \!\pm\! 0.00 \!\times\! 10^{-2}$} & $2.27 \!\pm\! 0.17 \!\times\! 10^{-4}$ & $2.63 \!\pm\! 0.00 \!\times\! 10^{-2}$ \\
    REACT-DBP & $3.02 \!\pm\! 0.01 \!\times\! 10^{-1}$ & $3.70 \!\pm\! 0.23 \!\times\! 10^{-1}$ & $4.18 \!\pm\! 0.27 \!\times\! 10^{-1}$ & $3.25 \!\pm\! 0.11 \!\times\! 10^{0}$ & $8.86 \!\pm\! 0.03 \!\times\! 10^{-1}$ & $1.14 \!\pm\! 0.03 \!\times\! 10^{-1}$ & $8.19 \!\pm\! 0.24 \!\times\! 10^{-2}$ & $1.28 \!\pm\! 0.04 \!\times\! 10^{-1}$ \\
    REACT-VM & $3.65 \!\pm\! 1.26 \!\times\! 10^{-2}$ & {\boldmath$4.96 \!\pm\! 0.00 \!\times\! 10^{-4}$} & {\boldmath$5.09 \!\pm\! 0.00 \!\times\! 10^{-4}$} & {\boldmath$1.93 \!\pm\! 0.00 \!\times\! 10^{-2}$} & {\boldmath$9.74 \!\pm\! 0.00 \!\times\! 10^{-1}$} & {\boldmath$2.57 \!\pm\! 0.00 \!\times\! 10^{-2}$} & {\boldmath$2.62 \!\pm\! 0.03 \!\times\! 10^{-4}$} & {\boldmath$2.64 \!\pm\! 0.00 \!\times\! 10^{-2}$} \\
    \bottomrule
  \end{tabular}%
  }
  \textbf{Rugged Müller-Brown potential}\\[2pt]
  \resizebox{\textwidth}{!}{%
  \begin{tabular}{l c c c c c c c c}
    \toprule
    Method & Mean $q$ error & $\nu_\mathrm{R}$ & $k_{AB}$ & $k_{BA}$ & $p_A$ & $p_B$ & $Z_{AB}$ & $K_{\mathrm{eq}}$ \\
    \midrule
    FEM & N/A & $8.43\!\times\!10^{-3}$ & $9.36\!\times\!10^{-3}$ & $8.54\!\times\!10^{-2}$ & $9.01\!\times\!10^{-1}$ & $9.88\!\times\!10^{-2}$ & $4.92\!\times\!10^{-3}$ & $1.10\!\times\!10^{-1}$ \\
    \midrule
    BKE & $8.34 \!\pm\! 2.28 \!\times\! 10^{-3}$ & $8.49 \!\pm\! 0.01 \!\times\! 10^{-3}$ & $9.42 \!\pm\! 0.02 \!\times\! 10^{-3}$ & $8.61 \!\pm\! 0.01 \!\times\! 10^{-2}$ & {\boldmath$9.01 \!\pm\! 0.00 \!\times\! 10^{-1}$} & $9.86 \!\pm\! 0.00 \!\times\! 10^{-2}$ & $5.14 \!\pm\! 0.03 \!\times\! 10^{-3}$ & $1.09 \!\pm\! 0.00 \!\times\! 10^{-1}$ \\
    BKE-IS & $1.84 \!\pm\! 0.28 \!\times\! 10^{-2}$ & $8.83 \!\pm\! 0.00 \!\times\! 10^{-3}$ & $9.80 \!\pm\! 0.00 \!\times\! 10^{-3}$ & $8.95 \!\pm\! 0.00 \!\times\! 10^{-2}$ & {\boldmath$9.01 \!\pm\! 0.00 \!\times\! 10^{-1}$} & {\boldmath$9.87 \!\pm\! 0.00 \!\times\! 10^{-2}$} & {\boldmath$4.95 \!\pm\! 0.07 \!\times\! 10^{-3}$} & {\boldmath$1.10 \!\pm\! 0.00 \!\times\! 10^{-1}$} \\
    Semigroup & $1.20 \!\pm\! 0.16 \!\times\! 10^{-2}$ & $8.88 \!\pm\! 0.02 \!\times\! 10^{-3}$ & $9.85 \!\pm\! 0.03 \!\times\! 10^{-3}$ & $9.00 \!\pm\! 0.01 \!\times\! 10^{-2}$ & {\boldmath$9.01 \!\pm\! 0.00 \!\times\! 10^{-1}$} & $9.86 \!\pm\! 0.01 \!\times\! 10^{-2}$ & $4.81 \!\pm\! 0.48 \!\times\! 10^{-3}$ & $1.09 \!\pm\! 0.00 \!\times\! 10^{-1}$ \\
    REACT-DBP & $2.43 \!\pm\! 0.46 \!\times\! 10^{-1}$ & $2.50 \!\pm\! 2.19 \!\times\! 10^{-1}$ & $3.13 \!\pm\! 2.88 \!\times\! 10^{-1}$ & $1.33 \!\pm\! 0.83 \!\times\! 10^{0}$ & $8.40 \!\pm\! 0.45 \!\times\! 10^{-1}$ & $1.60 \!\pm\! 0.45 \!\times\! 10^{-1}$ & $7.52 \!\pm\! 3.12 \!\times\! 10^{-2}$ & $1.95 \!\pm\! 0.66 \!\times\! 10^{-1}$ \\
    REACT-VM & $5.68 \!\pm\! 0.74 \!\times\! 10^{-3}$ & {\boldmath$8.45 \!\pm\! 0.00 \!\times\! 10^{-3}$} & {\boldmath$9.37 \!\pm\! 0.00 \!\times\! 10^{-3}$} & {\boldmath$8.55 \!\pm\! 0.01 \!\times\! 10^{-2}$} & {\boldmath$9.01 \!\pm\! 0.00 \!\times\! 10^{-1}$} & {\boldmath$9.87 \!\pm\! 0.01 \!\times\! 10^{-2}$} & $5.17 \!\pm\! 0.09 \!\times\! 10^{-3}$ & {\boldmath$1.10 \!\pm\! 0.00 \!\times\! 10^{-1}$} \\
    \bottomrule
  \end{tabular}%
  }
  \vspace{4mm}
\end{table*}

\Cref{fig:vm_training} shows the training of REACT-VM on the three overdamped systems over the three independent runs of \Cref{tab:main_result}: the training loss and the error of the reaction rate against the finite-element reference, against the number of energy evaluations.

\begin{figure}[t]
\centering
\includegraphics[width=\textwidth]{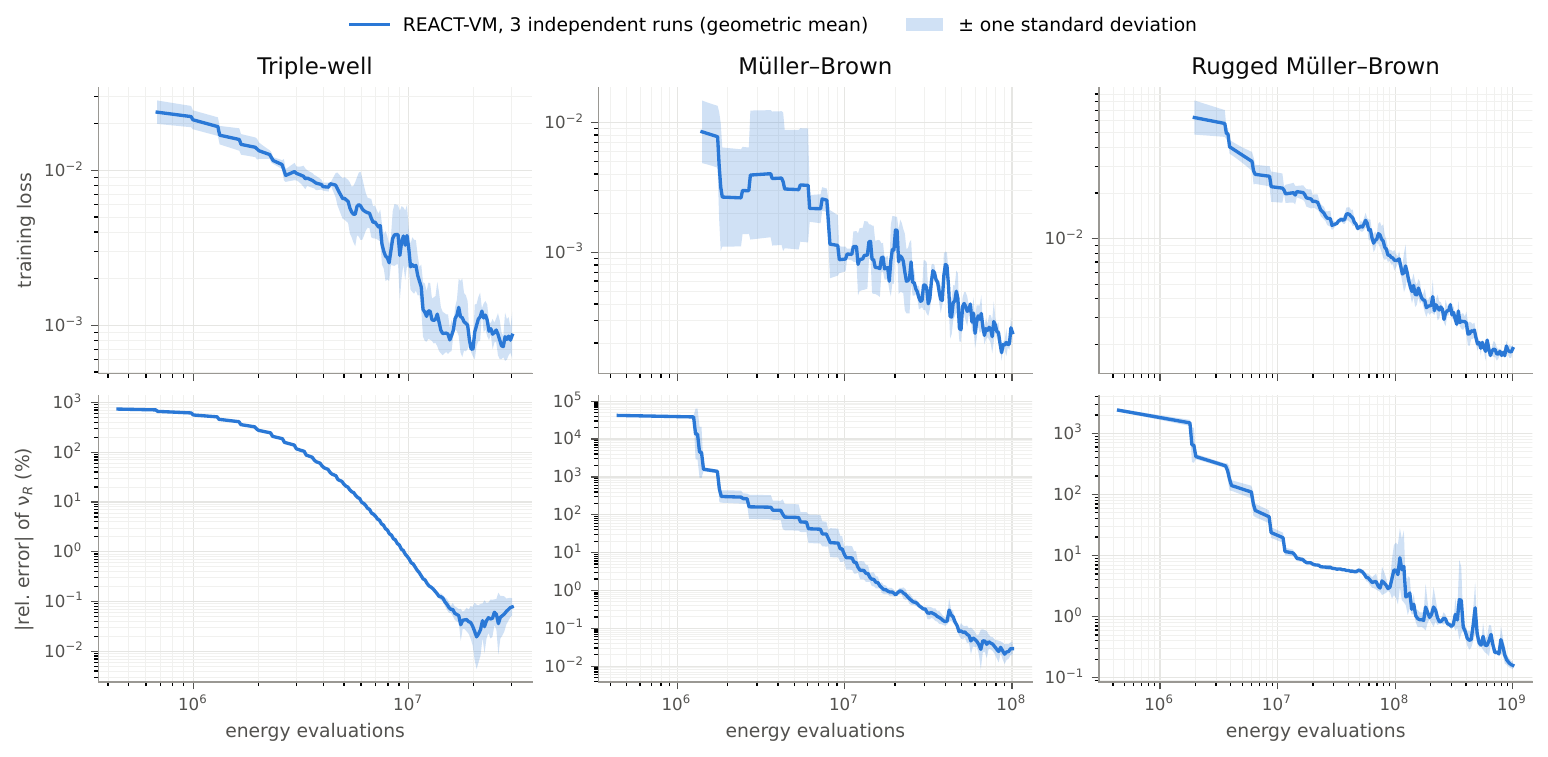}
\caption{Training of REACT-VM on the three overdamped systems (\Cref{tab:main_result}) against the number of energy evaluations. Top: training loss, running mean over ten evaluations of the minibatch loss. Bottom: absolute relative error of $\nu_\mathrm{R}$ against the finite-element reference. Lines: geometric mean over three independent runs; shaded: $\pm$ one standard deviation of the logarithm. The final values of all quantities, including $k_{AB}$ and $k_{BA}$, are given in \Cref{tab:main_result}.}
\label{fig:vm_training}
\end{figure}

We include additional visualizations on the learned committor function and the errors compared with the ground truth committor function.

\begin{figure}
\centering
    \begin{subfigure}{0.9\textwidth}
        \centering
        \includegraphics[width=\linewidth]{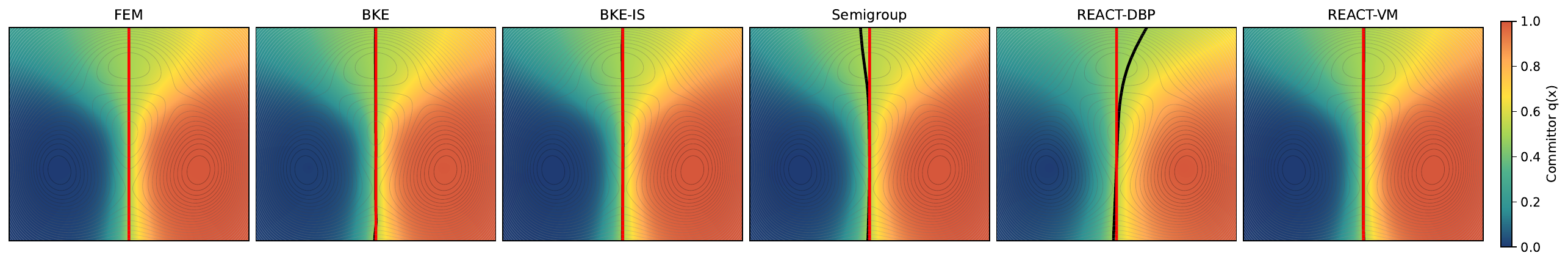}
        \caption{Triple-well potential}
    \end{subfigure}

    \medskip

    \begin{subfigure}{0.9\textwidth}
        \centering
        \includegraphics[width=\linewidth]{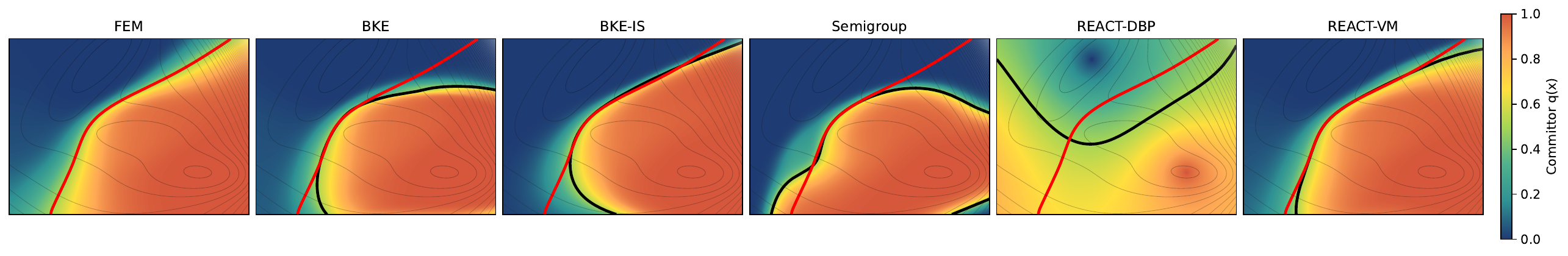}
        \caption{Müller-Brown potential}
    \end{subfigure}

    \medskip

    \begin{subfigure}{0.9\textwidth}
        \centering
        \includegraphics[width=\linewidth]{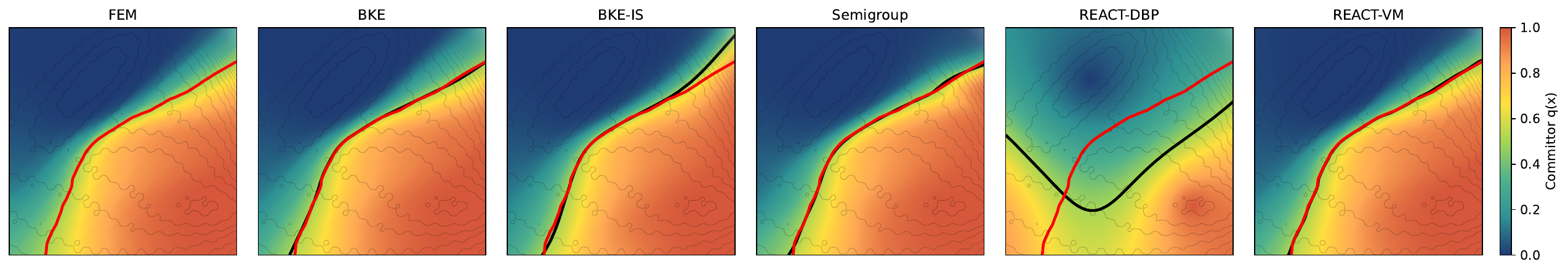}
        \caption{Rugged Müller-Brown potential}
    \end{subfigure}

    \caption{Ground truth and learned committors on the three overdamped systems. Red curves mark the finite-element $q=0.5$ transition surface; black curves mark the $q=0.5$ surface learned by each method. All REACT-VM visualizations use $\kappa=1.0$.}
    \label{fig:viz_committor_all}
\end{figure}

\begin{figure}
\centering
{\small
\setlength{\tabcolsep}{3pt}
\begin{tabular}{@{}c@{\hspace{3pt}}c@{\hspace{3pt}}c@{\hspace{3pt}}c@{\hspace{3pt}}c@{}}
    & \textbf{TPS paths} & \textbf{REACT-VM paths} & \textbf{TPS hitting times} & \textbf{REACT-VM hitting times} \\
    \shortstack{\textbf{Triple-}\\\textbf{well}} &
    \includegraphics[width=0.20\textwidth]{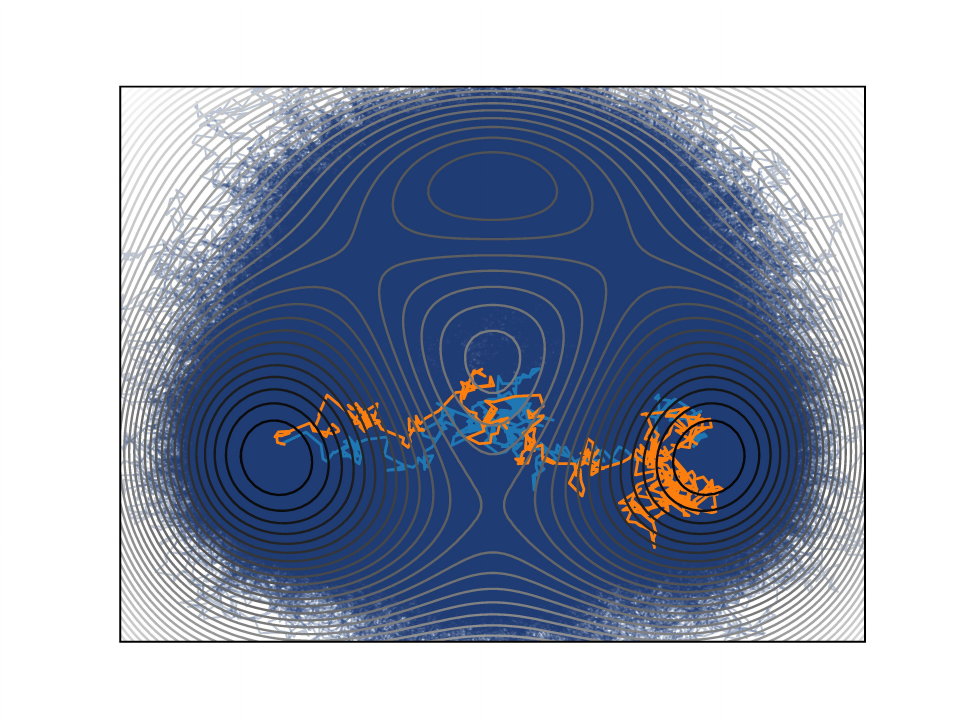} &
    \includegraphics[width=0.20\textwidth]{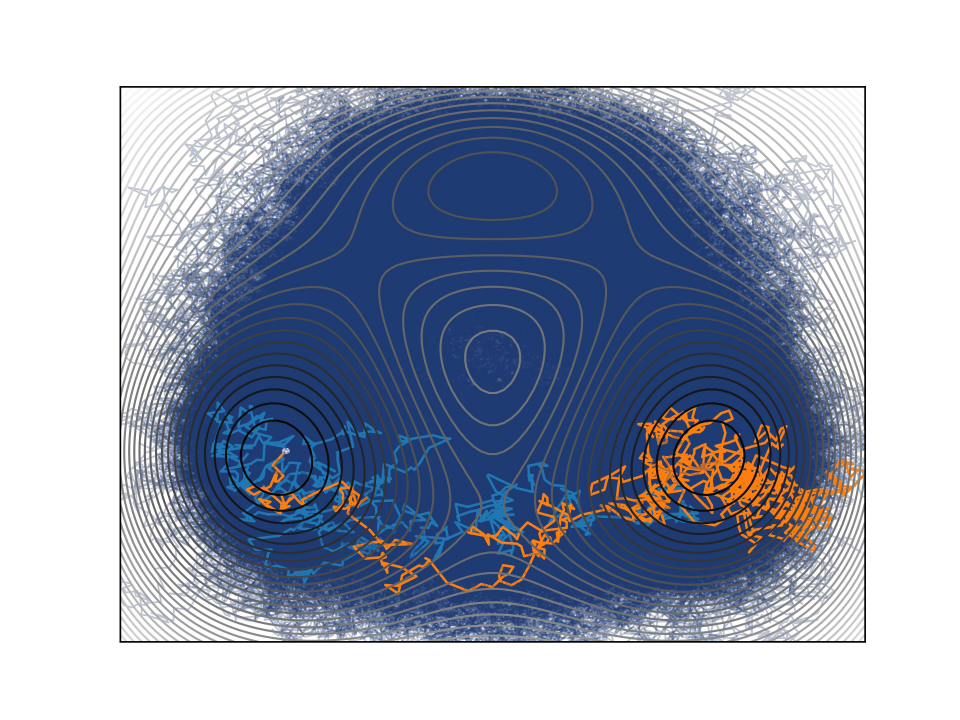} &
    \includegraphics[width=0.20\textwidth]{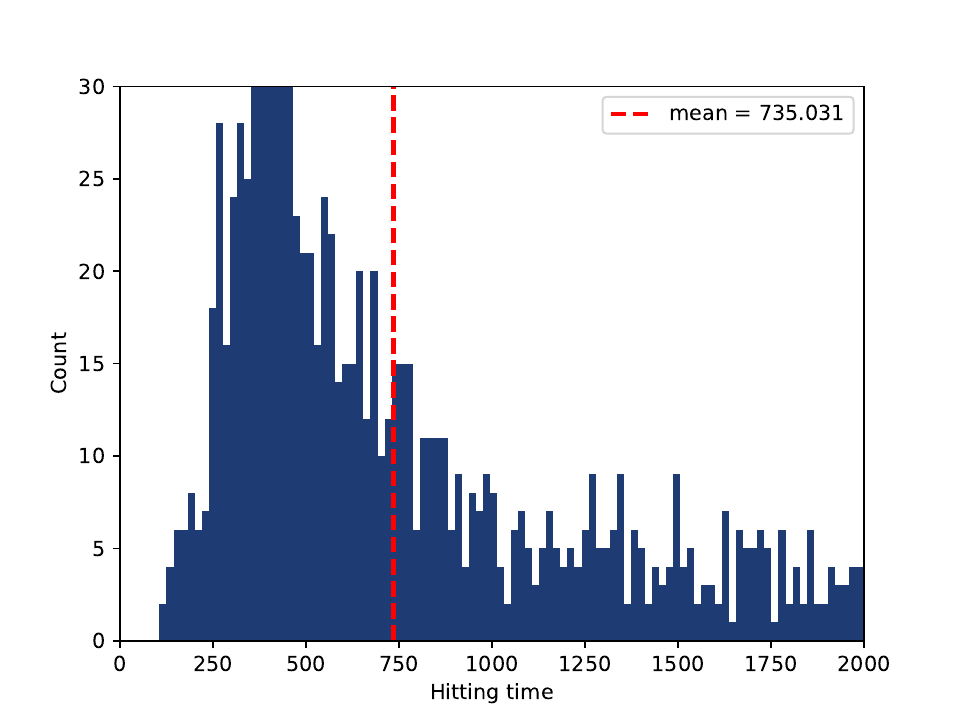} &
    \includegraphics[width=0.20\textwidth]{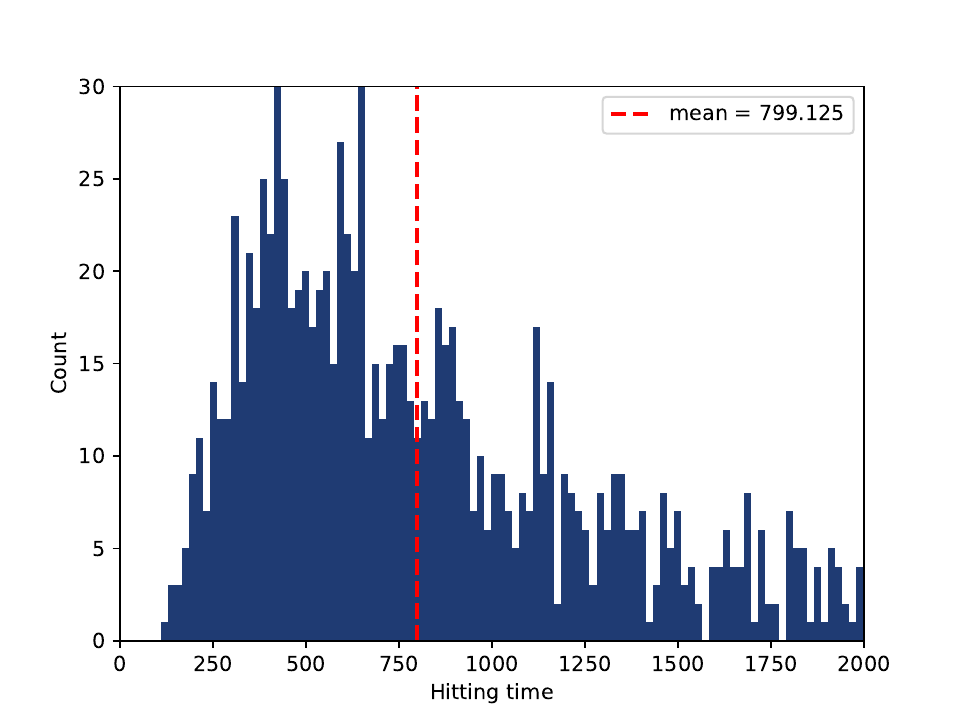} \\
    \shortstack{\textbf{Müller-}\\\textbf{Brown}} &
    \includegraphics[width=0.20\textwidth]{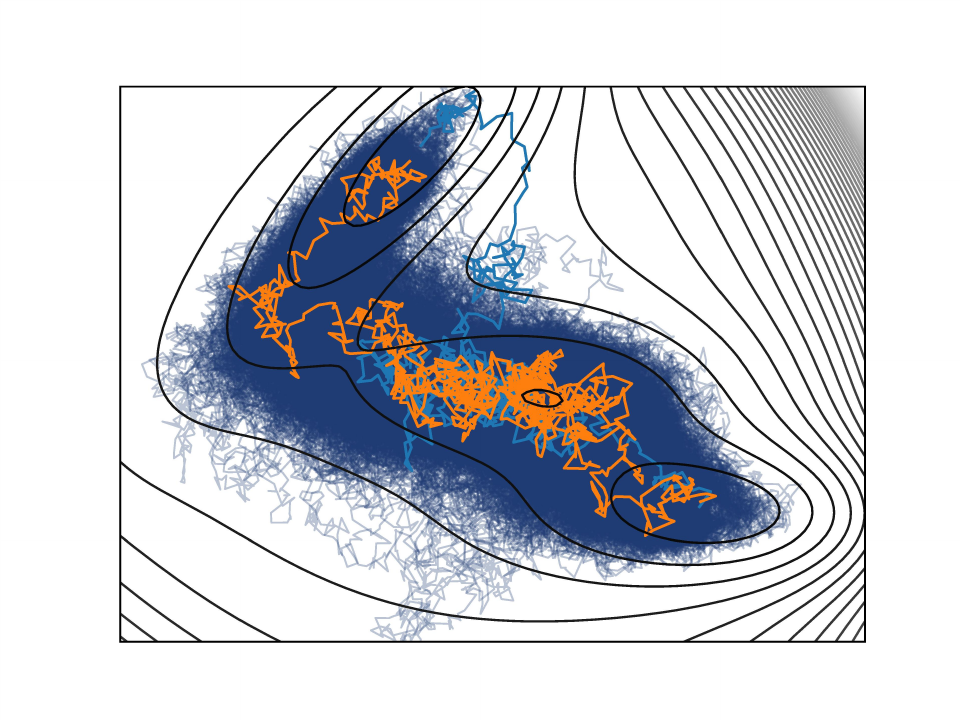} &
    \includegraphics[width=0.20\textwidth]{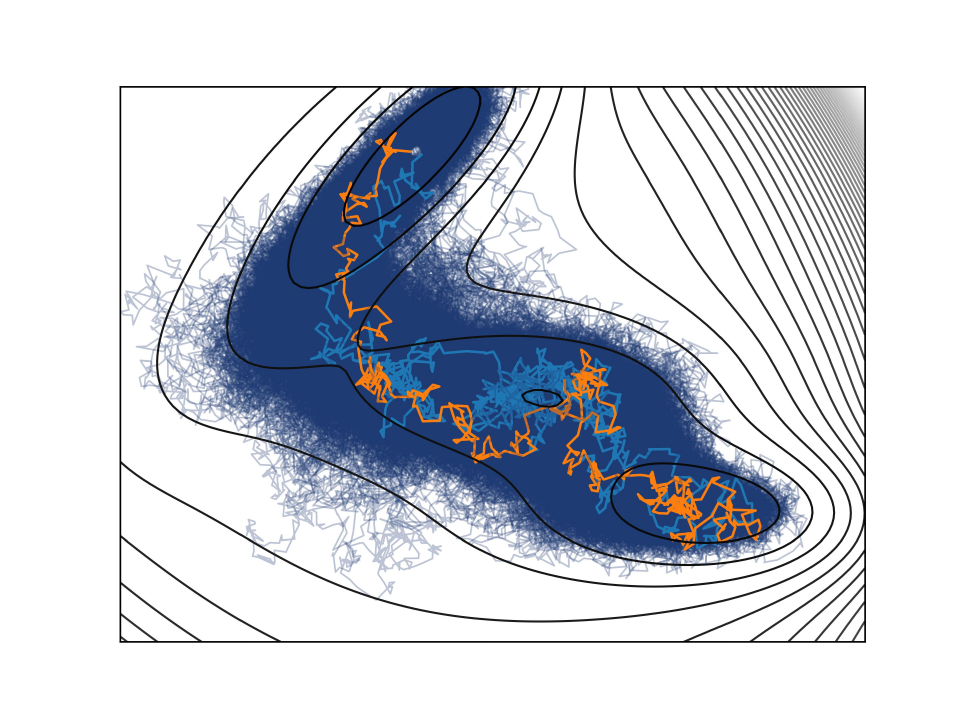} &
    \includegraphics[width=0.20\textwidth]{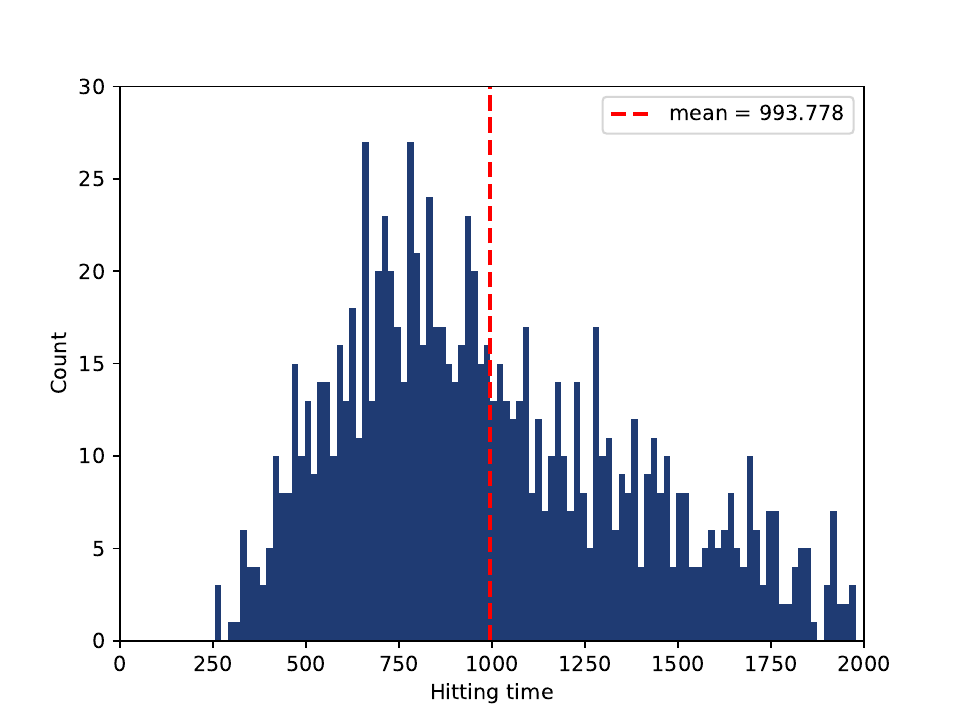} &
    \includegraphics[width=0.20\textwidth]{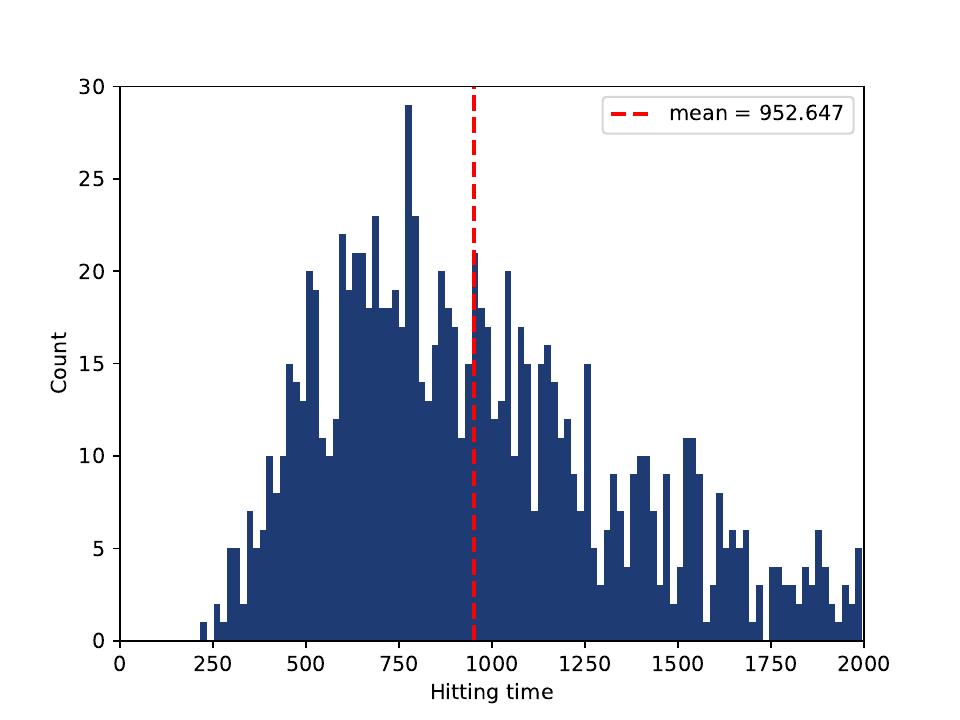} \\
    \shortstack{\textbf{Rugged}\\\textbf{Müller-}\\\textbf{Brown}} &
    \includegraphics[width=0.20\textwidth]{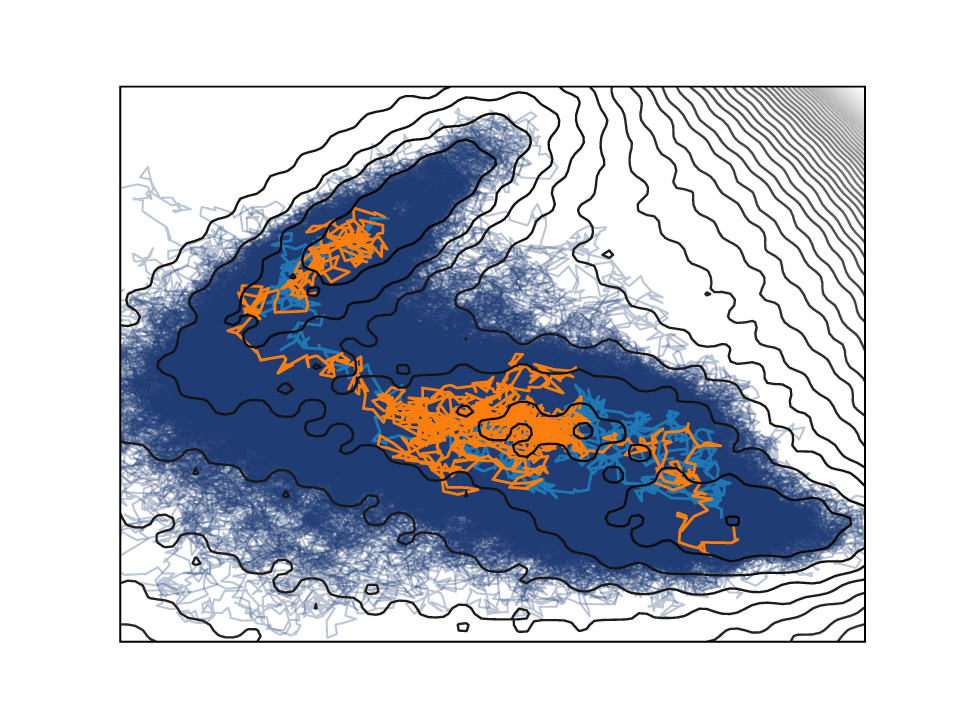} &
    \includegraphics[width=0.20\textwidth]{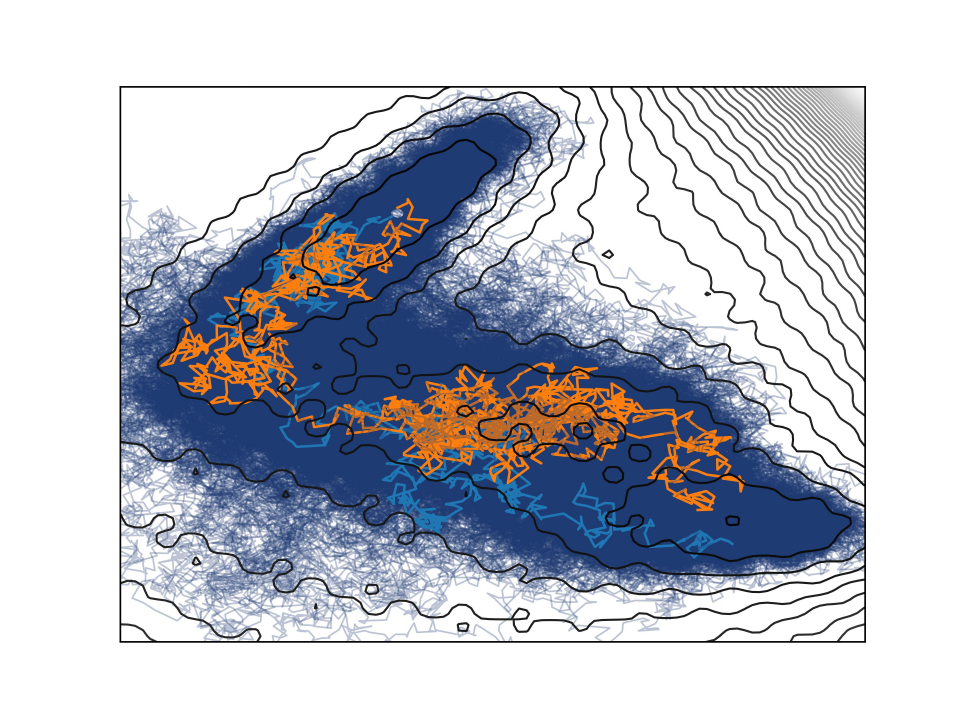} &
    \includegraphics[width=0.20\textwidth]{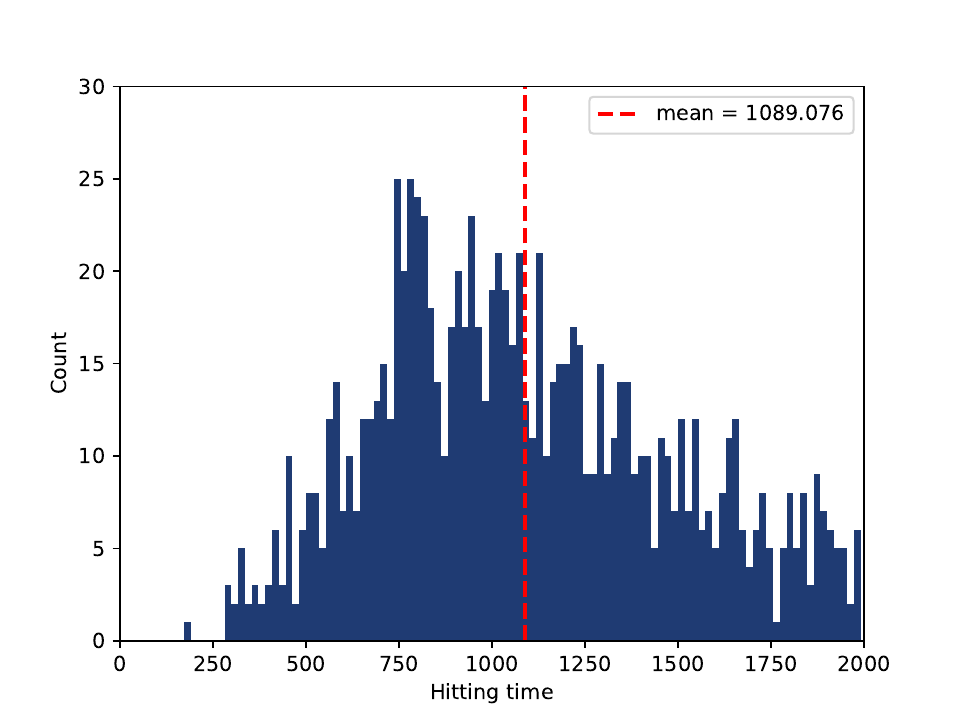} &
    \includegraphics[width=0.20\textwidth]{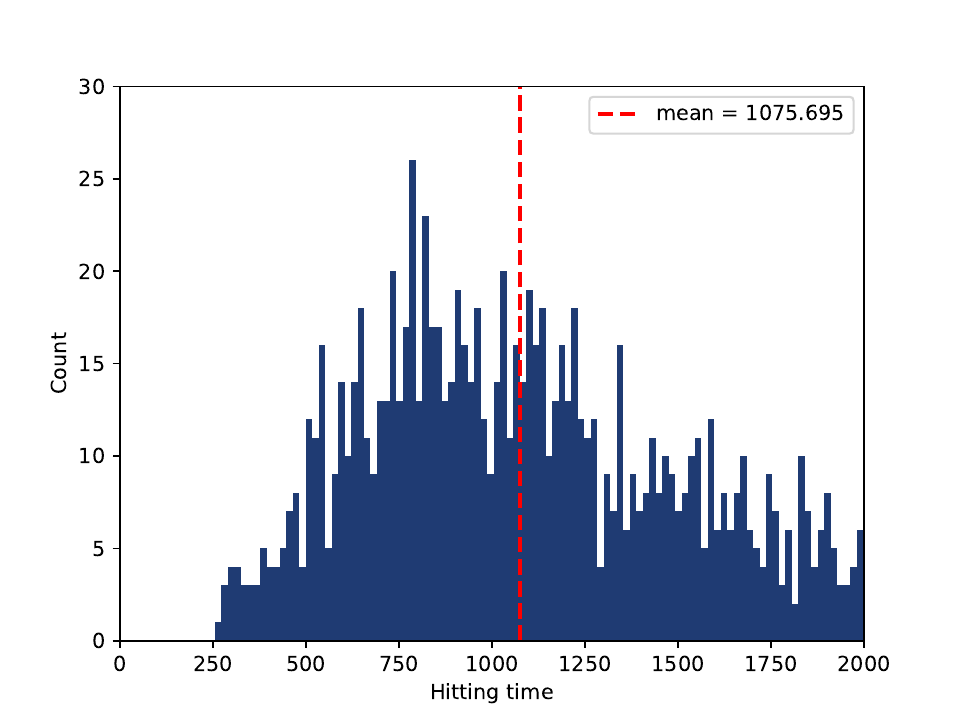}
\end{tabular}}
\caption{Transition-path ensembles and hitting-time distributions for TPS and REACT-VM on the three overdamped potentials. Two representative trajectories are highlighted in each path ensemble. REACT-VM results use $\kappa=1.0$ for all systems.}
\label{fig:viz_tps}
\end{figure}

\begin{figure}
\centering
    \begin{subfigure}{0.31\textwidth}
        \centering
        \includegraphics[width=\linewidth]{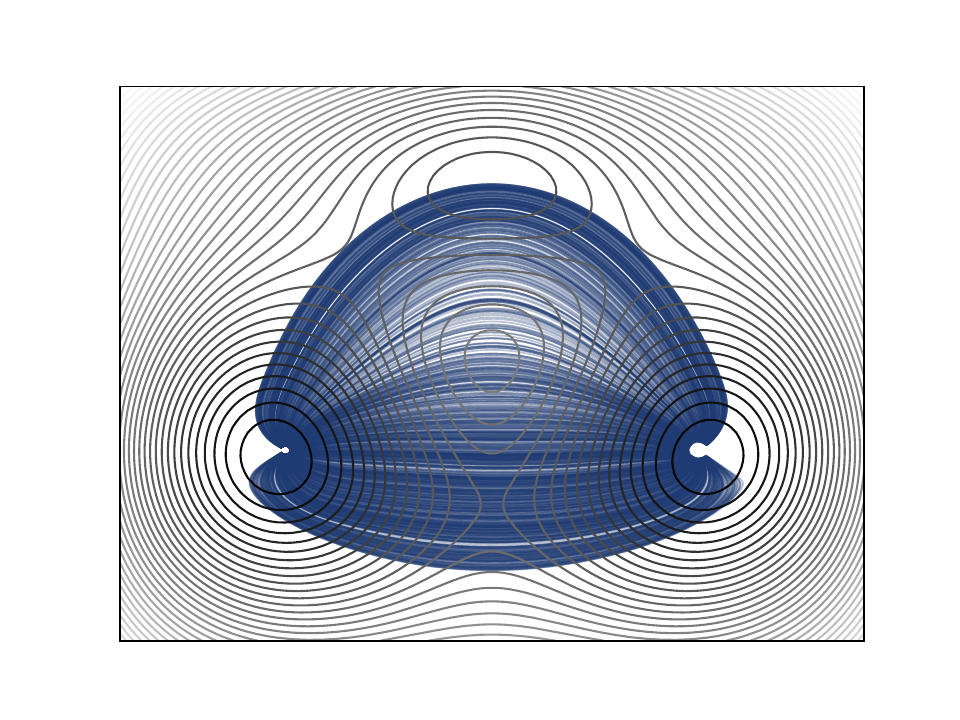}
        \caption{Triple-well potential}
    \end{subfigure}
    \hfill
    \begin{subfigure}{0.31\textwidth}
        \centering
        \includegraphics[width=\linewidth]{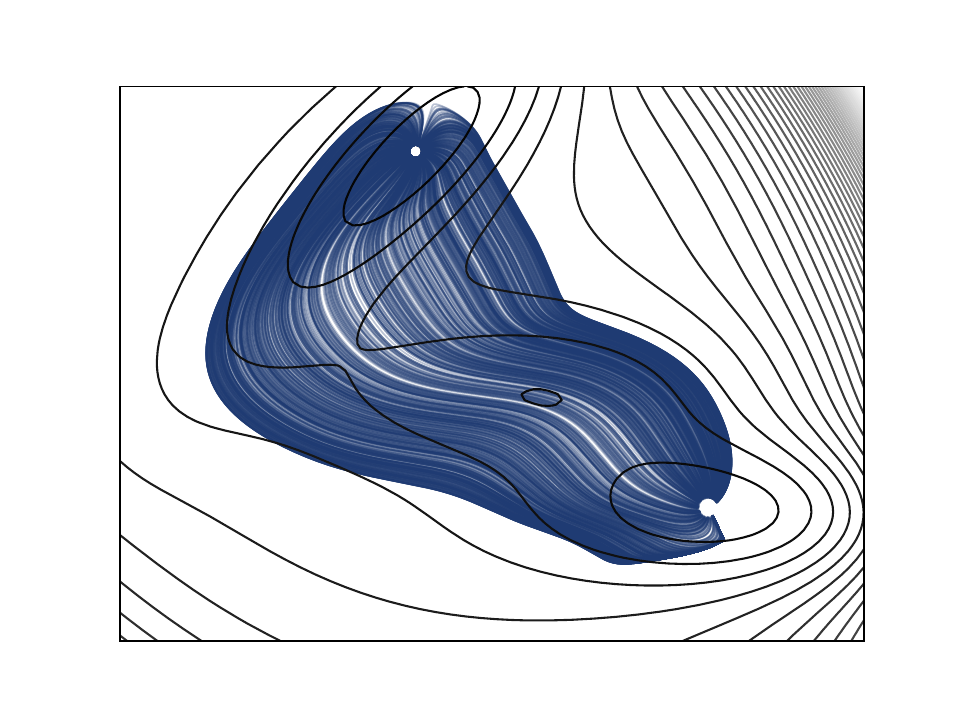}
        \caption{Müller-Brown potential}
    \end{subfigure}
    \hfill
    \begin{subfigure}{0.31\textwidth}
        \centering
        \includegraphics[width=\linewidth]{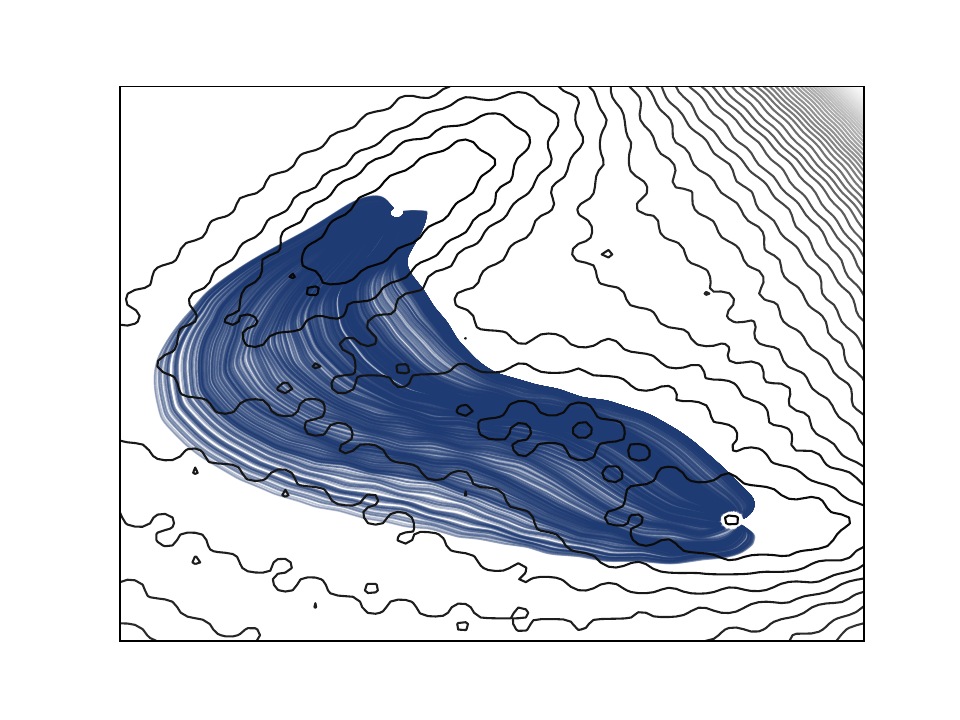}
        \caption{Rugged Müller-Brown}
    \end{subfigure}

    \caption{Reactive-current visualizations (streamlines of the $\kappa=0$ dynamics) on the three overdamped potentials.}
    \label{fig:viz_flux}
\end{figure}

\begin{figure}
\centering

    \begin{subfigure}{0.9\textwidth}
        \centering
        \includegraphics[width=\linewidth]{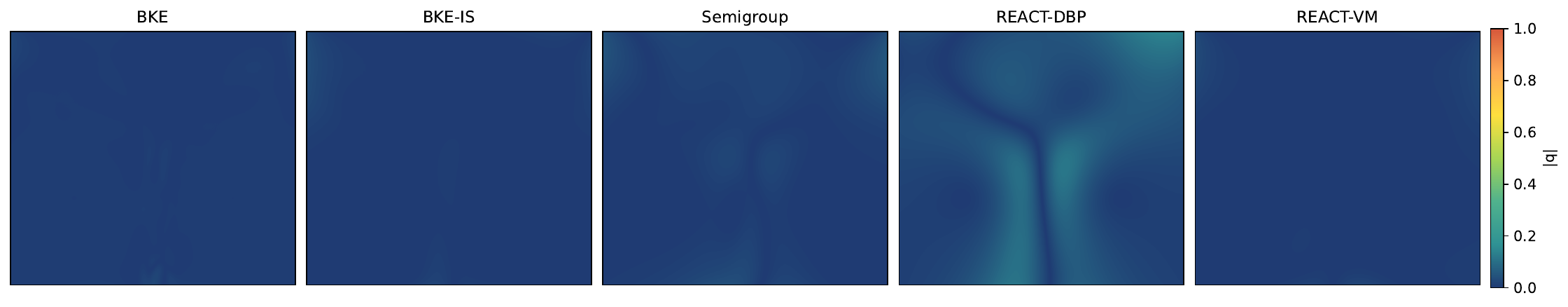}
        \caption{Triple-well potential}
    \end{subfigure}

    \medskip

    \begin{subfigure}{0.9\textwidth}
        \centering
        \includegraphics[width=\linewidth]{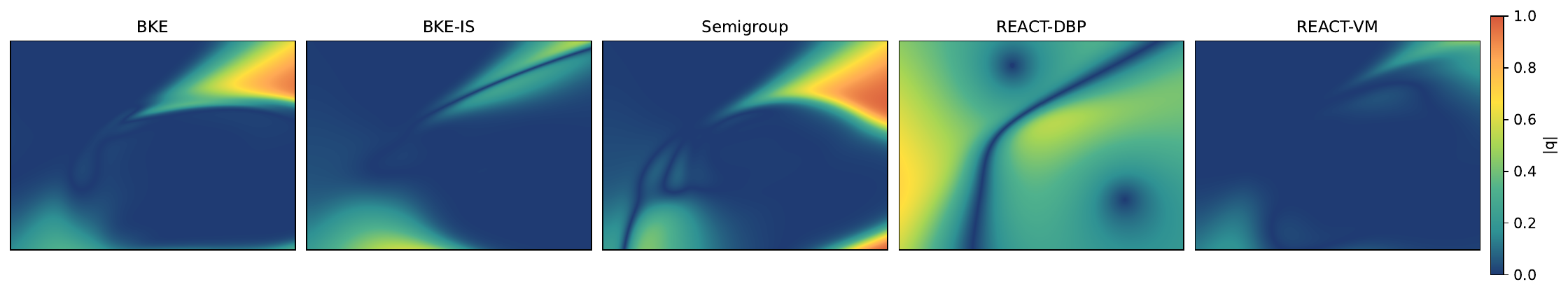}
        \caption{Müller-Brown potential}
    \end{subfigure}

    \medskip

    \begin{subfigure}{0.9\textwidth}
        \centering
        \includegraphics[width=\linewidth]{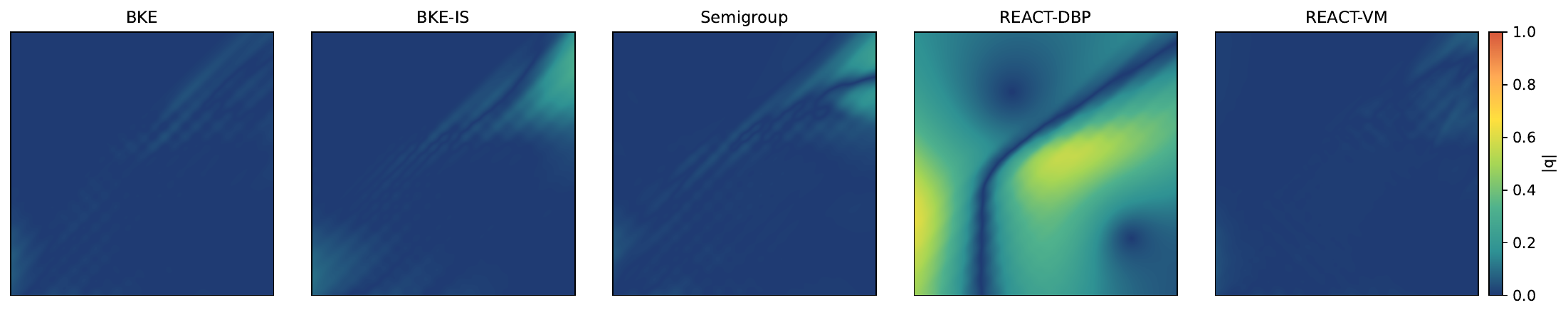}
        \caption{Rugged Müller-Brown potential}
    \end{subfigure}

    \caption{Mean absolute error maps between ground truth (finite-element) and learned committor functions on the three overdamped systems.}
    \label{fig:error_committor}
\end{figure}

\begin{figure}
\centering

    \begin{subfigure}{0.9\textwidth}
        \centering
        \includegraphics[width=\linewidth]{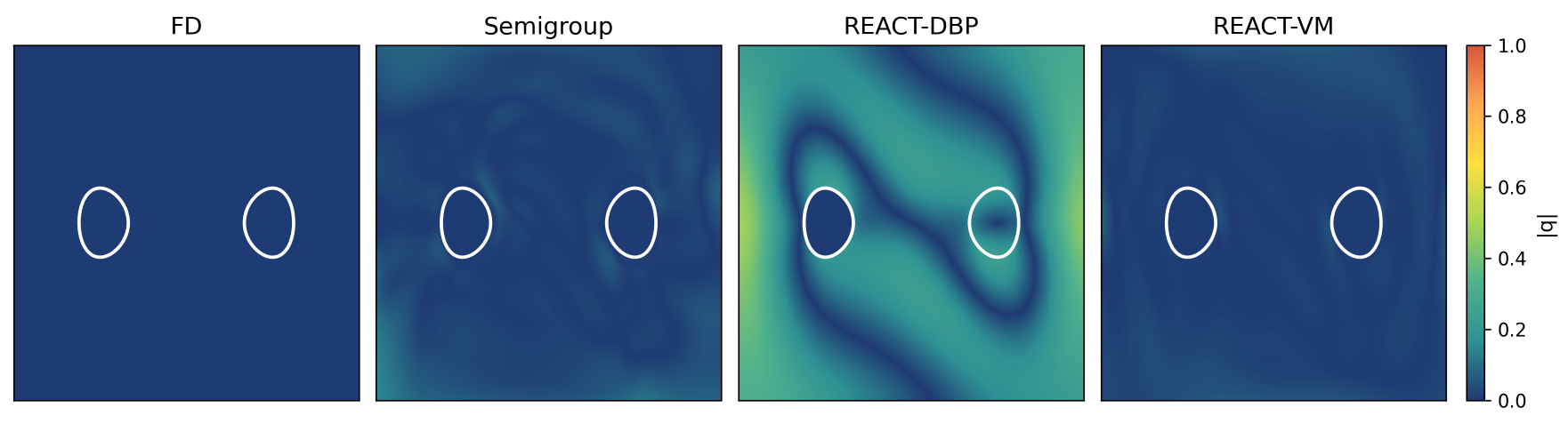}
    \end{subfigure}

    \caption{Mean absolute error map between ground truth (finite-difference) and learned committor functions on the underdamped double-well system for the semigroup baseline, REACT-DBP and REACT-VM.}
    \label{fig:error_committor_underdamped_doublewell}
\end{figure}

\begin{figure}
\centering

    \begin{subfigure}{0.9\textwidth}
        \centering
        \includegraphics[width=\linewidth]{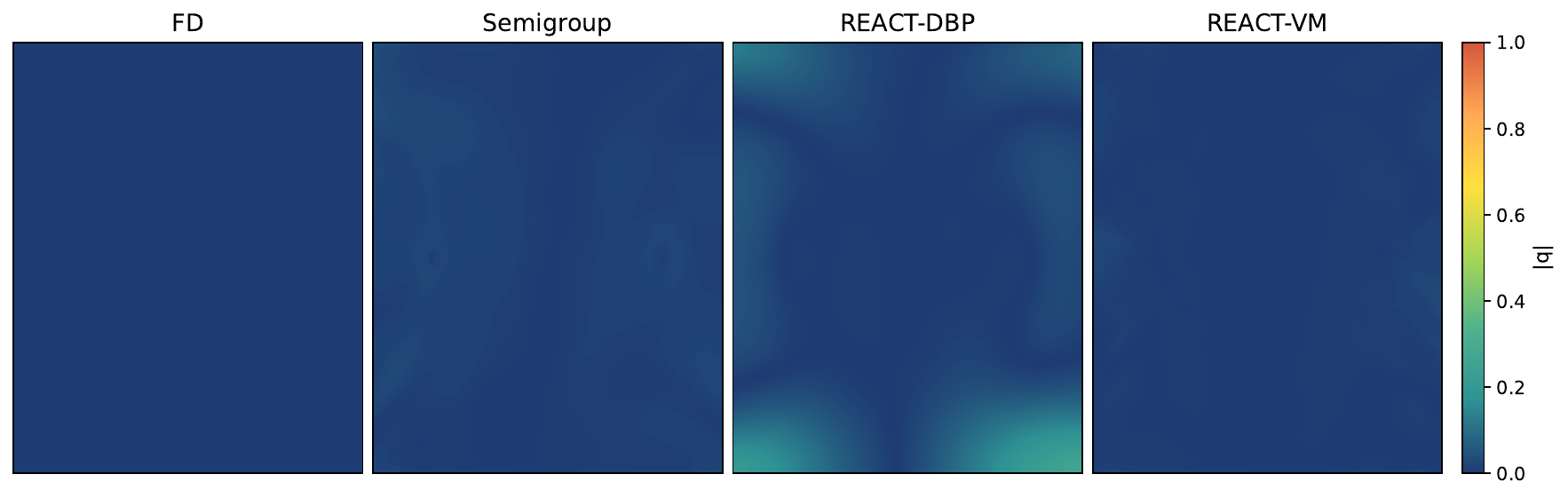}
    \end{subfigure}

    \caption{Mean absolute error map between ground truth (finite-difference) and learned committor functions on the Maier-Stein system for the semigroup baseline, REACT-DBP and REACT-VM.}
    \label{fig:error_committor_maier}
\end{figure}

\begin{figure}
\centering
\captionsetup[subfigure]{font=footnotesize,skip=1pt}
\begin{subfigure}[c]{0.40\textwidth}
    \centering
    \includegraphics[width=\linewidth]{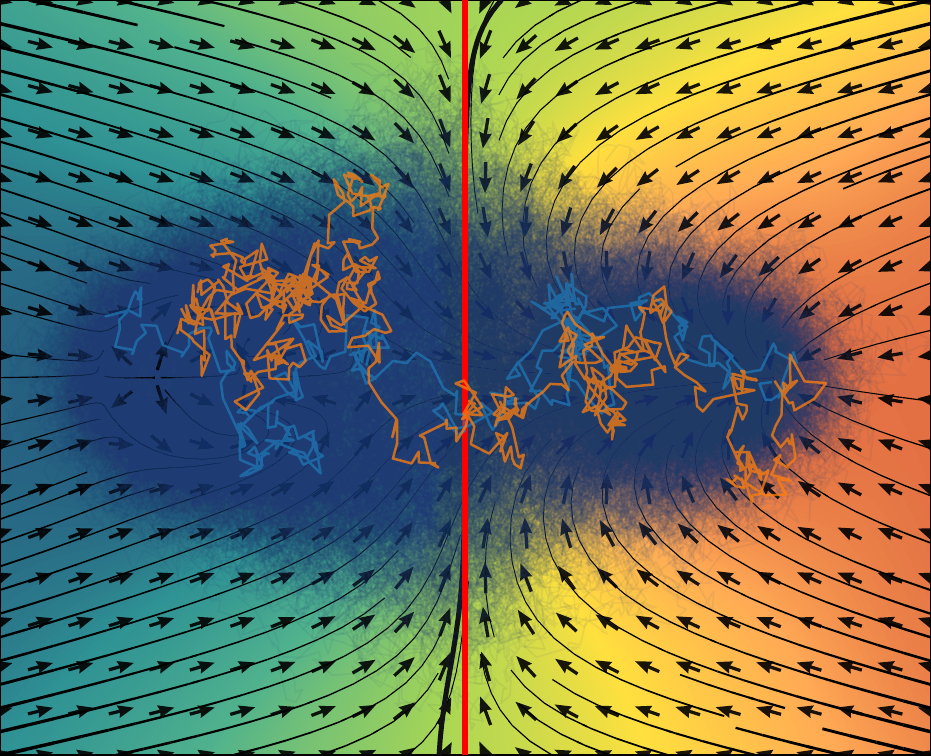}
    \caption{REACT-DBP}
\end{subfigure}
\hspace{6mm}
\begin{subfigure}[c]{0.47\textwidth}
    \centering
    \includegraphics[width=\linewidth]{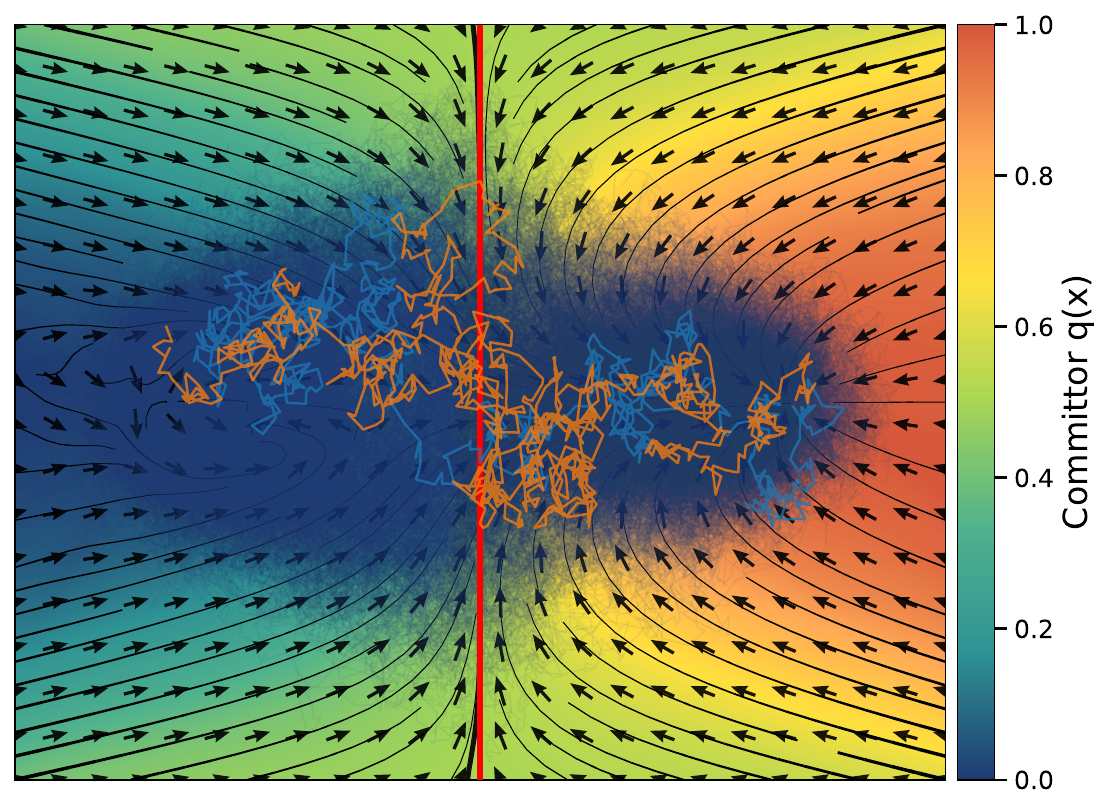}
    \caption{REACT-VM}
\end{subfigure}
\caption{Transition path ensemble sampled by REACT for the Maier-Stein system, overlaid with the reactive force field.}
\label{fig:tp_maier}
\end{figure}

\begin{figure}
\centering

    \begin{subfigure}{0.9\textwidth}
        \centering
        \includegraphics[width=\linewidth]{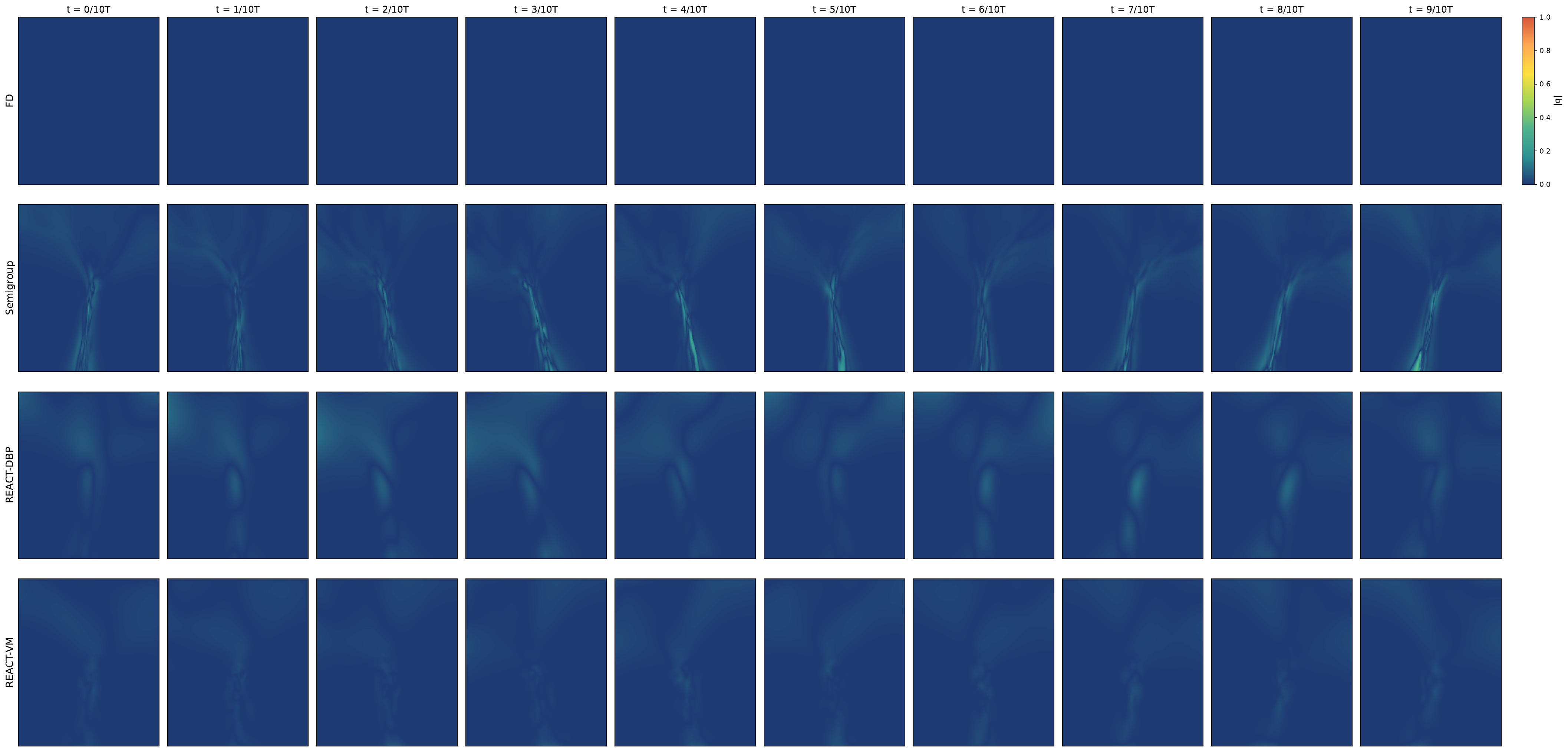}
    \end{subfigure}

    \caption{Mean absolute error map between ground truth (finite-difference) and learned committor functions on the periodically-driven triple-well system for the semigroup baseline, REACT-DBP and REACT-VM (committors averaged over three runs).}
    \label{fig:error_committor_periodic}
\end{figure}

\clearpage
\subsection{Alanine dipeptide}
\label{appendix:aldp}

\paragraph{System and dynamics.} Alanine dipeptide ($22$ atoms) in vacuum is modelled with the CHARMM27 force field~\citep{mackerell1998all,mackerell2004extending}. It follows underdamped Langevin dynamics at $300$\,K with friction $\gamma=0.3\,\mathrm{ps}^{-1}$, integrated by BAOAB~\citep{leimkuhler2013rational} with a time step of $1$\,fs. The sets $A$ (C7$_{\mathrm{eq}}$) and $B$ (C7$_{\mathrm{ax}}$) are the boxes of half-width $10^\circ$ around $(\phi,\psi)=(-150^\circ,170^\circ)$ and $(90^\circ,-50^\circ)$ in the plane of the backbone dihedral angles. Every trajectory is stopped when it enters $A\cup B$ or after $T=0.5$\,ps.

\paragraph{Reference committor and evaluation.} The reference contains $54{,}520$ configurations, selected so that their committors cover $[0,1]$ uniformly. The committor of each configuration is estimated by shooting, as the fraction of $1{,}000$ trajectories, started with momenta drawn from the Maxwell-Boltzmann distribution, that enter $B$ before $A$. Since the momenta are redrawn for every trajectory, this is the momentum-averaged committor, and we compare it with the learned committor averaged over Maxwell-Boltzmann momenta, by the Pearson correlation $r$ and the RMSE.

\paragraph{Setting of the comparison.} We compare REACT-VM with the semigroup baseline in the same setting: network, start points, replay buffer and optimizer, so that only the loss and the trajectories it consumes differ. We measure the cost by the simulated time, the total length of all simulated trajectories ($1$\,ns corresponds to $10^6$ energy evaluations). The network is an MLP with hidden widths $32$-$64$-$128$-$64$-$32$ whose inputs are the distances between the $10$ heavy atoms and their time derivatives, so that it is invariant to rigid motions and depends on the momenta. The output of the network has the form~\eqref{eq:q_parameterization}, with the distances to the centres of $A$ and $B$ in the $(\phi,\psi)$ plane; we set $\xi=e^{-2}$ for REACT-VM to allow more smooth reward near each basin. Each epoch draws new start points: $30\%$ next to $A$ and $B$, and $70\%$ from stored trajectory states, spread evenly over the current committor estimate and then drawn in equal halves with probability proportional to $\hat{q}_{\xi}$ and to $1-\hat{q}_{\xi}$. Both methods use Adam with learning rate $10^{-3}$ and cosine decay, gradient clipping, and an exponential moving average of the parameters, which is the model we evaluate.

\paragraph{REACT-VM.} Each epoch simulates $256$ trajectories, one per start point, with the $\kappa$-rescaled reactive dynamics~\eqref{eq:reactive_SDE_kappa_gen} driven by the current network and steered towards $B$ or $A$ for the two halves of the start points; for underdamped Langevin dynamics, $\kappa$ rescales only the friction and the noise. The noise level is drawn per trajectory, $\kappa\sim\mathcal{U}[1,20]$, motivated by the two-dimensional ablation (\Cref{appendix:ablations}), where different values of $\kappa$ give the best results on different systems; the trajectories that start next to $A$ or $B$ use $\kappa=1$. A trajectory that has reached neither set at the horizon $T$ is continued, with the same target and $\kappa$, for up to three further times, and each continuation is stored as a trajectory of its own. The trajectories are stored with their sampling drift in a replay buffer of $64{,}000$ trajectories, and each of the $400$ updates per epoch applies the VM loss~\eqref{eq:final_vm} to $256$ of them.

\paragraph{Semigroup loss.} The loss is that of~\Cref{appendix:baseline}. Each epoch simulates five unbiased trajectories of at most $0.5$\,ps from each of $102$ start points, about as many integration steps as REACT-VM; the buffer holds $64{,}000$ start points with their trajectories, and the number and size of the updates are those of REACT-VM.

\subsection{Chignolin}
\label{appendix:chignolin}

\paragraph{System and dynamics.} The mini-protein chignolin (CLN025, sequence YYDPETGTWY, $166$ atoms)~\citep{honda2008crystal} is modelled with the CHARMM27 force field~\citep{mackerell1998all,mackerell2004extending} and solvated by $2{,}100$ TIP3P water molecules~\citep{jorgensen1983comparison} and two Na$^+$ ions in a cubic box of side $4.12$\,nm, $6{,}468$ atoms in total. Electrostatics use particle-mesh Ewald summation~\citep{essmann1995smooth} with a real-space cutoff of $0.95$\,nm. The system follows underdamped Langevin dynamics at $340$\,K with friction $\gamma=0.3\,\mathrm{ps}^{-1}$, integrated by BAOAB~\citep{leimkuhler2013rational} with a time step of $1$\,fs; the water molecules are rigid~\citep{miyamoto1992settle} and the protein is unconstrained. The sets $A$ (folded) and $B$ (unfolded) are $\{Q\geq0.8\}$ and $\{Q\leq0.2\}$, where $Q=\frac{1}{74}\sum_{k=1}^{74}\big[1+\exp\big(50\,\mathrm{nm}^{-1}(r_k-1.8\,r_k^0)\big)\big]^{-1}$ is the fraction of native contacts~\citep{best2013native} over $74$ pairs of heavy atoms, with $r_k$ their distance and $r_k^0$ its value in the native structure. Every trajectory is stopped when it enters $A\cup B$ or after $T=0.5$\,ps.

\paragraph{Reference committor and evaluation.} The test set contains $2{,}134$ conformers. Their committors are estimated by shooting: for each conformer the protein is held fixed while the solvent is equilibrated around it ($100$\,ps, and another $10$\,ps before each trial), and each trial starts from the conformer with velocities drawn from the Maxwell-Boltzmann distribution at $340$\,K; its outcome is the first of the two sets that $Q$ enters. The trials use the same force field, integrated with a time step of $2$\,fs and constrained hydrogens. Each conformer was shot $20$ times, and $30$ more times where the outcomes of the first $20$ disagreed ($68{,}403$ trials, $78\,\mu$s of simulation in total). As for alanine dipeptide, the reference is the momentum-averaged committor, and we compare it with the learned committor averaged over Maxwell-Boltzmann momenta ($512$ draws per conformer), set to $0$ and $1$ in $A$ and $B$, by the Pearson correlation $r$ and the RMSE over all $2{,}134$ conformers.

\paragraph{Setting of the comparison.} As for alanine dipeptide, REACT-VM and the semigroup baseline share the network, the start points, the replay schedule and the optimizer, and we measure the cost by the simulated time ($1$\,ns corresponds to $10^6$ energy evaluations). The network is an MLP with hidden widths $32$-$64$-$128$-$64$-$32$; its inputs are the $74$ native-contact distances, the $45$ further distances between C$_\alpha$ atoms and the time derivatives of all $119$ distances, so that it is invariant to rigid motions and depends on the momenta. All these distances are between heavy atoms of the protein, so the network sees neither the solvent nor the hydrogen atoms.  The reference committors are defined consistently, as the solvent is re-equilibrated around each conformer before the trials. The output of the network has the form~\eqref{eq:q_parameterization}, with the distances $|Q-1|$ and $|Q|$ to the centres of $A$ and $B$ along $Q$, and $\xi=e^{-5}$. Each round draws new start points: $30\%$ are reference frames just outside $A$ and $B$ ($0.7\leq Q<0.8$ and $0.2<Q\leq0.35$, half each, with fresh momenta), and $70\%$ are stored trajectory states, spread evenly over the current committor estimate and then drawn in equal halves with probability proportional to $\hat{q}_{\xi}$ and to $1-\hat{q}_{\xi}$. Both methods use Adam with learning rate $10^{-3}$ and cosine decay over the budget of simulated time, gradient clipping, and an exponential moving average of the parameters, which is the model we evaluate.

\paragraph{REACT-VM.} Each round simulates $1{,}024$ trajectories, one per start point, with the $\kappa$-rescaled reactive dynamics~\eqref{eq:reactive_SDE_kappa_gen} driven by the current network and steered towards $B$ or $A$ for the two halves of the start points; the trajectories that start from the reference frames use $\kappa=1$ and the others $\kappa\sim\mathcal{U}[1,20]$. During the first $40$\,ns of simulated time the sampling dynamics also include the force $c\,\nabla_r\log\hat{q}_{\mathrm{target}}$ on the protein, capped at $500\,k_BT/\mathrm{nm}$ per atom, with $c=100\,k_BT$ until $10$\,ns and then decreasing along a cosine to $0$ at $40$\,ns. The $16$ trajectories of each round that enter $A$ or $B$ with the largest difference between $\hat{q}$ at their start and the value reached are also copied into an archive of $512$ trajectories, which provides $12.5\%$ of every batch.
The trajectories are stored in a replay buffer of $2{,}048$ trajectories, of which each round replaces half, and each of the $4{,}000$ updates per round applies the VM loss~\eqref{eq:final_vm} to $512$ of them. The budget is $2\times10^8$ energy evaluations ($200$\,ns).

\paragraph{Semigroup loss.} The loss is that of~\Cref{appendix:baseline}. Each round simulates five unbiased trajectories of at most $0.5$\,ps from each of $4{,}090$ start points drawn by the same rule; the buffer holds $8{,}180$ start points with their trajectories, and each of the $4{,}000$ updates per round uses $2{,}048$ of them. The budget is $10^9$ energy evaluations ($1\,\mu$s).

\newpage
\bibliography{references}

\begin{thebibliography}{75}
\providecommand{\natexlab}[1]{#1}
\providecommand{\url}[1]{\texttt{#1}}
\expandafter\ifx\csname urlstyle\endcsname\relax
  \providecommand{\doi}[1]{doi: #1}\else
  \providecommand{\doi}{doi: \begingroup \urlstyle{rm}\Url}\fi

\bibitem[Arrhenius(1889)]{arrhenius1889dissociationswarme}
Svante Arrhenius.
\newblock {\"U}ber die dissociationsw{\"a}rme und den einfluss der temperatur
  auf den dissociationsgrad der elektrolyte.
\newblock \emph{Zeitschrift f{\"u}r physikalische Chemie}, 4\penalty0
  (1):\penalty0 96--116, 1889.

\bibitem[Baratta et~al.(2024)Baratta, Dean, Dokken, Habera, HALE, Richardson,
  Rognes, Scroggs, Sime, and Wells]{baratta2023dolfinx}
Igor~A Baratta, Joseph~P Dean, J{\o}rgen~S Dokken, Michal Habera, Jack HALE,
  Chris~N Richardson, Marie~E Rognes, Matthew~W Scroggs, Nathan Sime, and
  Garth~N Wells.
\newblock Dolfinx: the next generation fenics problem solving environment,
  2024.
\newblock URL \url{https://doi.org/10.5281/zenodo.10447665}.

\bibitem[Best et~al.(2013)Best, Hummer, and Eaton]{best2013native}
Robert~B. Best, Gerhard Hummer, and William~A. Eaton.
\newblock Native contacts determine protein folding mechanisms in atomistic
  simulations.
\newblock \emph{Proceedings of the National Academy of Sciences}, 110\penalty0
  (44):\penalty0 17874--17879, 2013.

\bibitem[Bolhuis et~al.(2000)Bolhuis, Dellago, and
  Chandler]{bolhuis2000reaction}
Peter~G Bolhuis, Christoph Dellago, and David Chandler.
\newblock Reaction coordinates of biomolecular isomerization.
\newblock \emph{Proceedings of the National Academy of Sciences}, 97\penalty0
  (11):\penalty0 5877--5882, 2000.

\bibitem[Bolhuis et~al.(2002)Bolhuis, Chandler, Dellago, and
  Geissler]{bolhuis2002transition}
Peter~G Bolhuis, David Chandler, Christoph Dellago, and Phillip~L Geissler.
\newblock Transition path sampling: Throwing ropes over rough mountain passes,
  in the dark.
\newblock \emph{Annual review of physical chemistry}, 53\penalty0 (1):\penalty0
  291--318, 2002.

\bibitem[Chen et~al.(2023)Chen, Hoskins, Khoo, and Lindsey]{chen2023committor}
Yian Chen, Jeremy Hoskins, Yuehaw Khoo, and Michael Lindsey.
\newblock Committor functions via tensor networks.
\newblock \emph{Journal of Computational Physics}, 472:\penalty0 111646, 2023.

\bibitem[Dellago et~al.(1998)Dellago, Bolhuis, Csajka, and
  Chandler]{dellago1998transition}
Christoph Dellago, Peter~G Bolhuis, F{\'e}lix~S Csajka, and David Chandler.
\newblock Transition path sampling and the calculation of rate constants.
\newblock \emph{The Journal of chemical physics}, 108\penalty0 (5):\penalty0
  1964--1977, 1998.

\bibitem[Domingo-Enrich(2024)]{domingo2024taxonomy}
Carles Domingo-Enrich.
\newblock A taxonomy of loss functions for stochastic optimal control.
\newblock \emph{arXiv preprint arXiv:2410.00345}, 2024.

\bibitem[Domingo-Enrich et~al.(2024)Domingo-Enrich, Han, Amos, Bruna, and
  Chen]{domingo2024stochastic}
Carles Domingo-Enrich, Jiequn Han, Brandon Amos, Joan Bruna, and Ricky~TQ Chen.
\newblock Stochastic optimal control matching.
\newblock \emph{Advances in Neural Information Processing Systems},
  37:\penalty0 112459--112504, 2024.

\bibitem[Domingo-Enrich et~al.(2025)Domingo-Enrich, Drozdzal, Karrer, and
  Chen]{domingoadjoint2025}
Carles Domingo-Enrich, Michal Drozdzal, Brian Karrer, and Ricky~TQ Chen.
\newblock Adjoint matching: Fine-tuning flow and diffusion generative models
  with memoryless stochastic optimal control.
\newblock In \emph{The Thirteenth International Conference on Learning
  Representations}, 2025.

\bibitem[Du et~al.(1998)Du, Pande, Grosberg, Tanaka, and
  Shakhnovich]{du1998transition}
R.~Du, V.~S. Pande, A.~Y. Grosberg, T.~Tanaka, and E.~I. Shakhnovich.
\newblock On the transition coordinate for protein folding.
\newblock \emph{J. Chem. Phys.}, 108:\penalty0 334--350, 1998.

\bibitem[Du et~al.(2024)Du, Plainer, Brekelmans, Duan, Noe, Gomes,
  Aspuru-Guzik, and Neklyudov]{du2024doob}
Yuanqi Du, Michael Plainer, Rob Brekelmans, Chenru Duan, Frank Noe, Carla~P
  Gomes, Alan Aspuru-Guzik, and Kirill Neklyudov.
\newblock Doob's lagrangian: A sample-efficient variational approach to
  transition path sampling.
\newblock \emph{Advances in Neural Information Processing Systems},
  37:\penalty0 65791--65822, 2024.

\bibitem[E et~al.(2017)E, Han, and Jentzen]{weinan2017deep}
Weinan E, Jiequn Han, and Arnulf Jentzen.
\newblock Deep learning-based numerical methods for high-dimensional parabolic
  partial differential equations and backward stochastic differential
  equations.
\newblock \emph{Communications in Mathematics and Statistics}, 5\penalty0
  (4):\penalty0 349--380, 2017.

\bibitem[Essmann et~al.(1995)Essmann, Perera, Berkowitz, Darden, Lee, and
  Pedersen]{essmann1995smooth}
Ulrich Essmann, Lalith Perera, Max~L. Berkowitz, Tom Darden, Hsing Lee, and
  Lee~G. Pedersen.
\newblock A smooth particle mesh {E}wald method.
\newblock \emph{The Journal of Chemical Physics}, 103\penalty0 (19):\penalty0
  8577--8593, 1995.

\bibitem[Eyring(1935)]{eyring1935activated}
Henry Eyring.
\newblock The activated complex in chemical reactions.
\newblock \emph{The Journal of chemical physics}, 3\penalty0 (2):\penalty0
  107--115, 1935.

\bibitem[Haarnoja et~al.(2018)Haarnoja, Zhou, Abbeel, and
  Levine]{haarnoja2018soft}
Tuomas Haarnoja, Aurick Zhou, Pieter Abbeel, and Sergey Levine.
\newblock Soft actor-critic: Off-policy maximum entropy deep reinforcement
  learning with a stochastic actor.
\newblock \emph{arXiv preprint arXiv:1801.01290}, 2018.

\bibitem[Han et~al.(2018)Han, Jentzen, and E]{han2018solving}
Jiequn Han, Arnulf Jentzen, and W.~E.
\newblock Solving high-dimensional partial differential equations using deep
  learning.
\newblock \emph{Proceedings of the National Academy of Sciences}, 115\penalty0
  (34):\penalty0 8505--8510, 2018.

\bibitem[Hartmann and Sch{\"u}tte(2012)]{hartmann2012efficient}
Carsten Hartmann and Christof Sch{\"u}tte.
\newblock Efficient rare event simulation by optimal nonequilibrium forcing.
\newblock \emph{Journal of Statistical Mechanics: Theory and Experiment},
  2012\penalty0 (11):\penalty0 P11004, 2012.

\bibitem[Hartmann et~al.(2013)Hartmann, Banisch, Sarich, Badowski, and
  Sch{\"u}tte]{hartmann2013characterization}
Carsten Hartmann, Ralf Banisch, Marco Sarich, Tomasz Badowski, and Christof
  Sch{\"u}tte.
\newblock Characterization of rare events in molecular dynamics.
\newblock \emph{Entropy}, 16\penalty0 (1):\penalty0 350--376, 2013.

\bibitem[Hartmann et~al.(2019)Hartmann, Kebiri, Neureither, and
  Richter]{hartmann2019variational}
Carsten Hartmann, Omar Kebiri, Lara Neureither, and Lorenz Richter.
\newblock Variational approach to rare event simulation using least-squares
  regression.
\newblock \emph{Chaos: An Interdisciplinary Journal of Nonlinear Science},
  29\penalty0 (6):\penalty0 063107, 06 2019.

\bibitem[Hartmann et~al.(2026)Hartmann, Jöster, Schütte, Sikorski, and
  Weber]{hartmann2026importance}
Carsten Hartmann, Annika Jöster, Christof Schütte, Alexander Sikorski, and
  Marcus Weber.
\newblock Importance sampling of unbounded random stopping times: computing
  committor functions and exit rates without reweighting.
\newblock \emph{arXiv preprint arXiv:2601.01489}, 2026.

\bibitem[Hasyim et~al.(2022)Hasyim, Batton, and
  Mandadapu]{hasyim2022supervised}
Muhammad~R Hasyim, Clay~H Batton, and Kranthi~K Mandadapu.
\newblock Supervised learning and the finite-temperature string method for
  computing committor functions and reaction rates.
\newblock \emph{The Journal of Chemical Physics}, 157\penalty0 (18), 2022.

\bibitem[Helfmann et~al.(2020)Helfmann, Ribera~Borrell, Sch{\"u}tte, and
  Koltai]{helfmann2020extending}
Luzie Helfmann, Enric Ribera~Borrell, Christof Sch{\"u}tte, and P{\'e}ter
  Koltai.
\newblock Extending transition path theory: Periodically driven and finite-time
  dynamics.
\newblock \emph{Journal of nonlinear science}, 30\penalty0 (6):\penalty0
  3321--3366, 2020.

\bibitem[Holdijk et~al.(2023)Holdijk, Du, Hooft, Jaini, Ensing, and
  Welling]{holdijk2024stochastic}
Lars Holdijk, Yuanqi Du, Ferry Hooft, Priyank Jaini, Berend Ensing, and Max
  Welling.
\newblock Stochastic optimal control for collective variable free sampling of
  molecular transition paths.
\newblock \emph{Advances in Neural Information Processing Systems}, 36, 2023.

\bibitem[Honda et~al.(2008)Honda, Akiba, Kato, Sawada, Sekijima, Ishimura,
  Ooishi, Watanabe, Odahara, and Harata]{honda2008crystal}
Shinya Honda, Toshihiko Akiba, Yusuke~S. Kato, Yoshito Sawada, Masakazu
  Sekijima, Miyuki Ishimura, Ayako Ooishi, Hideki Watanabe, Takayuki Odahara,
  and Kazuaki Harata.
\newblock Crystal structure of a ten-amino acid protein.
\newblock \emph{Journal of the American Chemical Society}, 130\penalty0
  (46):\penalty0 15327--15331, 2008.

\bibitem[Hua et~al.(2024)Hua, Lauri{\`e}re, and
  Vanden-Eijnden]{hua2024efficient}
Mengjian Hua, Mathieu Lauri{\`e}re, and Eric Vanden-Eijnden.
\newblock An efficient on-policy deep learning framework for stochastic optimal
  control.
\newblock \emph{arXiv preprint arXiv:2410.05163}, 2024.

\bibitem[Hummer(2004)]{hummer2004transition}
Gerhard Hummer.
\newblock From transition paths to transition states and rate coefficients.
\newblock \emph{The Journal of chemical physics}, 120\penalty0 (2):\penalty0
  516--523, 2004.

\bibitem[Ikeda and Watanabe(2014)]{ikeda2014stochastic}
N.~Ikeda and S.~Watanabe.
\newblock \emph{Stochastic Differential Equations and Diffusion Processes}.
\newblock North-Holland Mathematical Library. North Holland, 2014.

\bibitem[Jorgensen et~al.(1983)Jorgensen, Chandrasekhar, Madura, Impey, and
  Klein]{jorgensen1983comparison}
William~L. Jorgensen, Jayaraman Chandrasekhar, Jeffry~D. Madura, Roger~W.
  Impey, and Michael~L. Klein.
\newblock Comparison of simple potential functions for simulating liquid water.
\newblock \emph{The Journal of Chemical Physics}, 79\penalty0 (2):\penalty0
  926--935, 1983.

\bibitem[Jung et~al.(2019)Jung, Covino, and Hummer]{jung2019artificial}
Hendrik Jung, Roberto Covino, and Gerhard Hummer.
\newblock Artificial intelligence assists discovery of reaction coordinates and
  mechanisms from molecular dynamics simulations.
\newblock \emph{arXiv preprint arXiv:1901.04595}, 2019.

\bibitem[Jung et~al.(2023)Jung, Covino, Arjun, Leitold, Dellago, Bolhuis, and
  Hummer]{jung2023machine}
Hendrik Jung, Roberto Covino, A~Arjun, Christian Leitold, Christoph Dellago,
  Peter~G Bolhuis, and Gerhard Hummer.
\newblock Machine-guided path sampling to discover mechanisms of molecular
  self-organization.
\newblock \emph{Nature Computational Science}, 3\penalty0 (4):\penalty0
  334--345, 2023.

\bibitem[Kang et~al.(2024)Kang, Trizio, and Parrinello]{kang2024computing}
Peilin Kang, Enrico Trizio, and Michele Parrinello.
\newblock Computing the committor with the committor to study the transition
  state ensemble.
\newblock \emph{Nature Computational Science}, 4\penalty0 (6):\penalty0
  451--460, 2024.

\bibitem[Kappen and Ruiz(2016)]{kappen2016adaptive}
Hilbert~Johan Kappen and Hans~Christian Ruiz.
\newblock Adaptive importance sampling for control and inference.
\newblock \emph{Journal of Statistical Physics}, 162\penalty0 (5):\penalty0
  1244--1266, 2016.

\bibitem[Khoo et~al.(2019)Khoo, Lu, and Ying]{khoo2019solving}
Yuehaw Khoo, Jianfeng Lu, and Lexing Ying.
\newblock Solving for high-dimensional committor functions using artificial
  neural networks.
\newblock \emph{Research in the Mathematical Sciences}, 6:\penalty0 1--13,
  2019.

\bibitem[Leimkuhler and Matthews(2013)]{leimkuhler2013rational}
Benedict Leimkuhler and Charles Matthews.
\newblock Rational construction of stochastic numerical methods for molecular
  sampling.
\newblock \emph{Applied Mathematics Research eXpress}, 2013\penalty0
  (1):\penalty0 34--56, 2013.

\bibitem[Levenspiel(1998)]{levenspiel1998chemical}
Octave Levenspiel.
\newblock \emph{Chemical reaction engineering}.
\newblock John wiley \& sons, 1998.

\bibitem[Li et~al.(2020)Li, Khoo, Ren, and Ying]{li2020solving}
Haoya Li, Yuehaw Khoo, Yinuo Ren, and Lexing Ying.
\newblock Solving for high dimensional committor functions using neural network
  with online approximation to derivatives.
\newblock \emph{arXiv preprint arXiv:2012.06727}, 2020.

\bibitem[Li et~al.(2022)Li, Khoo, Ren, and Ying]{li2022semigroup}
Haoya Li, Yuehaw Khoo, Yinuo Ren, and Lexing Ying.
\newblock A semigroup method for high dimensional committor functions based on
  neural network.
\newblock In \emph{Mathematical and Scientific Machine Learning}, pages
  598--618. PMLR, 2022.

\bibitem[Li et~al.(2019)Li, Lin, and Ren]{li2019computing}
Qianxiao Li, Bo~Lin, and Weiqing Ren.
\newblock Computing committor functions for the study of rare events using deep
  learning.
\newblock \emph{The Journal of Chemical Physics}, 151\penalty0 (5), 2019.

\bibitem[Lin and Ren(2024)]{lin2024deep}
Bo~Lin and Weiqing Ren.
\newblock Deep learning method for computing committor functions with adaptive
  sampling.
\newblock \emph{arXiv preprint arXiv:2404.06206}, 2024.

\bibitem[Lindorff-Larsen et~al.(2011)Lindorff-Larsen, Piana, Dror, and
  Shaw]{lindorff2011fast}
Kresten Lindorff-Larsen, Stefano Piana, Ron~O. Dror, and David~E. Shaw.
\newblock How fast-folding proteins fold.
\newblock \emph{Science}, 334\penalty0 (6055):\penalty0 517--520, 2011.

\bibitem[Lu and Nolen(2015)]{lu2015reactive}
Jianfeng Lu and James Nolen.
\newblock Reactive trajectories and the transition path process.
\newblock \emph{Probability Theory and Related Fields}, 161:\penalty0 195--244,
  2015.
\newblock \doi{10.1007/s00440-014-0547-y}.

\bibitem[MacKerell et~al.(1998)MacKerell, Bashford, Bellott, Dunbrack,
  Evanseck, Field, Fischer, Gao, Guo, Ha, et~al.]{mackerell1998all}
Alexander~D MacKerell, D~Bashford, M~Bellott, Roland~L Dunbrack, Jeffrey~D
  Evanseck, Martin~J Field, Stefan Fischer, Jiali Gao, H~Guo, Sookhee Ha,
  et~al.
\newblock All-atom empirical potential for molecular modeling and dynamics
  studies of proteins.
\newblock \emph{The Journal of Physical Chemistry B}, 102\penalty0
  (18):\penalty0 3586--3616, 1998.

\bibitem[MacKerell et~al.(2004)MacKerell, Feig, and
  Brooks]{mackerell2004extending}
Alexander~D MacKerell, Michael Feig, and Charles~L Brooks.
\newblock Extending the treatment of backbone energetics in protein force
  fields: limitations of gas-phase quantum mechanics in reproducing protein
  conformational distributions in molecular dynamics simulations.
\newblock \emph{Journal of Computational Chemistry}, 25\penalty0 (11):\penalty0
  1400--1415, 2004.

\bibitem[Mitchell and Rotskoff(2024)]{mitchell2024committor}
Andrew~R Mitchell and Grant~M Rotskoff.
\newblock Committor guided estimates of molecular transition rates.
\newblock \emph{Journal of Chemical Theory and Computation}, 20\penalty0
  (21):\penalty0 9378--9393, 2024.

\bibitem[Miyamoto and Kollman(1992)]{miyamoto1992settle}
Shuichi Miyamoto and Peter~A. Kollman.
\newblock {SETTLE}: An analytical version of the {SHAKE} and {RATTLE} algorithm
  for rigid water models.
\newblock \emph{Journal of Computational Chemistry}, 13\penalty0 (8):\penalty0
  952--962, 1992.

\bibitem[N{\"u}sken and Richter(2021)]{nusken2021solving}
Nikolas N{\"u}sken and Lorenz Richter.
\newblock Solving high-dimensional hamilton--jacobi--bellman pdes using neural
  networks: perspectives from the theory of controlled diffusions and measures
  on path space.
\newblock \emph{Partial differential equations and applications}, 2\penalty0
  (4):\penalty0 48, 2021.

\bibitem[Nüsken and Richter(2023)]{nusken2023interpolating}
Nikolas Nüsken and Lorenz Richter.
\newblock Interpolating between bsdes and pinns: Deep learning for elliptic and
  parabolic boundary value problems.
\newblock \emph{Journal of Machine Learning}, 2:\penalty0 31--64, 06 2023.

\bibitem[Onsager(1938)]{onsager1938initial}
Lars Onsager.
\newblock Initial recombination of ions.
\newblock \emph{Physical Review}, 54\penalty0 (8):\penalty0 554, 1938.

\bibitem[Onsager(1944)]{onsager1944crystal}
Lars Onsager.
\newblock Crystal statistics. i. a two-dimensional model with an order-disorder
  transition.
\newblock \emph{Physical review}, 65\penalty0 (3-4):\penalty0 117, 1944.

\bibitem[Peters and Trout(2006)]{peters2006obtaining}
Baron Peters and Bernhardt~L Trout.
\newblock Obtaining reaction coordinates by likelihood maximization.
\newblock \emph{The Journal of chemical physics}, 125\penalty0 (5), 2006.

\bibitem[Pham(2009)]{pham2009continuous}
Huy{\^e}n Pham.
\newblock \emph{Continuous-time stochastic control and optimization with
  financial applications}, volume~61.
\newblock Springer Science \& Business Media, 2009.

\bibitem[Pidstrigach et~al.(2025)Pidstrigach, Baker, Domingo-Enrich,
  Deligiannidis, and N{\"u}sken]{pidstrigach2025conditioning}
Jakiw Pidstrigach, Elizabeth~Louise Baker, Carles Domingo-Enrich, George
  Deligiannidis, and Nikolas N{\"u}sken.
\newblock Conditioning diffusions using malliavin calculus.
\newblock In \emph{Forty-second International Conference on Machine Learning},
  2025.

\bibitem[Raissi et~al.(2019)Raissi, Perdikaris, and
  Karniadakis]{raissi2019physics}
Maziar Raissi, Paris Perdikaris, and George~Em Karniadakis.
\newblock Physics-informed neural networks: A deep learning framework for
  solving forward and inverse problems involving nonlinear partial differential
  equations.
\newblock \emph{Journal of Computational Physics}, 378:\penalty0 686--707,
  2019.

\bibitem[Ribera~Borrell et~al.(2025)Ribera~Borrell, Richter, and
  Sch{\"u}tte]{ribera2025random_horizons}
Enric Ribera~Borrell, Lorenz Richter, and Christof Sch{\"u}tte.
\newblock Reinforcement learning with random time horizons.
\newblock \emph{arXiv preprint arXiv:2506.00962}, 2025.

\bibitem[Rotskoff et~al.(2022)Rotskoff, Mitchell, and
  Vanden-Eijnden]{rotskoff2022active}
Grant~M Rotskoff, Andrew~R Mitchell, and Eric Vanden-Eijnden.
\newblock Active importance sampling for variational objectives dominated by
  rare events: Consequences for optimization and generalization.
\newblock In \emph{Mathematical and Scientific Machine Learning}, pages
  757--780. PMLR, 2022.

\bibitem[Seifert(2008)]{seifert2008stochastic}
Udo Seifert.
\newblock Stochastic thermodynamics: principles and perspectives.
\newblock \emph{The European Physical Journal B}, 64:\penalty0 423--431, 2008.

\bibitem[Singh and Limmer(2023)]{singh2023variational}
Aditya~N Singh and David~T Limmer.
\newblock Variational deep learning of equilibrium transition path ensembles.
\newblock \emph{The Journal of Chemical Physics}, 159\penalty0 (2), 2023.

\bibitem[Singh and Limmer(2024)]{singh2024splitting}
Aditya~N Singh and David~T Limmer.
\newblock Splitting probabilities as optimal controllers of rare reactive
  events.
\newblock \emph{The Journal of Chemical Physics}, 161\penalty0 (5), 2024.

\bibitem[Singh et~al.(2025)Singh, Das, and Limmer]{singh2025variational}
Aditya~N Singh, Avishek Das, and David~T Limmer.
\newblock Variational path sampling of rare dynamical events.
\newblock \emph{Annual Review of Physical Chemistry}, 76, 2025.

\bibitem[Strahan et~al.(2021)Strahan, Antoszewski, Lorpaiboon, Vani, Weare, and
  Dinner]{strahan2021long}
John Strahan, Adam Antoszewski, Chatipat Lorpaiboon, Bodhi~P Vani, Jonathan
  Weare, and Aaron~R Dinner.
\newblock Long-timescale predictions from short-trajectory data: A benchmark
  analysis of the trp-cage miniprotein.
\newblock \emph{Journal of chemical theory and computation}, 17\penalty0
  (5):\penalty0 2948, 2021.

\bibitem[Strahan et~al.(2023)Strahan, Finkel, Dinner, and
  Weare]{strahan2023predicting}
John Strahan, Justin Finkel, Aaron~R Dinner, and Jonathan Weare.
\newblock Predicting rare events using neural networks and short-trajectory
  data.
\newblock \emph{Journal of computational physics}, 488:\penalty0 112152, 2023.

\bibitem[Stroock and Varadhan(1972)]{stroock1972onthesupport}
Daniel~W. Stroock and S.~R.~S. Varadhan.
\newblock On the support of diffusion processes with applications to the strong
  maximum principle.
\newblock In \emph{Berkeley Symp. on Math. Statist. and Prob.}, volume~3, pages
  333--359, 1972.

\bibitem[Stroock and Varadhan(1997)]{stroock1997multidimensional}
D.W. Stroock and S.R.S. Varadhan.
\newblock \emph{Multidimensional Diffusion Processes}.
\newblock Grundlehren der mathematischen Wissenschaften. Springer Berlin
  Heidelberg, 1997.

\bibitem[Sun et~al.(2022)Sun, Vandermause, Batzner, Xie, Clark, Chen, and
  Kozinsky]{sun2022multitask}
Lixin Sun, Jonathan Vandermause, Simon Batzner, Yu~Xie, David Clark, Wei Chen,
  and Boris Kozinsky.
\newblock Multitask machine learning of collective variables for enhanced
  sampling of rare events.
\newblock \emph{Journal of Chemical Theory and Computation}, 18\penalty0
  (4):\penalty0 2341--2353, 2022.

\bibitem[Tang et~al.(2026)Tang, Pandey, Chen, Meg{\'\i}as, Dehez, and
  Chipot]{tang2026breaking}
Chenyu Tang, Mayank~Prakash Pandey, Cheng~Giuseppe Chen, Alberto Meg{\'\i}as,
  Fran{\c{c}}ois Dehez, and Christophe Chipot.
\newblock Breaking timescales with generative sampling of conformational
  transitions.
\newblock \emph{Nature}, 657\penalty0 (8133):\penalty0 944--949, 2026.

\bibitem[Trizio et~al.(2025)Trizio, Kang, and Parrinello]{trizio2025everything}
Enrico Trizio, Peilin Kang, and Michele Parrinello.
\newblock Everything everywhere all at once: a probability-based enhanced
  sampling approach to rare events.
\newblock \emph{Nature Computational Science}, pages 1--10, 2025.

\bibitem[Vanden-Eijnden and E(2006)]{vanden2006towards}
Eric Vanden-Eijnden and Weinan E.
\newblock Towards a theory of transition paths.
\newblock \emph{Journal of statistical physics}, 123\penalty0 (3):\penalty0
  503--523, 2006.

\bibitem[Vanden-Eijnden and E(2010)]{vanden2010transition}
Eric Vanden-Eijnden and Weinan E.
\newblock Transition-path theory and path-finding algorithms for the study of
  rare events.
\newblock \emph{Annual review of physical chemistry}, 61:\penalty0 391--420,
  2010.

\bibitem[Wang et~al.(2025)Wang, Tang, Wang, Wan, Ren, and
  Yang]{wang2025estimating}
Yueyang Wang, Kejun Tang, Xili Wang, Xiaoliang Wan, Weiqing Ren, and Chao Yang.
\newblock Estimating committor functions via deep adaptive sampling on rare
  transition paths.
\newblock \emph{arXiv preprint arXiv:2501.15522}, 2025.

\bibitem[Weare and Dinner(2026)]{weare2026machine}
Jonathan Weare and Aaron~R Dinner.
\newblock Machine learning kinetics from molecular dynamics data.
\newblock \emph{arXiv preprint arXiv:2609.17736}, 2026.

\bibitem[Yan et~al.(2022)Yan, Touchette, and Rotskoff]{yan2022learning}
Jiawei Yan, Hugo Touchette, and Grant~M Rotskoff.
\newblock Learning nonequilibrium control forces to characterize dynamical
  phase transitions.
\newblock \emph{Physical Review E}, 105\penalty0 (2):\penalty0 024115, 2022.

\bibitem[Yuan et~al.(2023)Yuan, Shah, Bentz, and Cameron]{yuan2023optimal}
Jiaxin Yuan, Amar Shah, Channing Bentz, and Maria Cameron.
\newblock Optimal control for sampling the transition path process and
  estimating rates.
\newblock \emph{arXiv preprint arXiv:2305.17112}, 2023.

\bibitem[Zhou and Lu(2025)]{zhou2025solving}
Mo~Zhou and Jianfeng Lu.
\newblock Solving time-continuous stochastic optimal control problems:
  Algorithm design and convergence analysis of actor-critic flow.
\newblock \emph{arXiv preprint arXiv:2402.17208}, 2025.

\bibitem[Zhu et~al.(2025)Zhu, Guo, Choi, Liu, Chen, and Tao]{zhu2025mdns}
Yuchen Zhu, Wei Guo, Jaemoo Choi, Guan-Horng Liu, Yongxin Chen, and Molei Tao.
\newblock {MDNS}: Masked diffusion neural sampler via stochastic optimal
  control.
\newblock In \emph{Advances in Neural Information Processing Systems}, 2025.
\newblock arXiv:2508.10684.

\end{thebibliography}

\end{document}